\documentclass[pdflatex,sn-mathphys-ay]{sn-jnl}% Math and Physical Sciences Author Year Reference Style
\usepackage{graphicx}%
\usepackage{multirow}%
\usepackage{amsmath,amssymb,amsfonts}%
\usepackage{amsthm}%
\usepackage{mathrsfs}%
\usepackage[title]{appendix}%
\usepackage{xcolor}%
\usepackage{textcomp}%
\usepackage{manyfoot}%
\usepackage{booktabs}%
\usepackage{algorithm}%
\usepackage{algorithmicx}%
\usepackage{algpseudocode}%
\usepackage{listings}%
\usepackage{subfigure}
\newcommand{\E}{\mathbb{E}}

\newcommand{\R}{\mathbb{R}}

\newcommand{\cA}{\mathcal{A}}
\newcommand{\cB}{\mathcal{B}}

\newcommand{\cE}{\mathcal{E}}
\newcommand{\cF}{\mathcal{F}}
\newcommand{\cG}{\mathcal{G}}

\newcommand{\cJ}{\mathcal{J}}
\newcommand{\cK}{\mathcal{K}}
\newcommand{\cL}{\mathcal{L}}
\newcommand{\cM}{\mathcal{M}}
\newcommand{\cN}{\mathcal{N}}

\newcommand{\cR}{\mathcal{R}}
\newcommand{\cS}{\mathcal{S}}
\newcommand{\cT}{\mathcal{T}}

\newcommand{\argmax}{\operatornamewithlimits{argmax}}
\newcommand{\argmin}{\operatornamewithlimits{argmin}}
\mathchardef\mhyphen="2D
\newcommand{\ex}{\mathbb{E}}

\newcommand{\sbr}[1]{\left( #1 \right)}
\newcommand{\mbr}[1]{\left[ #1 \right]}
\newcommand{\lbr}[1]{\left\{ #1 \right\}}
\newcommand{\abr}[1]{\left| #1 \right|}
\newcommand{\nbr}[1]{\left\| #1 \right\|}

\newcommand{\indicator}{\mathbf{1}}
\newcommand{\trace}{\textup{tr}}

\newcommand{\round}{\mathtt{ROUND}}

\newcommand{\seg}{\textup{seg}}
\newcommand{\bi}{\textup{B}}
\newcommand{\summing}{\textup{S}}

\newcommand{\seglsvits}{\mathtt{Seg\mbox{-}LSVI\mbox{-}TS}}
\newcommand{\biseglsvits}{\mathtt{Bi\mbox{-}Seg\mbox{-}LSVI\mbox{-}TS}}
\newcommand{\sumseglsvits}{\mathtt{Sum\mbox{-}Seg\mbox{-}LSVI\mbox{-}TS}}
\newcommand{\lsvitsq}{\mathtt{LSVI\mbox{-}TS\mbox{-}Q}}

\newcommand{\sig}{\operatorname{sig}}
\newcommand{\bitssegd}{\mathtt{SegBiTS\mbox{-}d}}

\newcommand{\uneqsegbitsd}{\mathtt{Uneq\mbox{-}SegBiTS\mbox{-}d} }
\newcommand{\edlinucbsegd}{\mathtt{E\mbox{-}LinUCB\mbox{-}d}}
\newcommand{\sumrm}{\mathrm{sum}}

\usepackage{mathtools}

\newcommand{\compilehidecomments}{ture}
\ifthenelse{ \equal{\compilehidecomments}{true} }{
	\newcommand{\yihan}[1]{}
    \newcommand{\fengxu}[1]{}
    \newcommand{\siwei}[1]{}
    \newcommand{\shie}[1]{}
    \newcommand{\gal}[1]{}
}{
	\newcommand{\yihan}[1]{{\color{red} [\text{Yihan:} #1]}}
    \newcommand{\fengxu}[1]{{\color{brown} [\text{Fengxu:} #1]}}
    \newcommand{\siwei}[1]{{\color{blue} [\text{Siwei:} #1]}}
    \newcommand{\shie}[1]{{\color{purple} [\text{Shie:} #1]}}
    \newcommand{\gal}[1]{{\color{orange} [\text{Gal:} #1]}}
}

\theoremstyle{thmstyleone}%
\newtheorem{theorem}{Theorem}%  meant for continuous numbers
\newtheorem{proposition}[theorem]{Proposition}% 
\newtheorem{assumption}{Assumption}
\newtheorem{lemma}[theorem]{Lemma}

\theoremstyle{thmstyletwo}%

\theoremstyle{thmstylethree}%

\begin{document}

\title[Article Title]{Reinforcement Learning with Segment Reward Feedback under Linear Function Approximation}


\author[1]{\fnm{Fengxu} \sur{Liu}}\email{z23v2p@gmail.com}

\author[2]{\fnm{Siwei} \sur{Wang}}\email{siweiwang@microsoft.com}

\author[3]{\fnm{Gal}\sur{Dalal}}\email{gdalal@nvidia.com}

\author[3,4]{\fnm{Shie} \sur{Mannor}}\email{shie@ee.technion.ac.il}

\author[5]{\fnm{Yihan} \sur{Du}}\email{yihan\_du@sutd.edu.sg}

\affil[1]{\orgname{National University of Singapore}}

\affil[2]{\orgname{Microsoft Research Asia}}

\affil[3]{\orgname{NVIDIA Research}}

\affil[4]{\orgname{Technion}}

\affil[5]{\orgdiv{ESD},
  \orgname{Singapore University of Technology and Design}}


\abstract{
Classical reinforcement learning (RL) assumes that a reward is observed for every visited state-action pair. However, in real-world applications such as autonomous driving, such fine-grained feedback can be costly or difficult to collect, whereas trajectory-level feedback may be too sparse for efficient learning. To provide a general feedback model bridging these two extremes and handle large state spaces, we study RL with segment reward feedback under linear function approximation. Our work answers how the granularity of segment feedback and the choice of segmentation influence learning.
For equal-length segments with known transitions, we design algorithms $\bitssegd$ and $\edlinucbsegd$ for binary and sum feedback types, respectively. They adopt posterior sampling with planning to achieve computational efficiency and the E-optimal experimental design to attain near-optimality. Nearly matching lower bounds are established.
For equal-length segments with unknown transitions, we develop a unified $\seglsvits$ framework with two instantiations for binary and sum feedback, which carefully integrates the posterior estimated reward parameters into least-squares value iteration.
These results reveal a fundamental insight: under binary feedback, increasing the number of segments significantly reduces the regret through an exponential factor, while surprisingly, under sum feedback, the granularity of segments does not affect learning much.
Finally, to investigate whether segmenting according to state-action features can further expedite learning, we design an algorithm $\uneqsegbitsd$ that allows arbitrary segmentations. The resulting regret bound shows that under the usual elliptical potential analysis, the influence of state-action features on the regret appears only through logarithmic factors, and equal segmentation achieves the best performance.
}




\maketitle

\section{Introduction}
\label{sec1}

Reinforcement learning (RL) provides a general framework for sequential
decision making, in which an agent interacts with an unknown environment with the goal of maximizing the long-term reward.  Classical RL typically assumes
that the reward associated with every visited state-action pair is observed
immediately.  While this reward feedback model is analytically convenient,
state-action-wise rewards can be difficult and costly to obtain in real-world applications, particularly when performance is assessed by humans or through a delayed evaluation process. For example, in robotic manipulation, it is often easier to evaluate whether a robot successfully completes a task, such as folding clothes, than to assign a reliable reward to every individual action. In autonomous driving, evaluating each state-action pair according to multiple criteria such as safety, comfort, and efficiency can also be difficult and onerous. 
These practical applications has
motivated the study of trajectory-level feedback, where only a
single reward signal is revealed after an entire episode (Efroni et
al.,~\citeyear{efroni2021reinforcement}; Chatterji et
al.,~\citeyear{chatterji2021theory}).

Trajectory-level feedback reduces the frequency of reward queries, but it also aggregates the reward information from an entire episode into a single observation. When trajectories
are long, such a signal may provide little information about which portions of
the behavior were beneficial or detrimental to the overall performance. To resolve this challenge, segment feedback offers an
intermediate feedback granularity: an episode is divided into several
segments, and one reward signal is observed for each segment.  This segment feedback model provides a general paradigm 
interpolating between trajectory-level signals and standard state-action-wise feedback. It can naturally be applied to realistic scenarios where trajectories are presented to
human evaluators in shorter clips or bags, such as autonomous driving.  
% (Tang et al.,~\citeyear{tang2024reinforcement}; Du et al.,~\citeyear{pmlr-v267-du25e})
Under this model, there is a fundamental question:
\emph{how do the granularity of segment feedback and the choice of segmentation influence learning performance?}

Du et al.~(\citeyear{pmlr-v267-du25e}) initiated the theoretical study of equal-length segment feedback in episodic tabular Markov decision processes (MDPs).
%, characterizing how the number of segments affects learning under binary and sum feedback. 
% However, their algorithm design and analysis are restricted to the tabular representation and fixed equal-length trajectory partitions. 
However, the tabular representation becomes prohibitive in large or continuous state spaces, which is common in real-world applications, while fixed equal-length partitions leave open whether more flexible trajectory partitioning can be utilized to improve learning. 
For realistic applications with large or continuous state spaces, linear function approximation provides a widely-adopted formulation: in a linear MDP (Jin et al.,~\citeyear{jin20_linearmdp}), rewards and transition kernels are represented through a known $d$-dimensional feature map. Existing work on sparse feedback with function approximation~(Cassel et al.,~\citeyear{pmlr-v235-cassel24a}) mainly considered a single aggregate observation for the entire trajectory. Consequently, it remains unclear whether segment-level reward information can be  efficiently exploited in linear MDPs, how the granularity of segment feedback affects learning performance, and whether choosing segmentation according to the realized state-action features can provably benefit learning.

Motivated by these questions, in this work, we investigate RL with segment feedback under linear function approximation. We adopt the linear MDP formulation in (Jin et al.~\citeyear{jin20_linearmdp}), which enables generalization across large or potentially infinite state-action spaces through known feature representations. 
%Building on the fixed equal-length segment-feedback model of Du et al.~\citeyear{pmlr-v267-du25e}, 
Under the framework of segment reward feedback in linear MDPs, we first study equal-length segments and binary and sum feedback with known and unknown transitions. Specifically, binary feedback refers to a Bernoulli outcome   generated by the cumulative reward over the segment, and sum feedback is a noisy observation of the cumulative reward over the segment.
Beyond the equal-length segment setting, we further consider a more general variable segment setting with binary feedback, where both the number of segments and the way a trajectory is partitioned can be chosen adaptively by learning algorithms. We aim to rigorously quantify how the reward feedback type, unknown transition dynamics, and segmentation scheme impact the regret performance of learning algorithms.

This work faces several significant technical challenges: (i) When the transitions are unknown, designing a computationally and statistically efficient algorithm is highly nontrivial because reward and transition uncertainties are represented at different temporal scales. A per-step reward bonus is difficult to control because the reward covariance matrix is constructed by aggregate segment features and is generally incompatible with the transition covariance matrices constructed by individual state-action features, potentially leading to an unbounded regret analysis. A segment-level reward bonus, on the other hand, is history dependent and cannot be incorporated into the state-action-wise Markovian Bellman recursion. Finally, omitting reward bonuses altogether prevents the resulting value estimates from being optimistic, breaking the optimism-based regret analysis. (ii) Establishing
dimension-dependent lower bounds requires encoding multiple dimensions of the reward parameter into linear features, while ensuring that sparse segment
observations cannot easily distinguish them. For binary feedback, the
construction must additionally quantify the information loss caused by the
sigmoid curvature in order to demonstrate that the exponential dependence in the regret bound is inevitable. (iii) Under the variable segment setting, both the number of segments and the way a trajectory is partitioned may affect learning through the sigmoid curvatures and aggregated segment features. Hence, the analysis must rigorously account for all possible adaptive, episode-dependent segmentations.

To address these challenges, we develop sample efficient algorithms
and establish regret guarantees for different
settings. Table~\ref{tab:main-results} summarizes the settings and our results. Our main contributions are summarized as follows.

\begin{enumerate}
\item We study RL with segment reward feedback with linear
function approximation, which provides a general feedback paradigm emcompassing per-state-action feedback and
trajectory-level feedback as two extremes. Under this framework, we consider both binary and sum feedback types with
known and unknown transitions.

\item For equal-length segments with known transitions, we propose algorithms $\bitssegd$ and $\edlinucbsegd$ for binary and sum feedback, respectively. $\bitssegd$ employs posterior sampling with efficient linear MDP planning and achieves a regret bound with dependence on $\exp(\frac{Hr_{\max}}{2m})$,
% due to the sigmoid curvature, 
where $H$ is the length of each episode, $m$ is the number of segments, and $r_{\max}$ is the  universal upper bound of rewards. 
In contrast, $\edlinucbsegd$ uses the E-optimal experimental design to refine the conditioning of the covariance matrix and achieves a near-optimal regret guarantee that does not depend on $m$ (when ignoring logarithmic factors). 
%These results show that under binary feedback, increasing $m$ substantially reduces the regret, while surprisingly, under sum feedback, increasing $m$ does not improve learning much.

\item We establish dimension-dependent regret lower bounds for both binary
and sum segment feedback in linear MDPs with known transitions. 
% To the best of our knowledge, these are the first regret lower bounds for RL with fixed equal-length segment feedback under linear function approximation. 
The lower bound for binary feedback demonstrates 
that the exponential dependence on $H/m$ is unavoidable through a Bernoulli KL-divergence analysis, while the lower bound for sum feedback validates that our upper bound is near-optimal and $m$ does not affect the regret up to logarthmic factors.

\item For equal-length segments with unknown transitions, we develop a  unified $\seglsvits$ framework and provide computationally and sample efficient instantiations for both binary and sum feedback. By integrating posterior sampling of reward parameters with optimistic least-squares value iteration, the framework resolves the mismatch between segment-level reward uncertainty and state-action-wise value iteration. While transition learning introduces additional polynomial terms, the resulting regret bounds preserve the same insights as in the known transition setting: 
\emph{under binary feedback, increasing the number of segments $m$ significantly reduces the regret due to the $\exp(\frac{Hr_{\max}}{2m})$ factor, while surprisingly, under sum feedback, increasing $m$ does not improve learning much.}
%the binary-feedback bound retains the factor $\exp(\frac{Hr_{\max}}{2m})$, whereas the sum-feedback bound has no explicit dependence on $m$.

\item Finally, we investigate whether choosing segmentation according to the realized state-action features can further accelerate learning. We introduce a variable segment and binary feedback model and design an algorithm $\uneqsegbitsd$, which accommodates adaptive, episode-dependent partitions and accounts for the sigmoid curvature of each individual segment. Its regret guarantee shows that, \emph{under the common elliptical potential analysis}, the effect of the realized state-action features is absorbed into logarithmic factors and thus \emph{feature-dependent segmentation does not substantially benefit learning},
%, whereas the segment lengths explicitly govern the dominant  terms. 
and furthermore, 
%for a fixed number of segments, 
\emph{equal segmentation achieves the best performance}.
\end{enumerate}

\begin{table}[h]
\caption{Summary of our results.}\label{tab:main-results}%
\begin{tabular}{@{}lllll@{}}
\toprule
Result & Segmentation & Feedback Type & Transition & Regret\\
\midrule
Theorem~\ref{thm:seg-bits-upper} & Equal-length & Binary & Known & {\centering
\(\displaystyle
\widetilde{O}\!\left(
\begin{aligned}
&\exp\!\left(\frac{H r_{\max}}{2m}\right)
\nu(K)d
\sqrt{
Km
\max\!\left\{
\frac{H^2}{m\alpha\lambda},
1
\right\}}
\end{aligned}
\right)
\)
\par}
\\

Theorem~\ref{thm:seg-bits-lower} & Equal-length & Binary & Known & {\centering
\(\displaystyle
\Omega\left(
\exp\left(\frac{Hr_{\max}}{2m}\right)
d\sqrt{Km}
\right)
\)
\par}
\\

Theorem~\ref{thm:regret_sum_fixed_segment_d_known} & Equal-length & Sum & Known & {\centering
\(\displaystyle
\widetilde{O}\left(
d\sqrt{HK}
 \right)
\)
\par}
\\

Theorem~\ref{thm:regret_sum_fixed_segment_d_known_lower} & Equal-length & Sum & Known & {\centering
\(\displaystyle
 \Omega\left(
        d\sqrt{HK}
    \right)
\)
\par}
\\

Theorem~\ref{thm:seg_binary_unknown_tran_d}
& Equal-length
& Binary
& Unknown
&
{\centering
\(
\begin{aligned}
\widetilde{O}\!\Biggl(
&\exp\!\left(\frac{H r_{\max}}{2m}\right)
\nu(K)d
\sqrt{
Km
\max\!\left\{
\frac{H^2}{m\alpha\lambda},
1
\right\}
}
\\[-1mm]
&\quad{}
+
d^2H^2\sqrt{K}
\left(
\frac{\nu(K)}{\sqrt{\lambda}}
+
r_{\max}
\right)
\Biggr)
\end{aligned}
\)
\par}
\\

Theorem~\ref{thm:seg_sum_unknown_tran_d}
& Equal-length
& Sum
& Unknown
&
{\centering
\(\displaystyle
\begin{aligned}
\widetilde{O}\!\Biggl(
&d^{5/2}H^2\sqrt{K}
\left(
r_{\max}
+
\frac{1}{\sqrt{\lambda}}
\right)
\\[-1mm]
&\quad{}
+
d\left(
1+r_{\max}\sqrt{\lambda}
\right)
\sqrt{
KHd
\max\!\left\{
\frac{H}{\lambda},
1
\right\}
}
\Biggr)
\end{aligned}
\)
\par}
\\

Theorem~\ref{thm:regret-binary-variable-known}
& Variable
& Binary
& Known
&
{\centering
\(\displaystyle
\widetilde O\Bigg(
        \nu_U(K)d
        \sqrt{
            \sum_{k=1}^{K}
            \max\left\{
                \frac{
                    \sum_{i=1}^{m_k}(\ell_i^k)^2
                }{4\lambda},
                1
            \right\}
            \sum_{i=1}^{m_k}
            \frac{1}{\sig'(\ell_i^k r_{\max})}
        }
    \Bigg)
\)
\par}
\\
\botrule
\end{tabular}
\footnotetext{\textbf{Notation.}
Equal-length and variable refer to \(m\) equal-length
segments per episode and arbitrary, episode-dependent segmentations, respectively. 
Here \(K\), \(d\) and \(\lambda\) denote the number of episodes, feature dimension, and
regularization parameter, respectively. Under variable segmentation, \(m_k\) and \(\ell_i^k\) denote the
number of segments and the length of the
\(i\)-th segment in episode \(k\), respectively.
Define the sigmoid function \(\sig(x) \coloneqq 1/(1+\exp(-x))\), where \(\alpha\), \(\nu(K)\) and \(\nu_U(K)\) are defined in Eq.~\eqref{eq:def_alpha}, Eq.~\eqref{eq:def_nu} and Eq.~\eqref{eq:nu-unequal}, respectively. \(\widetilde O(\cdot)\) hides logarithmic factors.
%and \(\Omega(\cdot)\) denote upper bounds up to logarithmic factors and asymptotic lower bounds.
}
\end{table}

\section{Related Work}\label{sec2}

\subsection{RL with Sparse Feedback}\label{subsec2.1}
A line of work relaxes the standard per-step reward observation assumption by replacing instantaneous rewards with a single trajectory-level reward signal. Efroni et al.~(\citeyear{efroni2021reinforcement}) investigated episodic tabular RL with trajectory-level sum feedback and developed UCB- and Thompson-sampling-based algorithms with regret guarantees. More recently, Zhang et al.~(\citeyear{pmlr-v267-zhang25l}) established a near-minimax-optimal regret guarantee for episodic tabular RL with trajectory-level sum feedback using tighter reward confidence regions and a reference transition model. Chatterji et al.~(\citeyear{chatterji2021theory}) considered once-per-episode binary feedback by directly maximizing the probability of producing feedback one, while the resulting nonlinear objective may induce non-Markovian optimal policies and pose computational challenges. Cohen et al.~(\citeyear{pmlr-v134-cohen21a}) addressed tabular RL with adversarial aggregate losses, and proposed a computationally efficient algorithm. For this adversarial aggregate feedback setting, Lancewicki and Mansour~(\citeyear{pmlr-v267-lancewicki25a}) designed policy optimization algorithms that achieve near-optimal regret under known transitions and improve the guarantee in (Cohen et al.,~\citeyear{pmlr-v134-cohen21a}) under unknown transitions. Ito et al.~(\citeyear{NEURIPS2025_0a50497b}) further proposed best-of-both-worlds algorithms that attain logarithmic regret under stochastic losses while retaining square-root regret under adversarial losses.
Zhuo et al.~(\citeyear{Zhuo2026LearningMD}) studied fully bandit feedback in episodic MDPs, where only the total reward of each episode is observed while the state-action trajectory remains unobserved, and showed that exponential dependence on the horizon $H$ is unavoidable in worst-case regret.

Trajectory-level feedback has also been studied under function approximation. In the offline setting, Xu et al.~(\citeyear{10596317}) investigated RL with trajectory-wise rewards and introduced a pessimistic value iteration method based on reward decomposition. Cassel et al.~(\citeyear{pmlr-v235-cassel24a}) focused on aggregate bandit feedback in online linear MDPs and developed computationally efficient algorithms with $\widetilde{O}(\mathrm{poly}(d,H)\sqrt{K})$ regret.

More generally, there have been several works considering feedback defined over segments or bags of a trajectory.
%, rather than over the trajectory as a whole. 
Tang et al.~(\citeyear{tang2024reinforcement}) introduced an empirical framework for RL from bagged rewards and developed a transformer-based method, but did not provide regret guarantees. Gao et al.~(\citeyear{gaoharnessing}) investigated bagged decision times with non-Markovian transitions within each bag, focusing primarily on modeling the transition structure rather than recovering step-wise rewards from aggregated observations. Du et al.~(\citeyear{pmlr-v267-du25e}) studied episodic tabular MDPs with equal-length segment feedback, covering both sum and binary feedback, and provided regret guarantees under both known and unknown transitions.

Prior works on binary and sum segment feedback have been primarily restricted to tabular MDPs, while function approximation results have focused on whole-trajectory signals. %The case of segment-level feedback under linear function approximation remains less understood.
A general treatment of segment-level feedback under function approximation has remained an important unresolved problem, which we systematically address in this work.

\subsection{RL with Linear Function Approximation}\label{subsec2.2}

RL with linear function approximation has been extensively studied to handle large or continuous state spaces. A widely-adopted model is linear MDPs~(Yang and Wang~\citeyear{pmlr-v97-yang19b}; Jin et al.,~\citeyear{jin20_linearmdp}), which represent transitions as linear combinations of unknown measures, with coefficients given by known state-action features. Yang and Wang~(\citeyear{pmlr-v97-yang19b}) investigated discounted MDPs with linearly parameterized transitions and developed a parametric Q-learning algorithm under a generative model and an anchor state-action condition. Yang and Wang~(\citeyear{pmlr-v119-yang20h}) subsequently proposed MatrixRL for episodic RL
%with regret guarantees independent of the state-action space size, 
and extended this approach to kernel representations. Under the linear MDP model, optimism-based approaches combine least-squares value iteration with confidence bonuses to encourage exploration. Jin et al.~(\citeyear{jin20_linearmdp}) introduced the LSVI-UCB framework, and several subsequent works improved its dependence on the feature dimension and episode horizon, leading to nearly minimax-optimal regret guarantees (Hu et al.,~\citeyear{pmlr-v162-hu22a}; He et al.,~\citeyear{he2023nearly}). 
% Randomized exploration offers an alternative to deterministic optimism. Osband et al.~(\citeyear{pmlr-v48-osband16}) introduced randomized least-squares value iteration(LSVI), while Zanette et al.~(\citeyear{pmlr-v108-zanette20a}) and Ishfaq et al.~(\citeyear{ishfaq2021randomized}) provided regret analyses for randomized value-function methods under function approximation. 
%
Randomized exploration offers an alternative to deterministic optimism and has been used for value-based approaches with function approximation (Osband et al.,~\citeyear{pmlr-v48-osband16}; Zanette et al.,~\citeyear{pmlr-v108-zanette20a}; Ishfaq et al.,~\citeyear{ishfaq2021randomized}).

%Several other works have established minimax and instance-dependent regret lower bounds for linear MDPs (Zhou et al.,~\citeyear{pmlr-v134-zhou21a}; He et al.,~\citeyear{pmlr-v139-he21c}). 
A related but distinct model is the linear mixture MDP, whose transition kernel is a linear combination of known basis kernels with unknown coefficients (Jia et al., \citeyear{jia2020model}; Ayoub
et al., \citeyear{ayoub2020model}; Zhou et al., \citeyear{zhou2021provably}). 
He et al.~(\citeyear{pmlr-v139-he21c}) established instance-dependent logarithmic regret bounds for both linear MDPs and linear mixture MDPs. 

The linear function approximation literature typically assumes that the reward associated with every visited state-action pair is observed individually. Hence, their algorithms and results fail to address decision scenarios where rewards are revealed only through sparse observations. Our work bridges the above two lines of research by investigating binary and sum segment feedback in linear MDPs, with provably efficient algorithms, dimension-dependent lower bounds, and an extension to variable segment and binary feedback.

\section{Problem Formulation}
\label{sec3}

We study episodic RL in linear MDPs under segment reward feedback. In this framework, the agent does
not observe per-step rewards as in classical RL. Instead, reward feedback is observed only at the end of each trajectory segment.
% The framework separates two aspects of the feedback mechanism: how the
% trajectory is partitioned into segments and what reward information is revealed
% for each segment. 
Under this framework, we consider two feedback types: sum feedback and binary feedback.
We first investigate equal-length segments with binary and sum feedback. Then,
we extend the binary segment feedback model to variable segments where the number and
lengths of segments may vary across episodes.

\subsection{Episodic Linear MDPs}

Consider an episodic MDP $\cM=(\cS,\cA,H,r,P,\rho)$, where $\cS$ is a
measurable state space, $\cA$ is a finite action space, $H$ is the episode
horizon, $r:\cS\times\cA\to[-r_{\max},r_{\max}]$ is the reward function,
$\rho\in\Delta(\cS)$ is the initial-state distribution, and
$P=\{P_h\}_{h=1}^H$ is a collection of time-inhomogeneous transition
kernels.  For any $h\in[H]$ and $(s,a,s') \in \cS \times \cA \times \cS$, $P_h(s' \mid s,a)$ is the probability of transitioning to state $s'$ if action $a$ is taken in state $s$ at step $h$.  The reward
function is time-homogeneous and unknown.

We define a policy $\pi=\{\pi_h\}_{h=1}^H$ as a collection of mappings
$\pi_h:\cS\to\cA$.  For any policy $\pi$, the state value function and state-action value function
are defined as 
\begin{align}
V_h^\pi(s)
&\coloneqq
\E\left[
\sum_{t=h}^H r(s_t,a_t)
\,\middle|\,
s_h=s,\pi
\right],
\\
Q_h^\pi(s,a)
&\coloneqq
\E\left[
\sum_{t=h}^H r(s_t,a_t)
\,\middle|\,
s_h=s,a_h=a,\pi
\right] ,
\end{align}
which stand for the expected cumulative reward obtained till the end of this episode under policy $\pi$, starting from $s$ and $(s,a)$ at step $h$, respectively.
Since the action space and episode horizon are both finite, there always exists an optimal policy $\pi^*$  that satisfies $V_h^{\pi^*}(s)=\sup_\pi V_h^{\pi}(s)$ for all $h \in [H]$ and $s \in \cS$~(Puterman \citeyear{puterman2014markov}), and we denote $V_h^*(s)\coloneqq V_h^{\pi^*}(s)$.

The agent interacts with the MDP as follows. At the beginning of each
episode $k$, an initial state $s_1^k\sim\rho$ is picked by the environment, and the agent chooses
a policy $\pi^k$. At each step $h\in[H]$, the agent observes the current state
$s_h^k$ and chooses an action
$a_h^k=\pi_h^k(s_h^k)$. Then, the agent transitions to the next state 
$s_{h+1}^k\sim P_h(\cdot\mid s_h^k,a_h^k).$

The realized trajectory in episode $k$ is denoted by
$\tau^k \coloneqq (s_1^k,a_1^k,\dots,s_H^k,a_H^k).$
All state-action pairs in $\tau^k$ are observable to the agent, whereas reward
information is revealed only through the feedback associated with each segment of the realized trajectory, which we will introduce shortly.

In this work, we focus on the linear MDP model~(Jin et al.~\citeyear{jin20_linearmdp}).  Let
$\psi:\cS\times\cA\to\R^d$ be a known feature map. We make the following structural assumption.

\begin{assumption}[Linear MDP]
\label{assump:linear_mdp}
There exist a vector $\theta^*\in\R^d$ and, for every $h\in[H]$, a vector of
signed measures
\[
\mu_h^*
=
\bigl(
{\mu_h^*}^{(1)},\ldots,{\mu_h^*}^{(d)}
\bigr)^\top
\]
over $\cS$ such that, for every $(s,a)\in\cS\times\cA$,
\begin{align}
r(s,a)
&=
\langle\psi(s,a),\theta^*\rangle,
\\
P_h(\cdot\mid s,a)
&=
\langle\psi(s,a),\mu_h^*(\cdot)\rangle.
\end{align}
Assume the following normalization conditions:
\[
\|\psi(s,a)\|_2\le1, \ \forall (s,a) \in \cS \times \cA,
\qquad
\|\theta^*\|_2\le r_{\max}\sqrt d,
\qquad
\|\mu_h^*(\cS)\|_2\le\sqrt d.
\]
\end{assumption}

Throughout the paper, the reward parameter $\theta^*$ is unknown.  We consider
both known transition and unknown transition settings.  In the
former setting, the transition kernel $P$ is known to the agent, whereas in the
latter setting $P$ must be learned from the observed state transitions.

For any trajectory or trajectory segment $\tau$, define its aggregate feature
as $\phi^\tau \coloneqq \sum_{(s,a)\in\tau}\psi(s,a)$.  For any policy $\pi$, define
the expected trajectory feature under $\pi$ as
$\phi^\pi \coloneqq \E[\sum_{h=1}^H\psi(s_h,a_h)\mid\pi]$ and $\phi^\pi(s_1) \coloneqq \E[\sum_{h=1}^H\psi(s_h,a_h)\mid\pi,s_1]$.  Assumption~\ref{assump:linear_mdp}
implies $\|\phi^\pi\|_2\le H$.

\subsection{Equal-length Segment Feedback}

We first introduce the equal-length segment setting, which provides a fundamental model
for investigating how the granularity of reward feedback influences learning performance in linear
MDPs. 
% Following (Du et al.~\citeyear{pmlr-v267-du25e}), 
In this setting, each episode is divided
into $m$ segments of equal length $H/m$, where we assume that
$m$ divides $H$ for convenience. For each
$i\in[m]$, the $i$-th segment of episode $k$ is denoted by
\[
\tau_i^k
\coloneqq
\bigl(
(s_h^k,a_h^k)
\bigr)_{h=\frac{H}{m}(i-1)+1}^{\frac{H}{m}i}.
\]
Since $\|\psi(s,a)\|_2\le1$, the aggregate feature of $\tau_i^k$ satisfies
$\|\phi^{\tau_i^k}\|_2\le H/m$.  The agent observes reward feedback only at the end of each segment, instead of each step.  We
consider the following two feedback types.

\textbf{Binary segment feedback.} 
In the binary feedback model, let $\sig(x):=1/(1+\exp(-x))$.  At the end of segment $i\in[m]$, the agent
observes a binary outcome $y_i^k$ that satisfies
\[
y_i^k\mid\tau_i^k
\sim
\operatorname{Bernoulli}
\left(
\sig\bigl((\phi^{\tau_i^k})^\top\theta^*\bigr)
\right).
\]
Hence, the probability of receiving reward feedback $1$ is determined by the
cumulative reward over the corresponding segment. 

\textbf{Sum segment feedback.}
In the sum feedback model,  at each step $h$ of episode $k$, the environment generates an instantaneous random reward
$R(s_h^k,a_h^k)=r(s_h^k,a_h^k)+\varepsilon_h^k$, where $\varepsilon_h^k$ is a zero-mean and $1$-sub-Gaussian noise
that is independent of the  transition distribution. The instantaneous random reward $R(s_h^k,a_h^k)$ is not revealed to the agent. Instead, the agent observes only the aggregate random reward of each segment $i\in[m]$:
\begin{align*}
R_i^k
&:=
\sum_{t=\frac{H}{m}(i-1)+1}^{\frac{H}{m}i}
R(s_t^k,a_t^k)
=
(\phi^{\tau_i^k})^\top\theta^*
+
\sum_{t=\frac{H}{m}(i-1)+1}^{\frac{H}{m}i}
\varepsilon_t^k .
\end{align*}
When $m=H$, the sum segment feedback model degenerates to classical RL
in linear MDPs (Jin et al.,~\citeyear{jin20_linearmdp}).  When $m=1$, the binary and sum segment models
reduce to RL with binary~(Chatterji et al.,~\citeyear{chatterji2021theory}) and sum trajectory feedback~(Efroni et al.,~\citeyear{efroni2021reinforcement}), respectively.

\subsection{Variable and Binary Segment Feedback}
\label{sec:variable segment}
We further extend the binary feedback model by allowing the number and
lengths of the segments to vary across episodes, with both determined by the
algorithm. After the policy has been executed for the entire episode $k$, the
algorithm observes the trajectory $\tau^k$ and decides the number of segments
$m_k$ and the partition of segments. Let $0=t_0^k<t_1^k<\cdots<t_{m_k}^k=H$ denote the end step indices of the $m_k$ segments, and define $\ell_i^k \coloneqq t_i^k-t_{i-1}^k$ as the length of the $i$-th segment in episode $k$.

Once the partition has been determined, the environment reveals one binary
outcome for each realized segment. Similarly, for every $i\in[m_k]$,
the agent observes
\[
y_i^k\mid\tau_i^k
\sim
\operatorname{Bernoulli}
\left(
\sig\bigl((\phi^{\tau_i^k})^\top\theta^*\bigr)
\right).
\]
The equal-length and binary segment feedback model is recovered by taking $m_k=m$ and
$t_i^k=iH/m$.  
%All other components of the problem, including the underlying linear MDP and the policy class, remain unchanged.

\subsection{Learning Objective}

The learning objective is the same under all feedback models considered
above.  Over $K$ episodes, the agent aims to minimize the cumulative regret
\[
\cR(K)
\coloneqq
\sum_{k=1}^K \mathbb E_{s^k_1 \sim \rho}
\left(
V_1^*(s_1^k)-V_1^{\pi^k}(s_1^k)
\right).
\] 

Prior work for binary trajectory feedback~(Chatterji et al.,~\citeyear{chatterji2021theory}) maximizes the expected probability of generating feedback one, whose nonlinear dependence on the cumulative reward may necessitate non-Markovian policies. In contrast, we use binary feedback to infer the underlying reward and maximize the expected cumulative reward under the standard MDP formulation. This distinction allows us to restrict attention to Markovian policies  and design computationally efficient algorithms.

\section{Equal-length Segment Feedback with Known Transitions}\label{sec4}
In this section, we first consider the equal-length segment and known transition setting, which allows us to isolate the statistical difficulty of learning the unknown reward parameter from segment feedback from the standard difficulty of learning the MDP transition dynamics. Under this setting, we study binary and sum feedback, and develop efficient algorithms with dimension-dependent regret guarantees, complemented by lower bounds that demonstrate the tightness of key dependencies.

\subsection{Binary Feedback}

For binary feedback, we propose algorithm $\bitssegd$, which drives exploration by perturbing logistic reward estimates with Gaussian noise calibrated to segment-level uncertainty and performing linear MDP planning under the perturbed rewards. Our regret analysis quantifies the benefit of finer segmentation through a sigmoid curvature term, showing that the regret bound decreases rapidly as the number of segments increases. A complementary lower bound establishes that the exponential dependence on segment length is unavoidable.

\subsubsection{Algorithm $\bitssegd$ and Regret Upper Bound}
\label{sec:known-binary-upper}

Binary segment feedback reveals the unknown reward parameter only through a Bernoulli observation whose parameter is the sigmoid of the aggregate reward for each segment. Since each segment has
length $H/m$, the aggregate reward over each segment is bounded by $Hr_{\max}/m$ in absolute value.
Accordingly, we define the uniform inverse sigmoid-curvature parameter
\begin{align}\label{eq:def_alpha}
\alpha \coloneqq \exp\left(\frac{Hr_{\max}}{m}\right)
+ \exp\left(-\frac{Hr_{\max}}{m}\right) +2 .
\end{align}
For any $k\geq 0$, define
\begin{align}
\omega(k) \coloneqq \sqrt{\lambda} \left( r_{\max}\sqrt{d} +\frac{1}{2} \right)
+ \frac{d}{\sqrt{\lambda}} \log \left(
\frac{4}{\delta'} \left(
1+\frac{H^2k}{4d\lambda m}
\right) \right),
\label{eq:def_omega}
\end{align}
and
\begin{align}
\nu(k) \coloneqq \frac{m \sqrt{ \lambda}}{H}
\Bigg(
1 + \frac{H r_{\max} \sqrt{d}}{m}
+ \frac{H}{m\sqrt{\lambda}}
\sqrt{ 1 + \frac{H r_{\max} \sqrt{d}}{m} }
\omega(k)
+ \frac{H^2}{m^2 \lambda} \omega(k)^2
\Bigg)^{\frac{3}{2}} .
\label{eq:def_nu}
\end{align}
Here $\omega(k)$ controls the stochastic error of the regularized logistic regression 
estimator, and $\nu(k)$ determines the confidence radius used to
calibrate posterior sampling.

Using these quantities, we propose algorithm $\bitssegd$, a computationally efficient Thompson-sampling-style algorithm~(Thompson, \citeyear{thompson1933likelihood}) which adopts randomized posterior sampling and linear MDP planning and achieves a $\widetilde{O}(\sqrt{K})$ regret rate. 

\begin{algorithm}[t]
    \caption{$\bitssegd$}
    \label{alg:seg-bits-d}
    \begin{algorithmic}[1]
        \State \textbf{Input:}
        $\delta,\delta'\coloneqq \frac{\delta}{3},
        \lambda$.

        \For{$k=1,\dots,K$}

        \State
        $\begin{aligned}[t]
        \hat{\theta}_{k-1}
        \gets \arg\min_{\theta}\;
        &-\sum_{k'=1}^{k-1}\sum_{i=1}^{m}
        \Big(
        y_i^{k'}\log
        \sig\!\left((\phi^{\tau_i^{k'}})^\top\theta\right)
        \\
        &\qquad
        +(1-y_i^{k'})
        \log\!\left(
        1-\sig\!\left((\phi^{\tau_i^{k'}})^\top\theta\right)
        \right)
        \Big)
        +\frac{\lambda}{2}\|\theta\|_2^2 .
        \end{aligned}$
        \label{line:segd-estimate}

        \State
        $\Sigma_{k-1}
        \gets
        \sum_{k'=1}^{k-1}\sum_{i=1}^{m}
        \phi^{\tau_i^{k'}}
        (\phi^{\tau_i^{k'}})^\top
        +\alpha\lambda I$.
        \label{line:segd-covariance}

        \State Sample
        $\xi_k\sim
        \mathcal{N}
        \left(
        0,\,
        \alpha\nu(k-1)^2\Sigma_{k-1}^{-1}
        \right)$,
        where $\alpha$ and $\nu(k-1)$ are defined in
        Eqs.~\eqref{eq:def_alpha} and \eqref{eq:def_nu}, respectively.
        \label{line:segd-sample}

        \State
        $\widetilde{\theta}_k
        \gets
        \hat{\theta}_{k-1}+\xi_k$.
        \label{line:segd-perturb}

        \State
        $\pi^k
        \gets
        \arg\max_{\pi}
        (\phi^\pi)^\top\widetilde{\theta}_k$.
        \label{line:plan_opt_policy}

        \State
        Play episode $k$ with policy $\pi^k$.
        Observe trajectory $\tau^k$ and binary segment feedback
        $\{y_i^k\}_{i=1}^m$.
        \label{line:segd-feedback}

        \EndFor
    \end{algorithmic}
\end{algorithm}

Algorithm~\ref{alg:seg-bits-d} presents the pseudocode of $\bitssegd$.
At each episode $k$, $\bitssegd$ first performs regularized logistic regression
to obtain the reward parameter estimate $\hat{\theta}_{k-1}$,
 and computes the segment covariance matrix
$\Sigma_{k-1}$ from past binary feedback
(Line~\ref{line:segd-estimate}--\ref{line:segd-covariance}).
Then, it samples a Gaussian noise
$\xi_k\sim\mathcal{N}(0,\alpha\nu(k-1)^2\Sigma_{k-1}^{-1})$ and calculates the posterior estimated reward
parameter $\widetilde{\theta}_k=\hat{\theta}_{k-1}+\xi_k$
(Line~\ref{line:segd-sample}--\ref{line:segd-perturb}).
Then, $\bitssegd$ computes the optimal policy for the linear MDP with reward
parameter $\widetilde{\theta}_k$,
and collects the resulting trajectory and binary segment feedback
(Line~\ref{line:plan_opt_policy}--\ref{line:segd-feedback}).

The step $\pi^k=\arg\max_{\pi}(\phi^\pi)^\top\widetilde{\theta}_k$ in Line~\ref{line:plan_opt_policy} can be regarded as episodic linear MDP planning under known transition kernels and reward function $r(s,a)=\psi(s,a)^\top \widetilde{\theta}_k$. This can be efficiently implemented using existing episodic linear MDP planning algorithms, e.g.,  (Wagenmaker et al.,~\citeyear{pmlr-v162-wagenmaker22b}).% \yihan{I revised this paragraph. We need to discuss more on this point in the meeting.}

% We assume access to the known transitions through a generative model that samples from $P_h(\cdot\mid s,a)$ for any $(h,s,a)$. Existing methods for finite-horizon linear MDPs can compute arbitrarily accurate policies with sample and computational complexities polynomial in the problem parameters and independent of $|\mathcal{S}|$, while reusing simulated transition data across reward functions (Wagenmaker et al.,~\citeyear{pmlr-v162-wagenmaker22b}). Thus, computing a $1/K$-optimal policy per episode adds only a constant to the cumulative regret.

Next, we provide a regret upper bound for algorithm $\bitssegd$.

\begin{theorem}
\label{thm:seg-bits-upper}
Under Assumption~\ref{assump:linear_mdp}, with probability at least $1-\delta$, for any $K>0$, the regret of Algorithm~\ref{alg:seg-bits-d} satisfies
\begin{align}
\mathcal{R}(K)
=
\widetilde{O}
\Bigg(
    &\exp\left(\frac{Hr_{\max}}{2m}\right)
    \nu(K)\sqrt{d}
    \Bigg(
        \sqrt{
            Kmd
            \max
            \left\{
                \frac{H^2}{m\alpha\lambda},
                1
            \right\}
        }
        +
        H\sqrt{\frac{K}{\alpha\lambda}}
    \Bigg)
\Bigg).
\label{eq:seg-bits-upper-bound}
\end{align}

\end{theorem}

Under binary segment feedback, increasing the number of segments \(m\) substantially reduces the regret bound, as captured by the dominating exponential factor. This exponential dependence originates from the information loss induced by the curvature of the sigmoid function.
Specifically, when $m$ is small, each segment is long and the sigmoid output is close to either $0$ or $1$. In this regime, the sigmoid function curve is nearly flat, and it is statistically difficult to infer the underlying reward parameter. Otherwise, when $m$ is large, each segment is short and the sigmoid function curve is steep, making the feedback more informative about the reward parameter. This learning difficulty is characterized by
the sigmoid-curvature factor $\sqrt{\alpha}
=
\Theta(\exp\left(\frac{Hr_{\max}}{2m}\right)
)$.
%, which gives rise to the dominant exponential factor in the regret bound.

\subsubsection{Regret Lower Bound}
Now we establish the first regret lower bound for linear MDPs with binary segment feedback to the best of our knowledge. It demonstrates that the exponential dependence on the segment length in the regret guarantee is indispensable.

\begin{theorem}
\label{thm:seg-bits-lower}
Consider RL with binary segment feedback and known transitions. Suppose that $r_{\max}\ge 1/2$ and
$
K
\ge
\frac{(d-1)^2}
{m\,\sig'\!\left(\frac{H}{m}r_{\max}-\frac12\right)}
\max\left\{
\frac{m^2}{64H^2r_{\max}^2},
\frac14
\right\}$.
% with $|\mathcal A|=2^{d-1}$ satisfying Assumption~\ref{assump:linear_mdp} 
For any algorithm, there exists an instance such that the regret satisfies
\[
\mathbb{E} \left[\cR(K)\right]=
\Omega\left(
\exp\left(\frac{Hr_{\max}}{2m}\right)
d\sqrt{Km}
\right).
\label{eq:seg-bits-lower-bound}
\]
where $\mathbb{E} \left[\cdot\right]$ denotes the expectation over the randomness in the reward feedback and the algorithm.
\end{theorem}

Theorem~\ref{thm:seg-bits-lower} exhibits that the exponential dependence on $\frac{Hr_{\max}}{m}$ in Theorem~\ref{thm:seg-bits-upper} is intrinsic to binary segment feedback. In particular, as the number of segments $m$ increases, the regret bound decreases rapidly through the exponential factor $\exp(\frac{Hr_{\max}}{2m})$. Therefore, \emph{under binary feedback, increasing the number of segments accelerates learning. }

% \yihan{add a sentence to summarize/highlight proof novelty, e.g., ``We develop a novel lower bound analysis, which ... and derives the expoential dependence that is essential to the problem''}
% In the analysis, we view the averaged action within each segment as an effective linear bandit action, and bound the KL divergence between neighboring Bernoulli distributions using the curvature of the sigmoid function. Then, we relate this KL divergence to the probability of misidentifying the reward parameter for each direction. 
% By introducing the derivative of the sigmoid function $\sig'(x)=\frac{1}{\exp(x)+\exp(-x)+2}$, we derive the claimed exponential dependence $\exp(\frac{Hr_{\max}}{2m})$ in the regret lower bound. 
The main challenge of the lower bound proof is to derive the exponential dependence on the segment length $\frac{H}{m}$. 
We construct hard instances where the averaged action within each segment is viewed as an effective linear bandit action, and the aggregate segment reward lies in the nearly flat region of the sigmoid function. In this region, the KL divergence of Bernoulli observations under neighbouring instances is exponentially small in $\frac{Hr_{\max}}{m}$. Keeping the KL divergence bounded by a constant, the resulting choice of the reward gap $\Delta$ in the regret analysis gives rise to the exponential dependence. 
Below we provide a proof sketch of Theorem~\ref{thm:seg-bits-lower}.

\begin{proof}[Proof sketch.]
Construct the following linear MDP instance: $\cS=\{s_1,\ldots,s_N\}$ and $\cA=\{-1,1\}^{d-1}$, and for any $s_j\in\cS$,
\[
\psi(s_j,a)
=
\frac{1}{\sqrt d}
\begin{pmatrix}
a\\1
\end{pmatrix},
\qquad
\mu_h^*(s_j)
=
\begin{pmatrix}
\mathbf 0_{d-1}\\ \sqrt d/N
\end{pmatrix},
\qquad
\theta_\gamma^*
=
\begin{pmatrix}
\sqrt d\,\Delta\gamma\\
\sqrt d(r_{\max}-(d-1)\Delta)
\end{pmatrix},
\]
where $\gamma\in\{-1,1\}^{d-1}$. This gives $
P_h(s_{j'}\mid s_j,a)=\frac{1}{N}
$ for every $(s_j,a,s_{j'})$ and
$
r_\gamma(s_j,a)
=
r_{\max}-(d-1)\Delta
+
\Delta\langle a,\gamma\rangle.
$
Thus, the first $d-1$ coordinates encode $d-1$ independent components of the reward parameter, while the last coordinate induces uniform transitions and a common positive reward baseline. The optimal action is $a=\gamma$, and hence learning an optimal policy requires identifying all signs of $\gamma$.

For the $t$-th segment, let
$
\bar A_t
\coloneqq
\frac{1}{H/m}
\sum_{\ell=1}^{H/m}a_{t,\ell}
$. The averaged action $\bar A_t$ can be regarded as an action in a $(d-1)$-dimensional linear bandit problem, since the aggregate segment reward depends on $\gamma$ only through the inner product $\langle \bar A_t,\gamma\rangle$. 
Conditional on the chosen actions, the binary feedback is Bernoulli with mean $
\sig(
\frac{H}{m}
\left[
r_{\max}-(d-1)\Delta
+
\Delta\langle\bar A_t,\gamma\rangle
\right]
).
$ 
We choose
\begin{align}
\Delta
=
\frac{1}{
8\cdot\frac{H}{m}
\sqrt{
Km\,\sig'\!\left(
\frac{H}{m}r_{\max}-\frac12
\right)
}
}. \label{eq:def_Delta}
\end{align}
Under the stated condition on $K$ in Theorem~\ref{thm:seg-bits-lower}, this choice ensures that every aggregate segment reward satisfies
$$
\frac{H}{m}
\left(
r_{\max}-(d-1)\Delta
+
\Delta\langle\bar A_t,\gamma\rangle
\right)
\ge
\frac{H}{m}
\left(
r_{\max}-2(d-1)\Delta
\right)
\ge
\frac{H}{m}r_{\max}-\frac12
\ge 0.
$$
Hence, all aggregate segment rewards lie in the region where \(\sig'(\cdot)\) is decreasing. 

Then, we analyze the statistical information available for distinguishing neighboring instances. Fix $i\in[d-1]$, and let $\gamma'$ differ from $\gamma$ only in coordinate $i$. The segment aggregate rewards under these two parameters differ by $
2\frac{H}{m}\Delta
\left|
\bar A_{t,i}
\right|.
$
For any $x,y\in\mathbb R$, the KL divergence of the Bernoulli distributions parameterized by $x$ and $y$ satisfies

\begin{align}
\operatorname{KL}
\left(
\mathcal B(\sig(x)),
\mathcal B(\sig(y))
\right)
\le
\frac12
\sup_{u\in[x,y]}
\sig'(u)(x-y)^2. \label{eq:kl_bernoulli}
\end{align}

% Let $\mathbb P_\gamma$ denote the probability measure induced by any algorithm under the instance parameterized by $\gamma$. 

Let \(\mathbb D_\gamma\) denote the distribution of the interaction history generated by the algorithm under the instance parameterized by \(\gamma\).
Since all aggregate segment rewards are at least $\frac{H}{m}r_{\max}-\frac12$ and $\sig'(\cdot)$ is decreasing on $[0,\infty)$, using Eq.~\eqref{eq:kl_bernoulli}, we have
%following analysis of (Lattimore and Szepesv{\'a}r,~\citeyear{lattimore2020bandit}) yields 

$$
\begin{aligned}
\operatorname{KL}
\left(
\mathbb D_\gamma,
\mathbb D_{\gamma'}
\right)
&\le
\frac12
\sig'\!\left(
\frac{H}{m}r_{\max}-\frac12
\right)
\sum_{t=1}^{Km}
\mathbb E_\gamma
\left[
4
\left(\frac{H}{m}\right)^2
\Delta^2\bar A_{t,i}^2
\right]
\\
&\le
2Km\,
\sig'\!\left(
\frac{H}{m}r_{\max}-\frac12
\right)
\left(\frac{H}{m}\right)^2
\Delta^2 
\\
&=
\frac{1}{32}.
\end{aligned}
$$

% Therefore, even after all $Km$ segment observations, neighboring instances remain difficult to distinguish. The exponential dependence arises from the sigmoid curvature in the preceding KL bound. In particular,
Here the derivative of the sigmoid function induces the exponential factor:
$$
\begin{aligned}
\sig'\!\left(
\frac{H}{m}r_{\max}-\frac12
\right)
&=
\frac{1}{
\exp\!\left(
\frac{H}{m}r_{\max}-\frac12
\right)
+
2
+
\exp\!\left(
-\frac{H}{m}r_{\max}+\frac12
\right)
}
\\
&=
\Theta\!\left(
\exp\!\left(
-\frac{Hr_{\max}}{m}
\right)
\right).
\end{aligned}
$$

% Consequently,

% $$
% \frac{1}{
% \sqrt{
% \sig'\!\left(
% \frac{H}{m}r_{\max}-\frac12
% \right)
% }
% }
% =
% \Theta\!\left(
% \exp\!\left(
% \frac{Hr_{\max}}{2m}
% \right)
% \right).
% $$

% Thus, although the total KL divergence is bounded by a constant, the exponentially small curvature allows us to choose
% Following the classical lower bound analytical procedure~(Lattimore and Szepesv{\'a}r,~\citeyear{lattimore2020bandit}), 
% Recall that we choose
% $
% \Delta
% =
% \Theta\!\left(
% \frac{
% \exp\!\left(
% \frac{Hr_{\max}}{2m}
% \right)
% }{
% \frac{H}{m}\sqrt{Km}
% }
% \right)
% $.
% Recall the choice of $\Delta$ in Eq.~\eqref{eq:def_Delta}, which gives rise to the exponential dependence in the lower bound. 
Finally, the Bretagnolle-Huber inequality implies that the sign of each coordinate under two neighbouring instances cannot be reliably distinguished. 
%Averaging over $\gamma\in\{-1,1\}^{d-1}$, we can obtain that, a constant fraction of the \(d-1\) action coordinates have incorrect signs over the \(Km\) segments in expectation.
Averaging over $\gamma\in\{-1,1\}^{d-1}$ shows that the expected number of coordinate-wise sign errors across the $Km$ segments is at least a constant fraction of $(d-1)Km$. Since each coordinate-wise error incurs a  regret of $\frac{H}{m}\Delta$, we have

$$
\begin{aligned}
\mathbb{E} \left[ \cR(K)\right]
&=
\Omega\!\left(
dKm\frac{H}{m}\Delta
\right)
\\
&=
\Omega\!\left(
d
\sqrt{
\frac{Km}{
\sig'\!\left(
\frac{H}{m}r_{\max}-\frac12
\right)
}
}
\right)
\\
&=
\Omega\!\left(
d
\exp\!\left(
\frac{Hr_{\max}}{2m}
\right)
\sqrt{Km}
\right) ,
\end{aligned}
$$
which establishes the claimed lower bound.
%and shows that the exponential dependence on $\frac{H}{m}$ is unavoidable under binary segment feedback.
\end{proof}

\subsection{Sum Feedback}
For sum feedback, we develop algorithm $\edlinucbsegd$, which uses the E-optimal experimental design to refine the conditioning of the segment covariance matrix during the initial exploration phase. We provide regret upper and lower bounds  to demonstrate the near-optimality of $\edlinucbsegd$.
%, yielding near-optimal dependence on $d$, $H$, and $K$, with no polynomial dependence on $m$. 

\subsubsection{Algorithm $\edlinucbsegd$ and Regret Upper Bound}
\label{sec:known-sum-upper}

Under sum segment feedback, the reward observation collected from segment $\tau_i^k$ takes the linear form
$    R_i^k
    =
    \langle \phi^{\tau_i^k},\theta^*\rangle
    +
    \varepsilon_i^k.$
Moreover, since the transition kernels are known, 
% the expected episode feature $\phi^\pi$ can be evaluated for any policy $\pi$, and its expected cumulative reward is $(\phi^\pi)^\top\theta^*$. Consequently, 
reward learning in this setting can be viewed as a linear bandit problem (Dani et al., \citeyear{dani2008colt-stochastic}; Abbasi-Yadkori et al., \citeyear{abbasi2011improved}) where each episode selects a policy and produces $m$ reward observations. 
%This motivates an optimistic algorithm based on LinUCB~(Abbasi-Yadkori et al., \citeyear{abbasi2011improved}).
However, a direct application of existing linear bandit algorithms may result in overly large confidence intervals due to a poorly conditioned covariance matrix, leading to a suboptimal regret bound. 

To address this issue and accurately reveal the influence of segments on the
regret performance, $\edlinucbsegd$ adopts the E-optimal design~(Pukelsheim,
\citeyear{pukelsheim2006optimal}) in the initial exploration phase.
The procedure of $\edlinucbsegd$ is presented in
Algorithm~\ref{alg:ed_linucb_seg_d}.
Specifically, the E-optimal design step finds a distribution $w^*$ over
policies that minimizes the spectral norm of the inverse expected covariance
matrix, or equivalently, maximizes the minimum eigenvalue of the expected
covariance matrix (Line~\ref{line:e_optimal_design}).
The distribution $w^*$ is then converted into a sequence of $K_0$ exploration
policies using the rounding procedure $\round$ in~(Allen-Zhu et al.,
\citeyear{allen2021near})
(Lines~\ref{line:K_0} and~\ref{line:round}).
Next, $\edlinucbsegd$ plays these $K_0$ policies
(Line~\ref{line:initial_exploration}).
This initial exploration phase guarantees a sufficiently well-conditioned
covariance matrix.
After this phase, $\edlinucbsegd$ constructs a regularized least-squares
estimate for reward parameter $\theta^*$ (Line~\ref{line:hat_theta}),
and selects the policy with the highest optimistically estimated cumulative
reward (Line~\ref{line:pi_k_elinucb}).

\begin{algorithm}[H]
    \caption{$\edlinucbsegd$}
    \label{alg:ed_linucb_seg_d}
    \begin{algorithmic}[1]
        \State {\bfseries Input:}
        $\delta,\delta' \coloneqq \frac{\delta}{3},
        \lambda \coloneqq \frac{H}{r_{\max}^2 m},$
        rounding procedure $\round$,
        rounding approximation parameter $\gamma \coloneqq \frac{1}{10}$.
        $\beta(k) \coloneqq
        \sqrt{\frac{Hd}{m}
        \left(
        \log(1+\frac{kH^2}{\lambda d m})
        +2\log(\frac{1}{\delta'})
        \right)}
        +r_{\max}\sqrt{\lambda d}, \forall k>0$.

        \State \label{line:e_optimal_design}
        Let $w^* \in \triangle_{\Pi}$ and $z^*$ be the optimal solution
        and optimal value of the optimization:
        \begin{align}
            \min_{w \in \triangle_{\Pi}}
            \Bigg\|
            \Bigg(
            \sum_{\pi \in \Pi} w(\pi)
            \bigg(
            \sum_{i=1}^{m}
            \ex_{\tau_i \sim \pi}
            \mbr{\phi^{\tau_i}(\phi^{\tau_i})^\top}
            \bigg)
            \Bigg)^{-1}
            \Bigg\|
            \label{eq:ed_opt}
        \end{align}

        \State
        $K_0 \leftarrow
        \lceil \max\{
        26(1+\gamma)^2(z^*)^2H^4
        \log(\frac{2d}{\delta'}),
        \ \frac{d}{\gamma^2}
        \} \rceil$
        \label{line:K_0}\;

        \State
        $(\pi^1,\dots,\pi^{K_0}) \leftarrow
        \round(
        \{\sum_{i=1}^{m}
        \ex_{\tau_i \sim \pi}
        \mbr{\phi^{\tau_i}(\phi^{\tau_i})^\top}
        \}_{\pi \in \Pi},
        w^*, \gamma, K_0)$
        \label{line:round}\;

        \State
        Play $K_0$ episodes with policies $\pi^1,\dots,\pi^{K_0}$.
        Observe trajectories $\tau^1,\dots,\tau^{K_0}$ and rewards
        $\{R^1_i\}_{i=1}^{m},\dots,\{R^{K_0}_i\}_{i=1}^{m}$
        \label{line:initial_exploration}\;

        \For{$k=K_0+1,\dots,K$}

        \State
        $\hat{\theta}_{k-1} \leftarrow
        (\lambda I
        +\sum_{k'=1}^{k-1}\sum_{i=1}^{m}
        \phi^{\tau^{k'}_i}(\phi^{\tau^{k'}_i})^\top)^{-1}
        \sum_{k'=1}^{k-1}\sum_{i=1}^{m}
        \phi^{\tau^{k'}_i}R^{k'}_i$
        \label{line:hat_theta}\;

        \State
        $\Sigma_{k-1} \leftarrow
        \lambda I
        +\sum_{k'=1}^{k-1}\sum_{i=1}^{m}
        \phi^{\tau^{k'}_i}(\phi^{\tau^{k'}_i})^\top$
        \label{line:Sigma}\;

        \State
        $\pi^k \leftarrow
        \argmax_{\pi \in \Pi}
        ((\phi^{\pi})^\top\hat{\theta}_{k-1}
        +\beta(k-1)\cdot
        \|\phi^{\pi}\|_{(\Sigma_{k-1})^{-1}})$
        \label{line:pi_k_elinucb}\;

        \State
        Play episode $k$ with policy $\pi^k$.
        Observe trajectory $\tau^k$ and sum segment feedback
        $\{R^k_i\}_{i=1}^{m}$
        \label{line:play_episode_elinucb}\;

        \EndFor
    \end{algorithmic}
\end{algorithm}

Now we state the regret guarantee for algorithm $\edlinucbsegd$.

\begin{theorem}
\label{thm:regret_sum_fixed_segment_d_known}
Under Assumption~\ref{assump:linear_mdp}, with probability at least $1-\delta$, for any $K>0$, the regret of Algorithm~\ref{alg:ed_linucb_seg_d} satisfies
\[
\cR(K)
=
O\left(
d\sqrt{HK}
\log\left(
\left(
1+\frac{KH r_{\max}^2}{dm}
\right)
\frac{1}{\delta}
\right)
+
(z^*)^2H^5
\log\left(
\frac{d}{\delta}
\right)
+
dH
\right).
\]
\end{theorem}

When $K$ is sufficiently large, the dominant term in Theorem~\ref{thm:regret_sum_fixed_segment_d_known} is
$\widetilde O(d\sqrt{HK})$, while the second order terms arise from the E-optimal exploration phase. As shown by the lower bound in the next subsection, this upper bound for algorithm $\edlinucbsegd$ is near-optimal. 

This result conveys a different insight on the impact of segments from that under binary feedback: increasing the number of segments does not substantially decrease the regret. Specifically, a larger $m$ provides more reward observations, but meanwhile, under shorter segments, the segment feature \(\phi^{\tau^{k'}_i}\) contributed to the covariance matrix \(\Sigma_{k}\) is smaller, inflating the reward estimation uncertainty \(\|\phi^{\pi}\|_{(\Sigma_{k})^{-1}}\). These two effects cancel out when estimating the cumulative expected reward over an entire episode. Consequently, $m$ appears only in the logarithmic factor, and increasing $m$ yields no polynomial improvement in regret performance.

\subsubsection{Regret Lower Bound}
\label{sec:known-sum-lower}

We now establish a lower bound to demonstrate the tightness of our upper bound.
%to corroborate that the dependence on $K$, $H$, and $m$ in  Theorem~\ref{thm:regret_sum_fixed_segment_d_known} cannot be improved in general (when ignoring logarithmic factors).

\begin{theorem}
\label{thm:regret_sum_fixed_segment_d_known_lower}
Consider RL with sum segment feedback and known transitions. Suppose that
$K \geq \frac{(d-1)^2}{H r_{\max}^2}$.
 For any algorithm, there exists an instance such that the regret satisfies 
\[
   \mathbb{E}\left[\cR(K)\right]= \Omega\left(
        d\sqrt{HK}
    \right),
\]
where $\mathbb{E} \left[\cdot\right]$ denotes the expectation over the randomness in the reward feedback and the algorithm.
\end{theorem}

Theorem~\ref{thm:regret_sum_fixed_segment_d_known_lower} validates that the regret upper bound for algorithm $\edlinucbsegd$ (Theorem~\ref{thm:regret_sum_fixed_segment_d_known}) nearly matches the lower bound. More importantly, this lower bound confirms that the essential regret for reward learning under sum segment feedback is independent of $m$. Therefore, unlike binary feedback, \emph{under sum feedback, increasing the number of segments does not  substantially reduce the regret}.

\section{Equal-length Segment Feedback with Unknown Transitions}

With known transitions, a policy can be evaluated once a candidate reward parameter is available in a manner as in linear bandits. When the transitions are unknown, however, the agent must learn the reward parameter and transition dynamics simultaneously. This introduces a mismatch between segment-level reward feedback and step-wise value iteration, preventing a straightforward application of existing linear MDP algorithms, e.g., (Jin et al.~\citeyear{jin20_linearmdp}). To bridge this mismatch, we develop $\seglsvits$, a unified framework that establishes step-wise Bellman recursion using multiple Thompson-sampled reward models inferred from segment feedback. 
For binary and sum feedback, it respectively performs logistic and linear regression with posterior sampling to estimate the reward parameter using segment-level observations, and integrates the estimated reward parameters into  optimistic least-squares value iteration to ensure optimism for both reward and transition uncertainty.

\subsection{A Unified $\seglsvits$ Framework}
\label{sec:unknown-unified-framework}

Adapting least-squares value iteration to integrate segment reward feedback is highly non-trivial due to the following challenges: 
% The main challenge is therefore to achieve reward optimism while retaining a step-wise Markovian Bellman recursion.
(i) The reward parameter is estimated from segment features 
$\phi^{\tau_i^k}=\sum_{h\in\tau_i^k}\psi(s_h^k,a_h^k)$, so the reward uncertainty is characterized by a segment-level covariance matrix $\Sigma_{k-1}$. Introducing a per-step reward bonus of the form $\|\psi(s,a)\|_{\Sigma_{k-1}^{-1}}$ will require controlling the accumulated reward bonus over each step using the segment-level covariance matrix. However, 
%since each segment feature is a sum of step-wise features, 
there is generally no positive-semidefinite ordering relationship between the segment-level covariance matrix and the step-wise covariance matrix, and hence the usual elliptical potential argument fails. (ii) Alternatively, a bonus defined directly on a segment feature depends on the within-segment transition and cannot be incorporated into a Markovian state-action value function. (iii) Finally, omitting reward exploration bonuses leaves the transition bonus insufficient to guarantee optimism and breaks the analysis of optimistic least-squares value iteration.
%with respect to the unknown reward parameter. 

%To isolate the common algorithmic structure, 
To present the unified algorithmic framework compatible with different feedback types, 
let
$\mathsf{F}\in\{\mathrm{bi},\mathrm{sum}\}$ denote the binary and sum feedback, respectively, and let $\mathcal D_{k-1}^{\mathsf F}$ contain all trajectories and segment reward observations collected before episode $k$. We use
\[
\bigl(
    \hat{\theta}_{k-1}^{\mathsf F},
    \Sigma_{k-1}^{\mathsf F},
    \rho_{k-1}^{\mathsf F}
\bigr)
\leftarrow
\textsc{RewardEst}_{\mathsf F}
\bigl(\mathcal D_{k-1}^{\mathsf F}\bigr)
\]
to represent the feedback-specific reward estimator, covariance matrix, and sampling radius. We also let $B_{\mathsf F}$ and $\beta_{\mathsf F}$ denote the clipping and transition bonus parameters, respectively. Their concrete forms will be specified for binary and sum feedback separately. With this notation, algorithms for both feedback types are instances of the following framework.

\begin{algorithm}[H]
\caption{A Unified $\seglsvits$ Framework}
\label{alg:unified-seg-lsvi-ts}
\begin{algorithmic}[1]
    \State {\bfseries Input:}
    feedback type $\mathsf F$,
    $\delta$, $\delta'\coloneqq\delta/6$, $\lambda$,
       $n
        \coloneqq
        \left\lceil
            2\sqrt{2\pi e}
            \log\left(\frac{K}{\delta'}\right)
        \right\rceil,$
    $B_{\mathsf F}$, and $\beta_{\mathsf F}$.

    \For{$k=1,\dots,K$}
        \State Receive the initial state $s_1^k$.

        \State
        $
        \bigl(
            \hat{\theta}_{k-1}^{\mathsf F},
            \Sigma_{k-1}^{\mathsf F},
            \rho_{k-1}^{\mathsf F}
        \bigr)
        \leftarrow
        \textsc{RewardEst}_{\mathsf F}
        \bigl(\mathcal D_{k-1}^{\mathsf F}\bigr)
        $.

        \For{$j=1,\dots,n$}
            \State Sample $g_k^{(j)}\sim\mathcal N(0,I_d)$.

            \State
            $
            \widetilde{\theta}_k^{(j)}
            \leftarrow
            \hat{\theta}_{k-1}^{\mathsf F}
            +
            \rho_{k-1}^{\mathsf F}
            \bigl(\Sigma_{k-1}^{\mathsf F}\bigr)^{-1/2}
            g_k^{(j)}
            $.

            \State
            $
            \bigl(
                \{Q_h^{k,(j)}\}_{h=1}^H,
                \{V_h^{k,(j)}\}_{h=1}^{H}
            \bigr)
            \leftarrow
            \lsvitsq
            \left(
                \widetilde{\theta}_k^{(j)},
                k,
                B_{\mathsf F},
                \beta_{\mathsf F}
            \right)
            $.
        \EndFor

        \State
        $
        j^+
        \leftarrow
        \arg\max_{j\in[n]}
        V_1^{k,(j)}(s_1^k)
        $.

        \For{$h=1,\dots,H$}
            \State
            $
            a_h^k
            \leftarrow
            \arg\max_{a\in\mathcal A}
            Q_h^{k,(j^+)}(s_h^k,a)
            $.

            \State Take action $a_h^k$ and observe $s_{h+1}^k$.
        \EndFor

        \State Observe the segment feedback under model $\mathsf F$ and update $\mathcal D_k^{\mathsf F}$.
    \EndFor
\end{algorithmic}
\end{algorithm}

\begin{algorithm}[H]
\caption{$\lsvitsq$
$(\theta,k,B,\beta)$}
\label{alg:unified-lsvi-ts-q}
\begin{algorithmic}[1]
    \State {\bfseries Input:}
    $\theta$, $k$, $B$, and $\beta$.
    \State Initialize $V_{H+1}^k(\cdot)\leftarrow 0$, $\lambda'=1$.

    \For{$h=H,\dots,1$}
        \State
            $
            \Lambda_h^{k-1}
            \leftarrow
            \lambda'I
            +
            \sum_{k'=1}^{k-1}
            \psi(s_h^{k'},a_h^{k'})
            \psi(s_h^{k'},a_h^{k'})^\top
            $.
        \State
        $
        \mathbf w_h^k
        \leftarrow
        (\Lambda_h^{k-1})^{-1}
        \sum_{k'=1}^{k-1}
        \psi(s_h^{k'},a_h^{k'})
        [
            \psi(s_h^{k'},a_h^{k'})^\top\theta
            +
            V_{h+1}^k(s_{h+1}^{k'})
        ]
        $.

        \State
        $
        Q_h^k(\cdot,\cdot)
        \leftarrow
        \operatorname{clip}_{[-HB,HB]}
        (
            (\mathbf w_h^k)^\top\psi(\cdot,\cdot)
            +
            \beta
            \sqrt{
                \psi(\cdot,\cdot)^\top
                \bigl(\Lambda_h^{k-1}\bigr)^{-1}
                \psi(\cdot,\cdot)
            }
        )
        $, where 
        $   \operatorname{clip}_{[L,U]}(x)
        :=
        \min\{U,\max\{L,x\}\}$ for any $x,L,U\in\mathbb R$ such that $L\leq U$
        \State
        $
        V_h^k(s)
        \leftarrow
        \max_{a\in\mathcal A}Q_h^k(s,a)
        $.
    \EndFor

    \State {\bfseries Output:}
    $\{Q_h^k\}_{h=1}^H$ and $\{V_h^k\}_{h=1}^{H}$.
\end{algorithmic}
\end{algorithm}

Algorithm~\ref{alg:unified-seg-lsvi-ts} illustrates the procedure of $\seglsvits$.
For each of the $n=O(\log(K/\delta))$ sampled reward parameters, $\seglsvits$ computes a candidate policy under the induced linear reward function by performing $H$-step least-squares Bellman recursion. The transition data and covariance matrices are shared across all reward parameter candidates.
%, with candidate-specific regression targets. 
$\seglsvits$ then selects the policy with the highest optimistic initial-state value.
%, without augmenting the state space with segment histories or enumerating states.

The key idea of $\seglsvits$ is to achieve overall optimism by connecting the reward optimism of Thompson sampling with the transition optimism of least-squares value iteration in subroutine $\lsvitsq$. Conditional on the past, each sampled parameter $\widetilde{\theta}_k^{(j)}$ has a constant probability of being optimistic in the direction $\phi^{\pi^*}(s_1^k)$. Hence, with $n=O(\log(K/\delta))$ independent samples, with high probability, there exists $j'\in[n]$ such that \[ \left(\phi^{\pi^*}(s_1^k)\right)^\top \widetilde{\theta}_k^{(j')} \geq \left(\phi^{\pi^*}(s_1^k)\right)^\top \theta^* . \] 
For this candidate reward parameter  $\widetilde{\theta}_k^{(j')}$, the transition bonus in $\lsvitsq$ guarantees the optimism of its corresponding value function. With high probability, there exists a candidate $j'\in[n]$ such that the selection of $j^+$, the optimism of value iteration, and the optimism of the sampled reward parameter imply
\[
V_1^{k,(j^+)}(s_1^k)
\geq V_1^{k,(j')}(s_1^k)
\geq \left(\phi^{\pi^*}(s_1^k)\right)^\top
     \widetilde{\theta}_k^{(j')}
\geq \left(\phi^{\pi^*}(s_1^k)\right)^\top \theta^*.
\]
In $\seglsvits$, the Thompson sampling in the outer framework handles reward uncertainty at the segment level, while the subroutine $\lsvitsq$ handles transition uncertainty at the step level. This coupling restores optimism without introducing an explicit reward bonus that depends on within-segment transition into the Bellman recursion. Therefore, this algorithmic framework effectively resolves the aforementioned challenges due to the mismatch between segment-level reward uncertainty and step-wise least-squares value iteration. 
%Moreover, transition errors are controlled by the step-wise covariance
%matrices $\{\Lambda_h^{k-1}\}_{h=1}^H$, whereas reward errors are
%aggregated into segment-level quantities and controlled by
%$\Sigma_{k-1}^{\mathsf F}$. Hence, the analysis does not require a
%positive-semidefinite comparison between the two covariance structures.

\subsection{Binary Feedback}
\label{sec:unknown-binary-feedback}

We first instantiate the unified framework for binary segment feedback and refer to the resulting algorithm as  $\biseglsvits$. Recall that
$\alpha=\exp(Hr_{\max}/m)+\exp(-Hr_{\max}/m)+2$, and $\nu(k)$ is defined as in Eq.~\eqref{eq:def_nu}. Based on the binary feedback collected before episode $k$, we set
\begin{equation}
\label{eq:binary-unknown-reward-module}
\begin{aligned}
\hat{\theta}_{k-1}^{\mathrm{bi}}
\coloneqq {}&
\argmin_{\theta} \Bigg(
-\sum_{k'=1}^{k-1}\sum_{i=1}^{m}
\bigg(
y_i^{k'}
\log\left(
\sig\left(
(\phi^{\tau_i^{k'}})^\top\theta
\right)
\right)
\\[-2mm]
&\hspace{28mm}
+
(1-y_i^{k'})
\log\left(
1-
\sig\left(
(\phi^{\tau_i^{k'}})^\top\theta
\right)
\right)
\bigg)
+
\frac{\lambda}{2}\|\theta\|_2^2 \Bigg)
,
\\
\Sigma_{k-1}^{\mathrm{bi}}
\coloneqq{}&
\sum_{k'=1}^{k-1}\sum_{i=1}^{m}
\phi^{\tau_i^{k'}}
(\phi^{\tau_i^{k'}})^\top
+
\alpha\lambda I,
\\
\rho_{k-1}^{\mathrm{bi}}
\coloneqq{}&
\sqrt{\alpha}\,\nu(k-1).
\end{aligned}
\end{equation}
The corresponding clipping and transition bonus parameters are
\begin{equation}
\label{eq:binary-unknown-lsvi-parameters}
\begin{aligned}
B_{\mathrm{bi}}
\coloneqq{}&
\frac{\nu(K)}{\sqrt{\lambda}}
\left(
\sqrt d
+
2\sqrt{
\log\left(\frac{Kn}{\delta'}\right)
}
+
1
\right)
+
r_{\max}\sqrt d,
\\
\beta_{\mathrm{bi}}
\coloneqq{}&
c'_\beta dH B_{\mathrm{bi}}
\sqrt{
\log\left(
\frac{2dKH}{\delta'}
\right)
},
\end{aligned}
\end{equation}
respectively, where $c'_\beta>0$ is a sufficiently large absolute constant. Substituting Eqs.~\eqref{eq:binary-unknown-reward-module} and
\eqref{eq:binary-unknown-lsvi-parameters} into
Algorithm~\ref{alg:unified-seg-lsvi-ts} yields its
binary-feedback instantiation, and we present its regret guarantee below.

\begin{theorem}
\label{thm:seg_binary_unknown_tran_d}
Suppose Assumption~\ref{assump:linear_mdp} holds. There exists a sufficiently large absolute constant $c'_\beta>0$
such that, for any $K > 0$, with probability at least $1-\delta$,
the regret of Algorithm~\ref{alg:unified-seg-lsvi-ts}
under binary segment feedback, instantiated with
Eqs.~\eqref{eq:binary-unknown-reward-module} and
\eqref{eq:binary-unknown-lsvi-parameters}, satisfies
\begin{align*}
\cR(K)
=
\widetilde O
\Bigg(
&
\exp\left(\frac{Hr_{\max}}{2m}\right)
\nu(K)d
\sqrt{
Km
\max\left\{
\frac{H^2}{m\alpha\lambda},
1
\right\}
}
+
d^2H^2\sqrt K
\left(
\frac{\nu(K)}{\sqrt{\lambda}}
+
r_{\max}
\right)
\Bigg).
\end{align*}
\end{theorem}

The first and second terms in Theorem~\ref{thm:seg_binary_unknown_tran_d} are due to reward and transition learning, respectively.
Notably, the reward learning term exhibits an exponential dependence on the segment length $H/m$, arising from the sigmoid curvature under  binary feedback, as in the known transition setting (Theorem~\ref{thm:seg-bits-upper}). Learning the transition dynamics contributes an additional polynomial term. Consequently, when the reward learning term dominates, increasing the number of segments reduces the regret bound and accelerates learning.

%Technically, the two terms correspond to reward-model error and transition-learning error. Summing the reward error over each episode converts it into a segment-level quantity controlled by $\Sigma_{k-1}^{\mathrm{bi}}$, yielding the first term. The LSVI bonuses and martingale fluctuations are instead controlled stage by stage through $\{\Lambda_h^{k-1}\}_{h=1}^H$, producing the second term without requiring a positive-semidefinite comparison between the two covariance structures. Nevertheless, the two errors remain coupled because the LSVI targets contain the sampled rewards, which accounts for the factor $\nu(K)/\sqrt{\lambda}+r_{\max}.$

\subsection{Sum Feedback}
\label{sec:unknown-sum-feedback}

We next instantiate the unified framework for sum segment feedback and refer to the resulting algorithm as  $\sumseglsvits$.
In this setting, each segment observation is linear in the reward parameter, and the noise in each observation has sub-Gaussian scale $\sqrt{H/m}$. Therefore, normalizing the segment feature and observation 
by $\sqrt{H/m}$ yields a ridge regression problem with unit-scale 
sub-Gaussian noise. Based on the normalized ridge regression with sum feedback collected 
before episode $k$, we set
\begin{equation}
\label{eq:sum-unknown-reward-module}
\begin{aligned}
\widetilde{\beta}(k)
\coloneqq{}&
\sqrt{
    d\left(\log\left(
        1+\frac{kH}{\lambda d}
    \right)
    +
    2\log\left(
        \frac{1}{\delta'}
    \right) \right)
}
+
r_{\max}\sqrt{\lambda d},
\qquad k\geq 0,
\\
\hat{\theta}_{k-1}^{\mathrm{sum}}
\coloneqq{}&
\bigl(
    \Sigma_{k-1}^{\mathrm{sum}}
\bigr)^{-1}
\sum_{k'=1}^{k-1}
\sum_{i=1}^{m}
\frac{m}{H}
\phi^{\tau_i^{k'}}
R_i^{k'},
\\
\Sigma_{k-1}^{\mathrm{sum}}
\coloneqq{}&
\lambda I
+
\sum_{k'=1}^{k-1}
\sum_{i=1}^{m}
\frac{m}{H}
\phi^{\tau_i^{k'}}
\bigl(\phi^{\tau_i^{k'}}\bigr)^\top,
\\
\rho_{k-1}^{\mathrm{sum}}
\coloneqq{}&
\widetilde{\beta}(k-1).
\end{aligned}
\end{equation}
The corresponding clipping and transition bonus parameters are
\begin{equation}
\label{eq:sum-unknown-lsvi-parameters}
\begin{aligned}
B_{\mathrm{sum}}
\coloneqq{}&
\frac{\widetilde{\beta}(K)}{\sqrt{\lambda}}
\left(
    \sqrt d
    +
    2\sqrt{
        \log\left(
            \frac{Kn}{\delta'}
        \right)
    }
    +
    1
\right)
+
r_{\max}\sqrt d,
\\
\beta_{\mathrm{sum}}
\coloneqq{}&
c_\beta dH B_{\mathrm{sum}}
\sqrt{
    \log\left(
        \frac{2dKH}{\delta'}
    \right)
},
\end{aligned}
\end{equation}
respectively, where $c_\beta>0$ is a sufficiently large absolute constant.
Substituting Eqs.~\eqref{eq:sum-unknown-reward-module} and
\eqref{eq:sum-unknown-lsvi-parameters} into
Algorithm~\ref{alg:unified-seg-lsvi-ts} yields its
sum-feedback instantiation, whose regret performance is stated below.

\begin{theorem}
\label{thm:seg_sum_unknown_tran_d}
Suppose Assumption~\ref{assump:linear_mdp} holds. There exists a sufficiently large absolute constant $c_\beta>0$
such that, for any $K>0$, with probability at least $1-\delta$,
the regret of Algorithm~\ref{alg:unified-seg-lsvi-ts}
under sum segment feedback, instantiated with Eqs.~\eqref{eq:sum-unknown-reward-module} and
\eqref{eq:sum-unknown-lsvi-parameters}, satisfies
\begin{align*}
\cR(K) = \widetilde O\Bigg(
&
d^{5/2}H^2\sqrt K
\left(
    r_{\max}
    +
    \frac{1}{\sqrt{\lambda}}
\right)
+
d
\left(
    1+r_{\max}\sqrt{\lambda}
\right)
\sqrt{
    KHd
    \max\left\{
        \frac{H}{\lambda},
        1
    \right\}
}
\Bigg).
\end{align*}
\end{theorem}

The first and second terms in Theorem~\ref{thm:seg_sum_unknown_tran_d} correspond to transition and reward learning errors, respectively. 
In contrast to binary feedback, the reward learning error under sum feedback depends only polynomially on the problem parameters, and the higher order term arises from learning the unknown transition dynamics. 
% Its magnitude is nevertheless coupled with reward estimation through $B_{\mathrm{sum}}$, since the sampled rewards enter the LSVI regression targets and determine the scale of the transition bonuses. 
The normalized ridge regression for the reward parameter in Eq.~\eqref{eq:sum-unknown-reward-module} reduces the reward confidence radius from $\widetilde O(\sqrt{Hd/m})$ to $\widetilde O(\sqrt d)$, preventing the segment noise scale $\sqrt{H/m}$ from propagating into the transition learning term. As a result, the regret bound has no explicit dependence on $m$, and therefore increasing the number of segments does not substantially decrease the regret, consistent with the known transition setting (Theorem~\ref{thm:regret_sum_fixed_segment_d_known}).

% Technically, the two terms correspond to transition-learning error and reward-model error. The first term captures the accumulated LSVI bonuses. Uniform control of the sampled reward models gives $B_{\mathrm{sum}}=\widetilde O(d(r_{\max}+1/\sqrt{\lambda}))$, while the LSVI confidence parameter and the elliptical-potential argument contribute the remaining $d$ and $\sqrt d$ factors, yielding the $d^{5/2}$ dependence. The second term arises from the accumulated reward-model error, which becomes a segment-level quantity and is controlled by the normalized covariance matrix $\Sigma_{k-1}^{\mathrm{sum}}$.

\section{Variable and Binary Segment Feedback with Known Transitions}
\label{sec:variable-segment-binary}

In this section, we generalize the binary feedback and known transition model to the variable segment setting, as introduced in
Section~\ref{sec:variable segment}. The main
motivation is to investigate whether partitioning segments  according to the realized
state-action features can further shrink the regret and expedite learning. 

Developing a provably efficient algorithm and a rigorous 
analysis for the variable segment setting is challenging because the segment lengths may vary both within and
across episodes, and hence the Bernoulli observations no longer share a common
curvature parameter. Applying a single curvature bound based on the longest
segment would treat all observations according to the least informative segment and obscure the effect of the partition. Moreover, the covariance
matrix is constructed from features associated with the partitioned segments,
and thus the feature decomposition used in the elliptical potential argument must
remain consistent with the episode-dependent segmentation.

To address these difficulties, we present algorithm $\uneqsegbitsd$, which builds upon algorithm $\bitssegd$  but adopts a segment-specific sigmoid curvature to handle variable segments. Specifically, we construct
the covariance matrix using the weight
$\sig'(\ell_i^k r_{\max})$ tailored to the length of each segment, and sample a reward parameter according to the
resulting confidence geometry. Here $\ell_i^k$ denotes the length of the $i$-th segment in episode $k$.

\begin{algorithm}[t]
\caption{$\uneqsegbitsd$}
\label{alg:uneq-segbits-d}
\begin{algorithmic}[1]
\State \textbf{Input:} $\delta$, $\delta':=\frac{\delta}{3}$, $\lambda$.
\For{$k=1,\ldots,K$}
\State 
$\begin{aligned}[t]
\hat{\theta}_{k-1}
\gets
\arg\min_{\theta}\;
&-\sum_{k'=1}^{k-1}\sum_{i=1}^{m_{k'}}
\Bigg(
y_i^{k'}
\log\sig\left(
(\phi^{\tau_i^{k'}})^\top\theta
\right) +
\\
&\qquad
(1-y_i^{k'})
\log\left(
1-\sig\left(
(\phi^{\tau_i^{k'}})^\top\theta
\right)
\right)
\Bigg)
+
\frac{\lambda}{2}\|\theta\|_2^2 .
\end{aligned}
$
\State \label{eq:uneq:covariance matrix}
$\displaystyle
\Sigma_{k-1}^{U}
\leftarrow
\sum_{k'=1}^{k-1}\sum_{i=1}^{m_{k'}}
\sig'(\ell_i^{k'}r_{\max}) \, 
\phi^{\tau_i^{k'}}
(\phi^{\tau_i^{k'}})^\top
+
\lambda I .
$
\State  Sample
$\displaystyle
\xi_k
\sim
\mathcal N\left(
0,
\nu_U(k-1)^2 \, (\Sigma_{k-1}^{U})^{-1}
\right)$,
where $\nu_U(k-1)$ is defined in Eq.~\eqref{eq:nu-unequal}.
\State 
$\displaystyle
\widetilde{\theta}_k
\leftarrow
\hat{\theta}_{k-1}+\xi_k .
$
\State 
$\displaystyle
\pi^k
\leftarrow
\argmax_{\pi}
(\phi^\pi)^\top\widetilde{\theta}_k .
$
\State  Play episode $k$ with policy $\pi^k$ and observe trajectory $\tau^k$.
\State  Choose the number of segments $m_k$ and a segment partition of trajectory $\tau^k$, and observe binary
segment feedback $\{y_i^k\}_{i=1}^{m_k}$.
\EndFor
\end{algorithmic}
\end{algorithm}

Algorithm~\ref{alg:uneq-segbits-d} gives the pseudocode of $\uneqsegbitsd$.
% The algorithm follows the same posterior-sampling principle as $\bitssegd$.
% The essential difference is the covariance matrix in Line~\ref{eq:uneq:covariance matrix}. 
Different from algorithm $\bitssegd$, $\uneqsegbitsd$ employs the covariance matrix euipped with adaptive weights to variable segments.
Since
$\|\phi^{\tau_i^k}\|_2\leq \ell_i^k$ and the reward of each state-action pair
is bounded by $r_{\max}$, for every $\theta\in\Theta$,
$    |(\phi^{\tau_i^k})^\top\theta|
    \leq
    \ell_i^k r_{\max}.$
The symmetry and monotonicity of $\sig'$ therefore imply
$    \sig'\bigl((\phi^{\tau_i^k})^\top\theta\bigr)
    \geq
    \sig'(\ell_i^k r_{\max}).$
Thus, $\Sigma_{k-1}^{U}$ provides a deterministic lower bound on the
curvature of the logistic loss while preserving the differences in the information revealed by segments of different lengths. 

For any $k>0$, define
\[
    L(k)
    :=
    \max_{k'\leq k,\,  i\in[m_{k'}]}
    \ell_i^{k'},
\]
\begin{align}
    \omega_U(k)
    &:=
    \frac{\sqrt{\lambda}}{2}
    +
    \frac{d}{\sqrt{\lambda}}
    \log\left[
        \frac{4}{\delta'}
        \left(
            1+
            \frac{
                \sig'_{\max}
                \sum_{k'=1}^{k}
                \sum_{i=1}^{m_{k'}}
                (\ell_i^{k'})^2
            }{\lambda d}
        \right)
    \right]
    +
    r_{\max}\sqrt{\lambda d},
    \label{eq:omega-unequal}
\end{align}
and
\begin{align}
    \nu_U(k)
    &:=
    \frac{\sqrt{\lambda}}{L(k)}
    \Bigg[
        1+L(k)r_{\max}\sqrt d
        +
        \frac{L(k)}{\sqrt{\lambda}}
        \Bigg(
            \sqrt{1+L(k)r_{\max}\sqrt d}\,
            \omega_U(k)
            +
            \frac{L(k)}{\sqrt{\lambda}}
            \omega_U(k)^2
        \Bigg)
    \Bigg]^{3/2}.
    \label{eq:nu-unequal}
\end{align}

The following theorem provides a regret guarantee of algorithm $\uneqsegbitsd$ for arbitrary
segmentations.

\begin{theorem}
\label{thm:regret-binary-variable-known}
Under Assumption~\ref{assump:linear_mdp}, with probability at least $1-\delta$, for any $K>0$, the regret of Algorithm~\ref{alg:uneq-segbits-d} 
satisfies
\begin{align*}
    \cR(K)=\widetilde O\Bigg(
        H\nu_U(K)\sqrt{\frac{dK}{\lambda}}
        +
        \nu_U(K)d
        \sqrt{
            \sum_{k=1}^{K}
            \max\left\{
                \frac{
                    \sum_{i=1}^{m_k}(\ell_i^k)^2
                }{4\lambda},
                1
            \right\}
            \sum_{i=1}^{m_k}
            \frac{1}{\sig'(\ell_i^k r_{\max})}
        }
    \Bigg).
\end{align*}
\end{theorem}

The main implication of Theorem~\ref{thm:regret-binary-variable-known} is that it explicitly captures the exponential dependence associated with each individual segment length through the factor:
% under variable segments. In particular, the segment-dependent exponential factor is 
\[ \sum_{i=1}^{m_k} \frac{1}{\sig'(\ell_i^k r_{\max})} = \sum_{i=1}^{m_k} \left( e^{\ell_i^k r_{\max}} +2 +e^{-\ell_i^k r_{\max}} \right). \] 
Thus, up to polynomial factors, the regret depends on $\sqrt{ \sum_{k=1}^{K} \sum_{i=1}^{m_k} e^{\ell_i^k r_{\max}} }.$ Each segment therefore contributes an exponential cost determined by its own length, and a small number of long segments can substantially enlarge the regret bound. 
In addition, we note that under a usual elliptical potential argument, the feature decomposition should be aligned with  the realized segment partition. Then, when the same segment-wise feature decomposition is used in both the elliptical norms and covariance matrix, the effect of state-action features is absorbed into logarithmic factors and only the dependence on segment lengths remains in the dominant term.

% pairing each segment feature with its corresponding curvature weight $\sig'(\ell_i^k r_{\max})$. Under this decomposition, the realized state--action features enter through their segment-wise sums, and their detailed geometry affects the regret bound only through the log-determinant of the associated covariance matrix. This dependence is therefore absorbed into the logarithmic factors suppressed by $\widetilde O$, whereas the segment lengths enter the dominant segment-dependent terms explicitly. 

% Consequently, comparing different partitions with the same $m_k$ reduces to comparing the two length-dependent quantities appearing in Theorem~\ref{thm:regret-binary-variable-known}. Both quantities favor balanced segment lengths. Indeed, 

To understand the influence of segmentation more explicitly, we fix an episode $k$ and the number of segments $m_k$.
Since $\sum_{i=1}^{m_k}\ell_i^k=H$ and $1/\sig'(\ell r_{\max})$ is convex in $\ell$, Jensen's inequality gives
\[
    \sum_{i=1}^{m_k}
    \frac{1}{\sig'(\ell_i^k r_{\max})}
    \geq
    \frac{m_k}{
        \sig'\left(
            \frac{H}{m_k}r_{\max}
        \right)
    }.
\]
Similarly, the convexity of function $\ell^2$ implies
\[
    \sum_{i=1}^{m_k}(\ell_i^k)^2
    \geq
    \frac{H^2}{m_k}.
\]
Both inequalities attain equality when
$\ell_i^k=H/m_k$ for every $i\in[m_k]$. Therefore, when the exponential
curvature term dominates, equal segmentation minimizes the dominant
segment-dependent term in the regret bound. 
% If $H$ is not divisible by
% $m_k$, the most balanced integer partition gives the corresponding
% minimizer. 
This conclusion is drawn under the
usual elliptical potential analysis, and whether there exists a more advanced algorithm or analysis that can improve the regret performance through unequal segmentation, e.g.,  exploiting the state-action features to partion differently, remains an open problem. 
In a nutshell, we show in Theorem~\ref{thm:regret-binary-variable-known} that, \emph{under standard logistic-regression-based algorithm design and elliptical potential analysis, equal segmentation achieves the best regret performance for binary feedback}, that is, unequal segmentation does not bring significant additional benefits.

Below we provide a proof sketch to explain how the effect of state-action features is absorbed into logarthimic factors and how the dependence on individual segment lengths is derived through a curvature-weighted
elliptical potential argument.

% \yihan{This proof sketch is too brief (it describes important points using only texts). should be as concrete as that for the binary feedback lower bound. 
% \\
% Should expand important points with big equations: (i) how the state-action features are absorbed to the elliptical potential analysis. (ii) the whole analytical procedure, including weighted concentration for logistic regression, and how to deal with the weighted confidence intervals using the elliptical potential analysis.
% \\
% After reading this proof sketch, the reader should almost know how to prove.
% }

\begin{proof}[Proof sketch.]
The analysis begins with the observation that each segment feature aggregates the underlying state-action features:
\[
\phi^{\tau_i^k}
=\sum_{h=t_{i-1}^k+1}^{t_i^k}\psi(s_h^k,a_h^k),
\qquad
\|\phi^{\tau_i^k}\|_2\leq\ell_i^k.
\]
While the partition may depend on the realized trajectory, it is fixed before the binary outcomes are revealed. Hence, conditional on the realized trajectory and the chosen partition, the feedback noises have zero mean and variances $\sig'((\phi^{\tau_i^k})^\top\theta^*)$. The bound $|(\phi^{\tau_i^k})^\top\theta^*|\leq\ell_i^k r_{\max}$ implies
\[
\sig'(\ell_i^k r_{\max})
\leq \sig'((\phi^{\tau_i^k})^\top\theta^*).
\]
Consequently, the algorithm's covariance matrix weighted by the segment-length-dependent curvature is dominated by the covariance matrix weighted by the true feedback variances:
\[
\Sigma_{k-1}^U
=\lambda I+\sum_{k'=1}^{k-1}\sum_{i=1}^{m_{k'}}
\sig'(\ell_i^{k'}r_{\max})
\phi^{\tau_i^{k'}}(\phi^{\tau_i^{k'}})^\top
\preceq
\lambda I+\sum_{k'=1}^{k-1}\sum_{i=1}^{m_{k'}}
\sig'((\phi^{\tau_i^{k'}})^\top\theta^*)
\phi^{\tau_i^{k'}}(\phi^{\tau_i^{k'}})^\top.
\]
A variance-weighted self-normalized concentration argument (Faury et al.,~\citeyear{faury2020improved}), combined with the self-concordance property of the logistic loss, yields the following concentration bound: with high probability,
\[
|\phi^\top(\hat\theta_{k-1}-\theta^*)|
\leq \nu_U(k-1)\|\phi\|_{(\Sigma_{k-1}^U)^{-1}},
\qquad \forall k\geq1,\ \forall\phi\in\mathbb R^d.
\]

Conditional on the history prior to sampling $\xi_k$, the Gaussian perturbation used for posterior sampling is optimistic along $\phi^{\pi^*}$ with constant probability. The posterior-sampling decomposition and martingale concentration then reduce the regret to a sum of trajectory-feature norms under the inverse weighted covariance matrice:
\[
\mathcal R(K)
=\widetilde O\left(
\nu_U(K)\sqrt d\left(
\sum_{k=1}^K\|\phi^{\tau^k}\|_{(\Sigma_{k-1}^U)^{-1}}
+H\sqrt{\frac K\lambda}\right)
+Hr_{\max}\sqrt K\right).
\]

To control these norms, we absorb the curvature weights into the segment features. These weighted features appear exactly in the matrix update:
\[
\Sigma_k^U-\Sigma_{k-1}^U
=\sum_{i=1}^{m_k}
\bigl(\sqrt{\sig'(\ell_i^k r_{\max})}\phi^{\tau_i^k}\bigr)
\bigl(\sqrt{\sig'(\ell_i^k r_{\max})}\phi^{\tau_i^k}\bigr)^\top.
\]
Since all segments in an episode use the same pre-episode matrix, the elliptical potential argument is applied at the episode level. Using the fact that $\sig'(x)\leq1/4$ for any $x \in \mathbb{R}$ and the segment-feature norm bound gives
\[
\sum_{i=1}^{m_k}
\|\sqrt{\sig'(\ell_i^k r_{\max})}\phi^{\tau_i^k}\|_{(\Sigma_{k-1}^U)^{-1}}^2
\leq
2\max\left\{\frac{\sum_{i=1}^{m_k}(\ell_i^k)^2}{4\lambda},1\right\}
\log\left( \frac{\det(\Sigma_k^U)}{\det(\Sigma_{k-1}^U)} \right) .
\]
The triangle inequality transfers the trajectory norm to segment norms. 
% Applying Cauchy--Schwarz within and across episodes removes the curvature weights, while the log-determinant increments telescope:
Applying the Cauchy--Schwarz inequality to the weighted segment norms introduces the inverse curvature factor \(1/\sig'(\ell_i^k r_{\max})\), while summing over episodes allows the log-determinant increments to telescope:
\begin{align*}
\sum_{k=1}^K\|\phi^{\tau^k}\|_{(\Sigma_{k-1}^U)^{-1}}
&\leq\sum_{k=1}^K\sum_{i=1}^{m_k}
\frac{\|\sqrt{\sig'(\ell_i^k r_{\max})}\phi^{\tau_i^k}\|_{(\Sigma_{k-1}^U)^{-1}}}
{\sqrt{\sig'(\ell_i^k r_{\max})}}\\
&\leq
\sqrt{2\sum_{k=1}^K
\max\left\{\frac{\sum_{i=1}^{m_k}(\ell_i^k)^2}{4\lambda},1\right\}
\sum_{i=1}^{m_k}\frac{1}{\sig'(\ell_i^k r_{\max})}}
\sqrt{d\log\left(1+\frac{\sum_{k=1}^K\sum_{i=1}^{m_k}(\ell_i^k)^2}{4\lambda d}\right)}.
\end{align*}
% Thus, the inverse-curvature factors arise from removing the feature weights, whereas the polynomial segment-length factors arise from feature aggregation and episode-level matrix updates. Substituting this bound into the regret decomposition and absorbing the martingale terms yields
Substituting this bound into the preceding regret decomposition bound yields:
\[ 
\mathcal R(K)=\widetilde O\left(  H\nu_U(K)\sqrt{\frac{dK}{\lambda}}
        +
\nu_U(K)d\sqrt{\sum_{k=1}^K
\max\left\{\frac{\sum_{i=1}^{m_k}(\ell_i^k)^2}{4\lambda},1\right\}
\sum_{i=1}^{m_k}\frac{1}{\sig'(\ell_i^k r_{\max})}}
\right).
\]
\end{proof}

\section{Experiments}

In this section, we provide experimental results to examine how the granularity of segment feedback  influences learning performance under binary and sum feedback. Unless otherwise specified, all experiments use the same linear MDP instance with $|\mathcal S|=30$, $|\mathcal A|=6$, $H=24$, $r_{\max}=0.1$, $\delta=0.1$ and $K=20000$. The only exception is the  experiment for the known transition and sum feedback setting, which uses a smaller instance with $|\mathcal S|=2$, $|\mathcal A|=2$, $H=24$, $r_{\max}=0.1$, $\delta=0.1$ and $K=5000$, since the design of algorithm $\edlinucbsegd$ focuses more on optimality and it is computationally inefficient. In the equal-length segment setting, we consider the number of segments $m\in\{1,2,3,4,6,8,12,24\}$. Each algorithm is repeated for $20$ independent runs. We report the average regret across runs with two-sided $95\%$ Student-$t$ confidence intervals.

Figure~\ref{fig:experimental-results}(a) presents the experimental results for the equal-length segment and binary feedback setting. The regret of algorithms $\bitssegd$ and $\biseglsvits$ decreases substantially as $m$ increases.
%, with the largest improvement occurring for small $m$. 
This shows that shorter segments provide more informative binary observations by moving closer to the steep region of the sigmoid function, matching the exponential dependence on $H/m$ in Theorems~\ref{thm:seg-bits-upper} and~\ref{thm:seg_binary_unknown_tran_d}.

Figure~\ref{fig:experimental-results}(b) reports the  results for the equal-length segment and sum feedback setting. The regret of algorithms $\edlinucbsegd$ and $\sumseglsvits$ varies only mildly with $m$, indicating that increasing the number of segments does not yield  pronounced improvement under sum segment feedback. This agrees with Theorems~\ref{thm:regret_sum_fixed_segment_d_known} and~\ref{thm:seg_sum_unknown_tran_d}, whose leading regret terms have no explicit dependence on $m$.

Finally, we evaluate algorithm $\uneqsegbitsd$ under the variable segment and binary feedback setting. For computational convenience, we use the same number of segments $m_k$ and the same trajectory partition across different episodes $k$. In the two subfigures of Figure~\ref{fig:experimental-results}(c), we set $m_k=4$ and $m_k=6$, respectively, and evaluate the algorithm from equal to increasingly unequal segmentation. 
% Each result is averaged over $10$ independent runs. 
As shown in Figure~\ref{fig:experimental-results}(c), equal segmentation achieves the smallest regret, while the regret generally increases as the segmentation becomes more uneven, particularly when it contains a very long segment. This is consistent with our theoretical result in Theorem~\ref{thm:regret-binary-variable-known}, where longer segments incur a larger sigmoid curvature cost and equal segmentation minimizes the dominant segment-length-dependent terms for fixed $m_k$.

\begin{figure*}[t]
    \centering

    \subfigure[Equal-length segment and binary feedback]{%
        \begin{minipage}{\textwidth}
            \centering
            \includegraphics[width=0.42\linewidth]
            {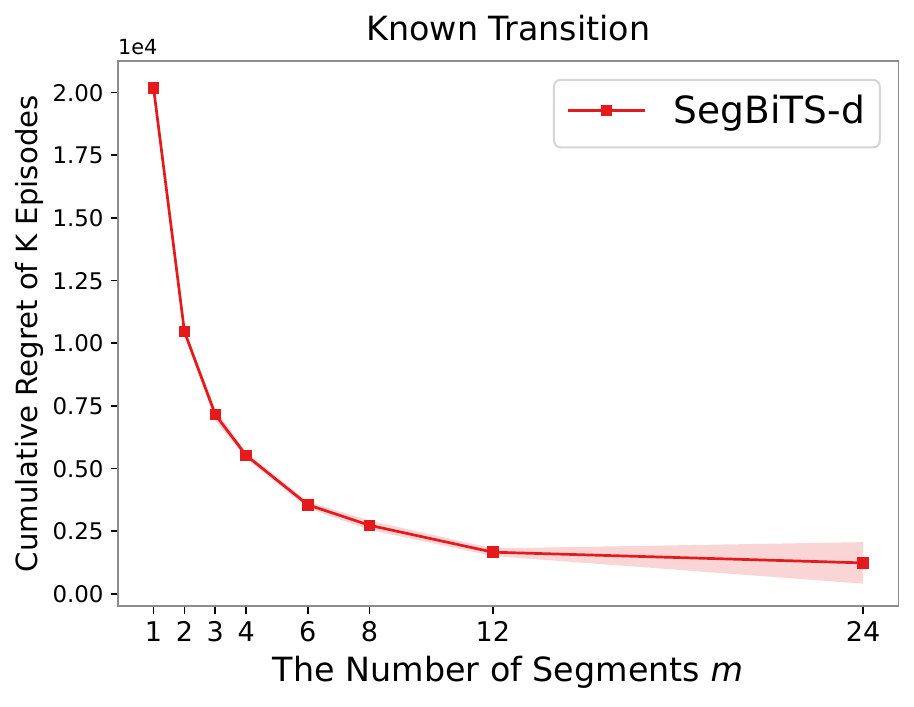}
            \hspace{0.02\linewidth}
            \includegraphics[width=0.42\linewidth]
            {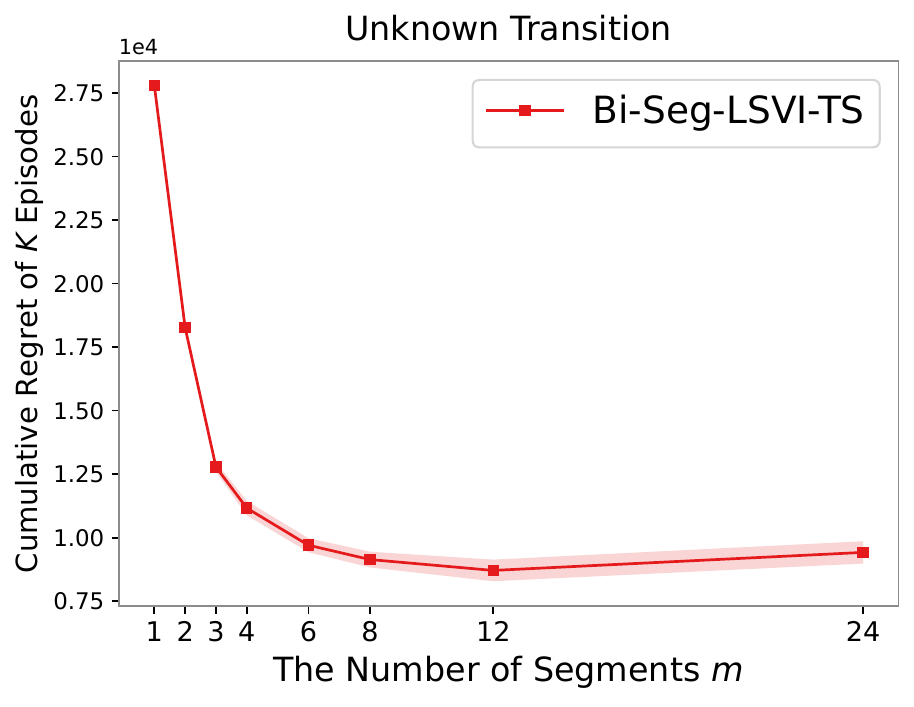}
        \end{minipage}%
    }

    \par\medskip

    \subfigure[Equal-length segment and sum feedback]{%
        \begin{minipage}{\textwidth}
            \centering
            \includegraphics[width=0.42\linewidth]
            {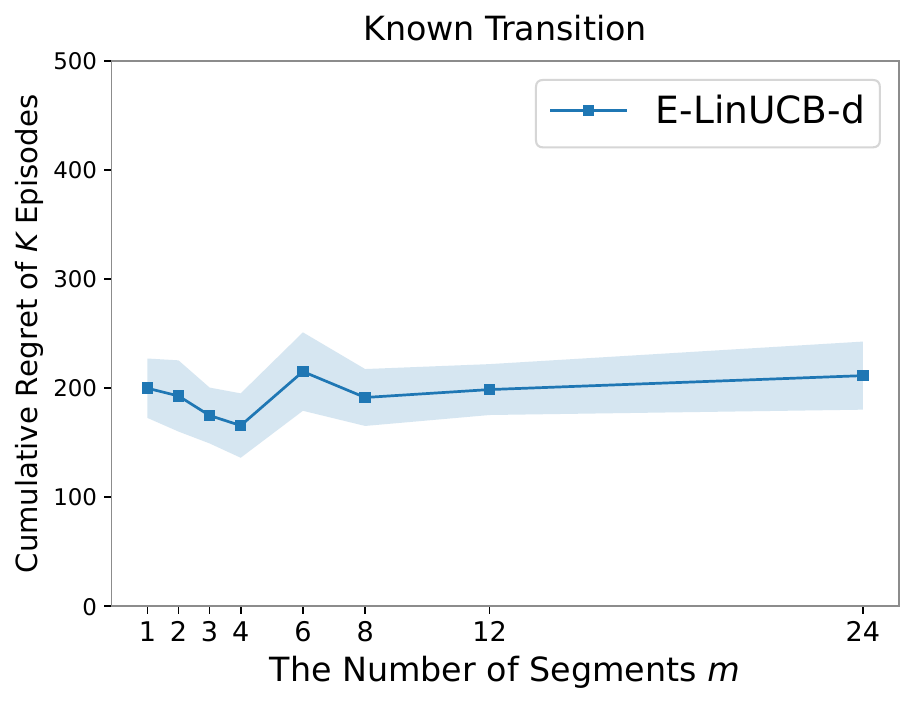}
            \hspace{0.02\linewidth}
            \includegraphics[width=0.42\linewidth]
            {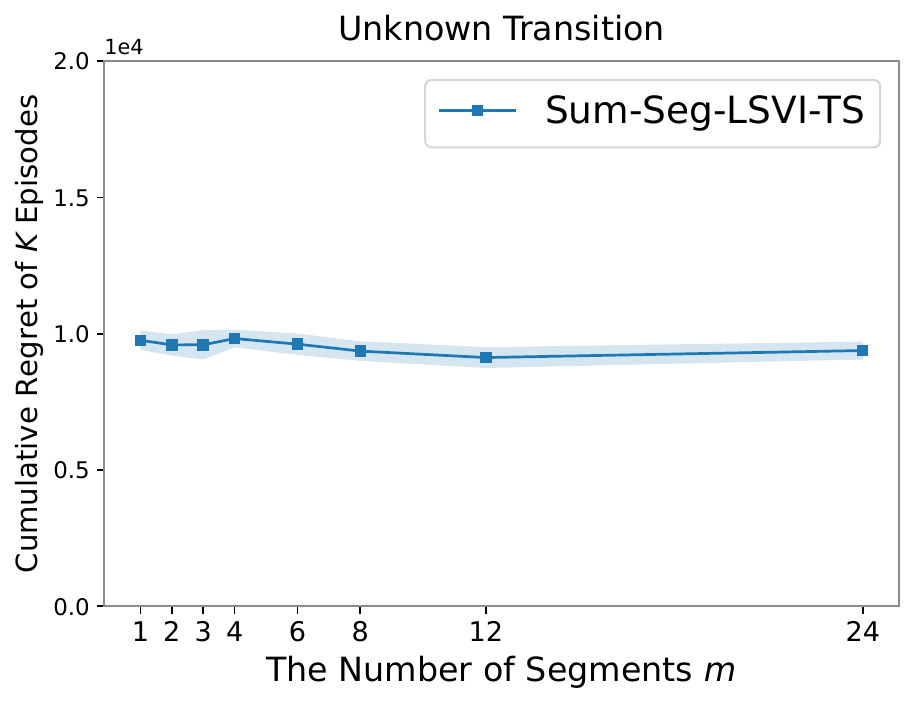}
        \end{minipage}%
    }

    \par\medskip

    \subfigure[Variable segment and binary feedback with known
    transitions. The X-axis ranges from equal to increasingly
    unequal segmentation.]{%
        \begin{minipage}{\textwidth}
            \centering
            \includegraphics[width=0.42\linewidth]
            {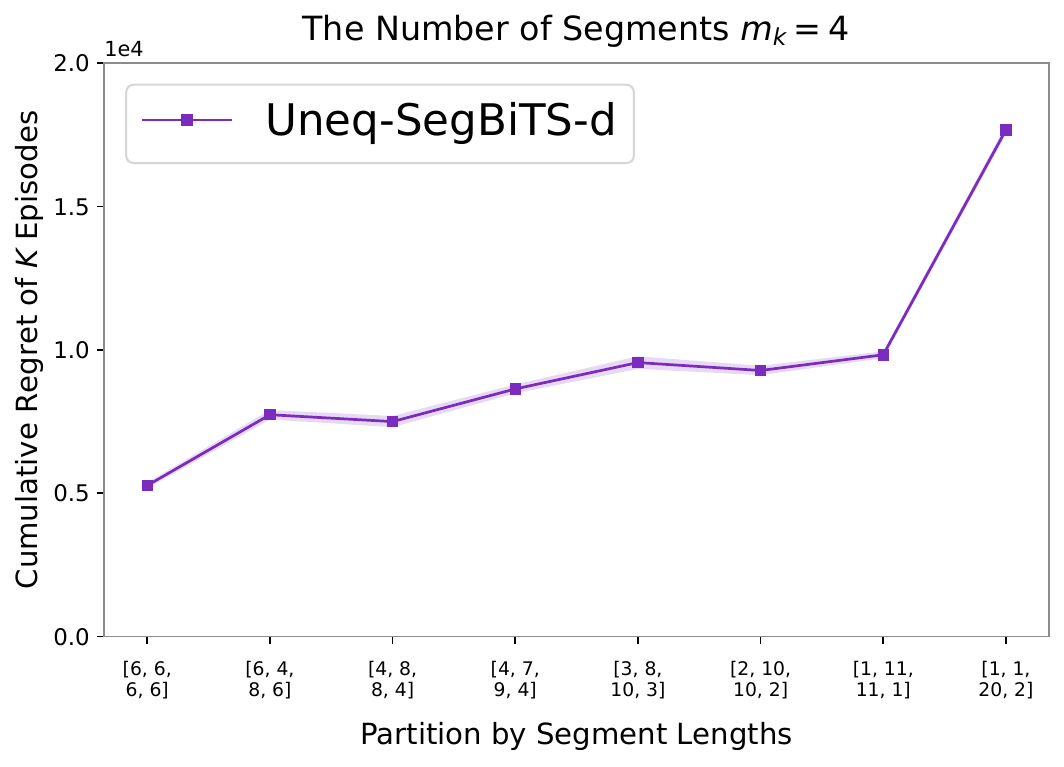}
            \hspace{0.02\linewidth}
            \includegraphics[width=0.42\linewidth]
            {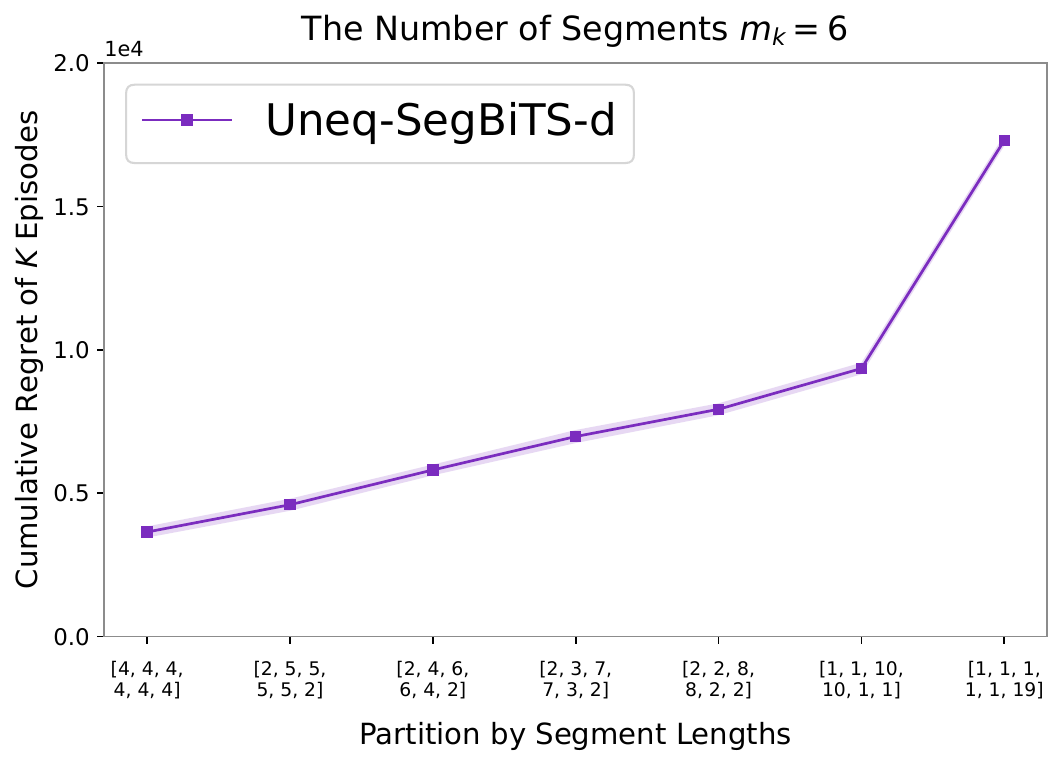}
        \end{minipage}%
    }

    \caption{Experimental results for RL with segment feedback
    in linear MDPs.}
    \label{fig:experimental-results}
\end{figure*}

\section{Conclusion and Future Work}

In this work, we study RL with segment reward feedback under linear function approximation. For equal-length segments with known and unknown transitions, we develop algorithms for both binary and sum feedback types and provide regret guarantees. Lower bounds under known transitions are established. 
These results reveal an insight: increasing the number of segments accelerates learning under binary feedback, but does not substantially expedite learning under sum feedback. 
We further consider variable segments with binary feedback. The resulting regret bound shows that, under the common elliptical potential analysis, equal segmentation minimizes the regret.

An important direction for future work is to investigate whether the current finding in the variable segment setting, i.e., the preference for equal segmentation, reflects an intrinsic property of the problem or stems from the limitations of the current analysis.
%determine whether this preference for equal segmentation reflects an intrinsic property of the problem or a limitation of the current analysis. 
It is interesting to develop new algorithms and more advanced analytical tools beyond the elliptical potential framework, as well as a lower bound, to fundamentally determine whether partitioning a trajectory according to the realized state-action features can provably improve learning.

\clearpage
\begin{appendices}
\numberwithin{theorem}{section}
\numberwithin{assumption}{section}

\section{Experimental Details}\label{secA}

In this section, we detail the instances used in our experiments.

For the fixed equal-length binary-feedback experiments with both known and
unknown transitions, the unknown-transition sum-feedback experiment, and
the variable-segment binary-feedback experiment, we consider the linear MDP instance
illustrated in Figure~\ref{fig:experiment-instance-other}. There are 30
states and 6 actions. The states are arranged into five cyclic layers, with
six states in each layer. The agent starts from a state chosen uniformly
from the first layer. There are two optimal actions, $a_1$ and $a_2$, and
four suboptimal actions, $a_3,\ldots,a_6$, in every state. For any
$s\in\cS$, the reward function satisfies
\[
r(s,a_1)=r(s,a_2)=r_{\max},
\qquad
r(s,a)=0,\quad a\in\{a_3,a_4,a_5,a_6\}.
\]
At each step, independently of the current state and selected action, the
agent transitions uniformly to one of the six states in the next layer.
After the fifth layer, the transition returns to the first layer. Thus, for
every state and action, each state in the next cyclic layer is reached with
probability $1/6$.

\begin{figure*}[h]
    \centering
    \includegraphics[
        width=\textwidth,
        trim=300pt 135pt 220pt 115pt,
        clip
    ]{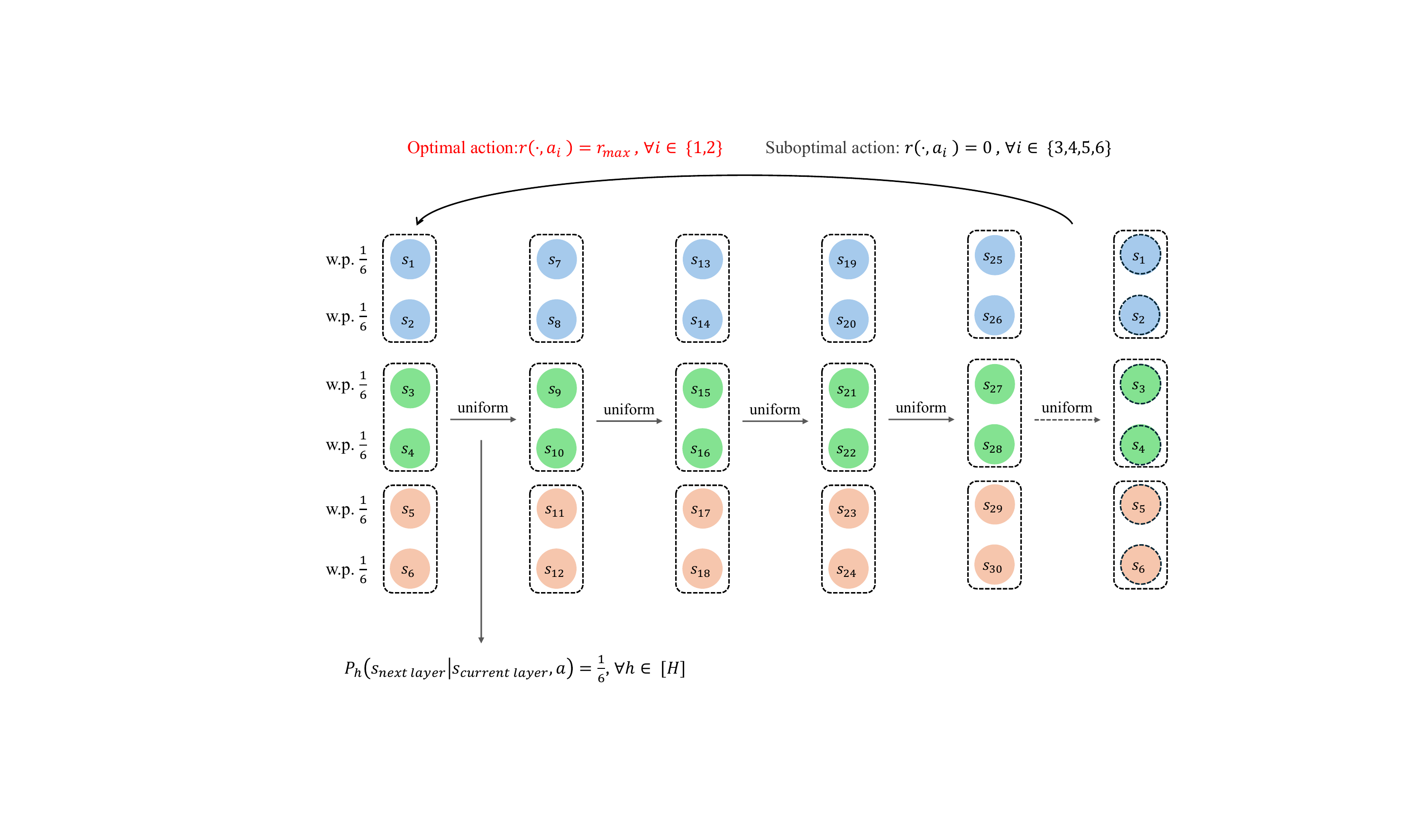}
    \caption{The 30-state linear MDP instance used in experiments.}
    \label{fig:experiment-instance-other}
\end{figure*}

For the known-transition sum-feedback experiment, we consider the smaller
linear MDP illustrated in Figure~\ref{fig:experiment-instance-sum}. There
are two states, denoted by $s^+$ and $s^-$, and two actions, denoted by
$a^+$ and $a^-$. The initial state is chosen uniformly from
$\{s^+,s^-\}$. For the optimal action $a^+$, we have
\[
r(s^+,a^+)=r_{\max},
\qquad
r(s^-,a^+)=0,
\]
whereas for the suboptimal action $a^-$,
\[
r(s^+,a^-)=0,
\qquad
r(s^-,a^-)=-r_{\max}.
\]
In either state, under the optimal action $a^+$, the agent transitions to
$s^+$ and $s^-$ with probabilities $0.75$ and $0.25$, respectively. Under
the suboptimal action $a^-$, it transitions to $s^+$ and $s^-$ with
probabilities $0.25$ and $0.75$, respectively. These transition
probabilities are the same at every step.

\begin{figure*}[t]
    \centering
    \includegraphics[
        width=0.55\textwidth,
        trim=525pt 345pt 503pt 165pt,
        clip
    ]{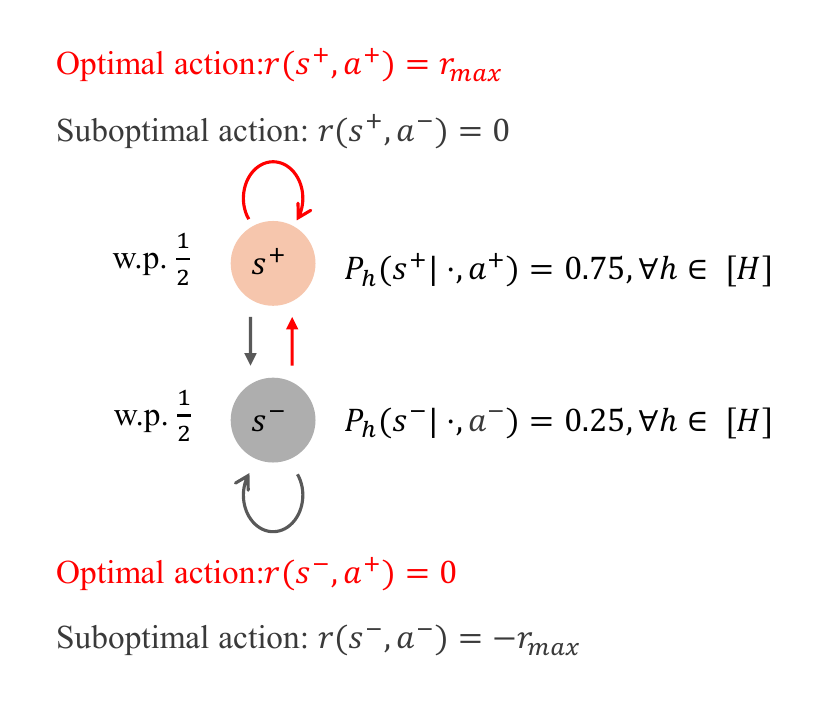}
    \caption{The two-state linear MDP instance used in the
    known-transition sum-feedback experiment.}
    \label{fig:experiment-instance-sum}
\end{figure*}

\section{Proofs for Fixed Equal-Length Binary Feedback with Known Transitions}\label{secB}

\subsection{Proof for the Regret Upper Bound }\label{app:upper_binary_known}

 First, we prove the confidence interval of the estimate $\hat{\theta}_{k-1}$.
    For any $k>0$ and $\theta \in \Theta$, define
\begin{align}
	Z_k &:= \sum_{k'=1}^{k} \sum_{i=1}^{m} \varepsilon_{k',i} \cdot \phi^{\tau^{k'}_i} , \nonumber
	\\
	g_{k}(\theta) &:= \sum_{k'=1}^{k} \sum_{i=1}^{m} \sig( (\phi^{\tau^{k'}_i})^\top\theta ) \cdot \phi^{\tau^{k'}_i} + \lambda \theta , \label{eq:def_g_k}
	\\
	\Lambda_{k}(\theta) &:=  \sum_{k'=1}^{k} \sum_{i=1}^{m} \sig'( (\phi^{\tau^{k'}_i})^\top\theta ) \cdot \phi^{\tau^{k'}_i} (\phi^{\tau^{k'}_i})^\top + \lambda I . \label{eq:def_Lambda_k}
\end{align}

\begin{lemma} \label{lemma:lin_Lambda_k}
	For any $k>0$ and $\theta \in \Theta$, we have
	\begin{align*}
		\det(\Lambda_k(\theta)) \leq \sbr{ \frac{H^2 \sig'_{\max} k}{d m} + \lambda  }^{d} .
	\end{align*}
\end{lemma}
\begin{proof}
	For any $k>0$, by AM-GM inequality, we have
	\begin{align*}
		\det(\Lambda_k(\theta)) &\leq \sbr{ \frac{\trace(\Lambda_k(\theta))}{d} }^{d}
		\\
		&\leq \sbr{ \frac{1}{d} \cdot \sbr{ k m \cdot \sig'_{\max} \cdot \sbr{\frac{H}{m}}^2 + \lambda d } }^{d}
		\\
		&= \sbr{ \frac{H^2 \sig'_{\max} k}{d m} + \lambda  }^{d} .
	\end{align*}
\end{proof}

For each $k>0$, let $F_k$ denote the filtration containing all information accumulated through the end of episode $k$. For each $i\in[m]$, let $\widetilde F_{i,k}$ denote the filtration containing the information available immediately before $y_i^k$ is observed in episode $k$. Since $\pi^k$ is determined before the observation of $y_i^k$, it is $\widetilde F_{i,k}$-measurable.

For each $k>0$ and $i\in[m]$, define the binary-feedback noise as $\varepsilon_{k,i}:=y_i^k-\sig((\phi^{\tau_i^k})^\top\theta^*)$. Its conditional variance given $\widetilde F_{i,k}$ is $v_{k,i}^2:=\ex[\varepsilon_{k,i}^2\mid\widetilde F_{i,k}]=\sig((\phi^{\tau_i^k})^\top\theta^*)\bigl(1-\sig((\phi^{\tau_i^k})^\top\theta^*)\bigr)=\sig'((\phi^{\tau_i^k})^\top\theta^*)$.

\begin{lemma}[Concentration of Noises under Binary Feedback]\label{lemma:lin_noise_seq_martingale}
	With probability at least $1-\delta'$, for any $k>0$,
	\begin{align*}
		\nbr{\sum_{k'=1}^{k} \sum_{i=1}^{m} \varepsilon_{k',i} \cdot \phi^{\tau^{k'}_i}}_{\Lambda_{k}^{-1}(\theta^*)}
		\leq
		\frac{\sqrt{\lambda}}{2}
		+
		\frac{d}{\sqrt{\lambda}}
		\log
		\sbr{
			\frac{4}{\delta'}
			\cdot
			\sbr{
				1
				+
				\frac{H^2 \sig'_{\max} k}{d m \lambda}
			}
		}.
	\end{align*}
\end{lemma}

\begin{proof}
	According to Theorem 1 in (Faury et al.,~\citeyear{faury2020improved}), we have that with probability at least $1-\delta'$, for any $k>0$,
	\begin{align*}
		\nbr{\sum_{k'=1}^{k} \sum_{i=1}^{m} \varepsilon_{k',i} \cdot \phi^{\tau^{k'}_i}}_{\Lambda_{k}^{-1}(\theta^*)}
		&\leq
		\frac{\sqrt{\lambda}}{2}
		+
		\frac{2}{\sqrt{\lambda}}
		\log
		\sbr{
			\frac{
				\det(\Lambda_{k}(\theta^*))^{\frac{1}{2}}
				\cdot
				\lambda^{-\frac{d}{2}}
			}{\delta'}
		}
		+
		\frac{2}{\sqrt{\lambda}} d \log(2)
		\\
		&\overset{\textup{(a)}}{\leq}
		\frac{\sqrt{\lambda}}{2}
		+
		\frac{2}{\sqrt{\lambda}}
		\log
		\sbr{
			\frac{1}{\delta'}
			\sbr{
				1
				+
				\frac{H^2 \sig'_{\max} k}{d m \lambda}
			}^{\frac{d}{2}}
		}
		+
		\frac{2}{\sqrt{\lambda}} d \log(2)
		\\
		&\leq
		\frac{\sqrt{\lambda}}{2}
		+
		\frac{d}{\sqrt{\lambda}}
		\log
		\sbr{
			\frac{1}{\delta'}
			\sbr{
				1
				+
				\frac{H^2 \sig'_{\max} k}{d m \lambda}
			}
		}
		+
		\frac{2}{\sqrt{\lambda}} d \log(2)
		\\
		&\leq
		\frac{\sqrt{\lambda}}{2}
		+
		\frac{d}{\sqrt{\lambda}}
		\log
		\sbr{
			\frac{4}{\delta'}
			\sbr{
				1
				+
				\frac{H^2 \sig'_{\max} k}{d m \lambda}
			}
		},
	\end{align*}
	where inequality (a) uses Lemma~\ref{lemma:lin_Lambda_k}.
\end{proof}

Define event
\begin{align*}
	\cE:=\lbr{ \nbr{g_k(\hat{\theta}_k) - g_k(\theta^*) }_{\Lambda_k^{-1}(\theta^*)} \leq \omega(k), \  \forall k>0 } .
\end{align*}
\begin{lemma}
	It holds that
	\begin{align*}
		\Pr(\cE)\ge 1-\delta'.
	\end{align*}
\end{lemma}

\begin{proof}This proof is similar to that for Lemma 8 in (Faury et al.,~\citeyear{faury2020improved}).

	For each $k>0$, define
	\begin{align*}
		\cL_k(\theta)
		:=
		-
		\Biggl[
		\sum_{k'=1}^{k}
		\sum_{i=1}^{m}
		\Bigl(
		y_i^{k'}
		\cdot
		\log
		\bigl(
		\sig((\phi^{\tau_i^{k'}})^\top\theta)
		\bigr)
		+
		(1-y_i^{k'})
		\cdot
		\log
		\bigl(
		1-\sig((\phi^{\tau_i^{k'}})^\top\theta)
		\bigr)
		\Bigr)
		-
		\frac{\lambda}{2}\|\theta\|_2^2
		\Biggr].
	\end{align*}

	Recall that
		$\hat{\theta}_k=\arg\min_{\theta}\cL_k(\theta).$
	Since $\nabla\cL_k(\hat{\theta}_k)=0$ and
	\(
	\sig'(x)=\sig(x)(1-\sig(x)),
	\)
	the first-order optimality condition implies
	\begin{align*}
		\underbrace{
		\sum_{k'=1}^{k}
		\sum_{i=1}^{m}
		\sig\!\left((\phi^{\tau_i^{k'}})^\top\hat{\theta}_k\right)
		\phi^{\tau_i^{k'}}
		+
		\lambda\hat{\theta}_k
		}_{g_k(\hat{\theta}_k)}
		=
		\sum_{k'=1}^{k}
		\sum_{i=1}^{m}
		y_i^{k'}
		\phi^{\tau_i^{k'}}.
	\end{align*}

	Therefore,
	\begin{align}
		g_k(\hat{\theta}_k)-g_k(\theta^*)
		&=
		\sum_{k'=1}^{k}
		\sum_{i=1}^{m}
		y_i^{k'}
		\phi^{\tau_i^{k'}}
		-
		\left[
		\sum_{k'=1}^{k}
		\sum_{i=1}^{m}
		\sig\!\left((\phi^{\tau_i^{k'}})^\top\theta^*\right)
		\phi^{\tau_i^{k'}}
		+
		\lambda\theta^*
		\right]
		\nonumber\\
		&=
		\sum_{k'=1}^{k}
		\sum_{i=1}^{m}
		\varepsilon_{k',i}\phi^{\tau_i^{k'}}
		-
		\lambda\theta^*.
		\label{eq:lin_g_hat_theta-g_theta_star}
	\end{align}

	Combining Eq.\eqref{eq:lin_g_hat_theta-g_theta_star} with Lemma~\ref{lemma:lin_noise_seq_martingale}, we obtain that, with probability at least $1-\delta'$, for any $k>0$,
	\begin{align*}
		\nbr{g_k(\hat{\theta}_k)-g_k(\theta^*)}_{\Lambda_k^{-1}(\theta^*)}
		&\le
		\nbr{
		\sum_{k'=1}^{k}
		\sum_{i=1}^{m}
		\varepsilon_{k',i}\phi^{\tau_i^{k'}}
		}_{\Lambda_k^{-1}(\theta^*)}
		+
		r_{\max}\sqrt{\lambda d}
		\\
		&\le
		\frac{\sqrt{\lambda}}{2}
		+
		\frac{d}{\sqrt{\lambda}}
		\log
		\left[
		\frac{4}{\delta'}
		\left(
		1+\frac{H^2\sig'_{\max}k}{dm\lambda}
		\right)
		\right]
		+
		r_{\max}\sqrt{\lambda d}
		\\
		&=
		\omega(k).
	\end{align*}
\end{proof}

For any $\phi \in \R^{d}$ and $\theta_1,\theta_2 \in \Theta$, define
\begin{align*}
b(\phi,\theta_1,\theta_2)
:=
\int_{z=0}^{1}
\sig'\left(
(1-z) \cdot \phi^\top \theta_1
+
z \cdot \phi^\top \theta_2
\right)
dz .
\end{align*}

For any $k>0$ and $\theta_1,\theta_2 \in \Theta$, define
\begin{align*}
\Gamma_k(\theta_1,\theta_2)
:=
\sum_{k'=1}^{k}
\sum_{i=1}^{m}
b(\phi^{\tau_i^{k'}},\theta_1,\theta_2)
\cdot
\phi^{\tau_i^{k'}}
\left(\phi^{\tau_i^{k'}}\right)^\top
+
\lambda I .
\end{align*}

In the definitions of $b(\phi,\theta_1,\theta_2)$ and
$\Gamma_k(\theta_1,\theta_2)$, the parameters $\theta_1$ and $\theta_2$
play symmetric roles and can therefore be interchanged.

Recall that
\begin{align*}
\alpha
:=
\exp\left(\frac{H r_{\max}}{m}\right)
+
\exp\left(-\frac{H r_{\max}}{m}\right)
+
2 .
\end{align*}

Then, we have
\begin{align*}
\sup_{\tau^{\seg}, \theta}
\frac{1}{\sig'\left((\phi^{\tau^{\seg}})^\top \theta\right)}
\leq
\alpha ,
\end{align*}
where $\phi^{\tau^{\seg}}$ denotes the cumulative feature vector associated
with any possible trajectory segment $\tau^{\seg}$.

\begin{lemma} \label{lemma:lin_transform_Lambda_Sigma}
	For any $k\geq 1$ and $\theta \in \Theta$, we have
	\begin{align*}
		%			\Gamma_k(\theta_1,\theta_2) &\succeq \sbr{1+\frac{2Hr_{\max}}{m}}^{-1} \Lambda_k(\theta_1) ,
		%			\\
		\Sigma_k &\preceq \alpha \Lambda_k(\theta) .
		%			\\
		%			\Gamma_k(\theta_1,\theta_2) &\succeq \frac{1}{\alpha} \Sigma_k .
	\end{align*}
\end{lemma}

\begin{proof}
	By the definition of $\alpha$, for any possible trajectory segment
	$\tau^{\seg}$ and any $\theta \in \Theta$, we have
	\begin{align*}
		\frac{1}{\alpha}
		\leq
		\sig'\left((\phi^{\tau^{\seg}})^\top \theta\right).
	\end{align*}
    Therefore, for any $k\geq 1$ and $\theta \in \Theta$, it follows that
	\begin{align*}
		\Sigma_k
		&=
		\sum_{k'=1}^{k}
		\sum_{i=1}^{m}
		\phi^{\tau_i^{k'}}
		\left(\phi^{\tau_i^{k'}}\right)^\top
		+
		\alpha \lambda I
		\\
		&=
		\alpha
		\sbr{
		\sum_{k'=1}^{k}
		\sum_{i=1}^{m}
		\frac{1}{\alpha}
		\cdot
		\phi^{\tau_i^{k'}}
		\left(\phi^{\tau_i^{k'}}\right)^\top
		+
		\lambda I
		}
		\\
		&\preceq
		\alpha
		\sbr{
		\sum_{k'=1}^{k}
		\sum_{i=1}^{m}
		\sig'\left((\phi^{\tau_i^{k'}})^\top\theta\right)
		\cdot
		\phi^{\tau_i^{k'}}
		\left(\phi^{\tau_i^{k'}}\right)^\top
		+
		\lambda I
		}
		\\
		&=
		\alpha \Lambda_k(\theta).
	\end{align*}
\end{proof}
	
\begin{lemma} \label{lemma:lin_two_equalities}
	For any $\phi \in \R^{d}$ and $\theta_1,\theta_2 \in \Theta$, we have
	\begin{align*}
		\sig(\phi^\top \theta_1)
		-
		\sig(\phi^\top \theta_2)
		=
		b(\phi,\theta_2,\theta_1)
		\cdot
		\phi^\top(\theta_1-\theta_2).
	\end{align*}

	In addition, for any $k>0$ and $\theta_1,\theta_2 \in \Theta$, we have
	\begin{align*}
		\nbr{\theta_1-\theta_2}_{\Gamma_{k}(\theta_2,\theta_1)}
		=
		\nbr{g_k(\theta_1)-g_k(\theta_2)}_{\Gamma_{k}^{-1}(\theta_2,\theta_1)} .
	\end{align*}
\end{lemma}

\begin{proof}
	The first claim follows directly from the mean-value theorem in
	integral form.

	Using the first claim, for any $k>0$, we obtain
	\begin{align*}
		g_k(\theta_1)-g_k(\theta_2)
		&=
		\sum_{k'=1}^{k}
		\sum_{i=1}^{m}
		\sbr{
			\sig\left((\phi^{\tau_i^{k'}})^\top \theta_1\right)
			-
			\sig\left((\phi^{\tau_i^{k'}})^\top \theta_2\right)
		}
		\cdot
		\phi^{\tau_i^{k'}}
		+
		\lambda(\theta_1-\theta_2)
		\\
		&=
		\sum_{k'=1}^{k}
		\sum_{i=1}^{m}
		b(\phi^{\tau_i^{k'}},\theta_2,\theta_1)
		\cdot
		\phi^{\tau_i^{k'}}
		\left(\phi^{\tau_i^{k'}}\right)^\top
		(\theta_1-\theta_2)
		+
		\lambda(\theta_1-\theta_2)
		\\
		&=
		\Gamma_{k}(\theta_2,\theta_1)
		\cdot
		(\theta_1-\theta_2).
	\end{align*}
	Therefore,
	\begin{align*}
		\nbr{\theta_1-\theta_2}_{\Gamma_{k}(\theta_2,\theta_1)}
		&=
		\sqrt{
			(\theta_1-\theta_2)^\top
			\Gamma_{k}(\theta_2,\theta_1)
			(\theta_1-\theta_2)
		}
		\\
		&=
		\sqrt{
			(\theta_1-\theta_2)^\top
			\Gamma_{k}(\theta_2,\theta_1)
			\Gamma_{k}^{-1}(\theta_2,\theta_1)
			\Gamma_{k}(\theta_2,\theta_1)
			(\theta_1-\theta_2)
		}
		\\
		&=
		\nbr{g_k(\theta_1)-g_k(\theta_2)}_{\Gamma_{k}^{-1}(\theta_2,\theta_1)} ,
	\end{align*}
	which proves the second claim.
\end{proof}

Recall that for any $k>0$, $Z_k := \sum_{k'=1}^{k} \sum_{i=1}^{m} \varepsilon_{k',i} \cdot \phi^{\tau^{k'}_i}$.

\begin{lemma} \label{lemma:lin_eff_bound_Z_k}
	For any $k>0$, we have
	\begin{align*}
		\Gamma_k(\theta^*,\hat{\theta}_k)
		&\succeq
		\sbr{
			1
			+
			\frac{H r_{\max}\sqrt{d}}{m}
			+
			\frac{H}{m\sqrt{\lambda}}
			\nbr{Z_k}_{\Gamma_k^{-1}(\theta^*,\hat{\theta}_k)}
		}^{-1}
		\Lambda_k(\theta^*) ,
		\\
		\nbr{Z_k}_{\Gamma_k^{-1}(\theta^*,\hat{\theta}_k)}
		&\leq
		\sqrt{
			1+
			\frac{H r_{\max}\sqrt{d}}{m}
		}
		\nbr{Z_k}_{\Lambda_k^{-1}(\theta^*)}
		+
		\frac{H}{m\sqrt{\lambda}}
		\nbr{Z_k}_{\Lambda_k^{-1}(\theta^*)}^2 .
	\end{align*}

	Furthermore, assuming that event $\cE$ holds, we have
	\begin{align*}
		\nbr{Z_k}_{\Gamma_k^{-1}(\theta^*,\hat{\theta}_k)}
		\leq
		\sqrt{
			1+
			\frac{H r_{\max}\sqrt{d}}{m}
		}
		\cdot
		\omega(k)
		+
		\frac{H}{m\sqrt{\lambda}}
		\cdot
		\omega(k)^2 .
	\end{align*}
\end{lemma}

\begin{proof}
	From Eq.~\eqref{eq:lin_g_hat_theta-g_theta_star}, for any $k>0$, we have
	\begin{align*}
		g_k(\hat{\theta}_k)-g_k(\theta^*)
		=
		Z_k-\lambda\theta^* .
	\end{align*}

	Using Lemma~\ref{lemma:self_concordance}, for any $\phi \in \R^{d}$ satisfying $\|\phi\|_2\leq L_{\phi}$, we have
	\begin{align*}
		b(\phi,\theta^*,\hat{\theta}_k)
		&\geq
		\sbr{
			1+
			\abr{\phi^\top(\theta^*-\hat{\theta}_k)}
		}^{-1}
		\sig'(\phi^\top \theta^*)
		\\
		&=
		\sbr{
			1+
			\abr{
				\phi^\top
				\Gamma_k^{-1}(\theta^*,\hat{\theta}_k)
				\cdot
				\bigl(g_k(\theta^*)-g_k(\hat{\theta}_k)\bigr)
			}
		}^{-1}
		\sig'(\phi^\top \theta^*)
		\\
		&\geq
		\sbr{
			1+
			\nbr{\phi}_{\Gamma_k^{-1}(\theta^*,\hat{\theta}_k)}
			\nbr{
				g_k(\theta^*)-g_k(\hat{\theta}_k)
			}_{\Gamma_k^{-1}(\theta^*,\hat{\theta}_k)}
		}^{-1}
		\sig'(\phi^\top \theta^*)
		\\
		&\geq
		\sbr{
			1+
			\frac{L_{\phi}}{\sqrt{\lambda}}
			\nbr{
				g_k(\theta^*)-g_k(\hat{\theta}_k)
			}_{\Gamma_k^{-1}(\theta^*,\hat{\theta}_k)}
		}^{-1}
		\sig'(\phi^\top \theta^*)
		\\
		&=
		\sbr{
			1+
			\frac{L_{\phi}}{\sqrt{\lambda}}
			\nbr{
				Z_k-\lambda\theta^*
			}_{\Gamma_k^{-1}(\theta^*,\hat{\theta}_k)}
		}^{-1}
		\sig'(\phi^\top \theta^*)
		\\
		&\geq
		\sbr{
			1
			+
			L_{\phi}r_{\max}\sqrt{d}
			+
			\frac{L_{\phi}}{\sqrt{\lambda}}
			\nbr{Z_k}_{\Gamma_k^{-1}(\theta^*,\hat{\theta}_k)}
		}^{-1}
		\sig'(\phi^\top \theta^*) .
	\end{align*}
	Applying the preceding bound with
	$\phi=\phi^{\tau_i^{k'}}$ and $L_{\phi}=H/m$, we obtain
	\begin{align*}
		\Gamma_k(\theta^*,\hat{\theta}_k)
		&:=
		\sum_{k'=1}^{k}
		\sum_{i=1}^{m}
		b(\phi^{\tau_i^{k'}},\theta^*,\hat{\theta}_k)
		\cdot
		\phi^{\tau_i^{k'}}
		\bigl(\phi^{\tau_i^{k'}}\bigr)^\top
		+
		\lambda I
		\\
		&\succeq
		\sum_{k'=1}^{k}
		\sum_{i=1}^{m}
		\sbr{
			1
			+
			\frac{H r_{\max}\sqrt{d}}{m}
			+
			\frac{H}{m\sqrt{\lambda}}
			\nbr{Z_k}_{\Gamma_k^{-1}(\theta^*,\hat{\theta}_k)}
		}^{-1}
		\sig'(\phi^\top \theta^*)
		\cdot
		\phi^{\tau_i^{k'}}
		\bigl(\phi^{\tau_i^{k'}}\bigr)^\top
		+
		\lambda I
		\\
		&\succeq
		\sbr{
			1
			+
			\frac{H r_{\max}\sqrt{d}}{m}
			+
			\frac{H}{m\sqrt{\lambda}}
			\nbr{Z_k}_{\Gamma_k^{-1}(\theta^*,\hat{\theta}_k)}
		}^{-1}
		\Lambda_k(\theta^*) .
	\end{align*}

	This matrix inequality implies
	\begin{align*}
		\nbr{Z_k}_{\Gamma_k^{-1}(\theta^*,\hat{\theta}_k)}^2
		\leq
		\sbr{
			1
			+
			\frac{H r_{\max}\sqrt{d}}{m}
			+
			\frac{H}{m\sqrt{\lambda}}
			\nbr{Z_k}_{\Gamma_k^{-1}(\theta^*,\hat{\theta}_k)}
		}
		\nbr{Z_k}_{\Lambda_k^{-1}(\theta^*)}^2 .
	\end{align*}

	Equivalently,
	\begin{align*}
		\nbr{Z_k}_{\Gamma_k^{-1}(\theta^*,\hat{\theta}_k)}^2
		-
		\frac{H}{m\sqrt{\lambda}}
		\nbr{Z_k}_{\Lambda_k^{-1}(\theta^*)}^2
		\nbr{Z_k}_{\Gamma_k^{-1}(\theta^*,\hat{\theta}_k)}
		-
		\sbr{
			1+
			\frac{H r_{\max}\sqrt{d}}{m}
		}
		\nbr{Z_k}_{\Lambda_k^{-1}(\theta^*)}^2
		\leq
		0 .
	\end{align*}

	By solving this quadratic inequality and on event $\cE$, we obtain
	\begin{align*}
		\nbr{Z_k}_{\Gamma_k^{-1}(\theta^*,\hat{\theta}_k)}
		&\leq
		\sqrt{
			1+
			\frac{H r_{\max}\sqrt{d}}{m}
		}
		\nbr{Z_k}_{\Lambda_k^{-1}(\theta^*)}
		+
		\frac{H}{m\sqrt{\lambda}}
		\nbr{Z_k}_{\Lambda_k^{-1}(\theta^*)}^2
		\\
		&\leq
		\sqrt{
			1+
			\frac{H r_{\max}\sqrt{d}}{m}
		}
		\cdot
		\omega(k)
		+
		\frac{H}{m\sqrt{\lambda}}
		\cdot
		\omega(k)^2 .
	\end{align*}
\end{proof}

\begin{lemma}[Concentration of $\phi^\top \hat{\theta}_k$ under Binary Feedback] \label{lemma:lin_confence_interval_proj_free}
	Assume that event $\cE$ holds. Then, for any $k>0$ and $\phi \in \R^{d}$,
	\begin{align*}
		|\phi^\top \theta^* - \phi^\top \hat{\theta}_k| \leq \sqrt{\alpha} \cdot \nu(k) \nbr{\phi}_{\Sigma_k^{-1}}  .
	\end{align*}
\end{lemma}
\begin{proof}
	On event $\cE$. We have
	\begin{align*}
		&\quad\ |\phi^\top \theta^* - \phi^\top \hat{\theta}_k| 
		\\
		&\le 	\nbr{\phi}_{\Gamma_k^{-1}(\theta^*,\hat{\theta}_k)} \nbr{\theta^* - \hat{\theta}_k}_{\Gamma_k(\theta^*,\hat{\theta}_k)}
		\nonumber\\
		&\overset{\textup{(a)}}{\leq} \sqrt{ 1+  \frac{H r_{\max} \sqrt{d}}{m}   + \frac{H}{m\sqrt{\lambda}} \nbr{ Z_k }_{\Gamma_k^{-1}(\theta^*,\hat{\theta}_k)} } \nbr{\phi}_{\Lambda_k^{-1}(\theta^*)} \nbr{g_k(\theta^*) - g_k(\hat{\theta}_k)}_{\Gamma_k^{-1}(\theta^*,\hat{\theta}_k)}
		\nonumber\\
		&= \sqrt{ 1+  \frac{H r_{\max} \sqrt{d}}{m}   + \frac{H}{m\sqrt{\lambda}} \nbr{ Z_k }_{\Gamma_k^{-1}(\theta^*,\hat{\theta}_k)} } \nbr{\phi}_{\Lambda_k^{-1}(\theta^*)} \nbr{ Z_k-\lambda \theta^* }_{\Gamma_k^{-1}(\theta^*,\hat{\theta}_k)}
		\nonumber\\
		&\leq \sqrt{ 1 \!+\!  \frac{H r_{\max} \sqrt{d}}{m} \!+\! \frac{H}{m\sqrt{\lambda}} \nbr{ Z_k }_{\Gamma_k^{-1}(\theta^*,\hat{\theta}_k)} } \nbr{\phi}_{\Lambda_k^{-1}(\theta^*)} \sbr{\nbr{ Z_k }_{\Gamma_k^{-1}(\theta^*,\hat{\theta}_k)} \!+\! r_{\max} \sqrt{\lambda d}  }
		\nonumber\\
		&= \frac{m \sqrt{\lambda}}{H} \sqrt{ 1+  \frac{H r_{\max} \sqrt{d}}{m}   + \frac{H}{m\sqrt{\lambda}} \nbr{ Z_k }_{\Gamma_k^{-1}(\theta^*,\hat{\theta}_k)} } \nbr{\phi}_{\Lambda_k^{-1}(\theta^*)} \cdot 
		\\
		&\quad \sbr{ \frac{H}{m \sqrt{\lambda}} \nbr{ Z_k }_{\Gamma_k^{-1}(\theta^*,\hat{\theta}_k)} + \frac{Hr_{\max}\sqrt{d} }{m } } 
		\nonumber\\
		&\leq \frac{m \sqrt{\lambda}}{H} \sbr{ 1+  \frac{H r_{\max} \sqrt{d}}{m}   + \frac{H}{m\sqrt{\lambda}} \nbr{ Z_k }_{\Gamma_k^{-1}(\theta^*,\hat{\theta}_k)} }^{\frac{3}{2}} \nbr{\phi}_{\Lambda_k^{-1}(\theta^*)}  
		\nonumber\\
		&\overset{\textup{(b)}}{\leq} \!\frac{m \sqrt{\alpha \lambda}}{H} \!\! \sbr{ 1 \!+\!  \frac{H r_{\max} \sqrt{d}}{m}   \!+\! \frac{H}{m\sqrt{\lambda}} \sbr{ \sqrt{ 1 \!+\!  \frac{H r_{\max} \sqrt{d}}{m} } \omega(k) \!+\! \frac{H}{m\sqrt{\lambda}} \omega(k)^2 } }^{\!\!\frac{3}{2}} \!\!\nbr{\phi}_{\Sigma_k^{-1}} 
		\\
		&= \sqrt{\alpha} \cdot \nu(k) \nbr{\phi}_{\Sigma_k^{-1}} ,
	\end{align*}
	where inequality (a) is due to Lemmas~\ref{lemma:lin_two_equalities} and \ref{lemma:lin_eff_bound_Z_k}, and inequality (b) follows from  Lemmas~\ref{lemma:lin_transform_Lambda_Sigma} and \ref{lemma:lin_eff_bound_Z_k}.
\end{proof}

\begin{lemma}[Gaussian Anti-Concentration] \label{lemma:lin_Gaussian_anti_concentration}
	Assume that event $\cE$ holds. Then, for any $k>0$ and $F_{k-1}$-measurable random variable $X \in \R^{d}$, we have
	\begin{align*}
		\Pr\mbr{
			X^\top \widetilde{\theta}_k
			\geq
			X^\top \theta^*
			\mid F_{k-1}
		}
		\geq
		\frac{1}{2\sqrt{2\pi e}} .
	\end{align*}
\end{lemma}

\begin{proof}
	This proof is originated from the analysis of Lemma 11 in ((Efroni et al., \citeyear{efroni2021reinforcement}).

	By Lemma~\ref{lemma:lin_confence_interval_proj_free}, on event $\cE$, for any $k>0$,
	\begin{align*}
		\left|
			X^\top \theta^*
			-
			X^\top \hat{\theta}_{k-1}
		\right|
		\leq
		\sqrt{\alpha}
		\cdot
		\nu(k-1)
		\nbr{X}_{\Sigma_{k-1}^{-1}} .
	\end{align*}

	It follows that
	\begin{align*}
		&\quad\ 
		\Pr\mbr{
			X^\top \widetilde{\theta}_k
			\geq
			X^\top \theta^*
			\mid F_{k-1}
		}
		\\
		&=
		\Pr\mbr{
			\frac{
				X^\top \widetilde{\theta}_k
				-
				X^\top \hat{\theta}_{k-1}
			}{
				\sqrt{\alpha}
				\cdot
				\nu(k-1)
				\nbr{X}_{\Sigma_{k-1}^{-1}}
			}
			\geq
			\frac{
				X^\top \theta^*
				-
				X^\top \hat{\theta}_{k-1}
			}{
				\sqrt{\alpha}
				\cdot
				\nu(k-1)
				\nbr{X}_{\Sigma_{k-1}^{-1}}
			}
			\mid F_{k-1}
		}.
	\end{align*}

	Conditioned on $F_{k-1}$,
	$X^\top \widetilde{\theta}_k-X^\top \hat{\theta}_{k-1}=X^\top \xi_k$
	is a Gaussian random variable with mean zero and standard deviation
	$\sqrt{\alpha}\cdot \nu(k-1)\nbr{X}_{\Sigma_{k-1}^{-1}}$.

	Moreover, on event $\cE$,
	\begin{align*}
		\frac{
			X^\top \theta^*
			-
			X^\top \hat{\theta}_{k-1}
		}{
			\sqrt{\alpha}
			\cdot
			\nu(k-1)
			\nbr{X}_{\Sigma_{k-1}^{-1}}
		}
		\leq
		\frac{
			\sqrt{\alpha}
			\cdot
			\nu(k-1)
			\nbr{X}_{\Sigma_{k-1}^{-1}}
		}{
			\sqrt{\alpha}
			\cdot
			\nu(k-1)
			\nbr{X}_{\Sigma_{k-1}^{-1}}
		}
		=
		1 .
	\end{align*}
	Hence,
	\begin{align*}
		\Pr\mbr{
			X^\top \widetilde{\theta}_k
			\geq
			X^\top \theta^*
			\mid F_{k-1}
		}
		&\geq
		\Pr\mbr{
			\frac{
				X^\top \widetilde{\theta}_k
				-
				X^\top \hat{\theta}_{k-1}
			}{
				\sqrt{\alpha}
				\cdot
				\nu(k-1)
				\nbr{X}_{\Sigma_{k-1}^{-1}}
			}
			>
			1
			\mid F_{k-1}
		}
		\\
		&=
		\Pr\mbr{
			\frac{
				X^\top \xi_k
			}{
				\sqrt{\alpha}
				\cdot
				\nu(k-1)
				\nbr{X}_{\Sigma_{k-1}^{-1}}
			}
			>
			1
			\mid F_{k-1}
		}
		\\
		&\overset{\textup{(a)}}{\geq}
		\frac{1}{2\sqrt{2\pi e}} ,
	\end{align*}
	where inequality \textup{(a)} follows from the standard Gaussian tail
	lower bound: if $Z\sim\cN(0,1)$, then
	\[
		\Pr[Z>z]
		\geq
		\frac{1}{\sqrt{2\pi}}
		\cdot
		\frac{z}{1+z^2}
		\exp\left(-\frac{z^2}{2}\right),
	\]
	applied with $z=1$ (Borjesson and Sundberg,~\citeyear{borjesson1979simple}).
\end{proof}

\begin{lemma} \label{lemma:lin_xi_xi'}
	Let $\xi_k,\xi'_k \in \R^{d}$ be i.i.d. random variables given $F_{k-1}$. Let $\tilde{p}$ be a $F_{k-1}$-measurable transition model, and $x_{k-1} \in \R^{d}$ be  a $F_{k-1}$-measurable random variable. For any policy $\pi$, denote the cumulative feature under policy $\pi$ on MDP $\tilde{P}$ by $\tilde{\phi}^{\pi}$. Let $\tilde{\pi}^k:= \argmax_{\pi} (\tilde{\phi}^{\pi})^\top (x_{k-1} + \xi_k)$. Then, we have
	\begin{align*}
		& \ex\mbr{ \sbr{ (\tilde{\phi}^{\tilde{\pi}^k})^\top \sbr{ x_{k-1} + \xi_k } - \ex\mbr{ (\tilde{\phi}^{\tilde{\pi}^k})^\top \sbr{ x_{k-1} + \xi_k } \ |\ F_{k-1}  } }^{+} \ \Big|\ F_{k-1} }
		\\
		&\leq \ex\mbr{ | (\tilde{\phi}^{\tilde{\pi}^k})^\top \xi_k | + | (\tilde{\phi}^{\tilde{\pi}^k})^\top \xi'_k| \ \Big|\ F_{k-1} } .
	\end{align*}
\end{lemma}
\begin{proof}
	This proof is originated from Lemma 12 in (Efroni et al.,~\citeyear{efroni2021reinforcement}).
	
	First, using the definition of $\tilde{\pi}^k$ and the fact that $\xi_k$ and $\xi'_k$ follow the same distribution, we have
	\begin{align}
		\ex\mbr{ (\tilde{\phi}^{\tilde{\pi}^k})^\top \sbr{ x_{k-1} + \xi_k } \ |\ F_{k-1}  } = \ex\mbr{ \max_{\pi} (\tilde{\phi}^{\pi})^\top \sbr{ x_{k-1} + \xi'_k } \ |\ F_{k-1}  } . \label{eq:connect_xi_xi'}
	\end{align}
	
	Then, since given $F_{k-1}$, $\xi_k$ and $\xi'_k$ are independent, we have
	\begin{align}
		\ex\mbr{ \max_{\pi} (\tilde{\phi}^{\pi})^\top \sbr{ x_{k-1} + \xi'_k } \ |\ F_{k-1}  } &= \ex\mbr{ \max_{\pi} (\tilde{\phi}^{\pi})^\top \sbr{ x_{k-1} + \xi'_k } \ |\ F_{k-1} , \xi_k, \tilde{\pi}_k }
		\nonumber\\
		&\geq \ex\mbr{ (\phi^{\tilde{\pi}_k})^\top \sbr{ x_{k-1} + \xi'_k } \ |\ F_{k-1}, \xi_k, \tilde{\pi}_k  } . \label{eq:relax_max_xi'}
	\end{align}
	
	Hence, combining Eqs.~\eqref{eq:connect_xi_xi'} and \eqref{eq:relax_max_xi'}, we have
	\begin{align*}
		&\quad\ \ex\mbr{ \sbr{(\phi^{\tilde{\pi}_k})^\top \sbr{ x_{k-1} + \xi_k } - \ex\mbr{ (\phi^{\tilde{\pi}_k})^\top \sbr{ x_{k-1} + \xi_k } \ |\ F_{k-1}  } }^{+} \ \Big|\ F_{k-1} }
		\\
		&\leq \ex\mbr{ \sbr{ (\phi^{\tilde{\pi}_k})^\top \sbr{ x_{k-1} + \xi_k }  - \ex\mbr{ (\phi^{\tilde{\pi}_k})^\top \sbr{ x_{k-1} + \xi'_k } \ |\ F_{k-1} , \xi_k, \tilde{\pi}_k  } }^{+} \ \Big|\ F_{k-1} } 
		\\
		&= \ex\mbr{ \sbr{ \ex\mbr{ (\phi^{\tilde{\pi}_k})^\top \sbr{ x_{k-1} + \xi_k }  -  (\phi^{\tilde{\pi}_k})^\top \sbr{ x_{k-1} + \xi'_k } \ |\ F_{k-1} , \xi_k, \tilde{\pi}_k  } }^{+} \ \Big|\ F_{k-1} } 
		\\
		&= \ex\mbr{ \sbr{ \ex\mbr{ (\phi^{\tilde{\pi}_k})^\top \xi_k   -  (\phi^{\tilde{\pi}_k})^\top \xi'_k \ |\ F_{k-1} , \xi_k, \tilde{\pi}_k  } }^{+} \ \Big|\ F_{k-1} }
		\\
		&\leq \ex\mbr{ \abr{ \ex\mbr{ (\phi^{\tilde{\pi}_k})^\top \xi_k   -  (\phi^{\tilde{\pi}_k})^\top \xi'_k \ |\ F_{k-1} , \xi_k, \tilde{\pi}_k  } } \ \Big|\ F_{k-1} }
		\\
		&\leq \ex\mbr{  \ex\mbr{ |(\phi^{\tilde{\pi}_k})^\top \xi_k   -  (\phi^{\tilde{\pi}_k})^\top \xi'_k| \ |\ F_{k-1} , \xi_k, \tilde{\pi}_k   } \ \Big|\ F_{k-1} }
		\\
		&= \ex\mbr{   | (\phi^{\tilde{\pi}_k})^\top \xi_k   -  (\phi^{\tilde{\pi}_k})^\top \xi'_k | \ \Big|\ F_{k-1} }
		\\
		&\leq \ex\mbr{   |(\phi^{\tilde{\pi}_k})^\top \xi_k| \ \Big|\ F_{k-1} } + \ex\mbr{   |  (\phi^{\tilde{\pi}_k})^\top \xi'_k| \ \Big|\ F_{k-1} } .
	\end{align*}
	
\end{proof}

For any $k>0$ and $\delta_k \in (0,1)$, define the event
\begin{align*}
	\cT_k(\delta_k)
	:=
	\lbr{
		\forall \phi \in \R^{d}:
		\ |\phi^\top \xi_k|
		\leq
		\sqrt{\alpha} \cdot \nu(k-1)
		\sbr{
			\sqrt{d}
			+
			2\sqrt{\log\sbr{\frac{1}{\delta_k}}}
		}
		\nbr{\phi}_{\Sigma_{k-1}^{-1}}
	}.
\end{align*}

\begin{lemma} \label{lemma:lin_X_top_xi_k}
	For any $k>0$ and $\delta_k \in (0,1)$, we have
	\begin{align*}
		\Pr\mbr{
			\cT_k(\delta_k)
			\mid F_{k-1}
		}
		\geq
		1-\delta_k .
	\end{align*}

	In addition, for a random variable $X \in \R^{d}$ satisfying
	$\|X\|_{\Sigma_{k-1}^{-1}}\leq L_X$, we have
	\begin{align*}
		\ex\mbr{
			|X^\top \xi_k|
			\mid F_{k-1}
		}
		&\leq
		\sqrt{\alpha} \cdot \nu(k-1)
		\sbr{
			\sqrt{d}
			+
			2\sqrt{\log\sbr{\frac{1}{\delta_k}}}
		}
		\ex\mbr{
			\nbr{X}_{\Sigma_{k-1}^{-1}}
			\mid F_{k-1}
		}
		\\
		&\quad
		+
		\sqrt{\alpha} \cdot \nu(k-1)
		\cdot
		L_X\sqrt{d\delta_k}.
	\end{align*}
\end{lemma}

\begin{proof}
	This proof is similar to the analysis of Lemma 13 in (Efroni et al.,~\citeyear{efroni2021reinforcement}).

	We first prove the first statement. For any $\phi \in \R^{d}$, by the Cauchy-Schwarz inequality,
	\begin{align}
		|\phi^\top \xi_k|
		&=
		\left|
			\phi^\top
			\Sigma_{k-1}^{-\frac{1}{2}}
			\Sigma_{k-1}^{\frac{1}{2}}
			\xi_k
		\right|
		\nonumber\\
		&\leq
		\nbr{
			\Sigma_{k-1}^{-\frac{1}{2}}\phi
		}_2
		\nbr{
			\Sigma_{k-1}^{\frac{1}{2}}\xi_k
		}_2
		\nonumber\\
		&=
		\sqrt{\alpha} \cdot \nu(k-1)
		\nbr{\phi}_{\Sigma_{k-1}^{-1}}
		\nbr{
			\frac{1}{\sqrt{\alpha}\cdot \nu(k-1)}
			\Sigma_{k-1}^{\frac{1}{2}}\xi_k
		}_2 .
		\label{eq:known_tran_phi_xi_decomposition}
	\end{align}

	Conditioned on $F_{k-1}$,
		$\frac{1}{\sqrt{\alpha}\cdot \nu(k-1)}
		\Sigma_{k-1}^{\frac{1}{2}}\xi_k$
	is a standard Gaussian vector in $\R^d$. Hence, its Euclidean norm is
	chi-distributed with parameter $d$ .

	By Lemma 1 in (Laurent and Massart,~\citeyear{laurent2000adaptive}), with probability at least
	$1-\delta_k$,
	\begin{align*}
		\nbr{
			\frac{1}{\sqrt{\alpha}\cdot \nu(k-1)}
			\Sigma_{k-1}^{\frac{1}{2}}\xi_k
		}_2
		&\leq
		\sqrt{
			d
			+
			2\sqrt{
				d\log\sbr{\frac{1}{\delta_k}}
			}
			+
			2\log\sbr{\frac{1}{\delta_k}}
		}
		\\
		&=
		\sqrt{
			\sbr{
				\sqrt{d}
				+
				\sqrt{\log\sbr{\frac{1}{\delta_k}}}
			}^2
			+
			\log\sbr{\frac{1}{\delta_k}}
		}
		\\
		&\leq
		\sqrt{d}
		+
		2\sqrt{\log\sbr{\frac{1}{\delta_k}}}.
	\end{align*}
	Combining this bound with Eq.~\eqref{eq:known_tran_phi_xi_decomposition}
	proves the first statement.

	Next, we prove the second statement.

	For a random variable $X \in \R^{d}$, we decompose
	\begin{align*}
		\ex\mbr{
			|X^\top \xi_k|
			\mid F_{k-1}
		}
		&=
		\ex\mbr{
			|X^\top \xi_k|
			\cT_k(\delta_k)
			\mid F_{k-1}
		}
		+
		\ex\mbr{
			|X^\top \xi_k|
			\left({{\cT}_k(\delta_k)}\right)^c
			\mid F_{k-1}
		}
		\\
		&\leq
		\sqrt{\alpha} \cdot \nu(k-1)
		\sbr{
			\sqrt{d}
			+
			2\sqrt{\log\sbr{\frac{1}{\delta_k}}}
		}
		\ex\mbr{
			\nbr{X}_{\Sigma_{k-1}^{-1}}
			\mid F_{k-1}
		}
		\\
		&\quad
		+
		\sqrt{
			\Pr\mbr{
				\left({{\cT}_k(\delta_k)}\right)^c
				\mid F_{k-1}
			}
			\cdot
			\ex\mbr{
				|X^\top \xi_k|^2
				\mid F_{k-1}
			}
		}
		\\
		&\overset{\textup{(a)}}{\leq}
		\sqrt{\alpha} \cdot \nu(k-1)
		\sbr{
			\sqrt{d}
			+
			2\sqrt{\log\sbr{\frac{1}{\delta_k}}}
		}
		\ex\mbr{
			\nbr{X}_{\Sigma_{k-1}^{-1}}
			\mid F_{k-1}
		}
		\\
		&\quad
		+
		\sqrt{\alpha} \cdot \nu(k-1)
		\sqrt{
			\delta_k
			\ex\mbr{
				\nbr{X}_{\Sigma_{k-1}^{-1}}^2
				\cdot
				\nbr{
					\frac{1}{\sqrt{\alpha}\cdot \nu(k-1)}
					\Sigma_{k-1}^{\frac{1}{2}}\xi_k
				}_2^2
				\mid F_{k-1}
			}
		}
		\\
		&\leq
		\sqrt{\alpha} \cdot \nu(k-1)
		\sbr{
			\sqrt{d}
			+
			2\sqrt{\log\sbr{\frac{1}{\delta_k}}}
		}
		\ex\mbr{
			\nbr{X}_{\Sigma_{k-1}^{-1}}
			\mid F_{k-1}
		}
		\\
		&\quad
		+
		\sqrt{\alpha} \cdot \nu(k-1)
		\sqrt{
			\delta_k L_X^2
			\ex\mbr{
				\nbr{
					\frac{1}{\sqrt{\alpha}\cdot \nu(k-1)}
					\Sigma_{k-1}^{\frac{1}{2}}\xi_k
				}_2^2
				\mid F_{k-1}
			}
		}
		\\
		&\overset{\textup{(b)}}{\leq}
		\sqrt{\alpha} \cdot \nu(k-1)
		\sbr{
			\sqrt{d}
			+
			2\sqrt{\log\sbr{\frac{1}{\delta_k}}}
		}
		\ex\mbr{
			\nbr{X}_{\Sigma_{k-1}^{-1}}
			\mid F_{k-1}
		}
		\\
		&\quad
		+
		\sqrt{\alpha} \cdot \nu(k-1)
		\cdot
		L_X\sqrt{d\delta_k}.
	\end{align*}
	Here inequality \textup{(a)} follows from the Cauchy-Schwarz inequality
	and the bound
	$\Pr\mbr{\left({{\cT}_k(\delta_k)}\right)^c\mid F_{k-1}}\leq \delta_k$.
	Inequality \textup{(b)} follows because, conditioned on $F_{k-1}$,
		$\nbr{
			\frac{1}{\sqrt{\alpha}\cdot \nu(k-1)}
			\Sigma_{k-1}^{\frac{1}{2}}\xi_k
		}_2$
	is chi-distributed with parameter $d$, and therefore
$		\ex\mbr{
			\nbr{
				\frac{1}{\sqrt{\alpha}\cdot \nu(k-1)}
				\Sigma_{k-1}^{\frac{1}{2}}\xi_k
			}_2^2
			\mid F_{k-1}
		}
		=
		d .$
\end{proof}

For any $k \geq 1$, let $\tau^k$ denote the trajectory in episode $k$.

Define event 
\begin{align}
	\cF^{\bi}_{\textup{KTran}}:=\Bigg\{& \abr{\sum_{k'=1}^{k} \sbr{ \ex\mbr{ \nbr{ \phi^{\tau^{k'}}}_{\Sigma_{k'-1}^{-1}} | F_{k'-1}} - \nbr{ \phi^{\tau^{k'}}}_{\Sigma_{k'-1}^{-1}} } } \leq 4H\sqrt{ \frac{k}{\alpha \lambda} \log\sbr{\frac{4k}{\delta'}} } ,
	\nonumber\\
	&\abr{\sum_{k'=1}^{k} \sbr{ \ex\mbr{ (\phi^{\pi^{k'}})^\top \theta^* | F_{k'-1} } - (\phi^{\pi^{k'}})^\top \theta^* }} \leq 4 H r_{\max} \sqrt{ k \log\sbr{\frac{4k}{\delta'}} } ,\ \forall k>0 \Bigg\} .
\end{align}

\begin{lemma} \label{lemma:lin_cF_binary}
	It holds that
	\begin{align*}
		\Pr \mbr{ \cF^{\bi}_{\textup{KTran}} } \geq 1-2\delta' .
	\end{align*}
\end{lemma}
\begin{proof}
	We prove the first inequality as follows.
	
	For any $k'\geq 1$, we have that $\nbr{ \phi^{\tau^{k'}}}_{\Sigma_{k'-1}^{-1}} \leq \frac{H}{\sqrt{\alpha \lambda}}$, and then $|\ex \left[\nbr{ \phi^{\tau^{k'}}}_{\Sigma_{k'-1}^{-1}} | F_{k'-1}\right] - \nbr{ \phi^{\tau^{k'}}}_{\Sigma_{k'-1}^{-1}}| \leq \frac{2H}{\sqrt{\alpha \lambda}}$.

	Using the Azuma-Hoeffding inequality, we have that for any fixed $k>0$, with probability at least $1-\frac{\delta'}{2k^2}$,
	\begin{align*}
		\abr{\sum_{k'=1}^{k} \sbr{ \ex \left[\nbr{ \phi^{\tau^{k'}}}_{\Sigma_{k'-1}^{-1}} | F_{k'-1}\right] - \nbr{ \phi^{\tau^{k'}}}_{\Sigma_{k'-1}^{-1}} } } &\leq \sqrt{ 2 \cdot \frac{4H^2}{\alpha \lambda} \cdot k \log\sbr{\frac{4k^2}{\delta'}} } .
	\end{align*}
	
	Since $\sum_{k=1}^{\infty} \frac{\delta'}{2k^2} \leq \delta'$, by a union bound over $k$, we have that with probability at least $1-\delta'$, for any $k\geq 1$,
	\begin{align*}
		\abr{\sum_{k'=1}^{k} \sbr{ \ex \left[\nbr{ \phi^{\tau^{k'}}}_{\Sigma_{k'-1}^{-1}} | F_{k'-1}\right] - \nbr{ \phi^{\tau^{k'}}}_{\Sigma_{k'-1}^{-1}} } } &\leq \sqrt{ 2 \cdot \frac{4H^2}{\alpha \lambda} \cdot k \log\sbr{\frac{4k^2}{\delta'}} } 
		\\
		&\leq 4H \sqrt{ \frac{k}{\alpha \lambda} \log\sbr{\frac{4k}{\delta'}} } .
	\end{align*}
	The second inequality can be obtained by a similar argument and the fact that $|(\phi^{\pi^{k}})^\top \theta^*| \leq Hr_{\max}$ for any $k>0$.
\end{proof}

\begin{lemma} \label{lemma:lin_sum_phi_seg_binary}
	For any $K\geq 1$, we have
	\begin{align*}
		\sum_{k=1}^{K} \sum_{i=1}^{m} \nbr{ \phi^{\tau^k_i}}_{\Sigma_{k-1}^{-1}} \leq \sqrt{ 2Km d \cdot \max\lbr{ \frac{H^2}{m \alpha \lambda}, 1} \cdot \log \sbr{ 1+ \frac{  KH^2 }{ \alpha \lambda d m } } } .
	\end{align*}
\end{lemma}
\begin{proof}		
	We have
	\begin{align}
		\sum_{k=1}^{K} \sum_{i=1}^{m} \nbr{ \phi^{\tau^k_i}}_{\Sigma_{k-1}^{-1}} &\leq \sqrt{ Km   \sum_{k=1}^{K} \sum_{i=1}^{m} \nbr{ \phi^{\tau^k_i}}_{\Sigma_{k-1}^{-1}}^2 }
		\nonumber\\
		&\overset{\textup{(a)}}{\leq} \sqrt{ 2Km \cdot \max\lbr{ \frac{H^2}{m \alpha \lambda}, 1} \cdot \sum_{k=1}^{K} \log \sbr{ 1+ \sum_{i=1}^{m} \nbr{ \phi^{\tau^k_i}}_{\Sigma_{k-1}^{-1}}^2} }
		\nonumber\\ 
		&\le \sqrt{ 2Km \cdot \max\lbr{ \frac{H^2}{m \alpha \lambda}, 1} \cdot \log \sbr{ \frac{\det(\Sigma_K)}{\det(\alpha \lambda I)} } } 
		\nonumber\\
		&\leq \sqrt{ 2Km d \cdot \max\lbr{ \frac{H^2}{m \alpha \lambda}, 1} \cdot \log \sbr{ 1+ \frac{  KH^2 }{ \alpha \lambda d m } } } , \label{eq:sum_indicator_derivation}
	\end{align}
	where inequality (a) is due to that for any $x \in [0,c]$ with constant $c\geq0$, it holds that $x \leq 2 \max\{c,1\} \cdot \log(1+x)$.
\end{proof}

\begin{proof}[Proof of Theorem~\ref{thm:seg-bits-upper}.]Let $\delta'=\frac{\delta}{3}$. Then,
		$\Pr\left[\cE \cap \cF^{\bi}_{\textup{KTran}}\right]
		\geq 1-\delta .$
	Therefore, it remains to prove the desired regret bound on the event
	$\cE \cap \cF^{\bi}_{\textup{KTran}}$.

	Assume that event $\cE \cap \cF^{\bi}_{\textup{KTran}}$ holds. Then,
	\begin{align}
		\cR(K)
		&=
		\sum_{k=1}^{K}
		\sbr{
			(\phi^{\pi^*})^\top \theta^*
			-
			(\phi^{\pi^k})^\top \theta^*
		}
		\nonumber\\
		&=
		\sum_{k=1}^{K}
		\sbr{
			\ex\mbr{
				(\phi^{\pi^*})^\top \theta^*
				-
				(\phi^{\pi^k})^\top \theta^*
				\mid F_{k-1}
			}
			+
			\ex\mbr{
				(\phi^{\pi^k})^\top \theta^*
				\mid F_{k-1}
			}
			-
			(\phi^{\pi^k})^\top \theta^*
		}
		\nonumber\\
		&\leq
		\sum_{k=1}^{K}
		\ex\mbr{
			(\phi^{\pi^*})^\top \theta^*
			-
			(\phi^{\pi^k})^\top \theta^*
			\mid F_{k-1}
		}
		+
		4Hr_{\max}
		\sqrt{
			K\log\sbr{\frac{4K}{\delta'}}
		}.
		\label{eq:lin_known_tran_regret_decomposition}
	\end{align}

	For the first term, we further decompose
	\begin{align}
		&\quad
		\sum_{k=1}^{K}
		\ex\mbr{
			(\phi^{\pi^*})^\top \theta^*
			-
			(\phi^{\pi^k})^\top \theta^*
			\mid F_{k-1}
		}
		\nonumber\\
		&=
		\sum_{k=1}^{K}
		\sbr{
			\ex\mbr{
				(\phi^{\pi^*})^\top \theta^*
				-
				(\phi^{\pi^k})^\top \widetilde{\theta}_k
				\mid F_{k-1}
			}
			+
			\ex\mbr{
				(\phi^{\pi^k})^\top \widetilde{\theta}_k
				-
				(\phi^{\pi^k})^\top \theta^*
				\mid F_{k-1}
			}
		}.
		\label{eq:lin_known_tran_ex_phi_star-phi_pi_k_theta}
	\end{align}

	We first show that
	\begin{align}
		\ex\mbr{
			(\phi^{\pi^*})^{\top}\theta^*
			-
			(\phi^{\pi^k})^{\top}\widetilde{\theta}_k
			\mid F_{k-1}
		}
		\leq
		2\sqrt{2\pi e}
		\cdot
		\ex\mbr{
			\sbr{
				(\phi^{\pi^k})^{\top}\widetilde{\theta}_k
				-
				\ex\mbr{
					(\phi^{\pi^k})^{\top}\widetilde{\theta}_k
					\mid F_{k-1}
				}
			}^{+}
			\mid F_{k-1}
		}.
		\label{eq:lin_known_tran_fist_term_leq_positive}
	\end{align}

	If
	\[
		\ex\mbr{
			(\phi^{\pi^*})^\top \theta^*
			-
			(\phi^{\pi^k})^\top \widetilde{\theta}_k
			\mid F_{k-1}
		}
		<0,
	\]
	then Eq.~\eqref{eq:lin_known_tran_fist_term_leq_positive} holds trivially.
	Otherwise, let
	\[
		z
		:=
		\ex\mbr{
			(\phi^{\pi^*})^\top \theta^*
			-
			(\phi^{\pi^k})^\top \widetilde{\theta}_k
			\mid F_{k-1}
		}.
	\]
	Then,
	\begin{align*}
		&\quad
		\ex\mbr{
			\sbr{
				(\phi^{\pi^k})^\top \widetilde{\theta}_k
				-
				\ex\mbr{
					(\phi^{\pi^k})^\top \widetilde{\theta}_k
					\mid F_{k-1}
				}
			}^{+}
			\mid F_{k-1}
		}
		\\
		&\geq
		z
		\Pr\mbr{
			(\phi^{\pi^k})^\top \widetilde{\theta}_k
			-
			\ex\mbr{
				(\phi^{\pi^k})^\top \widetilde{\theta}_k
				\mid F_{k-1}
			}
			\geq z
			\mid F_{k-1}
		}
		\\
		&\geq
		\ex\mbr{
			(\phi^{\pi^*})^\top \theta^*
			-
			(\phi^{\pi^k})^\top \widetilde{\theta}_k
			\mid F_{k-1}
		}
		\cdot
		\Pr\mbr{
			(\phi^{\pi^k})^\top \widetilde{\theta}_k
			\geq
			(\phi^{\pi^*})^\top \theta^*
			\mid F_{k-1}
		}
		\\
		&\overset{\textup{(a)}}{\geq}
		\ex\mbr{
			(\phi^{\pi^*})^\top \theta^*
			-
			(\phi^{\pi^k})^\top \widetilde{\theta}_k
			\mid F_{k-1}
		}
		\cdot
		\Pr\mbr{
			(\phi^{\pi^*})^\top \widetilde{\theta}_k
			\geq
			(\phi^{\pi^*})^\top \theta^*
			\mid F_{k-1}
		}
		\\
		&\overset{\textup{(b)}}{\geq}
		\ex\mbr{
			(\phi^{\pi^*})^\top \theta^*
			-
			(\phi^{\pi^k})^\top \widetilde{\theta}_k
			\mid F_{k-1}
		}
		\cdot
		\frac{1}{2\sqrt{2\pi e}},
	\end{align*}
	where inequality \textup{(a)} follows from the definition of $\pi^k$,
	and inequality \textup{(b)} follows from
	Lemma~\ref{lemma:lin_Gaussian_anti_concentration}. This proves
	Eq.~\eqref{eq:lin_known_tran_fist_term_leq_positive}.

	Let $\xi'_k \in \R^{d}$ be an independent copy of $\xi_k$ conditioned on
	$F_{k-1}$. Applying Lemma~\ref{lemma:lin_xi_xi'} with $\tilde{P}=P$,
	$x_{k-1}=\hat{\theta}_{k-1}$, and $\tilde{\pi}^k=\pi^k$, we obtain
	\begin{align*}
		&\quad
		\ex\mbr{
			(\phi^{\pi^*})^\top \theta^*
			-
			(\phi^{\pi^k})^\top \widetilde{\theta}_k
			\mid F_{k-1}
		}
		\\
		&\leq
		2\sqrt{2\pi e}
		\cdot
		\ex\mbr{
			\sbr{
				(\phi^{\pi^k})^\top \widetilde{\theta}_k
				-
				\ex\mbr{
					(\phi^{\pi^k})^\top \widetilde{\theta}_k
					\mid F_{k-1}
				}
			}^{+}
			\mid F_{k-1}
		}
		\\
		&\leq
		2\sqrt{2\pi e}
		\cdot
		\ex\mbr{
			\left|(\phi^{\pi^k})^\top \xi_k\right|
			+
			\left|(\phi^{\pi^k})^\top \xi'_k\right|
			\mid F_{k-1}
		}.
	\end{align*}

	Substituting the above bound into
	Eq.~\eqref{eq:lin_known_tran_ex_phi_star-phi_pi_k_theta}, and then applying
	Lemma~\ref{lemma:lin_X_top_xi_k} with
	$\delta_k=1/k^4$ and $L_X=H/\sqrt{\alpha\lambda}$, we have
	\begin{align}
		&\quad
		\sum_{k=1}^{K}
		\ex\mbr{
			(\phi^{\pi^*})^\top \theta^*
			-
			(\phi^{\pi^k})^\top \theta^*
			\mid F_{k-1}
		}
		\nonumber\\
		&=
		\sum_{k=1}^{K}
		\ex\mbr{
			(\phi^{\pi^*})^\top \theta^*
			-
			(\phi^{\pi^k})^\top \widetilde{\theta}_k
			+
			(\phi^{\pi^k})^\top \widetilde{\theta}_k
			-
			(\phi^{\pi^k})^\top \theta^*
			\mid F_{k-1}
		}
		\nonumber\\
		&\leq
		\sum_{k=1}^{K}
		\bigg(
		2\sqrt{2\pi e}
		\cdot
		\ex\mbr{
			\left|(\phi^{\pi^k})^\top \xi_k\right|
			+
			\left|(\phi^{\pi^k})^\top \xi'_k\right|
			\mid F_{k-1}
		}
		\nonumber\\
		&\qquad
		+
		\ex\mbr{
			(\phi^{\pi^k})^\top
			\sbr{
				\hat{\theta}_{k-1}
				+
				\xi_k
			}
			-
			(\phi^{\pi^k})^\top \theta^*
			\mid F_{k-1}
		}
		\bigg)
		\nonumber\\
		&\leq
		\sum_{k=1}^{K}
		\bigg(
		\sbr{2\sqrt{2\pi e}+1}
		\cdot
		\ex\mbr{
			\left|(\phi^{\pi^k})^\top \xi_k\right|
			\mid F_{k-1}
		}
		+
		2\sqrt{2\pi e}
		\cdot
		\ex\mbr{
			\left|(\phi^{\pi^k})^\top \xi'_k\right|
			\mid F_{k-1}
		}
		\nonumber\\
		&\qquad
		+
		\ex\mbr{
			(\phi^{\pi^k})^\top \hat{\theta}_{k-1}
			-
			(\phi^{\pi^k})^\top \theta^*
			\mid F_{k-1}
		}
		\bigg)
		\nonumber\\
		&\overset{\textup{(a)}}{\leq}
		\sum_{k=1}^{K}
		\Bigg(
		\sbr{4\sqrt{2\pi e}+2}
		\sqrt{\alpha}
		\cdot
		\nu(k-1)
		\sbr{
			\sqrt{d}
			+
			4\sqrt{\log\sbr{k}}
		}
		\cdot
		\ex\mbr{
			\nbr{\phi^{\pi^k}}_{\Sigma_{k-1}^{-1}}
			\mid F_{k-1}
		}
		\nonumber\\
		&\qquad
		+
		\sbr{4\sqrt{2\pi e}+1}
		\sqrt{\alpha}
		\cdot
		\nu(k-1)
		\frac{\sqrt{d}}{k^2}
		\cdot
		\frac{H}{\sqrt{\alpha\lambda}}
		\Bigg),
		\label{eq:lin_ex_phi_star_theta_star-phi_k_theta_star}
	\end{align}
	where inequality \textup{(a)} follows from
	Lemmas~\ref{lemma:lin_confence_interval_proj_free}
	and~\ref{lemma:lin_X_top_xi_k}.

	Moreover, by the definition of event $\cF^{\bi}_{\textup{KTran}}$ and
	Lemma~\ref{lemma:lin_sum_phi_seg_binary}, we have
	\begin{align}
		&\quad
		\sum_{k=1}^{K}
		\ex\mbr{
			\nbr{\phi^{\pi^k}}_{\Sigma_{k-1}^{-1}}
			\mid F_{k-1}
		}
        \\&\overset{\textup{(a)}}{\leq}\sum_{k=1}^{K}\ex\mbr{ \nbr{ \phi^{\tau^{k}}}_{\Sigma_{k-1}^{-1}} | F_{k-1}}
		\nonumber\\
		&=
		\sum_{k=1}^{K}
		\sbr{
			\ex\mbr{ \nbr{ \phi^{\tau^{k}}}_{\Sigma_{k-1}^{-1}} | F_{k-1}}
			-
			\nbr{ \phi^{\tau^{k}}}_{\Sigma_{k-1}^{-1}}
		}
		+
		\sum_{k=1}^{K}
		\nbr{ \phi^{\tau^{k}}}_{\Sigma_{k-1}^{-1}}
		\nonumber\\
		&\leq
		4H
		\sqrt{
			\frac{K}{\alpha\lambda}
			\log\sbr{\frac{4K}{\delta'}}
		}
		\nonumber\\
		&\quad
		+
		\sqrt{
			2Kmd
			\cdot
			\max\lbr{
				\frac{H^2}{m\alpha\lambda},
				1
			}
			\cdot
			\log\sbr{
				1+
				\frac{KH^2}{\alpha\lambda dm}
			}
		}.
		\label{eq:lin_ex_sum_phi_seg}
	\end{align}
where inequality (a) follows from Jensen's inequality, the identity
\(\phi^{\pi^k}=\E[\phi^{\tau^k}\mid \tilde F_k]\), and the tower property of conditional expectation.
	Combining
	Eq.~\eqref{eq:lin_ex_phi_star_theta_star-phi_k_theta_star},
	Eq.~\eqref{eq:lin_ex_sum_phi_seg}, and
	Eq.~\eqref{eq:lin_known_tran_regret_decomposition}, we obtain
	\begin{align*}
		\cR(K)
		&\leq
		\sbr{4\sqrt{2\pi e}+2}
		\sqrt{\alpha}
		\cdot
		\nu(K)
		\sbr{
			\sqrt{d}
			+
			4\sqrt{\log\sbr{K}}
		}
		\\
		&\quad
		\cdot
		\sbr{
			4H
			\sqrt{
				\frac{K}{\alpha\lambda}
				\log\sbr{\frac{4K}{\delta'}}
			}
			+
			\sqrt{
				2Kmd
				\max\lbr{
					\frac{H^2}{m\alpha\lambda},
					1
				}
				\log\sbr{
					1+
					\frac{KH^2}{\alpha\lambda dm}
				}
			}
		}
		\\
		&\quad
		+
		2\sbr{4\sqrt{2\pi e}+1}
		H
		\cdot
		\nu(K)
		\sqrt{
			\frac{d}{\lambda}
		}
		+
		4Hr_{\max}
		\sqrt{
			K\log\sbr{\frac{4K}{\delta'}}
		}
		\\
		&\overset{\textup{(a)}}{=}
		\widetilde{O}
		\Bigg(
			\exp\left(\frac{Hr_{\max}}{2m}\right)
			\cdot
			\nu(K)
			\sqrt{d}
			\sbr{
				\sqrt{
					Kmd
					\cdot
					\max\lbr{
						\frac{H^2}{m\alpha\lambda},
						1
					}
				}
				+
				H
				\sqrt{
					\frac{K}{\alpha\lambda}
				}
			}
		\Bigg),
	\end{align*}
	where in step \textup{(a)}, the last two terms are absorbed into the
	$\widetilde{O}(\cdot)$ notation.
\end{proof}

\subsection{Proof for the Regret Lower Bound}

In this section, we prove the regret lower bound with known transitions under fixed segment binary feedback.

\begin{proof}[Proof of Theorem~\ref{thm:seg-bits-lower}]
   Consider a linear MDP instance : $\cS=\{s_1,s_2,\dots,s_N\}$, $\cA=\{-1,1\}^{d-1}$, linear map $\psi(s_j,a)=(\sqrt{\frac{1}{d}} a^\top,\sqrt{\frac{1}{d}})^{\top}$, where $a\in\{-1,1\}^{d-1}$. For any $h\in [H]$, $\mu_h^*(s_j)=(\mathbf{0}_{d-1}^\top,\frac{\sqrt{d}}{N})^{\top}$.
    Under this construction, 
    $P_h(s_{j'}\mid s_j,a)=\langle \psi(s_j,a),\mu_h^*(s_{j'})\rangle
    =\frac{1}{N}$ for all $h\in[H]$ and $a\in\mathcal A$.

    We define $\Delta
:=
\frac{1}
{
8\cdot \frac{H}{m}
\sqrt{
Km\,\sig'\!\left(\frac{H}{m}r_{\max}-\frac{1}{2}\right)
}
}$. Then, we construct reward parameter $\theta^*_\gamma = (\sqrt{d}\Delta\gamma^\top,\sqrt{d}(r_{\max}-(d-1)\Delta))^\top$, where $\gamma\in \{-1,1\}^{d-1}$. We verify that the above construction indeed satisfies the Assumption~\ref{assump:linear_mdp}.

First,
for any $s_j\in \mathcal S$ and $a\in\mathcal A$, we have
\[
\|\psi(s_j,a)\|^2
=
\frac{1}{d}\|a\|^2 +\frac{1}{d}
=
\frac{d-1}{d}+\frac{1}{d}
=
1,
\]
where the second equality follows from $a\in\{-1,1\}^{d-1}$.

Moreover,
\[
\mu_h^*(\mathcal S)
=
\sum_{j=1}^N \mu_h^*(s_j)
=
(\mathbf{0}_{d-1}^\top,\sqrt{d})^{\top},
\]
and therefore
\[
\|\mu_h^*(\mathcal S)\|
=
\sqrt d.
\]

Finally, the reward is also linear in $\psi$. Indeed, for any $s_j\in \mathcal S$, 
\begin{align*}
r_\gamma(s_j,a)
&=
\left(r_{\max}-(d-1)\Delta\right)
+
\left\langle
\sqrt{\frac{1}{d}}a,
\sqrt{d}\Delta\gamma
\right\rangle \\
&=
r_{\max}-(d-1)\Delta+\Delta\langle a,\gamma\rangle .
\end{align*}
Furthermore,
\begin{align*}
\left\|\theta_\gamma^*\right\|^2
&=
\left\|
\sqrt d\Delta\gamma
\right\|^2
+
d\left(r_{\max}-(d-1)\Delta\right)^2 \\
&=
d\Delta^2\|\gamma\|^2
+
d\left(r_{\max}-(d-1)\Delta\right)^2 \\
&=
d(d-1)\Delta^2
+
d\left(r_{\max}-(d-1)\Delta\right)^2 \\
&=
d\left[
(d-1)\Delta^2
+
\left(r_{\max}-(d-1)\Delta\right)^2
\right].
\end{align*}
When $K
\ge
\frac{(d-1)^2 m}
{
64 H^2 r_{\max}^2
\sig'\!\left(\frac{H}{m}r_{\max}-\frac{1}{2}\right)
}$, we have $|r_\gamma(s_j,a)| \le r_{\max}$ and $\|\theta_\gamma^*\| \le r_{\max}\sqrt{d}$. According to all above, this linear MDP satisfies the Assumption~\ref{assump:linear_mdp}.

    If we set the segment length to be $\frac{H}{m}$, then for any algorithm $\mathbb A$, after running $K$ episodes, we obtain a total of $Km$ segments. In the $t$-th segment, the algorithm selects actions $a_{t,1},a_{t,2},\ldots,a_{t,\frac{H}{m}}\in\{-1,1\}^{d-1}$. We define the average action selected by the algorithm in the $t$-th segment as
    \[\bar A_t := \frac{1}{H/m}\sum_{\ell=1}^{H/m} a_{t,\ell}\]
    , where $\bar A_t\in[-1,1]^{d-1}$. In addition, conditional on selected actions, the feedback of $t$-th segment $y^{\gamma}_t\sim \cB\!\left(
    \sig\!\left(
    \frac{H}{m}
    \left[
    r_{\max}-(d-1)\Delta+\Delta\left\langle \bar A_t,\gamma\right\rangle
    \right]
    \right)
    \right).$ 
    
    If we further require $\frac{H}{m}(d-1)\Delta \le \frac{1}{4},$ we have
    
    \begin{align}
        \frac{H}{m}
    \left[
    r_{\max}-(d-1)\Delta+\Delta\left\langle \bar A_t,\gamma\right\rangle
    \right]
    \ge
    \frac{H}{m}
    \left[
    r_{\max}-2(d-1)\Delta
    \right]
    \ge \frac{H}{m}r_{\max} - \frac{1}{2}. \label{eq:lower_bound_logit}
    \end{align}
Therefore, under the condition
$
K
\ge
\frac{(d-1)^2}
{
m\,\sig'\!\left(\frac{H}{m}r_{\max}-\frac12\right)
}
\cdot
\max\left\{
\frac{m^2}{64H^2r_{\max}^2},
\frac14
\right\}
$, Assumption~\ref{assump:linear_mdp} holds and Eq.\eqref{eq:lower_bound_logit} also holds.

 Let $\mathbb D_\gamma$ denote the probability measure induced by algorithm $\mathbb A$
under the instance parameterized by $\gamma$. By Theorem 24.1 in (Lattimore and Szepesv{\'a}r,~\citeyear{lattimore2020bandit}) and Eq.~\eqref{eq:lower_bound_logit}, we have for any $\gamma,\gamma'\in\{-1,1\}^{d-1}$ that

\begin{align}
\operatorname{KL}(\mathbb{D}_\gamma,\mathbb{D}_{\gamma'})
&=
\E_\gamma\Bigg[
\sum_{t=1}^{Km}
\operatorname{KL}\Bigg(
\cB\!\left(
    \sig\!\left(
        \frac{H}{m}
        \left[
            r_{\max}-(d-1)\Delta
            +\Delta\langle\bar A_t,\gamma\rangle
        \right]
    \right)
\right),
\notag\\[-2pt]
&\qquad\qquad
\cB\!\left(
    \sig\!\left(
        \frac{H}{m}
        \left[
            r_{\max}-(d-1)\Delta
            +\Delta\langle\bar A_t,\gamma'\rangle
        \right]
    \right)
\right)
\Bigg)
\Bigg]
\notag\\
&\overset{\textup{(a)}}{\leq}
\frac{1}{2}
\sig'\!\left(
    \frac{H}{m}r_{\max}-\frac{1}{2}
\right)
\sum_{t=1}^{Km}
\E_\gamma\left[
    \left(
        \frac{H\Delta}{m}
        \langle\bar A_t,\gamma-\gamma'\rangle
    \right)^2
\right],
\label{eq:hyer_KL_Bi}
\end{align}
where inequality (a) follows from $\operatorname{KL}(\mathcal B(\sig(x)),\mathcal B(\sig(y)))\le\frac12 \sup_{u\in[x,y]}\sig'(u)(x-y)^2$. By Eq.\eqref{eq:lower_bound_logit}, both logits are at least $\frac{H}{m}r_{\max}-\frac12.$ When \(r_{\max}\ge 1/2\), this lower bound is nonnegative. Hence, using the monotonicity of \(\sig'\) on \([0,\infty)\), we get the claimed bound.

For $i \in [d-1]$ and $\gamma \in \{-1,1\}^{d-1}$, we define 
\begin{align*}
    p_{\gamma,i} = \mathbb D_\gamma \left(\sum_{t=1}^{Km}\indicator\{\operatorname{sign}(\bar A_{t,i})\neq\operatorname{sign}(\gamma_i) \}\ge \frac{Km}{2}\right)
\end{align*}

Now fixing $i \in [d-1]$ and $\gamma \in \{-1,1\}^{d-1}$, let $\gamma'_j = \gamma_j$ for $j\neq i$ and $\gamma'_i = -\gamma_i$. Using Eq.\eqref{eq:hyer_KL_Bi}, we have 
\begin{align}
     \operatorname{KL}(\mathbb D_\gamma,\mathbb D_{\gamma'}) &\le \frac{1}{2}\sig'\!\left(\frac{H}{m}r_{\max}-\frac{1}{2}\right)\sum_{t=1}^{Km}\mathbb{E}_{\gamma} \left[4\left(\frac{H}{m}\right)^2\Delta^2\bar A_{t,i}^2\right] \notag
     \\
     & \le 2\sig'\!\left(\frac{H}{m}r_{\max}-\frac{1}{2}\right)Km\left(\frac{H}{m}\right)^2\Delta^2\notag
     \\&\le
2\sig'\!\left(\frac{H}{m}r_{\max}-\frac{1}{2}\right)
Km\left(\frac{H}{m}\right)^2
\cdot
\frac{
1
}{
64\left(\frac{H}{m}\right)^2
Km\,\sig'\!\left(\frac{H}{m}r_{\max}-\frac{1}{2}\right)
}.\notag
    \\&\le \frac{1}{32}  \label{eq:KL_up}
\end{align}
 Using the Bretagnolle–Huber inequality and Eq.\eqref{eq:KL_up}, we have 
\begin{align*}
    p_{\gamma,i} +p_{\gamma',i} \ge \frac{1}{2}\exp\left(-\operatorname{KL}(\mathbb D_\gamma,\mathbb D_{\gamma'})\right) \ge \frac{1}{2}\exp\left(-\frac{1}{32}\right)
\end{align*}
Then taking sum over $i \in [d-1]$ and taking average over $\gamma \in \{-1,1\}^{d-1}$, we obtain
\begin{align*}
    \sum_{\gamma\in\{-1,1\}^{d-1}}\frac{1}{2^{d-1}}\sum_{i=1}^{d-1} p_{\gamma,i}\ge \frac{d-1}{4}\exp\left(-\frac{1}{32}\right)
\end{align*}

This show that there exists a $\gamma \in \{-1,1\}^{d-1}$ such that $\sum_{i=1}^{d-1} p_{\gamma,i}\ge \frac{d-1}{4}\exp\left(-\frac{1}{32}\right)$.By the definition of $p_{\gamma,i}$ and our construction, the regret for this $\gamma$ is 
\begin{align*}
\mathbb{E}\left[\cR(K)\right]
&=
\mathbb{E} \left[\sum_{k=1}^K\left(V^* - V^{\pi_k} \right) \right]
\\
&=
\mathbb{E}_{\gamma}
\left[
\sum_{k=1}^{K}
\sum_{h=1}^{H}
\Delta
\sum_{i=1}^{d-1}
\left(
\operatorname{sign}(\gamma_i) - a_{h,i}^{k}
\right)
\gamma_i
\right]
\\
&\ge
\frac{H}{m}\cdot\Delta\cdot
\sum_{i=1}^{d-1}
\mathbb{E}_{\gamma}
\left[
\sum_{t=1}^{Km}
\mathbf{1}
\left\{
\operatorname{sign}\!\left(\bar A_{t,i}\right)
\neq
\gamma_i
\right\}
\right]
\\
&\ge
\frac{H}{m}\cdot\Delta \cdot \frac{Km}{2} \sum_{i=1}^{d-1}\mathbb D_\gamma \left(\sum_{t=1}^{Km}\indicator\{\operatorname{sign}(\bar A_{t,i})\neq\operatorname{sign}(\gamma_i) \}\ge \frac{Km}{2}\right)
\\ &=\frac{H}{m}\cdot \Delta \cdot \frac{Km}{2}\sum_{i=1}^{d-1}p_{\gamma,i}
\\ &
\ge
\frac{H}{m}
\cdot
\frac{
1
}{
8\cdot \frac{H}{m}
\sqrt{
Km\,\sig'\!\left(\frac{H}{m}r_{\max}-\frac{1}{2}\right)
}
}
\cdot
\frac{Km(d-1)}{8}
\exp\!\left(-\frac{1}{32}\right)
\\&=
\frac{(d-1)}{64}
\exp\!\left(-\frac{1}{32}\right)
\sqrt{
\frac{Km}{
\sig'\!\left(\frac{H}{m}r_{\max}-\frac{1}{2}\right)
}
}.
\end{align*}
Therefore, we can construct a Linear MDP instance satisfying the assumptions, whose state space consists of $N$ states and whose action space has cardinality \(2^{d-1}\). 

For this instance, when $K
\ge
\frac{(d-1)^2}
{
m\,\sig'\!\left(\frac{H}{m}r_{\max}-\frac12\right)
}\cdot
\max\left\{
\frac{m^2}{64H^2r_{\max}^2},
\frac14
\right\}$ and $r_{\max}\ge \frac{1}{2}$, we have
\begin{align*}
    \cR(K) =
\Omega\!\left(
\exp\!\left(\frac{Hr_{\max}}{2m}\right)d
\sqrt{
Km
}
\right).
\end{align*}
\end{proof}

\section{Proofs for Fixed Equal-Length Sum Feedback with Known Transitions}\label{secC}

\subsection{Proof for the Regret Upper Bound }
In this section, we prove the regret upper bound for fixed equal-length sum feedback with known transitions.

Define event
\begin{align}
	& \cJ:=\Bigg\{ \nbr{ \sum_{k=1}^{K_0} \sbr{ \sum_{i=1}^{m} \phi^{\tau^k_i} (\phi^{\tau^k_i})^\top  - \ex_{\tau_i \sim \pi^k}\mbr{\sum_{i=1}^{m} \phi(\tau_i) \phi(\tau_i)^\top} } } 
	\nonumber\\
	&\leq \frac{4H^2}{m} \sqrt{K_0\log\sbr{\frac{2d}{\delta'}}}  + \frac{4H^2}{m} \log\sbr{\frac{2d}{\delta'}} \Bigg\} .
\end{align}

\begin{lemma}[Concentration of Initial Sampling] 
	It holds that
	\begin{align*}
		\Pr \mbr{ \cJ } \geq 1-\delta' .
	\end{align*}
\end{lemma}
\begin{proof}
	Note that the policies $\pi^1,\dots,\pi^{K_0}$ and the number $K_0$ are fixed before sampling. Thus, conditional on these policies, the only randomness comes from the trajectories generated under them. For each $k\in[K_0]$, we have
	$\ex[\sum_{i=1}^{m} \phi^{\tau^k_i} (\phi^{\tau^k_i})^\top]=\ex_{\tau_i \sim \pi^k}\mbr{\sum_{i=1}^{m} \phi(\tau_i) \phi(\tau_i)^\top}$, and
	$\|\sum_{i=1}^{m} \phi^{\tau^k_i} (\phi^{\tau^k_i})^\top\| \leq \frac{H^2}{m}$.
	The lemma then follows from the matrix Bernstein inequality (Theorem 6.1.1 in Tropp et al. \citeyear{tropp2015introduction}).
\end{proof}
Recall that $z^*$ is the optimal value of
$$
z^*
:=
\min_{w\in\Delta_{\Pi}}
\left\|
\left(
\sum_{\pi\in\Pi} w(\pi)
\sum_{i=1}^{m}
\E_{\tau_i\sim\pi}
\left[
\phi^{\tau_i}(\phi^{\tau_i})^\top
\right]
\right)^{-1}
\right\|.
$$
Here the expectation is over the trajectory segment induced by policy $\pi$ and the transition kernels $P=\{P_h\}_{h=1}^H$.
Let $w^*$ be an optimal solution attaining $z^*$.

\begin{lemma}[E-optimal Design] \label{lemma:seg_e_optimal_design_d}
	Assume that event $\cJ$ holds. Then, we have
	\begin{align*}
		\nbr{ \sbr{\sum_{k=1}^{K_0} \sum_{i=1}^{m} \phi^{\tau^k_i} (\phi^{\tau^k_i})^\top}^{-1} } \leq \frac{1}{H^2} .
	\end{align*}
\end{lemma}
\begin{proof}
	By the guarantee of the rounding procedure $\round$ (Theorem 1.1 in Allen-Zhu et al. \citeyear{allen2021near}) and the fact that $K_0 \geq \frac{d}{\gamma^2}$, we have
	\begin{align*}
		&\quad \nbr{ \sbr{\sum_{k=1}^{K_0}  \ex_{\tau_i \sim \pi^k}\mbr{\sum_{i=1}^{m} \phi(\tau_i) \phi(\tau_i)^\top} }^{-1} }
		\\
		&\leq (1+\gamma) \nbr{  \sbr{ K_0 \sum_{\pi \in \Pi} w^*(\pi) \cdot  \ex_{\tau_i \sim \pi^k}\mbr{\sum_{i=1}^{m} \phi(\tau_i) \phi(\tau_i)^\top} }^{-1} }
		\\
		&\leq \frac{(1+\gamma) z^*}{K_0} .
	\end{align*}

	Let $\sigma_{\min}(\cdot)$ denote the minimum eigenvalue. We next lower bound the minimum eigenvalue of the empirical design matrix obtained from the initial sampling phase. On the event $\cJ$, we have
	\begin{align}
		&\quad\ \sigma_{\min}\sbr{ \sum_{k=1}^{K_0} \sum_{i=1}^{m} \phi^{\tau^k_i} (\phi^{\tau^k_i})^\top }
		\nonumber\\
		&= \sigma_{\min}\!\sbr{ \sum_{k=1}^{K_0} \ex_{\tau_i \sim \pi^k}\mbr{\sum_{i=1}^{m} \phi(\tau_i) \phi(\tau_i)^{\!\top}} \!+\! \sum_{k=1}^{K_0} \sum_{i=1}^{m} \phi^{\tau^k_i} (\phi^{\tau^k_i})^{\!\top} \!-\! \sum_{k=1}^{K_0} \ex_{\tau_i \sim \pi^k}\!\mbr{\sum_{i=1}^{m} \phi(\tau_i) \phi(\tau_i)^{\!\top}} }
		\nonumber\\
		&\geq \!\sigma_{\min}\!\sbr{ \sum_{k=1}^{K_0} \ex_{\tau_i \sim \pi^k}\!\mbr{\sum_{i=1}^{m} \phi(\tau_i) \phi(\tau_i)^{\!\top}} } \!-\! \nbr{ \sum_{k=1}^{K_0} \sum_{i=1}^{m} \phi^{\tau^k_i} (\phi^{\tau^k_i})^{\!\top} \!-\! \sum_{k=1}^{K_0} \ex_{\tau_i \sim \pi^k}\!\mbr{\sum_{i=1}^{m} \phi(\tau_i) \phi(\tau_i)^{\!\top}} }
		\nonumber\\
		&\geq \frac{ K_0 }{(1+\gamma) z^*} - \frac{4H^2}{m}\sqrt{\log\sbr{\frac{2d}{\delta'}}} \cdot \sqrt{K_0} - \frac{4H^2}{m} \log\sbr{\frac{2d}{\delta'}} . \label{eq:seg_minimum_eig_derivation_d}
	\end{align}

	Let $x=\sqrt{K_0}$ and define
	\begin{align*}
		f(x) = \frac{ 1 }{(1+\gamma) z^*} \cdot x^2 - \frac{4H^2}{m} \sqrt{\log\sbr{\frac{2d}{\delta'}}} \cdot x - \frac{4H^2}{m} \log\sbr{\frac{2d}{\delta'}} - H^2 .
	\end{align*}
	
	By the elementary property of quadratic functions, if
	\begin{align}
		x \geq \frac{ \frac{4H^2}{m} \sqrt{\log\sbr{\frac{2d}{\delta'}}} + \sqrt{ \sbr{ \frac{4H^2}{m} \sqrt{\log\sbr{\frac{2d}{\delta'}}} }^2 + 4 \cdot \frac{ 1 }{(1+\gamma) z^*} \sbr{\frac{4H^2}{m} \log\sbr{\frac{2d}{\delta'}} + H^2 } } }{ 2 \cdot \frac{ 1 }{(1+\gamma) z^*} } , \label{eq:condition_quadratic_function_d}
	\end{align} 
	then we have $f(x)\geq0$.
	
	Thus, to ensure that Eq.~\eqref{eq:condition_quadratic_function_d} holds, it suffices to set 
	\begin{align*}
		K_0 &\geq \frac{(1+\gamma)^2 (z^*)^2 }{4} \cdot \Bigg( 2 \cdot \sbr{\frac{4H^2}{m}\sqrt{\log\sbr{\frac{2d}{\delta'}}}}^2 + 2 \cdot \sbr{ \frac{4H^2}{m}\sqrt{\log\sbr{\frac{2d}{\delta'}}} }^2 \\&\quad +  \frac{8  }{(1+\gamma) z^*} \cdot 5H^2 \log\sbr{\frac{2d}{\delta'}}   \Bigg)
		\\
		&=  \sbr{ \frac{ 16H^4 (1+\gamma)^2 (z^*)^2 }{m^2 } + 10 H^2 (1+\gamma) z^* } \log\sbr{\frac{2d}{\delta'}} .
	\end{align*}
	
	Furthermore, since $\|\sum_{\pi \in \Pi} w^*(\pi)  \ex_{\tau_i \sim \pi^k}\mbr{\sum_{i=1}^{m} \phi(\tau_i) \phi(\tau_i)^\top}\|_2 \leq H^2$ and hence $z^* \geq \frac{1}{H^2}$, to make the right-hand side of Eq.~\eqref{eq:seg_minimum_eig_derivation_d} no smaller than $H^2$, it suffices to set
	\begin{align*}
		K_0 \geq 26H^4 (1+\gamma)^2 (z^*)^2  \log\sbr{\frac{2d}{\delta'}} .
	\end{align*}
	
	Therefore, combining the definition of $K_0$ with Eq.~\eqref{eq:seg_minimum_eig_derivation_d}, we obtain
	\begin{align*}
\sigma_{\min}
\left(
\sum_{k=1}^{K_0}
\sum_{i=1}^m
\phi^{\tau_i^k}
(\phi^{\tau_i^k})^\top
\right)
\ge H^2.
\end{align*}
	
    which completes the proof.
\end{proof}

\begin{lemma} \label{lemma:log_det_seg_d}
	For any $k>0$,
	\begin{align*}
		\sum_{k'=1}^{k} \log\sbr{ 1+\sum_{i=1}^{m} \nbr{\phi^{\tau^{k'}_i}}_{\Sigma_{k'-1}^{-1}}^2 } \le \log \sbr{ \frac{\det(\Sigma_k)}{\det(\lambda I)} }
		\leq d \log \sbr{ 1+ \frac{  kH^2 }{ \lambda d m } } .
	\end{align*}
\end{lemma}
\begin{proof}
	For any $k>0$, it holds that
	\begin{align*}
		\det(\Sigma_k)&\overset{\textup{(a)}}{=}\det\sbr{\Sigma_{k-1}+ \sum_{i=1}^{m} \phi^{\tau^k_i} (\phi^{\tau^k_i})^\top}
		\\
		&=\det(\Sigma_{k-1}) \det\sbr{ I + \sum_{i=1}^{m} (\Sigma_{k-1})^{-\frac{1}{2}}   \phi^{\tau^k_i} (\phi^{\tau^k_i})^\top  (\Sigma_{k-1})^{-\frac{1}{2}} }
		\\
		&\overset{\textup{(b)}}{\ge} \det(\Sigma_{k-1}) \sbr{ 1+ \sum_{i=1}^{m} \nbr{\phi^{\tau^k_i}}_{\Sigma_{k-1}^{-1}}^2 }
		\\
		&\overset{\textup{(c)}}{\geq} \det(\lambda I) \prod_{k'=1}^{k} \sbr{ 1+ \sum_{i=1}^{m} \nbr{\phi^{\tau^{k'}_i}}_{\Sigma_{k'-1}^{-1}}^2 } ,
	\end{align*}
    where inequality (b) uses the fact that $\det(I+A)\geq 1+\trace(A)$ for any positive semi-definite matrix $A$, and inquality (c) comes from repeatedly applying steps (a)-(b).
    
	Taking the logarithm on both sides, we have
	\begin{align*}
		\log \det(\Sigma_k) \ge \log \det(\lambda I) +  \sum_{k'=1}^{k} \log\sbr{ 1+ \sum_{i=1}^{m} \nbr{\phi^{\tau^{k'}_i}}_{\Sigma_{k'-1}^{-1}}^2 } .
	\end{align*}
	Then,
	\begin{align*}
		\sum_{k'=1}^{k} \log\sbr{ 1+\sum_{i=1}^{m} \nbr{\phi^{\tau^{k'}_i}}_{\Sigma_{k'-1}^{-1}}^2 } &\le \log \sbr{ \frac{\det(\Sigma_k)}{\det(\lambda I)} }
		\\
		&\overset{\textup{(d)}}{\leq} \log \sbr{ \frac{\sbr{ \frac{\trace(\Sigma_k)}{d}}^{d} }{ \lambda^{d} } }
		\\
		&= d \log \sbr{ \frac{  \trace(\Sigma_k) }{ \lambda d } }
		\\
		&\leq d \log \sbr{ \frac{ \lambda d + km \cdot \frac{H^2}{m^2} }{ \lambda d } }
		\\
		&= d \log \sbr{ 1+\frac{ kH^2 }{ \lambda d m } } ,
	\end{align*}
	where inequality (d) uses the arithmetic mean-geometric mean inequality.
    	
    which completes the proof.
\end{proof}

\begin{lemma}[Elliptical Potential with Optimized Initialization]\label{lemma:ellip_potent_with_opt_init_seg_d}
	Assume that event $\cJ$ holds. 
	Then, for any $k \geq K_0+1$,
	\begin{align*}
		\sum_{i=1}^{m} \nbr{ \phi^{\tau^k_i}}_{\Sigma_{k-1}^{-1}}^2 \leq 1 .
	\end{align*}
	Furthermore, for any $K\geq K_0+1$,
	\begin{align*}
		\sum_{k=K_0+1}^{K} \sum_{i=1}^{m} \nbr{ \phi^{\tau^k_i}}_{\Sigma_{k-1}^{-1}} \leq \sqrt{2Km d \log \sbr{ 1+ \frac{  KH^2 }{ \lambda d m } }} .
	\end{align*}
\end{lemma}
\begin{proof}
	By Lemma~\ref{lemma:seg_e_optimal_design_d}, for any $k\geq K_0+1$, we have
	\begin{align*}
		\sum_{i=1}^{m} \nbr{ \phi^{\tau^k_i}}_{\Sigma_{k-1}^{-1}}^2
		&=
		\sum_{i=1}^{m}
		\nbr{ \phi^{\tau^k_i}}_{
		\sbr{
			\lambda I
			+
			\sum_{k'=1}^{K_0}
			\sum_{i=1}^{m}
			\phi^{\tau_i^{k'}}
			(\phi^{\tau_i^{k'}})^\top
			+
			\sum_{k'=K_0+1}^{k-1}
			\sum_{i=1}^{m}
			\phi^{\tau_i^{k'}}
			(\phi^{\tau_i^{k'}})^\top
		}^{-1}
		}^2
		\\
		&\leq
		\sum_{i=1}^{m}
		\nbr{ \phi^{\tau^k_i} }_{
		\sbr{
			\sum_{k'=1}^{K_0}
			\sum_{i=1}^{m}
			\phi^{\tau_i^{k'}}
			(\phi^{\tau_i^{k'}})^\top
		}^{-1}
		}^2
		\\
		&\leq
		m \cdot \frac{H^2}{m^2} \cdot \frac{1}{H^2}
		\\
		&\leq 1 .
	\end{align*}
	
	We then bound the cumulative elliptical norm by Cauchy--Schwarz, 
	\begin{align*}
		\sum_{k=K_0+1}^{K} \sum_{i=1}^{m} \nbr{ \phi^{\tau^k_i}}_{\Sigma_{k-1}^{-1}}
		&\leq
		\sqrt{
			K m
			\sum_{k=K_0+1}^{K}
			\sum_{i=1}^{m}
			\nbr{ \phi^{\tau^k_i}}_{\Sigma_{k-1}^{-1}}^2
		}
		\\
		&\overset{\textup{(a)}}{\leq}
		\sqrt{
			K m \cdot 2
			\sum_{k=K_0+1}^{K}
			\log \sbr{
				1
				+
				\sum_{i=1}^{m}
				\nbr{ \phi^{\tau^k_i}}_{\Sigma_{k-1}^{-1}}^2
			}
		}
		\\
		&\leq
		\sqrt{
			K m \cdot 2
			\sum_{k=1}^{K}
			\log \sbr{
				1
				+
				\sum_{i=1}^{m}
				\nbr{ \phi^{\tau^k_i}}_{\Sigma_{k-1}^{-1}}^2
			}
		}
		\\
		&\overset{\textup{(b)}}{\leq}
		\sqrt{
			2K m d
			\log \sbr{
				1
				+
				\frac{K H^2}{\lambda d m}
			}
		} ,
	\end{align*}
	where inequality (a) uses the fact that $x \leq 2\log(1+x)$ for any $x\in[0,1]$, and inequality (b) follows from Lemma~\ref{lemma:log_det_seg_d}.
    	
    which completes the proof.
\end{proof}
Define event 
\begin{align}
	\cK\!:=\!\lbr{ \nbr{\hat{\theta}_{k} \!-\! \theta^*}_{\Sigma_{k}} \!\leq\! \sqrt{ \frac{H d}{m} \left( \log\sbr{1 \!+\! \frac{kH^2}{\lambda dm}} \!+\! 2\log\sbr{\frac{1}{\delta'}} \right) } \!+\! r_{\max}\sqrt{\lambda d} ,\  \forall k>0 }\! . \label{eq:concentration_theta_seg_d}
\end{align}
\begin{lemma}[Concentration of $\hat{\theta}_k$ under Sum Feedback] \label{lemma:concentration_sum_d}
	It holds that
	\begin{align*}
		\Pr \mbr{ \cK } \geq 1-\delta' .
	\end{align*}
\end{lemma}

\begin{proof}
	For each segment, conditional on the observed trajectory, the sum feedback is $\frac{H}{m}$-sub-Gaussian. Moreover, by Assumption~\ref{assump:linear_mdp}, we have $\|\theta^*\|\leq r_{\max}\sqrt{d}$. Therefore, applying the self-normalized concentration inequality in Lemma 2 of (Abbasi-Yadkori et al., \citeyear{abbasi2011improved}) to the segment-level observations. Hence, the event $\cK$ holds with probability at least $1-\delta'$.
\end{proof}

For any $k\geq1$, Recall that $\tau^k$ denotes the trajectory in episode $k$ and $F_k$ denotes the filtration that includes all information up to the end of episode $k$.
Then, $\pi^{k+1}$ is $F_k$-measurable.

Define event 
\begin{align}
	&\cF^{\summing}_{\textup{opt}}:=\Bigg\{ \abr{\sum_{k'=K_0+1}^{k} \sbr{ \ex_{\tau \sim \pi^{k'}}\mbr{ \nbr{ \phi^{\tau}}_{\Sigma_{k'-1}^{-1}} | F_{k'-1}} - \nbr{ \phi^{\tau^{k'}}}_{\Sigma_{k'-1}^{-1}} } } 
	\nonumber\\
	&\hspace*{5em} \leq 4\sqrt{ k \log\sbr{\frac{4k}{\delta'}} } ,\ \forall k \geq K_0+1 \Bigg\} . \label{eq:def_martingale_d}
\end{align}

\begin{lemma}
	It holds that
	\begin{align*}
		\Pr\mbr{ \left(\cF^{\summing}_{\textup{opt}}\right)^c \cap \cJ }
		\leq \delta' .
	\end{align*}
\end{lemma}

\begin{proof}
	On the event $\cJ$, Lemma~\ref{lemma:ellip_potent_with_opt_init_seg_d} implies that for any $k'\geq K_0+1$,
	$\nbr{ \phi^{\tau^{k'}}}_{\Sigma_{k'-1}^{-1}} \leq 1$.
	Therefore,
	\[
	\left|
	\ex_{\tau \sim \pi^{k'}}
	\mbr{
		\nbr{ \phi^{\tau}}_{\Sigma_{k'-1}^{-1}}
		\mid F_{k'-1}
	}
	-
	\nbr{ \phi^{\tau^{k'}}}_{\Sigma_{k'-1}^{-1}}
	\right|
	\leq 2 .
	\]
	Moreover, conditional on $F_{k'-1}$, the policy $\pi^{k'}$ and the matrix $\Sigma_{k'-1}$ are fixed. 
	
	Hence, using the Azuma-Hoeffding inequality, for any fixed $k \geq K_0+1$, with probability at least $1-\frac{\delta'}{2k^2}$,
	\begin{align*}
		\abr{\sum_{k'=K_0+1}^{k} \!\! \sbr{ \ex_{\tau \sim \pi^{k'}}\mbr{ \nbr{ \phi^{\tau}}_{\Sigma_{k'-1}^{-1}} | F_{k'-1}} \!-\! \nbr{ \phi^{\tau^{k'}}}_{\Sigma_{k'-1}^{-1}} } } &\leq \sqrt{ 2 \cdot 4 (k-K_0-1) \log\sbr{\frac{4k^2}{\delta'}} } .
	\end{align*}
	
	Since $\sum_{k=K_0+1}^{\infty} \frac{\delta'}{2k^2} \leq \delta'$, by a union bound over all $k\geq K_0+1$, with probability at least $1-\delta'$, the above inequality holds simultaneously for all $k\geq K_0+1$.
	Therefore, on the event $\cJ$, with probability at least $1-\delta'$, for any $k\geq K_0+1$,
	\begin{align*}
		\abr{\sum_{k'=K_0+1}^{k} \sbr{ \ex_{\tau \sim \pi^{k'}}\mbr{ \nbr{ \phi^{\tau}}_{\Sigma_{k'-1}^{-1}} | F_{k'-1}} - \nbr{ \phi^{\tau^{k'}}}_{\Sigma_{k'-1}^{-1}} } } &\leq \sqrt{ 2 \cdot 4 (k-K_0-1) \log\sbr{\frac{4k^2}{\delta'}} } 
		\\
		&\leq 4 \sqrt{ k \log\sbr{\frac{4k}{\delta'}} } .
	\end{align*}
	
	which completes the proof.
\end{proof}

\begin{proof}[Proof of Theorem~\ref{thm:regret_sum_fixed_segment_d_known}]
Let $\delta'=\frac{\delta}{3}$. We have $\Pr\mbr{\cJ \cap \cK \cap \cF^{\summing}_{\textup{opt}}}
    \geq
    1-\delta .$
Therefore, it suffices to prove the regret bound on the event
$\cJ \cap \cK \cap \cF^{\summing}_{\textup{opt}}$.

Assume that event $\cJ \cap \cK \cap \cF^{\summing}_{\textup{opt}}$ holds.
Then, we have
\begin{align}
    \cR(K)
    &=
    \sum_{k=1}^{K}
    \sbr{
        (\phi^{\pi^*})^\top \theta^*
        -
        (\phi^{\pi^k})^\top \theta^*
    }
    \nonumber\\
    &\overset{\textup{(a)}}{\leq}
    \sum_{k=K_0+1}^{K}
    \sbr{
        (\phi^{\pi^*})^\top \hat{\theta}_{k-1}
        +
        \beta(k-1)
        \cdot
        \| \phi^{\pi^*} \|_{(\Sigma_{k-1})^{-1}}
        -
        (\phi^{\pi^k})^\top \theta^*
    }
    +
    2K_0 H r_{\max}
    \nonumber\\
    &\overset{\textup{(b)}}{\leq}
    \sum_{k=K_0+1}^{K}
    \sbr{
        (\phi^{\pi^k})^\top \hat{\theta}_{k-1}
        +
        \beta(k-1)
        \cdot
        \| \phi^{\pi^k} \|_{(\Sigma_{k-1})^{-1}}
        -
        (\phi^{\pi^k})^\top \theta^*
    }
    +
    2K_0 H r_{\max}
    \nonumber\\
    &\leq
    \sum_{k=K_0+1}^{K}
    2\beta(k-1)
    \cdot
    \| \phi^{\pi^k} \|_{(\Sigma_{k-1})^{-1}}
    +
    2K_0 H r_{\max}
    \nonumber\\
    &\leq
    2\beta(K)
    \sum_{k=K_0+1}^{K}
    \nbr{
        \ex_{\tau \sim \pi^k}
        \mbr{
            \phi^{\tau}
            \mid F_{k-1}
        }
    }_{(\Sigma_{k-1})^{-1}}
    +
    2K_0 H r_{\max}
    \nonumber\\
    &\overset{\textup{(c)}}{\leq}
    2\beta(K)
    \sum_{k=K_0+1}^{K}
    \ex_{\tau \sim \pi^k}
    \mbr{
        \nbr{ \phi^{\tau}}_{(\Sigma_{k-1})^{-1}}
        \mid F_{k-1}
    }
    +
    2K_0 H r_{\max}
    \nonumber\\
    &\leq
    2\beta(K)
    \!\!\!\!\!
    \sum_{k=K_0+1}^{K}
    \!\!\!\!\!
    \sbr{
        \ex_{\tau \sim \pi^k}
        \mbr{
            \nbr{ \phi^{\tau}}_{(\Sigma_{k-1})^{-1}}
            \mid F_{k-1}
        }
        -
        \nbr{ \phi^{\tau^{k}}}_{(\Sigma_{k-1})^{-1}}
        +
        \nbr{ \phi^{\tau^{k}}}_{(\Sigma_{k-1})^{-1}}
    }
    +
    2K_0 H r_{\max}
    \nonumber\\
    &\leq
    2\beta(K)
    \!\!\!\!
    \sum_{k=K_0+1}^{K}
    \!\!\!
    \sbr{
        \ex_{\tau \sim \pi^k}
        \mbr{
            \nbr{ \phi^{\tau}}_{(\Sigma_{k-1})^{-1}}
            \mid F_{k-1}
        }
        -
        \nbr{ \phi^{\tau^{k}}}_{(\Sigma_{k-1})^{-1}}
        +
        \sum_{i=1}^{m}
        \nbr{ \phi^{\tau_i^k}}_{(\Sigma_{k-1})^{-1}}
    }
    \nonumber\\
    &\quad
    +
    2K_0 H  r_{\max}.
    \label{eq:regret_decomposition_seg_d}
\end{align}
Here, inequality (a) follows from Eq.~\eqref{eq:concentration_theta_seg_d}, inequality (b) follows from the definition of $\pi^k$, and inequality (c) follows from Jensen's inequality.

Plugging Eq.~\eqref{eq:def_martingale_d} and Lemma~\ref{lemma:ellip_potent_with_opt_init_seg_d} into Eq.~\eqref{eq:regret_decomposition_seg_d}, and using $\lambda:=\frac{H}{r_{\max}^2 m}$, we obtain
\begin{align*}
    \cR(K)
    &\leq
    2
    \sbr{
        \sqrt{
            \frac{Hd}{m} \left(
            \log\sbr{
                1+\frac{KH^2}{\lambda d m}
            }
            +
            2\log\sbr{\frac{1}{\delta'}} \right)
        }
        +
        r_{\max}\sqrt{\lambda d}
    }
    \cdot
    \\
    &\quad
    \sbr{
        4\sqrt{
            K
            \log\sbr{\frac{4K}{\delta'}}
        }
        +
        \sqrt{
            2Km d
            \log \sbr{
                1+\frac{KH^2}{\lambda dm}
            }
        }
    }
    \\
    &\quad
    +
    2Hr_{\max}
    \left\lceil
        \max
        \lbr{
            26H^4(1+\gamma)^2(z^*)^2
            \log\sbr{\frac{2d}{\delta'}},
            \frac{d}{\gamma^2}
        }
    \right\rceil
    \\
    &=
    O\!\sbr{
        d\sqrt{HK}
        \log\sbr{
            \sbr{
                1+\frac{KH r_{\max}^2}{d}
            }
            \frac{1}{\delta}
        }
        +
        (z^*)^2 H^5
        \log\sbr{\frac{d}{\delta}}r_{\max}
        +
        dH
    } .
\end{align*}
This completes the proof.
\end{proof}

\subsection{Proof for the Regret Lower Bound }

 In this section, we prove the regret lower bound with known transitions under fixed segment sum feedback.

\begin{proof}[Proof of Theorem~\ref{thm:regret_sum_fixed_segment_d_known_lower}]
    Consider a linear MDP instance : $\cS=\{s_1,s_2,\dots,s_N\}$, $\cA=\{-1,1\}^{d-1}$, linear map $\psi(s_j,a)=(\sqrt{\frac{1}{d}} a^\top,\sqrt{\frac{1}{d}})^{\top}$, where $a\in\{-1,1\}^{d-1}$. For any $h\in [H]$, $\mu_h^*(s_j)=(\mathbf{0}_{d-1}^\top,\frac{\sqrt{d}}{N})^{\top}$.
    Under this construction, 
    $P_h(s_{j'}\mid s_j,a)=\langle \psi(s_j,a),\mu_h^*(s_{j'})\rangle
    =\frac{1}{N}$ for all $h\in[H]$ and $a\in\mathcal A$.

    We construct reward parameter $\theta^*_\gamma = (\frac{\sqrt{d}}{\sqrt{KH}}\gamma^\top,0)^\top$, where $\gamma\in \{-1,1\}^{d-1}$. Moreover, For any $(s,a)\in \mathcal S\times \mathcal A$, the reward distribution of $(s,a)$ is $ 
\mathcal N(r_{\gamma}(s,a),1)$. Then, verify that the above construction indeed satisfies the Assumption~\ref{assump:linear_mdp}.

    First,
for any $s_j\in \mathcal S$ and $a\in\mathcal A$, we have
\[
\|\psi(s_j,a)\|^2
=
\frac{1}{d}\|a\|^2 +\frac{1}{d}
=
\frac{d-1}{d}+\frac{1}{d}
=
1,
\]
where the second equality follows from $a\in\{-1,1\}^{d-1}$.

Moreover,
\[
\mu_h^*(\mathcal S)
=
\sum_{j=1}^N \mu_h^*(s_j)
=
(\mathbf{0}_{d-1}^\top,\sqrt{d})^{\top},
\]
and therefore
\[
\|\mu_h^*(\mathcal S)\|
=
\sqrt d.
\]

Finally, the reward is also linear in $\psi$. Indeed, for any $s_j\in \mathcal S$, 
\[
r_\gamma(s_j,a)
=
\bigl\langle \psi(s_j,a),\theta_\gamma^*\bigr\rangle
=
\left\langle
\sqrt{\frac{1}{d}}a,
\frac{\sqrt d}{\sqrt{KH}}\gamma
\right\rangle
=
\frac{1}{\sqrt{KH}}\langle a,\gamma\rangle .
\]
Furthermore,
\[
\|\theta_\gamma^*\|
=
\frac{\sqrt d}{\sqrt{KH}}\|\gamma\|
=
\frac{\sqrt d}{\sqrt{KH}}\sqrt{d-1}.
\]
When $K\ge \frac{(d-1)^2}{H r_{\max}^2}$, we have $|r_\gamma(s_j,a)| \le r_{\max}$ and $\|\theta_\gamma^*\| \le r_{\max}\sqrt{d}$. Thus, this linear MDP instane satisfies the Assumption~\ref{assump:linear_mdp}.

    If we set the segment length to be $\frac{H}{m}$, then for any algorithm $\mathbb A$, after running $K$ episodes, we obtain a total of $Km$ segments. In the $t$-th segment, the algorithm selects actions $a_{t,1},a_{t,2},\ldots,a_{t,\frac{H}{m}}\in\{-1,1\}^{d-1}.$
    
    We define the average action selected by the algorithm in the $t$-th segment as
    \[
    \bar A_t := \frac{1}{H/m}\sum_{\ell=1}^{H/m} a_{t,\ell}, \]
     where $\bar A_t\in[-1,1]^{d-1}$. In addition, conditional on selected actions, the reward of $t$-th segment $R^{\gamma}_t = \sum_{\ell=1}^{H/m} \left(r_{\gamma}(s_{t,\ell},a_{t,\ell})+\varepsilon_{t,\ell}\right)\sim \cN(\frac{H}{m}\cdot\frac{1}{\sqrt{KH}}\langle \bar A_t ,\gamma\rangle,\frac{H}{m})$. We can view this as one pull of a linear bandit by the algorithm in the $t$-th segment.

    Recall that $\mathbb D_\gamma$ denotes the probability measure induced by algorithm $\mathbb A$
under the instance parameterized by $\gamma$. By Theorem 24.1 in (Lattimore and Szepesv{\'a}r,~\citeyear{lattimore2020bandit}), we have for any $\gamma,\gamma'\in\{-1,1\}^{d-1}$ that
    \begin{align}
        \operatorname{KL}(\mathbb D_\gamma,\mathbb D_{\gamma'})
&=
\mathbb E_\gamma
\left[
\sum_{t=1}^{Km}
\operatorname{KL}\left(
\mathcal N\left(
\frac{H}{m}\cdot \frac{1}{\sqrt{KH}}
\langle \bar A_t,\gamma\rangle,
\frac{H}{m}
\right),
\mathcal N\left(
\frac{H}{m}\cdot \frac{1}{\sqrt{KH}}
\langle \bar A_t,\gamma'\rangle,
\frac{H}{m}
\right)
\right)
\right] \notag \\
&=
\mathbb E_\gamma
\left[
\sum_{t=1}^{Km}
\frac{
\left(
\frac{H}{m}\cdot \frac{1}{\sqrt{KH}}
\langle \bar A_t,\gamma-\gamma'\rangle
\right)^2
}{
2\cdot \frac{H}{m}
}
\right] \notag \\
&=
\frac{1}{2Km}
\sum_{t=1}^{Km}
\mathbb E_\gamma
\left[
\langle \bar A_t,\gamma-\gamma'\rangle^2
\right].\label{eq:hyer_KL}
    \end{align}
For $i \in [d-1]$ and $\gamma \in \{-1,1\}^{d-1}$,we define 
\begin{align*}
    p_{\gamma,i} = \mathbb D_\gamma \left(\sum_{t=1}^{Km}\indicator\{\operatorname{sign}(\bar A_{t,i})\neq\operatorname{sign}(\gamma_i) \}\ge \frac{Km}{2}\right)
\end{align*}
Now fixing $i \in [d-1]$ and $\gamma \in \{-1,1\}^{d-1}$, let $\gamma'_j = \gamma_j$ for $j\neq i$ and $\gamma'_i = -\gamma_i$. Using the Bretagnolle–Huber inequality and Eq.~\eqref{eq:hyer_KL}, we have 
\begin{align*}
    p_{\gamma,i} +p_{\gamma',i} \ge \frac{1}{2}\exp\left(-\frac{1}{2Km}
\sum_{t=1}^{Km}
\mathbb E_\gamma
\left[
\langle \bar A_t,\gamma-\gamma'\rangle^2
\right])\right) \ge \frac{1}{2}\exp\left(-2\right)
\end{align*}
Then taking sum over $i \in [d-1]$ and taking average over $\gamma \in \{-1,1\}^{d-1}$, we obtain
\begin{align*}
    \sum_{\gamma\in\{-1,1\}^{d-1}}\frac{1}{2^{d-1}}\sum_{i=1}^{d-1} p_{\gamma,i}\ge \frac{d-1}{4}\exp\left(-2\right)
\end{align*}
This show that there exists a $\gamma \in \{-1,1\}^{d-1}$ such that $\sum_{i=1}^{d-1} p_{\gamma,i}\ge \frac{d-1}{4}\exp\left(-2\right)$. By the definition of $p_{\gamma,i}$ and our construction, the regret for this $\gamma$ is 
\begin{align*}
    \mathbb{E}\left[\cR(K)\right] &= \mathbb{E} \left[\sum_{k=1}^K\left(V^* - V^{\pi_k} \right) \right]
    \\& = \E_{\gamma}\left[\sum_{k=1}^K\sum_{h=1}^H\sum_{i=1}^{d-1}\frac{1}{\sqrt{KH}}\left(\operatorname{sign}(\gamma_i)-a_{h,i}^k\right)\gamma_i\right]
    \\&
    =\frac{H}{m}\cdot\frac{1}{\sqrt{KH}}\sum_{i=1}^{d-1}\E_{\gamma}\left[\sum_{t=1}^{Km}\left(\operatorname{sign}(\gamma_i)-\bar A_{t,i}\right)\gamma_i\right]
    \\&\geq\frac{H}{m}\cdot\frac{1}{\sqrt{KH}}\sum_{i=1}^{d-1}\E_{\gamma}\left[\sum_{t=1}^{Km}\indicator\{\operatorname{sign}(\bar A_{t,i})\neq\operatorname{sign}(\gamma_i) \}\right]
    \\&\geq\frac{H}{m}\cdot\frac{1}{\sqrt{KH}} \cdot \frac{Km}{2} \sum_{i=1}^{d-1}\mathbb D_\gamma \left(\sum_{t=1}^{Km}\indicator\{\operatorname{sign}(\bar A_{t,i})\neq\operatorname{sign}(\gamma_i) \}\ge \frac{Km}{2}\right)
    \\ &=\frac{H}{m}\cdot\frac{1}{\sqrt{KH}} \cdot \frac{Km}{2}\sum_{i=1}^{d-1}p_{\gamma,i} 
    \\& \ge \frac{\exp\left(-2\right)}{8} \left(d-1\right) \sqrt{KH} .
\end{align*}
Therefore, we can construct a linear MDP instance satisfying the linear MDP assumption, whose state space consists of $N$ states and whose action space has cardinality \(2^{d-1}\).

For this instance, when $K \ge \frac{(d-1)^2}{Hr_{\max}^2}$, we have
\begin{align*}
    \mathbb{E}\left[\cR(K)\right] = \Omega\left(d\sqrt{HK}\right).
\end{align*}

which completes the proof    
\end{proof}

\section{Proofs for Fixed Equal-Length Segment Feedback with Unknown Transitions}\label{secD}

\noindent\textbf{Notation:} To avoid an overly lengthy presentation, we denote $\psi_h^k=\psi(s^k_h,a_h^k)$ and $P_hV(s,a)=\mathbb E_{s'\sim P_h(\cdot\mid s,a)}\left[V(s')\right]$.

\subsection{Proof for the Regret Upper Bound of Binary Feedback}\label{App:binary_unknown}
In this section, we provide proofs of regret upper bound of binary feedback with unknown transitions.

The proof proceeds by constructing a collection of auxiliary MDPs using the sampled reward parameters. We first show that the value function estimates produced by the algorithm are optimistic with respect to these auxiliary MDPs. A Thompson-sampling argument then transfers this optimism to the original MDP by ensuring that, with high probability, at least one sampled reward parameter is optimistic along the optimal policy. Finally, we decompose the regret into one-step estimation errors accumulated along the realized trajectories and control the reward and transition errors using their respective covariance matrices, which yields the desired regret upper bound.

Define the event
\begin{align}\label{eq:union_bound_Gaussian_noise}
    \mathcal{E}_{\mathrm{TS}}
    :=
    \left\{
        \forall k\in[K],\,j\in[n]:
        \nbr{g_k^{(j)}}_2
        \le
        \sqrt{d}
        +
        2\sqrt{
            \log\left(\frac{Kn}{\delta'}\right)
        }
    \right\},
\end{align}
where $g_k^{(j)}\sim\mathcal{N}(0,I_d)$.

\begin{lemma}[Universal bound of TS noise]
\label{lemma:universal_bound_of_TS_noise}
It holds that 
\begin{align*}
    \Pr[\mathcal{E}_{\mathrm{TS}}] \ge 1-\delta' .
\end{align*}
\end{lemma}

\begin{proof}
    For any fixed $k\in[K]$ and $j\in[n]$, let
    $g_k^{(j)}\sim \mathcal N(0,I_d)$ be a standard Gaussian vector.
    Then $\nbr{g_k^{(j)}}_2$ follows a chi-distribution with parameter $d$.
    By Lemma~1 in (Laurent and
Massart, \citeyear{laurent2000adaptive}), we have
    \begin{align*}
       \Pr\left(
       \nbr{g_k^{(j)}}_2
       \le
       \sqrt{d}
       +
       2\sqrt{\log{\frac{Kn}{\delta'}}}
       \right)
       \ge
       1-\frac{\delta'}{Kn}.
    \end{align*}
    Taking a union bound over all $(k,j)\in[K]\times[n]$, we obtain
    \begin{align*}
       \Pr\left(
       \forall k\in[K], j\in[n]:
       \nbr{g_k^{(j)}}_2
       \le
       \sqrt{d}
       +
       2\sqrt{\log{\frac{Kn}{\delta'}}}
       \right)
       \ge
       1-\delta' .
    \end{align*}
    Therefore, by the definition of $\mathcal{E}_{\mathrm{TS}}$, it follows that
    \begin{align*}
        \Pr[\mathcal{E}_{\mathrm{TS}}] \ge 1-\delta' .
    \end{align*}
    
    This completes the proof.
\end{proof}

In Algorithm~$\biseglsvits$, the sampled reward parameter is given by
\[
\widetilde{\theta}_k^{(j)}
=
\hat{\theta}_{k-1}^{\mathrm{bi}}
+
\rho_{k-1}^{\mathrm{bi}}
\bigl(\Sigma_{k-1}^{\mathrm{bi}}\bigr)^{-1/2}
g_k^{(j)}.
\]

Consider a auxiliary linear MDP $\cM^{(j)}_k$ whose reward function is
   $ r_k^{(j)}(s,a)
    =
    \langle \psi(s,a), \widetilde{\theta}_k^{(j)} \rangle$ while transition kernel is the same as that of the original MDP.
   
   Conditional on $\widetilde{\theta}_k^{(j)}$, for any policy $\pi$, define
\begin{align*}
    Q_{h}^{\pi;(k,j)}(s,a)&:= \mathbb{E}[\sum_{t=h}^{H} r_k^{(j)}(s_t,a_t) \mid s_h=s, a_h=a, \pi, \widetilde{\theta}_k^{(j)} ] ,
\\
    V_h^{\pi;(k,j)}(s)
    &:=
    \mathbb{E}
    [
        \sum_{t=h}^{H}
        r_k^{(j)}(s_t,a_t)
        \mid
        s_h=s,\ \pi,\ \widetilde{\theta}_k^{(j)}
    ].
\end{align*}

   Since $\widetilde{\theta}_k^{(j)}$ is random, the corresponding state-action value function and state value function are  random functions. Below, we provide the linearity of state action value function on Original and auxiliary linear MDP.

\begin{proposition}[Proposition 2.3 in Jin et al. \citeyear{jin20_linearmdp}] \label{proposition:linear_Q_prime}
For a linear MDP, for any policy $\pi$, there exist weights $\{\mathbf{w}_h^\pi\}_{h \in [H]}$ such that for any $(s, a, h) \in \mathcal{S} \times \mathcal{A} \times [H]$, we have $Q_h^\pi(s, a) = \langle \psi(s, a), \mathbf{w}_h^\pi \rangle$.
\end{proposition}

\begin{proposition}\label{prop:linearity_Q}
For random reward parameter  $\widetilde{\theta}_k^{(j)}$, for any policy $\pi$,
there exist random weights
$\{\mathbf{w}_{h}^{\pi;(k,j)}\}_{h\in[H]}$ such that for any
$(s,a,h)\in\mathcal S\times\mathcal A\times[H]$,
\[
    Q_{h}^{\pi;(k,j)}(s,a)
    =
    \langle \psi(s,a), \mathbf{w}_{h}^{\pi;(k,j)} \rangle .
\]
\end{proposition}
\begin{proof}
   By the Bellman equation in the MDP $\mathcal M_k^{(j)}$ and Assumption~\ref{assump:linear_mdp}, we have \begin{align*} Q_{h}^{\pi;(k,j)}(s,a) &= r_k^{(j)}(s,a) + (P_h V_{h+1}^{\pi;(k,j)})(s,a) \\ &= \langle \psi(s,a), \widetilde{\theta}_k^{(j)} \rangle + \left\langle \psi(s,a), \int_{\mathcal{S}} V_{h+1}^{\pi;(k,j)}(s') d\mu_h^*(s') \right\rangle . \end{align*} Therefore, \[ Q_{h}^{\pi;(k,j)}(s,a) = \langle \psi(s,a), \mathbf{w}_{h}^{\pi;(k,j)} \rangle , \] where the random weight vector $\mathbf{w}_{h}^{\pi;(k,j)}$ is defined as $\mathbf{w}_{h}^{\pi;(k,j)} = \widetilde{\theta}_k^{(j)} + \int_{\mathcal{S}} V_{h+1}^{\pi;(k,j)}(s') d\mu_h^*(s') . $

This completes the proof.
\end{proof}

Then, we provide the univerisal bound of sampled reward parameter. Recall the event about concerntration of logistic estimator

\begin{align*}
	\cE=\lbr{ \nbr{g_k(\hat{\theta}_{k}^{\mathrm{bi}}) - g_k(\theta^*) }_{\Lambda_k^{-1}(\theta^*)} \leq \omega(k), \  \forall k>0 } ,
\end{align*}
and

\[
B_{\mathrm{bi}}
:=
\frac{\nu(K)}{\sqrt{\lambda}}
\left(
\sqrt d
+
2\sqrt{
\log\left(\frac{Kn}{\delta'}\right)
}
+
1
\right)
+
r_{\max}\sqrt d.
\]
\begin{lemma}[Universal Bound of $\widetilde{\theta}_k^{(j)}$ under Binary Feedback]
\label{lemma:Bi_universal_bound_of_TS_theta}
	Assume that event $\mathcal{E}_{\mathrm{TS}}\cap\mathcal E$ holds. Then, we have
	\begin{align*}
		\forall k\in[K],j\in[n],\qquad
		\nbr{\widetilde{\theta}_k^{(j)}}_2
		\le
		\frac{
		\nu(K)
		}{
		\sqrt{\lambda}
		}
		\left(
		\sqrt d
		+
		2\sqrt{\log\left(\frac{Kn}{\delta'}\right)}
		+
		1
		\right)
		+
		r_{\max}\sqrt d
		= B_{\mathrm{bi}}.
	\end{align*}
\end{lemma}

\begin{proof}
	By the sampling rule,
	\[
	    \xi_k^{(j)}
	    =
	    \rho_{k-1}^{\mathrm{bi}}
	    \left(\Sigma_{k-1}^{\mathrm{bi}}\right)^{-1/2}
	    g_k^{(j)}.
	\]
	Since
	$\Sigma_{k-1}^{\mathrm{bi}}\succeq\alpha\lambda I$ and
	$\rho_{k-1}^{\mathrm{bi}}=\sqrt{\alpha}\nu(k-1)$, on event
	$\mathcal{E}_{\mathrm{TS}}\cap\mathcal E$, we have
	\begin{align*}
	    \left\|
	    \xi_k^{(j)}
	    \right\|_2
	    &\le
	    \rho_{k-1}^{\mathrm{bi}}
	    \left\|
	    \left(\Sigma_{k-1}^{\mathrm{bi}}\right)^{-1/2}
	    \right\|_{\mathrm{op}}
	    \left\|
	    g_k^{(j)}
	    \right\|_2
	    \\
	    &\le
	    \frac{
	    \nu(K)
	    }{
	    \sqrt{\lambda}
	    }
	    \left(
	    \sqrt d
	    +
	    2\sqrt{\log\left(\frac{Kn}{\delta'}\right)}
	    \right),
	    \qquad
	    \forall k\in[K],j\in[n].
	\end{align*}
	For any $(k,j)\in[K]\times[n]$, we have
	\begin{align*}
	    \left\|
	    \widetilde{\theta}_k^{(j)}
	    \right\|_2
	    &=
	    \left\|
	    \widetilde{\theta}_k^{(j)}
	    -
	    \theta^*
	    +
	    \theta^*
	    \right\|_2
	    \\
	    &\le
	    \left\|
	    \xi_k^{(j)}
	    \right\|_2
	    +
	    \left\|
	    \hat{\theta}_{k-1}^{\mathrm{bi}}
	    -
	    \theta^*
	    \right\|_2
	    +
	    \left\|
	    \theta^*
	    \right\|_2
	    \\
	    &\le
	    \frac{
	    \nu(K)
	    }{
	    \sqrt{\lambda}
	    }
	    \left(
	    \sqrt d
	    +
	    2\sqrt{\log\left(\frac{Kn}{\delta'}\right)}
	    \right)
	    +
	    \frac{\nu(K)}{\sqrt{\lambda}}
	    +
	    r_{\max}\sqrt d
	    \\
	    &\le
	    \frac{
	    \nu(K)
	    }{
	    \sqrt{\lambda}
	    }
	    \left(
	    \sqrt d
	    +
	    2\sqrt{\log\left(\frac{Kn}{\delta'}\right)}
	    +
	    1
	    \right)
	    +
	    r_{\max}\sqrt d
	    =
	    B_{\mathrm{bi}},
	\end{align*}
	where the second inequality follows from the definition of the event
	$\mathcal{E}_{\mathrm{TS}}$, Lemma~\ref{lemma:lin_confence_interval_proj_free},
	and Assumption~\ref{assump:linear_mdp}.

	This completes the proof.
\end{proof}

Next, we establish uniform bounds on the weight vectors associated with the value functions of the original MDP, the auxiliary MDPs, and those constructed by the algorithm.

\begin{lemma}[Bound on Weights of Value Functions of the Original MDP]\label{lemma:bound_weights_primal_MDP}
    Under Assumption~\ref{assump:linear_mdp}, for any policy $\pi$, let 
    $\{\mathbf{w}_h^{\pi}\}_{h\in[H]}$ be the corresponding weight vectors satisfying
    \begin{align*}
        Q_h^{\pi}(s,a)
        =
        \left\langle
            \psi(s,a),
            \mathbf{w}_h^{\pi}
        \right\rangle,
        \qquad
        \forall (s,a,h)\in\mathcal{S}\times\mathcal{A}\times[H].
    \end{align*}
    Then, for all $h\in[H]$, we have
    \begin{align*}
        \left\|
        \mathbf{w}_h^{\pi}
        \right\|_2
        \le
        2Hr_{\max}\sqrt{d}.
    \end{align*}
\end{lemma}

\begin{proof}
    This proof follows the analysis of Lemma B.1 in (Jin et al., \citeyear{jin20_linearmdp}).
    
    By Proposition~\ref{proposition:linear_Q_prime}, we have
    \begin{align*}
        \mathbf{w}_h^{\pi}
        =
        \theta^*
        +
        \int_{\cS}
        V_{h+1}^{\pi}(s')
        d\mu^*_h(s') .
    \end{align*}
    Since the reward at each step lies in $[-r_{\max},r_{\max}]$, for any state $s'$, the state value function satisfies
    \begin{align*}
        |V_{h+1}^{\pi}(s')|
        \le
        Hr_{\max} .
    \end{align*}
    Therefore,
    \begin{align*}
               \left\|
        \int_{\cS}
        V_{h+1}^{\pi}(s')
        d\mu^*_h(s')
        \right\|_2
        \le
        Hr_{\max}\sqrt{d}.
    \end{align*}
    Hence,
    \begin{align*}
        \left\|
        \mathbf{w}_h^{\pi}
        \right\|_2
        &=
        \left\|
        \theta^*
        +
        \int_{\cS}
        V_{h+1}^{\pi}(s')
        d\mu^*_h(s')
        \right\|_2
        \\
        &\le
        \left\|
        \theta^*
        \right\|_2
        +
        \left\|
        \int_{\cS}
        V_{h+1}^{\pi}(s')
        d\mu^*_h(s')
        \right\|_2
        \\
        &\le
        r_{\max}\sqrt{d}
        +
        Hr_{\max}\sqrt{d}
        \\
        &\le
        2Hr_{\max}\sqrt{d}.
    \end{align*}

    This completes the proof.
\end{proof}

\begin{lemma}[Bound on Weights of Value Functions of MDP $\mathcal M_k^{(j)}$ under Binary Feedback]
\label{lemma:Bi_bound_Vweights_mdp_j^k}
	Assume that event $\mathcal E\cap \mathcal E_{\mathrm{TS}}$ holds.
	Under Assumption~\ref{assump:linear_mdp}, for any policy $\pi$, let
	$\{\mathbf{w}_{h}^{\pi;(k,j)}\}_{h\in[H]}$ be the corresponding random
	weight vectors satisfying
	\begin{align*}
	    Q_{h}^{\pi;(k,j)}(s,a)
	    =
	    \left\langle
	    \psi(s,a),
	    \mathbf{w}_{h}^{\pi;(k,j)}
	    \right\rangle,
	    \qquad
	    \forall (s,a,h,k,j)
	    \in
	    \mathcal S\times\mathcal A\times[H]\times[K]\times[n].
	\end{align*}
	Then, for all $(h,k,j)\in[H]\times[K]\times[n]$, we have
	\begin{align*}
	    \left\|
	    \mathbf{w}_{h}^{\pi;(k,j)}
	    \right\|_2
	    \le
	    2HB_{\mathrm{bi}}\sqrt d.
	\end{align*}
\end{lemma}

\begin{proof}
	By Proposition~\ref{prop:linearity_Q}, for any
	$(h,k,j)\in[H]\times[K]\times[n]$, we have
	\begin{align*}
	    \mathbf{w}_{h}^{\pi;(k,j)}
	    =
	    \widetilde{\theta}_k^{(j)}
	    +
	    \int_{\mathcal S}
	    V_{h+1}^{\pi;(k,j)}(s')
	    d\mu_h^*(s').
	\end{align*}
	On event $\mathcal E\cap \mathcal E_{\mathrm{TS}}$, by
	Lemma~\ref{lemma:Bi_universal_bound_of_TS_theta} and
	Assumption~\ref{assump:linear_mdp}, we have
	\begin{align*}
	    \left|
	    \left\langle
	    \psi(s,a),
	    \widetilde{\theta}_k^{(j)}
	    \right\rangle
	    \right|
	    \le
	    \left\|
	    \widetilde{\theta}_k^{(j)}
	    \right\|_2
	    \left\|
	    \psi(s,a)
	    \right\|_2
	    \le
	    B_{\mathrm{bi}}.
	\end{align*}
	Thus, the reward of MDP $\mathcal M_k^{(j)}$ lies in
	$[-B_{\mathrm{bi}},B_{\mathrm{bi}}]$. Hence, for any state $s'$, the state value function
	satisfies
	\begin{align*}
	    \left|
	    V_{h+1}^{\pi;(k,j)}(s')
	    \right|
	    \le
	    HB_{\mathrm{bi}}.
	\end{align*}
	Therefore, by Assumption~\ref{assump:linear_mdp},
	\begin{align*}
	    \left\|
	    \int_{\mathcal S}
	    V_{h+1}^{\pi;(k,j)}(s')
	    d\mu_h^*(s')
	    \right\|_2
	    \le
	    HB_{\mathrm{bi}}\sqrt d.
	\end{align*}
	Hence,
	\begin{align*}
	    \left\|
	    \mathbf{w}_{h}^{\pi;(k,j)}
	    \right\|_2
	    &=
	    \left\|
	    \widetilde{\theta}_k^{(j)}
	    +
	    \int_{\mathcal S}
	    V_{h+1}^{\pi;(k,j)}(s')
	    d\mu_h^*(s')
	    \right\|_2
	    \\
	    &\le
	    \left\|
	    \widetilde{\theta}_k^{(j)}
	    \right\|_2
	    +
	    \left\|
	    \int_{\mathcal S}
	    V_{h+1}^{\pi;(k,j)}(s')
	    d\mu_h^*(s')
	    \right\|_2
	    \\
	    &\le
	    B_{\mathrm{bi}}
	    +
	    HB_{\mathrm{bi}}\sqrt d
	    \\
	    &\le
	    2HB_{\mathrm{bi}}\sqrt d.
	\end{align*}

	This completes the proof.  
\end{proof}

\begin{lemma}
\label{lemma:Bi_bound_alg_weight_MDP_j^k}
	Assume that event $\mathcal{E}_{\mathrm{TS}}\cap\mathcal E$ holds.
	Then, for any $(k,h,j)\in[K]\times[H]\times[n]$, the weight vector
	$\mathbf{w}_h^{k,(j)}$ used to construct the estimated state-action
	value function $Q_h^{k,(j)}$ in Algorithm~\ref{alg:unified-lsvi-ts-q}
	satisfies
	\begin{align*}
	    \left\|
	    \mathbf{w}_h^{k,(j)}
	    \right\|_2
	    \le
	    2HB_{\mathrm{bi}}
	    \sqrt{
	    \frac{dk}{\lambda'}
	    }.
	\end{align*}
\end{lemma}

\begin{proof}
	This proof follows the analysis of Lemma B.2 in (Jin et al., \citeyear{jin20_linearmdp}).

	On event $\mathcal{E}_{\mathrm{TS}}\cap\mathcal E$, for any
	$(k,h,j)\in[K]\times[H]\times[n]$, the reward induced by
	$\widetilde{\theta}_k^{(j)}$ is bounded by $B_{\mathrm{bi}}$. Moreover, the estimated
	state-action value function $Q_h^{k,(j)}$ is bounded by $HB_{\mathrm{bi}}$.
	Hence, the regression target in Algorithm~\ref{alg:unified-lsvi-ts-q}
	is bounded in absolute value by $2HB_{\mathrm{bi}}$.

	For any vector $\mathbf{v}\in\mathbb{R}^d$, we have
	\begin{align*}
	    \left|
	    \mathbf{v}^{\top}
	    \mathbf{w}_h^{k,(j)}
	    \right|
	    &=
	    \left|
	    \mathbf{v}^{\top}
	    \bigl(\Lambda_h^{k-1}\bigr)^{-1}
	    \sum_{k'=1}^{k-1}
	    \psi(s_h^{k'},a_h^{k'})
	    \left[
	    \psi(s_h^{k'},a_h^{k'})^{\top}
	    \widetilde{\theta}_k^{(j)}
	    +
	    \max_{a\in\mathcal A}
	    Q_{h+1}^{k,(j)}(s_{h+1}^{k'},a)
	    \right]
	    \right|
	    \\
	    &\le
	    2HB_{\mathrm{bi}}
	    \sum_{k'=1}^{k-1}
	    \left|
	    \mathbf{v}^{\top}
	    \bigl(\Lambda_h^{k-1}\bigr)^{-1}
	    \psi(s_h^{k'},a_h^{k'})
	    \right|
	    \\
	    &\le
	    2HB_{\mathrm{bi}}
	    \sqrt{
	    \left[
	    \sum_{k'=1}^{k-1}
	    \mathbf{v}^{\top}
	    \bigl(\Lambda_h^{k-1}\bigr)^{-1}
	    \mathbf{v}
	    \right]
	    \left[
	    \sum_{k'=1}^{k-1}
	    \psi(s_h^{k'},a_h^{k'})^{\top}
	    \bigl(\Lambda_h^{k-1}\bigr)^{-1}
	    \psi(s_h^{k'},a_h^{k'})
	    \right]
	    }
	    \\
	    &\le
	    2HB_{\mathrm{bi}}
	    \|\mathbf{v}\|_2
	    \sqrt{
	    \frac{dk}{\lambda'}
	    }.
	\end{align*}
	The last inequality follows from
	$\Lambda_h^{k-1}\succeq \lambda'I$ and
	\begin{align*}
	    \sum_{k'=1}^{k-1}
	    \psi(s_h^{k'},a_h^{k'})^{\top}
	    \bigl(\Lambda_h^{k-1}\bigr)^{-1}
	    \psi(s_h^{k'},a_h^{k'})
	    \le
	    d.
	\end{align*}
	Therefore, by the dual characterization of the Euclidean norm,
	\begin{align*}
	    \left\|
	    \mathbf{w}_h^{k,(j)}
	    \right\|_2
	    =
	    \max_{\mathbf{v}:\|\mathbf{v}\|_2=1}
	    \left|
	    \mathbf{v}^{\top}
	    \mathbf{w}_h^{k,(j)}
	    \right|
	    \le
	    2HB_{\mathrm{bi}}
	    \sqrt{
	    \frac{dk}{\lambda'}
	    }.
	\end{align*}

	This completes the proof.
\end{proof}

We prove a self-normalized concentration inequality for the value function estimation error under binary
feedback. Fix a constant $c'_{\beta}>0$. For an absolute constant $C_{LS}^{Bi}>0$ independent of
$c'_{\beta}$.
Define the event
\begin{align}
\mathcal E^{\mathrm{Bi}}_{\mathrm{LS}}
:=
\Bigg\{
&
\left\|
\sum_{k'=1}^{k-1}
\psi_h^{k'}
\left[
    V_{h+1}^{k,(j)}(s_{h+1}^{k'})
    -
    P_h V_{h+1}^{k,(j)}(s_h^{k'},a_h^{k'})
\right]
\right\|_{(\Lambda_h^{k-1})^{-1}}
\nonumber
\\
&\qquad\qquad
\le
C_{{LS}}^{Bi} d H B_{\mathrm{bi}} \sqrt{\log
    \left(
        \frac{2(c'_{\beta}+1)dKH}{\delta'}
    \right)},
\quad
\forall (k,h,j)\in[K]\times[H]\times[n]
\Bigg\}.
\label{eq:Bi_event_LS}
\end{align}

\begin{lemma}
\label{lemma:self_normalized_V_concentration}
Let $c'_{\beta}$ be the constant in the definition of
$\beta_{\mathrm{bi}}$, and let $\lambda'=1$. There exists an absolute
constant $C_{{LS}}^{Bi}>0$, independent of $c'_{\beta}$, such that
the event $\mathcal E^{\mathrm{Bi}}_{\mathrm{LS}}$ defined above satisfies
\begin{align*}
    \Pr
    \left[
        \cE
        \cap
        \mathcal E_{\mathrm{TS}}
        \cap
        \left(
        \mathcal E^{\mathrm{Bi}}_{\mathrm{LS}}
        \right)^c
    \right]
    \le
    \delta'.
\end{align*}
\end{lemma}

\begin{proof}
This proof follows the analysis of Lemma B.3 in
(Jin et al., \citeyear{jin20_linearmdp}).

On event $\mathcal E_{\mathrm{TS}}\cap\cE$, for all
$(k,h,j)\in[K]\times[H]\times[n]$,
Lemma~\ref{lemma:Bi_bound_alg_weight_MDP_j^k} gives
\begin{align*}
    \left\|
    \mathbf w_h^{k,(j)}
    \right\|_2
    \le
    2HB_{\mathrm{bi}}
    \sqrt{
    \frac{dK}{\lambda'}
    }.
\end{align*}
In addition, by the definition of $\Lambda_h^{k-1}$, we have
\[
    \lambda_{\min}
    \left(
    \Lambda_h^{k-1}
    \right)
    \ge
    \lambda'.
\]

Hence, on event $\mathcal E_{\mathrm{TS}}\cap\cE$, for all
$(k,h,j)\in[K]\times[H]\times[n]$, the function
$V_{h+1}^{k,(j)}$ belongs to the function class characterized in
Lemma~\ref{lemma:cover_num_LSVI_V}.

Combining Lemmas~\ref{lemma:uniform_self_normalized} and
\ref{lemma:cover_num_LSVI_V}, we obtain that, with probability at least
$1-\delta'$, the following implication holds: when
$\mathcal E_{\mathrm{TS}}\cap\cE$ occurs, for any fixed
$\epsilon>0$ and all $(k,h,j)\in[K]\times[H]\times[n]$,
\begin{align}
&
\left\|
\sum_{k'=1}^{k-1}
\psi_h^{k'}
\left[
    V_{h+1}^{k,(j)}(s_{h+1}^{k'})
    -
    P_h V_{h+1}^{k,(j)}(s_h^{k'},a_h^{k'})
\right]
\right\|_{(\Lambda_h^{k-1})^{-1}}^2
\nonumber
\\
&\le
4H^2B_{\mathrm{bi}}^2
\Bigg[
    \frac d2
    \log
    \left(
        \frac{k+\lambda'}{\lambda'}
    \right)
    +
    d
    \log
    \left(
        1+
        \frac{8HB_{\mathrm{bi}}\sqrt{dK}}
        {\epsilon\sqrt{\lambda'}}
    \right)
    +
    d^2
    \log
    \left(
        1+
        \frac{8d^{1/2}\beta_{\mathrm{bi}}^2}
        {\epsilon^2\lambda'}
    \right)
    +
    \log
    \left(
        \frac{H}{\delta'}
    \right)
\Bigg]
+
\frac{8k^2\epsilon^2}{\lambda'} .
\label{eq:Bi_concentration_value_bound}
\end{align}
Now set $\lambda'=1$ and choose
\begin{align*}
    \beta_{\mathrm{bi}}
    &=
    c'_{\beta}dHB_{\mathrm{bi}}
    \sqrt{
    \log
    \left(
        \frac{2dKH}{\delta'}
    \right)
    },
    \qquad
    \epsilon
    =
    \frac{dHB_{\mathrm{bi}}}{K}.
\end{align*}
Plugging these choices into
Eq.~\eqref{eq:Bi_concentration_value_bound}, we conclude that there
exists an absolute constant $C'_0>0$, independent of $c'_{\beta}$,
such that
\begin{align*}
&
\left\|
\sum_{k'=1}^{k-1}
\psi_h^{k'}
\left[
    V_{h+1}^{k,(j)}(s_{h+1}^{k'})
    -
    P_h V_{h+1}^{k,(j)}(s_h^{k'},a_h^{k'})
\right]
\right\|_{(\Lambda_h^{k-1})^{-1}}^2
\le
C'_0 d^2H^2B_{\mathrm{bi}}^2\log
    \left(
        \frac{2(c'_{\beta}+1)dKH}{\delta'}
    \right) .
\end{align*}
Choose the constant in the definition of
$\mathcal E^{\mathrm{Bi}}_{\mathrm{LS}}$ such that
\[
    C_{{LS}}^{Bi}
    \ge
    \sqrt{C'_0}.
\]
Then
\begin{align*}
&
\left\|
\sum_{k'=1}^{k-1}
\psi_h^{k'}
\left[
    V_{h+1}^{k,(j)}(s_{h+1}^{k'})
    -
    P_h V_{h+1}^{k,(j)}(s_h^{k'},a_h^{k'})
\right]
\right\|_{(\Lambda_h^{k-1})^{-1}}
\le
C_{{LS}}^{Bi}dHB_{\mathrm{bi}}\sqrt{\log
    \left(
        \frac{2(c'_{\beta}+1)dKH}{\delta'}
    \right)}.
\end{align*}
By the definition of $\mathcal E^{\mathrm{Bi}}_{\mathrm{LS}}$, therefore,
\begin{align*}
    \Pr
    \left[
        \cE
        \cap
        \mathcal E_{\mathrm{TS}}
        \cap
        \left(
        \mathcal E^{\mathrm{Bi}}_{\mathrm{LS}}
        \right)^c
    \right]
    \le
    \delta'.
\end{align*}

This completes the proof.
\end{proof}

Next, we establish upper bounds on the value function estimation errors of Algorithm~\ref{alg:unified-lsvi-ts-q} in the original and auxiliary MDPs.

\begin{lemma}[Upper Bound for estimator error in Original MDP]
\label{lemma:Bi_upper_bound_LSVI-TS}
    Assume that event 
    $\cE \cap \mathcal E_{\mathrm{TS}}
    \cap \mathcal E^{\mathrm{Bi}}_{\mathrm{LS}}$ holds.
    Choose the absolute constant $c'_\beta$ in the definition of
    $\beta_{\mathrm{bi}}$ sufficiently large.
    Then, for any policy $\pi$, for all
    $(s,a,h,k,j)\in
    \mathcal{S}\times\mathcal{A}\times[H]\times[K]\times[n]$, we have
    \begin{align*}
       \left|
        \left\langle
            \psi(s,a),
            \mathbf{w}_h^{k,(j)}
        \right\rangle
        -
        Q_h^{\pi}(s,a)
        -
        P_h
        \left(
            V_{h+1}^{k,(j)}
            -
            V_{h+1}^{\pi}
        \right)(s,a)
        -
        \left\langle
            \psi(s,a),
            \widetilde{\theta}_k^{(j)}-\theta^*
        \right\rangle
        \right|
        \\
        \le
        \beta_{\mathrm{bi}}
        \sqrt{
            \bigl(\psi(s,a)\bigr)^\top
            \bigl(\Lambda_h^{k-1}\bigr)^{-1}
            \psi(s,a)
        } .
    \end{align*}
\end{lemma}

\begin{proof}
   By Proposition~\ref{proposition:linear_Q_prime} and the Bellman
    equation, for any
    $(s,a,h)\in\mathcal{S}\times\mathcal{A}\times[H]$, the state-action
    value function satisfies
    \begin{align*}
        Q_h^{\pi}(s,a)
        &:=
        \left\langle
            \psi(s,a),
            \mathbf{w}_h^{\pi}
        \right\rangle
        =
        \left(
            \langle\psi,\theta^*\rangle
            +
            P_h V_{h+1}^{\pi}
        \right)(s,a).
    \end{align*}
    Therefore,
    \begin{align*}
        &\mathbf{w}_h^{k,(j)}-\mathbf{w}_h^{\pi}
        \\
        &=
        \bigl(\Lambda_h^{k-1}\bigr)^{-1}
        \sum_{k'=1}^{k-1}
        \psi_h^{k'}
        \left[
            \langle
                \psi_h^{k'},
                \widetilde{\theta}_k^{(j)}
            \rangle
            +
            V_{h+1}^{k,(j)}(s_{h+1}^{k'})
        \right]
        -
        \mathbf{w}_h^{\pi}
        \\
        &=
        \bigl(\Lambda_h^{k-1}\bigr)^{-1}
        \left\{
            -\lambda'\mathbf{w}_h^{\pi}
            +
            \sum_{k'=1}^{k-1}
            \psi_h^{k'}
            \left[
                V_{h+1}^{k,(j)}(s_{h+1}^{k'})
                -
                P_hV_{h+1}^{\pi}(s_h^{k'},a_h^{k'})
            \right]
        \right.
        \\
        &\hspace{5.5cm}\left.
            +
            \sum_{k'=1}^{k-1}
            \psi_h^{k'}
            \bigl(\psi_h^{k'}\bigr)^\top
            \bigl(
                \widetilde{\theta}_k^{(j)}-\theta^*
            \bigr)
        \right\}
        \\
        &=
        \underbrace{
            -\lambda'
            \bigl(\Lambda_h^{k-1}\bigr)^{-1}
            \mathbf{w}_h^{\pi}
        }_{q_1}
        +
        \underbrace{
            \bigl(\Lambda_h^{k-1}\bigr)^{-1}
            \sum_{k'=1}^{k-1}
            \psi_h^{k'}
            \left[
                V_{h+1}^{k,(j)}(s_{h+1}^{k'})
                -
                P_hV_{h+1}^{k,(j)}
                (s_h^{k'},a_h^{k'})
            \right]
        }_{q_2}
        \\
        &\quad
        +
        \underbrace{
            \bigl(\Lambda_h^{k-1}\bigr)^{-1}
            \sum_{k'=1}^{k-1}
            \psi_h^{k'}
            P_h
            \left(
                V_{h+1}^{k,(j)}
                -
                V_{h+1}^{\pi}
            \right)(s_h^{k'},a_h^{k'})
        }_{q_3}
        +
        \bigl(
            \widetilde{\theta}_k^{(j)}-\theta^*
        \bigr)
        \underbrace{
            -\lambda'
            \bigl(\Lambda_h^{k-1}\bigr)^{-1}
            \bigl(
                \widetilde{\theta}_k^{(j)}-\theta^*
            \bigr)
        }_{q_4}.
    \end{align*}

    We next control the four terms above separately. For the first term,
    \begin{align*}
        \left|
        \left\langle
            \psi(s,a),
            q_1
        \right\rangle
        \right|
        &=
        \lambda'
        \left|
        \left\langle
            \psi(s,a),
            \bigl(\Lambda_h^{k-1}\bigr)^{-1}
            \mathbf{w}_h^{\pi}
        \right\rangle
        \right|
        \\
        &\le
        \sqrt{\lambda'}
        \left\|
            \mathbf{w}_h^{\pi}
        \right\|_2
        \sqrt{
            \bigl(\psi(s,a)\bigr)^\top
            \bigl(\Lambda_h^{k-1}\bigr)^{-1}
            \psi(s,a)
        } .
    \end{align*}
    For the second term, on event
    $\mathcal E^{\mathrm{Bi}}_{\mathrm{LS}}$, by
    Eq.~\eqref{eq:Bi_event_LS}, we have
    \begin{align*}
        \left|
        \left\langle
            \psi(s,a),
            q_2
        \right\rangle
        \right|
        \le
        C_{{LS}}^{{Bi}}
        \cdot dHB_{\mathrm{bi}}\sqrt{\log
    \left(
        \frac{2(c'_{\beta}+1)dKH}{\delta'}
    \right)}
        \sqrt{
            \bigl(\psi(s,a)\bigr)^\top
            \bigl(\Lambda_h^{k-1}\bigr)^{-1}
            \psi(s,a)
        } .
    \end{align*}
    For the third term, using
    Assumption~\ref{assump:linear_mdp}, we obtain
    \begin{align*}
        \left\langle
            \psi(s,a),
            q_3
        \right\rangle
        &=
        \left\langle
            \psi(s,a),
            \bigl(\Lambda_h^{k-1}\bigr)^{-1}
            \sum_{k'=1}^{k-1}
            \psi_h^{k'}
            P_h
            \left(
                V_{h+1}^{k,(j)}
                -
                V_{h+1}^{\pi}
            \right)(s_h^{k'},a_h^{k'})
        \right\rangle
        \\
        &=
        \left\langle
            \psi(s,a),
            \bigl(\Lambda_h^{k-1}\bigr)^{-1}
            \sum_{k'=1}^{k-1}
            \psi_h^{k'}
            \bigl(\psi_h^{k'}\bigr)^\top
            \int
            \left(
                V_{h+1}^{k,(j)}
                -
                V_{h+1}^{\pi}
            \right)(s')
            d\mu_h^*(s')
        \right\rangle
        \\
        &=
        \underbrace{
            \left\langle
                \psi(s,a),
                \int
                \left(
                    V_{h+1}^{k,(j)}
                    -
                    V_{h+1}^{\pi}
                \right)(s')
                d\mu_h^*(s')
            \right\rangle
        }_{p_1}
        \\
        &\quad
        \underbrace{
            -
            \lambda'
            \left\langle
                \psi(s,a),
                \bigl(\Lambda_h^{k-1}\bigr)^{-1}
                \int
                \left(
                    V_{h+1}^{k,(j)}
                    -
                    V_{h+1}^{\pi}
                \right)(s')
                d\mu_h^*(s')
            \right\rangle
        }_{p_2}.
    \end{align*}
    By Assumption~\ref{assump:linear_mdp} and
    $B_{\mathrm{bi}}\ge r_{\max}$, we have
    \begin{align*}
        p_1
        &=
        P_h
        \left(
            V_{h+1}^{k,(j)}
            -
            V_{h+1}^{\pi}
        \right)(s,a),
        \qquad
        |p_2|
        \le
        4HB_{\mathrm{bi}}\sqrt{d\lambda'}
        \sqrt{
            \bigl(\psi(s,a)\bigr)^\top
            \bigl(\Lambda_h^{k-1}\bigr)^{-1}
            \psi(s,a)
        } .
    \end{align*}
    For the fourth term, on event
    $\cE\cap\mathcal E_{\mathrm{TS}}$, we have
    \begin{align*}
        \left|
        \left\langle
            \psi(s,a),
            q_4
        \right\rangle
        \right|
        &=
        \lambda'
        \left|
        \left\langle
            \psi(s,a),
            \bigl(\Lambda_h^{k-1}\bigr)^{-1}
            \bigl(
                \widetilde{\theta}_k^{(j)}
                -
                \theta^*
            \bigr)
        \right\rangle
        \right|
        \\
        &\le
        \lambda'
        \sqrt{
            \bigl(\psi(s,a)\bigr)^\top
            \bigl(\Lambda_h^{k-1}\bigr)^{-1}
            \psi(s,a)
        }
        \left\|
            \widetilde{\theta}_k^{(j)}
            -
            \theta^*
        \right\|_{\bigl(\Lambda_h^{k-1}\bigr)^{-1}}
        \\
        &\le
        \sqrt{\lambda'}B_{\mathrm{bi}}
        \sqrt{
            \bigl(\psi(s,a)\bigr)^\top
            \bigl(\Lambda_h^{k-1}\bigr)^{-1}
            \psi(s,a)
        } .
    \end{align*}
    Since
    \begin{align*}
        \left\langle
            \psi(s,a),
            \mathbf{w}_h^{k,(j)}
        \right\rangle
        -
        Q_h^{\pi}(s,a)
        &=
        \left\langle
            \psi(s,a),
            q_1+q_2+q_3+q_4
        \right\rangle
        +
        \left\langle
            \psi(s,a),
            \widetilde{\theta}_k^{(j)}-\theta^*
        \right\rangle,
    \end{align*}
    by Lemma~\ref{lemma:bound_weights_primal_MDP} and our choice of
    $\lambda'$, there exists an absolute constant $c'>0$, independent
    of $c'_\beta$, such that
    \begin{align}
        &
        \left|
        \left\langle
            \psi(s,a),
            \mathbf{w}_h^{k,(j)}
        \right\rangle
        -
        Q_h^{\pi}(s,a)
        -
        P_h
        \left(
            V_{h+1}^{k,(j)}
            -
            V_{h+1}^{\pi}
        \right)(s,a)
        -
        \left\langle
            \psi(s,a),
            \widetilde{\theta}_k^{(j)}-\theta^*
        \right\rangle
        \right|
        \notag
        \\
        &\le
        c'\cdot dHB_{\mathrm{bi}}\sqrt{\log
    \left(
        \frac{2(c'_{\beta}+1)dKH}{\delta'}
    \right)}
        \sqrt{
            \bigl(\psi(s,a)\bigr)^\top
            \bigl(\Lambda_h^{k-1}\bigr)^{-1}
            \psi(s,a)
        } .
        \label{eq:Bi_c'_bound}
    \end{align}
    It remains to ensure that the right-hand side above is bounded by
    the bonus term with coefficient $\beta_{\mathrm{bi}}$.
    It is sufficient to choose
    $c'_\beta$ such that
    \begin{align}
        c'
        \sqrt{
            \log\left(\frac{2dKH}{\delta'}\right)
            +
            \log(c'_\beta+1)
        }
        \le
        c'_\beta
        \sqrt{
            \log\left(\frac{2dKH}{\delta'}\right)
        } .
        \label{eq:Bi_cbeta_choice}
    \end{align}
    Since
    $\log\left(\frac{2dKH}{\delta'}\right)\in[\log 2,\infty)$
    and $c'$ is an absolute constant independent of $c'_\beta$, we can
    find a sufficiently large absolute constant $c'_\beta$ satisfying
    \[
        c'
        \sqrt{
            \log 2+\log(c'_\beta+1)
        }
        \le
        c'_\beta\sqrt{\log 2}.
    \]
    This choice guarantees Eq.~\eqref{eq:Bi_cbeta_choice}. Hence,
    \begin{align*}
        c'\cdot dHB_{\mathrm{bi}}\sqrt{\log
    \left(
        \frac{2(c'_{\beta}+1)dKH}{\delta'}
    \right)}
        \le
        \beta_{\mathrm{bi}} .
    \end{align*}
    Combining this inequality with Eq.~\eqref{eq:Bi_c'_bound} proves
    the result.

    This completes the proof.
\end{proof}

\begin{lemma}[Upper Bound for estimator error in $\cM_k^{(j)}$]
\label{lemma:Bi_upper_bound_LSVI_kj}
    Assume that event 
    $\cE \cap \mathcal E_{\mathrm{TS}}
    \cap \mathcal E^{\mathrm{Bi}}_{\mathrm{LS}}$ holds.
    Choose the absolute constant $c'_\beta$ in the definition of
    $\beta_{\mathrm{bi}}$ sufficiently large.
    Then, for any policy $\pi$, for all
    $(s,a,h,k,j)\in
    \mathcal{S}\times\mathcal{A}\times[H]\times[K]\times[n]$, we have
    \begin{align*}
        \left|
        \left\langle
            \psi(s,a),
            \mathbf{w}_h^{k,(j)}
        \right\rangle
        -
        Q_h^{\pi;(k,j)}(s,a)
        -
        P_h
        \left(
            V_{h+1}^{k,(j)}
            -
            V_{h+1}^{\pi;(k,j)}
        \right)(s,a)
        \right|
        \\
        \le
        \beta_{\mathrm{bi}}
        \sqrt{
            \bigl(\psi(s,a)\bigr)^\top
            \bigl(\Lambda_h^{k-1}\bigr)^{-1}
            \psi(s,a)
        } .
    \end{align*}
\end{lemma}

\begin{proof}
    This proof follows the analysis of Lemma~B.4 in
    (Jin et al.,~\citeyear{jin20_linearmdp}).
    
    Consider the MDP $\cM_k^{(j)}$ with reward parameter
    $\widetilde{\theta}_k^{(j)}$.
    By Proposition~\ref{prop:linearity_Q} and the Bellman equation, for
    any $(s,a,h)\in\mathcal{S}\times\mathcal{A}\times[H]$, the
    state-action value function satisfies
    \begin{align*}
        Q_h^{\pi;(k,j)}(s,a)
        &:=
        \left\langle
            \psi(s,a),
            \mathbf{w}_h^{\pi;(k,j)}
        \right\rangle
        \\
        &=
        \left(
            \langle
                \psi,
                \widetilde{\theta}_k^{(j)}
            \rangle
            +
            P_hV_{h+1}^{\pi;(k,j)}
        \right)(s,a).
    \end{align*}
    This implies
    \begin{align*}
        \mathbf{w}_h^{k,(j)}
        -
        \mathbf{w}_h^{\pi;(k,j)}
        &=
        \bigl(\Lambda_h^{k-1}\bigr)^{-1}
        \sum_{k'=1}^{k-1}
        \psi_h^{k'}
        \left[
            \langle
                \psi_h^{k'},
                \widetilde{\theta}_k^{(j)}
            \rangle
            +
            V_{h+1}^{k,(j)}(s_{h+1}^{k'})
        \right]
        -
        \mathbf{w}_h^{\pi;(k,j)}
        \\
        &=
        \bigl(\Lambda_h^{k-1}\bigr)^{-1}
        \left\{
            -\lambda'\mathbf{w}_h^{\pi;(k,j)}
            +
            \sum_{k'=1}^{k-1}
            \psi_h^{k'}
            \left[
                V_{h+1}^{k,(j)}(s_{h+1}^{k'})
                -
                P_hV_{h+1}^{\pi;(k,j)}
                (s_h^{k'},a_h^{k'})
            \right]
        \right\}
        \\
        &=
        \underbrace{
            -\lambda'
            \bigl(\Lambda_h^{k-1}\bigr)^{-1}
            \mathbf{w}_h^{\pi;(k,j)}
        }_{q_1}
        +
        \underbrace{
            \bigl(\Lambda_h^{k-1}\bigr)^{-1}
            \sum_{k'=1}^{k-1}
            \psi_h^{k'}
            \left[
                V_{h+1}^{k,(j)}(s_{h+1}^{k'})
                -
                P_hV_{h+1}^{k,(j)}
                (s_h^{k'},a_h^{k'})
            \right]
        }_{q_2}
        \\
        &\quad
        +
        \underbrace{
            \bigl(\Lambda_h^{k-1}\bigr)^{-1}
            \sum_{k'=1}^{k-1}
            \psi_h^{k'}
            P_h
            \left(
                V_{h+1}^{k,(j)}
                -
                V_{h+1}^{\pi;(k,j)}
            \right)(s_h^{k'},a_h^{k'})
        }_{q_3}.
    \end{align*}

    We now bound these three terms separately. For the first term,
    \begin{align*}
        \left|
        \left\langle
            \psi(s,a),
            q_1
        \right\rangle
        \right|
        &=
        \lambda'
        \left|
        \left\langle
            \psi(s,a),
            \bigl(\Lambda_h^{k-1}\bigr)^{-1}
            \mathbf{w}_h^{\pi;(k,j)}
        \right\rangle
        \right|
        \\
        &\le
        \sqrt{\lambda'}
        \left\|
            \mathbf{w}_h^{\pi;(k,j)}
        \right\|_2
        \sqrt{
            \bigl(\psi(s,a)\bigr)^\top
            \bigl(\Lambda_h^{k-1}\bigr)^{-1}
            \psi(s,a)
        } .
    \end{align*}
    For the second term, on event
    $\mathcal E^{\mathrm{Bi}}_{\mathrm{LS}}$, by
    Eq.~\eqref{eq:Bi_event_LS}, we have
    \begin{align*}
        \left|
        \left\langle
            \psi(s,a),
            q_2
        \right\rangle
        \right|
        \le
        C_{{LS}}^{{Bi}}
        \cdot dHB_{\mathrm{bi}}\sqrt{\log
    \left(
        \frac{2(c'_{\beta}+1)dKH}{\delta'}
    \right)}
        \sqrt{
            \bigl(\psi(s,a)\bigr)^\top
            \bigl(\Lambda_h^{k-1}\bigr)^{-1}
            \psi(s,a)
        } .
    \end{align*}
    For the third term, by Assumption~\ref{assump:linear_mdp}, we have
    \begin{align*}
        \left\langle
            \psi(s,a),
            q_3
        \right\rangle
        &=
        \left\langle
            \psi(s,a),
            \bigl(\Lambda_h^{k-1}\bigr)^{-1}
            \sum_{k'=1}^{k-1}
            \psi_h^{k'}
            P_h
            \left(
                V_{h+1}^{k,(j)}
                -
                V_{h+1}^{\pi;(k,j)}
            \right)(s_h^{k'},a_h^{k'})
        \right\rangle
        \\
        &=
        \left\langle
            \psi(s,a),
            \bigl(\Lambda_h^{k-1}\bigr)^{-1}
            \sum_{k'=1}^{k-1}
            \psi_h^{k'}
            \bigl(\psi_h^{k'}\bigr)^\top
            \int
            \left(
                V_{h+1}^{k,(j)}
                -
                V_{h+1}^{\pi;(k,j)}
            \right)(s')
            d\mu_h^*(s')
        \right\rangle
        \\
        &=
        \underbrace{
            \left\langle
                \psi(s,a),
                \int
                \left(
                    V_{h+1}^{k,(j)}
                    -
                    V_{h+1}^{\pi;(k,j)}
                \right)(s')
                d\mu_h^*(s')
            \right\rangle
        }_{p_1}
        \\
        &\quad
        \underbrace{
            -
            \lambda'
            \left\langle
                \psi(s,a),
                \bigl(\Lambda_h^{k-1}\bigr)^{-1}
                \int
                \left(
                    V_{h+1}^{k,(j)}
                    -
                    V_{h+1}^{\pi;(k,j)}
                \right)(s')
                d\mu_h^*(s')
            \right\rangle
        }_{p_2}.
    \end{align*}
    By Assumption~\ref{assump:linear_mdp}, we have
    \begin{align*}
        p_1
        &=
        P_h
        \left(
            V_{h+1}^{k,(j)}
            -
            V_{h+1}^{\pi;(k,j)}
        \right)(s,a),
        \qquad
        |p_2|
        \le
        4HB_{\mathrm{bi}}\sqrt{d\lambda'}
        \sqrt{
            \bigl(\psi(s,a)\bigr)^\top
            \bigl(\Lambda_h^{k-1}\bigr)^{-1}
            \psi(s,a)
        } .
    \end{align*}
    Since
    \begin{align*}
        \left\langle
            \psi(s,a),
            \mathbf{w}_h^{k,(j)}
        \right\rangle
        -
        Q_h^{\pi;(k,j)}(s,a)
        =
        \left\langle
            \psi(s,a),
            q_1+q_2+q_3
        \right\rangle ,
    \end{align*}
    on event $\cE\cap\mathcal E_{\mathrm{TS}}$, by
    Lemma~\ref{lemma:Bi_bound_Vweights_mdp_j^k} and our choice of
    $\lambda'$, there exists an absolute constant $c_0>0$, independent
    of $c'_\beta$, such that
    \begin{align}
        &
        \left|
        \left\langle
            \psi(s,a),
            \mathbf{w}_h^{k,(j)}
        \right\rangle
        -
        Q_h^{\pi;(k,j)}(s,a)
        -
        P_h
        \left(
            V_{h+1}^{k,(j)}
            -
            V_{h+1}^{\pi;(k,j)}
        \right)(s,a)
        \right|
        \notag
        \\
        &\le
        c_0\cdot dHB_{\mathrm{bi}}\sqrt{\log
    \left(
        \frac{2(c'_{\beta}+1)dKH}{\delta'}
    \right)}
        \sqrt{
            \bigl(\psi(s,a)\bigr)^\top
            \bigl(\Lambda_h^{k-1}\bigr)^{-1}
            \psi(s,a)
        }  \notag
    \\&
        \le
        \beta_{\mathrm{bi}}
        \sqrt{
            \bigl(\psi(s,a)\bigr)^\top
            \bigl(\Lambda_h^{k-1}\bigr)^{-1}
            \psi(s,a)
        } .
        \label{eq:Bi_c0_bound}
    \end{align}
    The last inequality follows by choosing $c'_\beta$ sufficiently
    large so that Eq.~\eqref{eq:Bi_cbeta_choice} and
    Eq.~\eqref{eq:Bi_c0_bound} hold simultaneously.

    This completes the proof.
\end{proof}

Define event
\[
    \mathcal Q^{\mathrm{Bi}}
    :=
    \left\{
    \forall k\in [K],
    \exists j\in[n]:
    \bigl(\phi^{\pi^*}(s_1^k)\bigr)^\top
    \widetilde{\theta}_k^{(j)}
    \ge
    \bigl(\phi^{\pi^*}(s_1^k)\bigr)^\top
    \theta^*
    \right\},
\]

\begin{lemma}\label{lemma:Bi_TS_p_optim}
Set
$    n
    =
    \left\lceil
        2\sqrt{2\pi e}
        \log\left(\frac{K}{\delta'}\right)
    \right\rceil .$
Then
\begin{align*}
    \Pr\left(
        \bigl(\mathcal Q^{\mathrm{Bi}}\bigr)^c
        \cap \cE
    \right)
    \leq \delta' .
\end{align*}
\end{lemma}

\begin{proof}
For each $k\in[K]$, define $ X_k := \phi^{\pi^*}(s_1^k)$.
Let $\cG_k$ denote the filtration including the history up to the end
of episode $k-1$ and the initial state $s_1^k$ of episode $k$. Then
$X_k$, $\hat{\theta}_{k-1}^{\mathrm{bi}}$, and
$\Sigma_{k-1}^{\mathrm{bi}}$ are fixed given $\cG_k$.

For any $k>0$, define the local concentration event
\[
    \cE^{k-1}
    =
    \left\{
        \left\|
            g_{k-1}\bigl(\hat{\theta}_{k-1}^{\mathrm{bi}}\bigr)
            -
            g_{k-1}(\theta^*)
        \right\|_{\Lambda_{k-1}^{-1}(\theta^*)}
        \leq \omega(k-1)
    \right\}.
\]
On event $\cE^{k-1}$, by
Lemma~\ref{lemma:lin_confence_interval_proj_free}, we have
\[
    \left|
        X_k^\top
        \left(
            \theta^*-\hat{\theta}_{k-1}^{\mathrm{bi}}
        \right)
    \right|
    \leq
    \sqrt{\alpha}\cdot\nu(k-1)
    \left\|X_k\right\|_{(\Sigma_{k-1}^{\mathrm{bi}})^{-1}} .
\]
In particular,
\[
    X_k^\top
    \left(
        \theta^*-\hat{\theta}_{k-1}^{\mathrm{bi}}
    \right)
    \leq
    \sqrt{\alpha}\cdot\nu(k-1)
    \left\|X_k\right\|_{(\Sigma_{k-1}^{\mathrm{bi}})^{-1}} .
\]

For each $j\in[n]$, conditional on $\cG_k$, we have
\[
    \widetilde{\theta}_k^{(j)}
    =
    \hat{\theta}_{k-1}^{\mathrm{bi}}
    +
    \xi_k^{(j)},
    \qquad
    \xi_k^{(j)}
    \sim
    \mathcal N\left(
        0,
        \alpha\cdot\nu(k-1)^2
        (\Sigma_{k-1}^{\mathrm{bi}})^{-1}
    \right),
\]
and the samples are conditionally independent. Therefore,
\[
    X_k^\top\xi_k^{(j)}
    \sim
    \mathcal N\left(
        0,
        \alpha\cdot\nu(k-1)^2
        \left\|X_k\right\|_{(\Sigma_{k-1}^{\mathrm{bi}})^{-1}}^2
    \right).
\]
Consequently, on $\cE^{k-1}$, following the same anti-concentration
argument as in Lemma~\ref{lemma:lin_Gaussian_anti_concentration}, we
obtain
\[
    \Pr\left[
        X_k^\top\widetilde{\theta}_k^{(j)}
        \geq
        X_k^\top\theta^*
        \,\middle|\,
        \cG_k
    \right]
    \geq
    \Pr(Z\geq 1)
    \geq
    \frac{1}{2\sqrt{2\pi e}}
    =:p_0,
\]
where $Z\sim\mathcal N(0,1)$. If
$\left\|X_k\right\|_{(\Sigma_{k-1}^{\mathrm{bi}})^{-1}}=0$, then the
same lower bound holds trivially.

Define
\[
    \mathcal Q_k^{\mathrm{Bi}}
    =
    \left\{
        \exists j\in[n]:
        X_k^\top\widetilde{\theta}_k^{(j)}
        \geq
        X_k^\top\theta^*
    \right\}.
\]
By the conditional independence of
$\{\widetilde{\theta}_k^{(j)}\}_{j\in[n]}$ given $\cG_k$, we have
\[
    \mathbf 1_{\cE^{k-1}}
    \Pr\left(
        \bigl(\mathcal Q_k^{\mathrm{Bi}}\bigr)^c
        \,\middle|\,
        \cG_k
    \right)
    \leq
    \mathbf 1_{\cE^{k-1}}(1-p_0)^n
    \leq
    \mathbf 1_{\cE^{k-1}}\exp(-np_0).
\]
Since $\cE^{k-1}$ is also measurable with respect to $\cG_k$, taking
the tower rule of expectation on both sides, we obtain
\[
    \Pr\left(
        \bigl(\mathcal Q_k^{\mathrm{Bi}}\bigr)^c
        \cap \cE^{k-1}
    \right)
    \leq
    \exp(-np_0).
\]
Moreover,
\[
    \cE
    =
    \bigcap_{k>0}\cE^{k-1}
    \subseteq
    \cE^{k-1}.
\]
Thus,
\[
    \Pr\left(
        \bigl(\mathcal Q_k^{\mathrm{Bi}}\bigr)^c
        \cap \cE
    \right)
    \leq
    \exp(-np_0).
\]
Since
\[
    \mathcal Q^{\mathrm{Bi}}
    =
    \bigcap_{k=1}^{K}\mathcal Q_k^{\mathrm{Bi}},
\]
by the union bound,
\[
    \Pr\left(
        \bigl(\mathcal Q^{\mathrm{Bi}}\bigr)^c
        \cap \cE
    \right)
    \leq
    \sum_{k=1}^{K}
    \Pr\left(
        \bigl(\mathcal Q_k^{\mathrm{Bi}}\bigr)^c
        \cap \cE
    \right)
    \leq
    K\exp(-np_0).
\]
By the choice of $n$, we have
\[
    K\exp(-np_0)\leq\delta'.
\]
Therefore,
\[
    \Pr\left(
        \bigl(\mathcal Q^{\mathrm{Bi}}\bigr)^c
        \cap \cE
    \right)
    \leq \delta'.
\]

This completes the proof.

\end{proof}

Let ${\pi^*}_k^{(j)}$, $Q_{h}^{{\pi^*}_k^{(j)};(k,j)}$, and 
$V_{h}^{{\pi^*}_k^{(j)};(k,j)}$ denote the optimal policy, state-action value function, and state value function of MDP $\cM_k^{(j)}$.

\begin{lemma}\label{lemma:Bi_UCB_MDP_jk}
   Assume that event 
   $\cE \cap \mathcal E_{\mathrm{TS}} \cap \mathcal E^{\mathrm{Bi}}_{\mathrm{LS}}$ holds.
   Then, for all 
   $(s,a,h,k,j)\in \mathcal{S} \times \mathcal{A} \times [H] \times [K] \times [n]$, we have
\[
Q_h^{k,(j)}(s,a) 
\ge 
Q_h^{{\pi^*}_k^{(j)};(k,j)}(s,a).
\]
\end{lemma}

\begin{proof}
This proof follows the analysis of Lemma~B.5 in (Jin et al., \citeyear{jin20_linearmdp}).

We prove the claim by backward induction on $h$.
At the last step $H$, we have 
\[
V_{H+1}^{k,(j)} = V_{H+1}^{{\pi^*}_k^{(j)};(k,j)} = 0 .
\]
Applying Lemma~\ref{lemma:Bi_upper_bound_LSVI_kj} with $\pi={\pi^*}_k^{(j)}$, we obtain
\[
\left| 
\left\langle 
\psi(s,a), 
\mathbf{w}_H^{k,(j)} 
\right\rangle 
-
Q_H^{{\pi^*}_k^{(j)};(k,j)}(s,a) 
\right|
\le
\beta_{\mathrm{bi}} 
\sqrt{ 
\bigl(\psi(s,a)\bigr)^\top 
\bigl(\Lambda_H^{k-1}\bigr)^{-1} 
\psi(s,a) 
}.
\]
Since $Q_H^{{\pi^*}_k^{(j)};(k,j)}(s,a) \in [-HB_{\mathrm{bi}},HB_{\mathrm{bi}}]$, it follows that
\[
Q_H^{{\pi^*}_k^{(j)};(k,j)}(s,a) 
\le 
\operatorname{clip}_{[-HB_{\mathrm{bi}}, HB_{\mathrm{bi}}]}
        \left(
        \bigl(\mathbf{w}_H^{k,(j)}\bigr)^\top \psi(s,a)
        +
        \beta_{\mathrm{bi}}
        \sqrt{
        \bigl(\psi(s,a)\bigr)^\top
        \bigl(\Lambda_H^{k-1}\bigr)^{-1}
        \psi(s,a)
        }
        \right)
=
Q_H^{k,(j)}(s,a).
\]

Now suppose that the claim holds at step $h+1$. We prove it for step $h$.
Again, applying Lemma~\ref{lemma:Bi_upper_bound_LSVI_kj} with $\pi={\pi^*}_k^{(j)}$, we have
\[
\left| 
\left\langle 
\psi(s,a), 
\mathbf{w}_h^{k,(j)} 
\right\rangle 
-
Q_h^{{\pi^*}_k^{(j)};(k,j)}(s,a)
-
P_h
\left(
V_{h+1}^{k,(j)}
-
V_{h+1}^{{\pi^*}_k^{(j)};(k,j)}
\right)(s,a)
\right|
\\
\le
\beta_{\mathrm{bi}} 
\sqrt{
\bigl(\psi(s,a)\bigr)^\top 
\bigl(\Lambda_h^{k-1}\bigr)^{-1} 
\psi(s,a)
}.
\]

By the induction hypothesis,
\[V_{h+1}^{k,(j)}(s')
\ge
V_{h+1}^{{\pi^*}_k^{(j)};(k,j)}(s'),\]
and hence
$P_h
\left(
V_{h+1}^{k,(j)}
-
V_{h+1}^{{\pi^*}_k^{(j)};(k,j)}
\right)(s,a)
\ge
0.$

Therefore,
\begin{align*}
    Q_h^{{\pi^*}_k^{(j)};(k,j)}(s,a) 
&\le 
\operatorname{clip}_{[-HB_{\mathrm{bi}}, HB_{\mathrm{bi}}]}
        \left(
        \bigl(\mathbf{w}_h^{k,(j)}\bigr)^\top \psi(s,a)
        +
        \beta_{\mathrm{bi}}
        \sqrt{
        \bigl(\psi(s,a)\bigr)^\top
        \bigl(\Lambda_h^{k-1}\bigr)^{-1}
        \psi(s,a)
        }
        \right) \\
&=
Q_h^{k,(j)}(s,a).
\end{align*}

This completes the proof.
\end{proof}

\begin{lemma}[Optimism]\label{lemma:Bi_LSVI_TS_optim}
Assume that event 
$\cE \cap \mathcal E_{\mathrm{TS}} \cap \mathcal E^{\mathrm{Bi}}_{\mathrm{LS}}\cap \mathcal Q^{\mathrm{Bi}}$ holds.
For any episode \(k\) and realized initial state $s_1^k$, let $j^+$ be selected by Algorithm~\ref{alg:unified-seg-lsvi-ts} such that
$    j^+
    =
    \arg\max_{j\in[n]}
    V_1^{k,(j)}(s_1^k).$
Then, we have
\begin{align*}
    V_1^{k,(j^+)}(s_1^k)
    \ge
    V_1^*(s_1^k).
\end{align*}
\end{lemma}

\begin{proof}
    On event $\cE \cap \mathcal E_{\mathrm{TS}} \cap \mathcal E^{\mathrm{Bi}}_{\mathrm{LS}}$, by Lemma~\ref{lemma:Bi_UCB_MDP_jk}, we have
    \[
        Q_h^{k,(j)}(s,a)
        \ge
        Q_h^{{\pi^*}_k^{(j)};(k,j)}(s,a),
        \qquad
        \forall (s,a,h,k,j)
        \in
        \mathcal{S}\times\mathcal{A}\times[H]\times[K]\times[n].
    \]
    Given $s_1^k$, it follows that
    \begin{align*}
        V_1^{k,(j)}(s^k_1)
        &=
        \max_a Q_1^{k,(j)}(s^k_1,a)
        \ge
        Q_1^{k,(j)}
        \left(
            s_1^k,
            {\pi^*}_k^{(j)}(s_1^k)
        \right)
\ge
        V_1^{{\pi^*}_k^{(j)};(k,j)}(s_1^k)
        \ge
        V_1^{\pi^*;(k,j)}(s_1^k),
    \end{align*}
    where the last inequality follows from the definition of the optimal state value function
    $V_1^{{\pi^*}_k^{(j)};(k,j)}$ in MDP $\cM_k^{(j)}$, and $\pi^*$ denotes the optimal policy of the original MDP.

    Moreover, under the same transition kernel, we can write
    $V_1^{\pi^*;(k,j)}(s_1^k)$ and $V_1^*(s_1^k)$ as
    \[
        V_1^{\pi^*;(k,j)}(s_1^k)
        =
        \left(\phi^{\pi^*}(s_1^k)\right)^\top
        \widetilde{\theta}_k^{(j)},
        \qquad
        V_1^*(s_1^k)
        =
        \left(\phi^{\pi^*}(s_1^k)\right)^\top
        \theta^* .
    \]
    On event $\mathcal Q^{\mathrm{Bi}}$, for the given episode $k$, there exists $j'\in[n]$ such that
    \begin{align*}
        \left(\phi^{\pi^*}(s_1^k)\right)^\top
        \widetilde{\theta}_k^{(j')}
        \ge
        \left(\phi^{\pi^*}(s_1^k)\right)^\top
        \theta^* .
    \end{align*}
    Hence,
    \begin{align*}
        V_1^{\pi^*;(k,j')}(s_1^k)
        \ge
        V_1^*(s_1^k).
    \end{align*}

    By the definition of $j^+$ in Algorithm~\ref{alg:unified-seg-lsvi-ts}, we have
    \begin{align*}
        V_1^{k,(j^+)}(s^k_1)
        \ge
        V_1^{k,(j')}(s^k_1)
        \ge
        V_1^{\pi^*;(k,j')}(s_1^k)
        \ge
        V_1^*(s_1^k).
    \end{align*}
    This proves the optimism.
\end{proof}

\begin{lemma}\label{lemma:Bi_recursive_formula}
    Assume that event 
    $\cE \cap \mathcal E_{\mathrm{TS}} \cap \mathcal E^{\mathrm{Bi}}_{\mathrm{LS}}$ holds.
    For any episode $k$, let $j^+$ and $\pi^k$ be generated by 
    Algorithm~\ref{alg:unified-seg-lsvi-ts} . Define
    \[
        \Gamma_h^k
        :=
        V_h^{k,(j^+)}(s_h^k)
        -
        V_h^{\pi^k}(s_h^k),
    \]
    and
    \[
        \zeta_{h+1}^k
        :=
        \mathbb{E}
        \left[
            \Gamma_{h+1}^k
            \mid
            s_h^k,a_h^k
        \right]
        -
        \Gamma_{h+1}^k .
    \]
    Then, for any $(k,h)\in[K]\times[H]$, we have
    \begin{align*}
        \Gamma_h^k
        \le
        \Gamma_{h+1}^k
        +
        \zeta_{h+1}^k
        +
        2\beta_{\mathrm{bi}}
        \sqrt{
            \bigl(\psi_h^k\bigr)^\top
            \bigl(\Lambda_h^{k-1}\bigr)^{-1}
            \psi_h^k
        } 
        +
        \left\langle
            \psi_h^k,
            \widetilde{\theta}_k^{(j^+)}-\theta^*
        \right\rangle .
    \end{align*}
\end{lemma}

\begin{proof}
    This proof follows the analysis of Lemma~B.6 in (Jin et al., \citeyear{jin20_linearmdp}).
    
    On event $\cE \cap \mathcal E_{\mathrm{TS}} \cap \mathcal {E}^{\mathrm{{Bi}}}_{\mathrm{LS}}$, by Lemma~\ref{lemma:Bi_UCB_MDP_jk} and its proof, the un-clipped optimistic estimate satisfies
    \[
        \left\langle
            \psi(s,a),
            \mathbf{w}_h^{k,(j)}
        \right\rangle
        +
        \beta_{\mathrm{bi}}
        \sqrt{
            \bigl(\psi(s,a)\bigr)^\top
            \bigl(\Lambda_h^{k-1}\bigr)^{-1}
            \psi(s,a)
        }
        \ge
        Q_h^{{\pi^*}_k^{(j)};(k,j)}(s,a)
        \ge
        -HB_{\mathrm{bi}} .
    \]
    Hence, the lower clipping operation is inactive. Therefore, using Lemma~\ref{lemma:Bi_upper_bound_LSVI-TS}, for any
    $(s,a,h,k)\in\mathcal{S}\times\mathcal{A}\times[H]\times[K]$, we have
    \begin{align*}
        Q_h^{k,(j^+)}(s,a)
        -
        Q_h^{\pi^k}(s,a)
        &\le
        P_h
        \left(
            V_{h+1}^{k,(j^+)}
            -
            V_{h+1}^{\pi^k}
        \right)(s,a)\\
        &+
        2\beta_{\mathrm{bi}}
        \sqrt{
            \bigl(\psi(s,a)\bigr)^\top
            \bigl(\Lambda_h^{k-1}\bigr)^{-1}
            \psi(s,a)
        }
        +
        \left\langle
            \psi(s,a),
            \widetilde{\theta}_k^{(j^+)}-\theta^*
        \right\rangle .
    \end{align*}
    By the action-selection rule in Algorithm~\ref{alg:unified-seg-lsvi-ts} and the definition of the state value function under $\pi^k$, we have
    \begin{align*}
        \Gamma_h^k
        &=
        V_h^{k,(j^+)}(s_h^k)
        -
        V_h^{\pi^k}(s_h^k)=
        Q_h^{k,(j^+)}(s_h^k,a_h^k)
        -
        Q_h^{\pi^k}(s_h^k,a_h^k).
    \end{align*}
    Substituting $(s,a)=(s_h^k,a_h^k)$ into the previous inequality, we have
    \begin{align*}
        \Gamma_h^k
        &\le
        P_h
        \left(
            V_{h+1}^{k,(j^+)}
            -
            V_{h+1}^{\pi^k}
        \right)(s_h^k,a_h^k)
        +
        2\beta_{\mathrm{bi}}
        \sqrt{
            \bigl(\psi_h^k\bigr)^\top
            \bigl(\Lambda_h^{k-1}\bigr)^{-1}
            \psi_h^k
        }
        +
        \left\langle
            \psi_h^k,
            \widetilde{\theta}_k^{(j^+)}-\theta^*
        \right\rangle
        \\
        &=
        \Gamma_{h+1}^k
        +
        \zeta_{h+1}^k
        +
        2\beta_{\mathrm{bi}}
        \sqrt{
            \bigl(\psi_h^k\bigr)^\top
            \bigl(\Lambda_h^{k-1}\bigr)^{-1}
            \psi_h^k
        }
        +
        \left\langle
            \psi_h^k,
            \widetilde{\theta}_k^{(j^+)}-\theta^*
        \right\rangle .
    \end{align*}

    This completes the proof.
\end{proof}

Although the initial state \(s_1^k\) is revealed before the Gaussian
noises in Algorithm~\ref{alg:unified-seg-lsvi-ts}, Gaussian noises are sampled independently of \(s_1^k\) conditional on the past history.

For the convenience of analysis, we use the equivalent revealing order in which the
Gaussian noises are exposed before \(s_1^k\). Let $\mathcal F_{k-1}^{+}$ denote the filtration including the history up to the end of episode $k-1$ and the Gaussian noises \(\{\xi_k^{(j)}\}_{j=1}^n\) for episode \(k\) are revealed, but before the initial state \(s_1^k\) is observed. For \(h\in[H]\), let $\mathcal F_{k,h}$ denote the filtration including $\mathcal F_{k-1}^{+}$, the trajectory information up to step \(h\) are revealed.

We use the notation of Lemma~\ref{lemma:Bi_recursive_formula}, define the event
\begin{align}
\mathcal{E}^{\mathrm{Bi}}_{\mathrm{KS}}
:=
\Bigg\{
&\left|
\sum_{k'=1}^{k}
\sum_{h=1}^{H}
\zeta_{h+1}^{k'}
\right|
\le
4HB_{\mathrm{bi}}
\sqrt{
2kH\log\left(\frac{2K}{\delta'}\right)
},
\nonumber\\
&\left|
\sum_{k'=1}^{k}
\left(
\left[
V_1^*(s_1^{k'})
-
V_1^{\pi^{k'}}(s_1^{k'})
\right]
-
\mathbb E_{s\sim\rho}
\left[
V_1^*(s)
-
V_1^{\pi^{k'}}(s)
\right]
\right)
\right|\notag
\\
&\le
2Hr_{\max}
\sqrt{
2k\log\left(\frac{2K}{\delta'}\right)
},
\quad
\forall k\in[K].
\Bigg\}. 
\label{eq:Bi_some_concerntration}
\end{align}

\begin{lemma}[Concentration of Martingale Difference Terms]\label{lemma:Bi_event_KS}
It holds that
\begin{align*}
    \Pr\left[ \mathcal{E}^{\mathrm{Bi}}_{\mathrm{KS}} \right] \ge 1-2\delta' .
\end{align*}
\end{lemma}

\begin{proof}
We prove the two inequalities in the definition of
$\mathcal E^{\mathrm{Bi}}_{\mathrm{KS}}$ separately.

First, we control the martingale difference sequence
$\{\zeta_{h+1}^{k}\}_{k,h}$. Recall that \(\mathcal F_{k,h}\) denotes the information available after the history up to episode \(k-1\), the Gaussian perturbations for episode \(k\), and the trajectory information up to step \(h\), namely \(s_1^k,a_1^k,\ldots,s_h^k,a_h^k\), are revealed. 

Conditional on \(\mathcal F_{k,h}\), the functions
\(V_{h+1}^{k,(j^+)}\) and \(V_{h+1}^{\pi^k}\) are fixed. Hence, by the definition of
\(\Gamma_{h+1}^k\) and \(\zeta_{h+1}^k\), we have
\[
\mathbb E
\left[
    \zeta_{h+1}^k
    \mid
    \mathcal F_{k,h}
\right]
=0.
\]
Moreover, since both state value functions are bounded in
\([-HB_{\mathrm{bi}},HB_{\mathrm{bi}}]\), we have
\[
|\zeta_{h+1}^k|
\le
4HB_{\mathrm{bi}}.
\]
Therefore, for any fixed \(k\in[K]\), by the Azuma--Hoeffding inequality,
\[
\Pr\left(
\left|
\sum_{k'=1}^{k}
\sum_{h=1}^{H}
\zeta_{h+1}^{k'}
\right|
>
4HB_{\mathrm{bi}}
\sqrt{
2kH\log\left(\frac{2K}{\delta'}\right)
}
\right)
\le
\frac{\delta'}{K}.
\]
Taking a union bound over \(k\in[K]\), we obtain that with probability at
least \(1-\delta'\), for all \(k\in[K]\),
\[
\left|
\sum_{k'=1}^{k}
\sum_{h=1}^{H}
\zeta_{h+1}^{k'}
\right|
\le
4HB_{\mathrm{bi}}
\sqrt{
2kH\log\left(\frac{2K}{\delta'}\right)
}.
\]

Next, we construct the concentration inequality for the randomness of the
initial states. We clarify the meaning of
\[
\mathbb E_{s\sim\rho}
\left[
V_1^*(s)-V_1^{\pi^k}(s)
\right].
\]
Here, \(\mathcal F_{k-1}^{+}\) denotes the information including the history up to episode \(k-1\) and the Gaussian perturbations \(\{\xi_k^{(j)}\}_{j=1}^n\) for episode \(k\) are revealed, but before the initial state \(s_1^k\) is observed. 

Conditional on \(\mathcal F_{k-1}^{+}\), the history and the Gaussian perturbations are fixed. For an initial state \(s\), \(V_1^{\pi^k}(s)\) denotes the value of the policy obtained by applying the same selection rule of Algorithm~\ref{alg:unified-seg-lsvi-ts} with initial state \(s\). Under this convention, the function
\[
s
\mapsto
V_1^*(s)-V_1^{\pi^k}(s)
\]
is fixed conditional on \(\mathcal F_{k-1}^{+}\). 

Define
\[
\Delta_k(s)
:=
V_1^*(s)-V_1^{\pi^k}(s).
\]
Since \(s_1^k\sim\rho\) is sampled independently of
\(\mathcal F_{k-1}^{+}\), we have
\[
\mathbb E
\left[
\Delta_k(s_1^k)
-
\mathbb E_{s\sim\rho}
[
\Delta_k(s)
]
\mid
\mathcal F_{k-1}^{+}
\right]
=0.
\]
Furthermore, since
$V_1^*(s),\ V_1^{\pi^k}(s)
\in
[-Hr_{\max},Hr_{\max}],$
we have
$0
\le
\Delta_k(s)
\le
2Hr_{\max},$
and therefore
\[
\left|
\Delta_k(s_1^k)
-
\mathbb E_{s\sim\rho}
[
\Delta_k(s)
]
\right|
\le
2Hr_{\max}.
\]
Thus, for any fixed \(k\in[K]\), application of the
Azuma--Hoeffding inequality gives
\[
\Pr\left(
\left|
\sum_{k'=1}^{k}
\left(
\Delta_{k'}(s_1^{k'})
-
\mathbb E_{s\sim\rho}
[
\Delta_{k'}(s)
]
\right)
\right|
>
2Hr_{\max}
\sqrt{
2k\log\left(\frac{2K}{\delta'}\right)
}
\right)
\le
\frac{\delta'}{K}.
\]
Taking a union bound over \(k\in[K]\), we obtain that with probability at
least \(1-\delta'\), for all \(k\in[K]\),
\[
\left|
\sum_{k'=1}^{k}
\left(
\left[
V_1^*(s_1^{k'})
-
V_1^{\pi^{k'}}(s_1^{k'})
\right]
-
\mathbb E_{s\sim\rho}
\left[
V_1^*(s)
-
V_1^{\pi^{k'}}(s)
\right]
\right)
\right|
\le
2Hr_{\max}
\sqrt{
2k\log\left(\frac{2K}{\delta'}\right)
}.
\]
Combining the two concentration events by a union bound yields
\[
\Pr\left[
\mathcal E^{\mathrm{Bi}}_{\mathrm{KS}}
\right]
\ge
1-2\delta' .
\]

This completes the proof.
\end{proof}

\begin{lemma}\label{lemma:elliptical_sum_bound}
Assume that $\lambda'=1$ and $\|\psi_h^k\|_2\le 1$ for all $(k,h)\in[K]\times[H]$. Then,
\begin{align*}
    \sum_{k=1}^{K}
    \sum_{h=1}^{H}
    \sqrt{
        \bigl(\psi_h^k\bigr)^\top
        \bigl(\Lambda_h^{k-1}\bigr)^{-1}
        \psi_h^k
    }
    \le
    H
    \sqrt{
        2dK
        \log\left(
            1+\frac{K}{d}
        \right)
    } .
\end{align*}
\end{lemma}

\begin{proof}This proof follows the analysis of Lemma~D.2 in (Jin et al., \citeyear{jin20_linearmdp})

Fix any $h\in[H]$. Since $\lambda'=1$, we have $\Lambda_h^0=I$. For each $k\in[K]$, define
\begin{align*}
    x_k
    :=
    \bigl(\psi_h^k\bigr)^\top
    \bigl(\Lambda_h^{k-1}\bigr)^{-1}
    \psi_h^k .
\end{align*}
Because $\Lambda_h^{k-1}\succeq I$ and $\|\psi_h^k\|_2\le 1$, we have $0\le x_k\le 1$. Hence, using the inequality
\[
    x
    \le
    2\log(1+x),
    \qquad
    \forall x\in[0,1],
\]
we obtain
\begin{align*}
    \sum_{k=1}^{K}
    x_k
    \le
    2
    \sum_{k=1}^{K}
    \log(1+x_k).
\end{align*}
By the matrix determinant lemma and the update rule
\[
    \Lambda_h^k
    =
    \Lambda_h^{k-1}
    +
    \psi_h^k
    \bigl(\psi_h^k\bigr)^\top,
\]
we have
\begin{align*}
    \det(\Lambda_h^k)
    &=
    \det(\Lambda_h^{k-1})
    \left(
        1+
        \bigl(\psi_h^k\bigr)^\top
        \bigl(\Lambda_h^{k-1}\bigr)^{-1}
        \psi_h^k
    \right)
    \\
    &=
    \det(\Lambda_h^{k-1})(1+x_k).
\end{align*}
Therefore,
\begin{align*}
    \sum_{k=1}^{K}
    \log(1+x_k)
    =
    \log
    \left[
        \frac{
            \det(\Lambda_h^K)
        }{
            \det(\Lambda_h^0)
        }
    \right].
\end{align*}
Since $\Lambda_h^0=I$, we have $\det(\Lambda_h^0)=1$. Moreover, by the AM--GM inequality applied to the eigenvalues of $\Lambda_h^K$,
\begin{align*}
    \det(\Lambda_h^K)
    \le
    \left(
        \frac{
            \operatorname{tr}(\Lambda_h^K)
        }{d}
    \right)^d .
\end{align*}
Using the definition of $\Lambda_h^K$ and $\|\psi_h^k\|_2\le 1$, we get
\begin{align*}
    \operatorname{tr}(\Lambda_h^K)
    &=
    \operatorname{tr}(I)
    +
    \sum_{k=1}^{K}
    \operatorname{tr}
    \left(
        \psi_h^k
        \bigl(\psi_h^k\bigr)^\top
    \right)
    \\
    &=
    d
    +
    \sum_{k=1}^{K}
    \|\psi_h^k\|_2^2
    \le
    d+K.
\end{align*}
Thus,
\begin{align*}
    \det(\Lambda_h^K)
    \le
    \left(
        1+\frac{K}{d}
    \right)^d .
\end{align*}
Combining the above bounds gives
\begin{align*}
    \sum_{k=1}^{K}
    \bigl(\psi_h^k\bigr)^\top
    \bigl(\Lambda_h^{k-1}\bigr)^{-1}
    \psi_h^k
    \le
    2d
    \log
    \left(
        1+\frac{K}{d}
    \right).
\end{align*}
By the Cauchy--Schwarz inequality,
\begin{align*}
    \sum_{k=1}^{K}
    \sqrt{
        \bigl(\psi_h^k\bigr)^\top
        \bigl(\Lambda_h^{k-1}\bigr)^{-1}
        \psi_h^k
    }
    &\le
    \sqrt{K}
    \sqrt{
    \sum_{k=1}^{K}
    \bigl(\psi_h^k\bigr)^\top
    \bigl(\Lambda_h^{k-1}\bigr)^{-1}
    \psi_h^k
    }
    \\
    &\le
    \sqrt{
        2dK
        \log\left(
            1+\frac{K}{d}
        \right)
    } .
\end{align*}
Finally, summing the above inequality over $h\in[H]$,we obtain
\begin{align*}
    \sum_{k=1}^{K}
    \sum_{h=1}^{H}
    \sqrt{
        \bigl(\psi_h^k\bigr)^\top
        \bigl(\Lambda_h^{k-1}\bigr)^{-1}
        \psi_h^k
    }
    &\le
    \sum_{h=1}^{H}
    \sqrt{
        2dK
        \log\left(
            1+\frac{K}{d}
        \right)
    }
=
    H
    \sqrt{
        2dK
        \log\left(
            1+\frac{K}{d}
        \right)
    } .
\end{align*}

This completes the proof.
\end{proof}

\begin{proof}[Proof of Theorem~\ref{thm:seg_binary_unknown_tran_d}] On the event
   $ \cE
    \cap
    \mathcal E_{\mathrm{TS}}
    \cap
    \mathcal E^{\mathrm{Bi}}_{\mathrm{LS}}
    \cap
    \mathcal Q^{\mathrm{Bi}}
    \cap
    \mathcal E^{\mathrm{Bi}}_{\mathrm{KS}},$
we have
\begin{align*}
    \cR(K) 
    &=
    \sum_{k=1}^K
    \left(
        V_1^*-V_1^{\pi^k}
    \right)
    \overset{\textup{(a)}}{\le}
    \sum_{k=1}^K
    \left(
        V_1^*(s_1^k)-V_1^{\pi^k}(s_1^k)
    \right)
    +
    2Hr_{\max}
    \sqrt{
        2K\log\left(\frac{2K}{\delta'}\right) 
    } 
    \\
    &\overset{\textup{(b)}}{\le}
    \sum_{k=1}^{K} \Gamma_1^k 
    +
    2Hr_{\max}
    \sqrt{
        2K\log\left(\frac{2K}{\delta'}\right) 
    } 
    \\
    &\overset{\textup{(c)}}{\le}
    \sum_{k=1}^{K}\sum_{h=1}^{H} \zeta_{h+1}^k
    +
    2\beta_{\mathrm{bi}}
    \sum_{k=1}^{K}\sum_{h=1}^{H}
    \sqrt{
        \bigl(\psi_h^k\bigr)^\top
        \bigl(\Lambda_h^{k-1}\bigr)^{-1}
        \psi_h^k
    }
    \\
    &\quad
    +
    \sum_{k=1}^{K}\sum_{i=1}^{m}
    \left\langle
        \phi^{\tau_i^k},
        \widetilde{\theta}_k^{(j^+)}-\theta^*
    \right\rangle
    +
    2Hr_{\max}
    \sqrt{
        2K\log\left(\frac{2K}{\delta'}\right) 
    } .
\end{align*}
Here, inequality \textup{(a)} follows from Eq.~\eqref{eq:Bi_some_concerntration}, inequality \textup{(b)} follows from Lemma~\ref{lemma:Bi_LSVI_TS_optim}, and inequality \textup{(c)} follows from Lemma~\ref{lemma:Bi_recursive_formula} and the identity
\[
    \sum_{h=1}^{H}\psi_h^k
    =
    \sum_{i=1}^{m}\phi^{\tau_i^k}.
\]

For the first term, by Eq.~\eqref{eq:Bi_some_concerntration}, we have
\begin{align*}
    \sum_{k=1}^{K}
    \sum_{h=1}^{H}
    \zeta_{h+1}^{k}
    \le
    4HB_{\mathrm{bi}}
    \sqrt{
        2KH\log\left(\frac{2K}{\delta'}\right)
    } .
\end{align*}
For the second term, by the Cauchy--Schwarz inequality and Lemma~\ref{lemma:elliptical_sum_bound}, we have
\begin{align*}
    \sum_{k=1}^{K}
    \sum_{h=1}^{H}
    \sqrt{
        (\psi_h^k)^\top
        (\Lambda_h^{k-1})^{-1}
        \psi_h^k
    }
    \le
    H
    \sqrt{
        2dK\log\left(1+\frac{K}{d}\right)
    }.
\end{align*}
For the third term, by the Cauchy--Schwarz inequality and Lemma~\ref{lemma:lin_sum_phi_seg_binary}, we have
\begin{align*}
    \sum_{k=1}^{K}
    \sum_{i=1}^{m}
    \left\langle
        \phi^{\tau_i^k},
        \widetilde{\theta}_k^{(j^+)}-\theta^*
    \right\rangle
    &\le
    \sum_{k=1}^{K}
    \sum_{i=1}^{m}
    \nbr{\phi^{\tau_i^k}}_{(\Sigma_{k-1}^{\mathrm{bi}})^{-1}}
    \nbr{
        \widetilde{\theta}_k^{(j^+)}
        -
        \theta^*
    }_{\Sigma_{k-1}^{\mathrm{bi}}}
    \\
    &\le
    \sum_{k=1}^{K}
    \sum_{i=1}^{m}
    \nbr{\phi^{\tau_i^k}}_{(\Sigma_{k-1}^{\mathrm{bi}})^{-1}}
    \left(
        \nbr{
            \hat{\theta}_{k-1}^{\mathrm{bi}}
            -
            \theta^*
        }_{\Sigma_{k-1}^{\mathrm{bi}}}
        +
        \sqrt{\alpha}\cdot\nu(k-1)
        \nbr{g_k^{(j^+)}}
    \right)
    \\
    &\le
    \sqrt{\alpha}\cdot\nu(K)
    \left(
        \sqrt{d}
        +
        2\sqrt{\log{\frac{Kn}{\delta'}}}
        +
        1
    \right)
\cdot
    \sqrt{
        2Km d
        \cdot
        \max\lbr{ \frac{H^2}{m \alpha \lambda}, 1}
        \cdot
        \log \sbr{
            1+ \frac{KH^2}{\alpha \lambda d m}
        }
    } .
\end{align*}
Combining the three bounds above, we have
\begin{align}
\cR(K)
&\le
2Hr_{\max}
\sqrt{
    2K\log\left(\frac{2K}{\delta'}\right) 
}
+
4HB_{\mathrm{bi}}
\sqrt{
    2KH\log\left(\frac{2K}{\delta'}\right)
}
+
2\beta_{\mathrm{bi}} H
\sqrt{
    2dK\log\left(1+\frac{K}{d}\right)
}
\notag
\\
&\quad
+
\sqrt{\alpha}\cdot\nu(K)
\left(
    \sqrt{d}
    +
    2\sqrt{\log{\frac{Kn}{\delta'}}}
    +
    1
\right)
\sqrt{
    2Km d
    \cdot
    \max\lbr{ \frac{H^2}{m \alpha \lambda}, 1}
    \cdot
    \log \sbr{
        1+ \frac{KH^2}{\alpha \lambda d m}
    }
}.
\label{Bi_regret_mid}
\end{align}
Plugging
\[
    B_{\mathrm{bi}}
    :=
    \frac{
        \nu(K)
    }{
        \sqrt{\lambda}
    }
    \left(
        \sqrt d
        +
        2\sqrt{\log\left(\frac{Kn}{\delta'}\right)}
        +
        1
    \right)
    +
    r_{\max}\sqrt{d},
\]
\[
    \beta_{\mathrm{bi}}
    =
    c'_{\beta} d H B_{\mathrm{bi}}
    \sqrt{
        \log\left(\frac{2dKH}{\delta'}\right)
    },
    \qquad
    n
    =
    \left\lceil
    2\sqrt{2\pi e}
    \log
    \left(
        \frac{K}{\delta'}
    \right)
    \right\rceil
\]
into Eq.~\eqref{Bi_regret_mid}, we obtain
\begin{align*}
\cR(K)
    &=
    \widetilde O
    \Bigg(
        \exp\sbr{\frac{Hr_{\max}}{2m}}
        \nu(K) d
        \sqrt{
            Km
            \cdot
            \max\lbr{
                \frac{H^2}{m \alpha \lambda},
                1
            }
        }
        +
        d^2 H^2 \sqrt K
        \left(
            \frac{\nu(K)}{\sqrt{\lambda}}
            +
            r_{\max}
        \right)
    \Bigg).
\end{align*}
It remains to lower bound the probability of the event. By the union bound,
\begin{align*}
&\Pr\left[
\mathcal E
\cap
\mathcal E_{\mathrm{TS}}
\cap
\mathcal E^{\mathrm{Bi}}_{\mathrm{LS}}
\cap
\mathcal Q^{\mathrm{Bi}}
\cap
\mathcal E^{\mathrm{Bi}}_{\mathrm{KS}}
\right]
\\
&\ge
1
-
\Pr(\mathcal E^c)
-
\Pr(\mathcal E_{\mathrm{TS}}^c)
-
\Pr\left[
\mathcal E
\cap
\mathcal E_{\mathrm{TS}}
\cap
(\mathcal E^{\mathrm{Bi}}_{\mathrm{LS}})^c
\right]
-
\Pr\left[
(\mathcal Q^{\mathrm{Bi}})^c
\cap
\mathcal E
\right]
-
\Pr\left[
(\mathcal E^{\mathrm{Bi}}_{\mathrm{KS}})^c
\right]
\\
&\ge 1-\delta
.
\end{align*}
\end{proof}

\subsection{Proof for the Regret Upper Bound of Sum Feedback}
In this section, we provide proofs of regret upper bound of sum feedback with unknown transitions. Apart from some differences in notation, the proof follows essentially the same arguments as those presented in Appendix~\ref{App:binary_unknown}.

Recall that, in Algorithm~$\sumseglsvits$, the sampled reward parameter is given by
\[
\widetilde{\theta}_k^{(j)}
=
\hat{\theta}_{k-1}^{\mathrm{sum}}
+
\rho_{k-1}^{\mathrm{sum}}
\bigl(\Sigma_{k-1}^{\mathrm{sum}}\bigr)^{-1/2}
g_k^{(j)}.
\]

Consider a auxiliary linear MDP $\cM^{(j)}_k$ whose reward function is
   $ r_k^{(j)}(s,a)
    =
    \langle \psi(s,a), \widetilde{\theta}_k^{(j)} \rangle$ while transition kernel is the same as that of the original MDP.
   
   Conditional on $\widetilde{\theta}_k^{(j)}$, for any policy $\pi$, recall
\begin{align*}
    Q_{h}^{\pi;(k,j)}(s,a)&:= \mathbb{E}[\sum_{t=h}^{H} r_k^{(j)}(s_t,a_t) \mid s_h=s, a_h=a, \pi, \widetilde{\theta}_k^{(j)} ] ,
\\
    V_h^{\pi;(k,j)}(s)
    &:=
    \mathbb{E}
    [
        \sum_{t=h}^{H}
        r_k^{(j)}(s_t,a_t)
        \mid
        s_h=s,\ \pi,\ \widetilde{\theta}_k^{(j)}
    ].
\end{align*}

Define event 
\begin{align}
\mathcal K_{\mathrm{Norm}} 
:=
\!\lbr{ 
\nbr{\hat{\theta}_{k}^{\mathrm{sum}} \!-\! \theta^*}_{\Sigma_{k}^{\mathrm{sum}}} 
\!\leq\! 
\sqrt{ 
d   \left(\log\sbr{1 \!+\! \frac{kH}{\lambda d}} 
\!+\! 
2\log\sbr{\frac{1}{\delta'}} \right) 
} 
\!+\! 
r_{\max}\sqrt{\lambda d} ,
\  \forall k>0 
}\! . 
\label{eq:concentration_theta_seg_weighted}
\end{align}

\begin{lemma}\label{lemma:concentration_theta_seg_weighted}
It holds that 
\begin{align*}
    \Pr[\mathcal K_{\mathrm{Norm}}] \ge 1-\delta' .
\end{align*}
\end{lemma}

\begin{proof}
    Since the sum feedback on each segment is \(\frac{H}{m}\)-sub-Gaussian conditional on the observed trajectory, and \(\|\theta^*\|_2\leq r_{\max}\sqrt{d}\), for any \(k>0\), define \(\tilde{\phi}^{\tau_i^k}=\frac{\phi^{\tau_i^k}}{\sqrt{H/m}}\) and the corresponding observation \(\tilde{R}_i^k=\frac{R_i^k}{\sqrt{H/m}}\). Then the normalized noise is \(1\)-sub-Gaussian. Applying Lemma~2 in (Abbasi-Yadkori et al., \citeyear{abbasi2011improved}) to the normalized features and observations, we can obtain this lemma.
    
    This completes the proof.
\end{proof}
Recall that 
\begin{align*}
    \mathcal{E}_{\mathrm{TS}}
    :=
    \left\{
    \forall k\in[K], j\in[n]:
    \nbr{g_k^{(j)}}_2
    \le
    \sqrt{d}
    +
    2\sqrt{\log{\frac{Kn}{\delta'}}}
    \right\},
\end{align*}
where $g_k^{(j)}\sim \mathcal N(0,I_d)$ and \[
 B_{\sumrm} := \frac{\widetilde{\beta}(K)(\sqrt{d}+2\sqrt{\log{\frac{Kn}{\delta'}}}+1)}{\sqrt{\lambda}}+r_{\max}\sqrt{d}.
\]

\begin{lemma}[Universal Bound of $\widetilde{\theta}_k^{(j)}$ under Sum Feedback]\label{lemma:universal_bound_of_TS_theta}
	Assume that event $\mathcal{E}_{\mathrm{TS}} \cap \mathcal K_{\mathrm{Norm}}$ holds. Then, we have
	\begin{align*}
		\forall k\in[K],j\in[n],\qquad \nbr{  \widetilde{\theta}_k^{(j)} }_2 \le \frac{\widetilde{\beta}(K)(\sqrt{d}+2\sqrt{\log{\frac{Kn}{\delta'}}}+1)}{\sqrt{\lambda}}+r_{\max}\sqrt{d} =B_{\sumrm}.
	\end{align*}
\end{lemma}

\begin{proof}
	By the sampling rule,
	$
	    \xi_k^{(j)}
	    =
	    \widetilde{\beta}(k-1)\bigl(\Sigma_{k-1}^{\mathrm{sum}}\bigr)^{-1/2}g_k^{(j)}.
	$
	Since $\Sigma_{k-1}^{\mathrm{sum}}\succeq \lambda I$, on event $\mathcal{E}_{\mathrm{TS}} \cap \mathcal K_{\mathrm{Norm}}$, we have
	\begin{align*}
	    \left\|
	    \xi_k^{(j)}
	    \right\|_2
	    &\le
	    \widetilde{\beta}(k-1)
	    \left\|
	    \bigl(\Sigma_{k-1}^{\mathrm{sum}}\bigr)^{-1/2}
	    \right\|_{\mathrm{op}}
	    \left\|
	    g_k^{(j)}
	    \right\|_2 \\
	    &\le
	    \frac{
	    \widetilde{\beta}(K)
	    \left(
	    \sqrt d
	    +
	    2\sqrt{\log\left(\frac{Kn}{\delta'}\right)}
	    \right)
	    }{\sqrt{\lambda}},
	    \qquad
	    \forall k\in[K],j\in[n].
	\end{align*}
    For any $(k,j)\in[K]\times[n]$, we have
    \begin{align*}
      \left\|
      \widetilde{\theta}_k^{(j)}
      \right\|_2
      &= 
      \left\|
      \widetilde{\theta}_k^{(j)} -\theta^*+\theta^*
      \right\|_2 \\
      &\le 
      \left\|
      {\xi}_k^{(j)}
      \right\|_2
      +
      \left\|
      \hat{\theta}_{k-1}^{\mathrm{sum}} -\theta^*
      \right\|_2 +
      \left\|
      \theta^*
      \right\|_2
      \\
      &\le\frac{\widetilde{\beta}(K)(\sqrt{d}+2\sqrt{\log{\frac{Kn}{\delta'}}})}{\sqrt{\lambda}}
      +\frac{\widetilde{\beta}(K)}{\sqrt{\lambda}}+
      r_{\max}\sqrt{d}
      \\&
      \le 
      \frac{\widetilde{\beta}(K)(\sqrt{d}+2\sqrt{\log{\frac{Kn}{\delta'}}}+1)}{\sqrt{\lambda}}
      +
      r_{\max}\sqrt{d},
    \end{align*}
    where the second inequality follows from the definition of the event $\mathcal{E}_{\mathrm{TS}}$, Eq.~\eqref{eq:concentration_theta_seg_weighted} and Assumption~\ref{assump:linear_mdp}.

    This completes the proof.
\end{proof}

\begin{lemma}[Bound on Weights of Value Functions of MDP $\cM_k^{(j)}$ under Sum feedback]\label{lemma:bound_Vweights_mdp_j^k}
    Assume that event $\mathcal{E}_{\mathrm{TS}} \cap \mathcal K_{\mathrm{Norm}}$ holds. 
    Under Assumption~\ref{assump:linear_mdp}, for any policy $\pi$, let 
    $\{\mathbf{w}_{h}^{\pi;(k,j)}\}_{h \in [H]}$ be the corresponding random weight vectors satisfying
    \begin{align*}
        Q_{h}^{\pi;(k,j)}(s,a)
        =
        \left\langle
            \psi(s,a),
            \mathbf{w}_{h}^{\pi;(k,j)}
        \right\rangle,
        \qquad
        \forall (s,a,h,k,j)
        \in
        \mathcal{S}\times\mathcal{A}\times[H]\times[K]\times[n].
    \end{align*}
    Then, for all $(h,k,j)\in[H]\times[K]\times[n]$, we have
    \begin{align*}
        \left\|
        \mathbf{w}_{h}^{\pi;(k,j)}
        \right\|_2
        \le
        2HB_{\mathrm{sum}}\sqrt{d}.
    \end{align*}
\end{lemma}

\begin{proof}
    By Proposition~\ref{prop:linearity_Q}, for any $(h,k,j)\in[H]\times[K]\times[n]$, we have
    \begin{align*}
        \mathbf{w}_{h}^{\pi;(k,j)}
        =
        \widetilde{\theta}_k^{(j)}
        +
        \int_{\mathcal{S}}
        V_{h+1}^{\pi;(k,j)}(s')
        d\mu_h^*(s') .
    \end{align*}
    On event $\mathcal{E}_{\mathrm{TS}} \cap \mathcal K_{\mathrm{Norm}}$, by Lemma~\ref{lemma:universal_bound_of_TS_theta} and Assumption~\ref{assump:linear_mdp}, we have
    \begin{align*}
       \left|
       \left\langle
       \psi(s,a),
       \widetilde{\theta}_k^{(j)}
       \right\rangle
       \right|
       \le
       \left\|
       \widetilde{\theta}_k^{(j)}
       \right\|_2
       \left\|
       \psi(s,a)
       \right\|_2
       \le
       B_{\mathrm{sum}} .
    \end{align*}
    Thus, the reward of MDP $\cM_k^{(j)}$ lies in $[-B_{\mathrm{sum}},B_{\mathrm{sum}}]$. Hence, for any state $s'$, the state value function satisfies
    \begin{align*}
        \left|
        V_{h+1}^{\pi;(k,j)}(s')
        \right|
        \le
        HB_{\mathrm{sum}} .
    \end{align*}
    Therefore,
    \begin{align*}
        \left\|
        \int_{\mathcal{S}}
        V_{h+1}^{\pi;(k,j)}(s')
        d\mu_h^*(s')
        \right\|_2
        \le
        HB_{\mathrm{sum}}\sqrt{d}.
    \end{align*}
    Hence,
    \begin{align*}
        \left\|
        \mathbf{w}_{h}^{\pi;(k,j)}
        \right\|_2
        &=
        \left\|
        \widetilde{\theta}_k^{(j)}
        +
        \int_{\mathcal{S}}
        V_{h+1}^{\pi;(k,j)}(s')
        d\mu_h^*(s')
        \right\|_2
        \\
        &\le
        \left\|
        \widetilde{\theta}_k^{(j)}
        \right\|_2
        +
        \left\|
        \int_{\mathcal{S}}
        V_{h+1}^{\pi;(k,j)}(s')
        d\mu_h^*(s')
        \right\|_2
        \\
        &\le
        B_{\mathrm{sum}}
        +
        HB_{\mathrm{sum}}\sqrt{d}
        \\
        &\le
        2HB_{\mathrm{sum}}\sqrt{d}.
    \end{align*}

    This completes the proof.
\end{proof}

\begin{lemma}
\label{lemma:bound_alg_weight_MDP_j^k}
    Assume that event $\mathcal{E}_{\mathrm{TS}} \cap \mathcal{K}_{\mathrm{Norm}}$ holds. 
    Then, for any $(k,h,j)\in[K]\times[H]\times[n]$, the weight vector 
    $\mathbf{w}_h^{k,(j)}$ used to construct the estimated state-action value function $Q_h^{k,(j)}$ in Algorithm~\ref{alg:unified-lsvi-ts-q} satisfies
    \begin{align*}
        \left\|
        \mathbf{w}_h^{k,(j)}
        \right\|_2
        \le
        2HB_{\mathrm{sum}}
        \sqrt{\frac{dk}{\lambda'}}.
    \end{align*}
\end{lemma}

\begin{proof}This proof follows the analysis of Lemma B.2 in (Jin et al., \citeyear{jin20_linearmdp}).
   
   On event $\mathcal{E}_{\mathrm{TS}} \cap \mathcal{K}_{\mathrm{{Norm}}}$, for any $(k,h,j)\in[K]\times[H]\times[n]$, the reward induced by $\widetilde{\theta}_k^{(j)}$ is bounded by $B_{\mathrm{sum}}$, and the estimated state value function $Q_h^{k,(j)}$ is bounded by $HB_{\mathrm{sum}}$. 
    Hence, the regression target in Algorithm~\ref{alg:unified-lsvi-ts-q} is bounded in absolute value by $2HB_{\mathrm{sum}}$.

    For any vector $\mathbf{u}\in\mathbb{R}^d$, we have
    \begin{align*}
        \left|
        \mathbf{u}^\top \mathbf{w}_h^{k,(j)}
        \right|
        &=
        \left|
        \mathbf{u}^\top
        \bigl(\Lambda_h^{k-1}\bigr)^{-1}
        \sum_{k'=1}^{k-1}
        \psi(s_h^{k'},a_h^{k'})
        \left[
        \psi(s_h^{k'},a_h^{k'})^\top
        \widetilde{\theta}_k^{(j)}
        +
        \max_{a\in\mathcal{A}}
        Q_{h+1}^{k,(j)}(s_{h+1}^{k'},a)
        \right]
        \right|
        \\
        &\le
        2HB_{\mathrm{sum}}
        \sum_{k'=1}^{k-1}
        \left|
        \mathbf{u}^\top
        \bigl(\Lambda_h^{k-1}\bigr)^{-1}
        \psi(s_h^{k'},a_h^{k'})
        \right|
        \\
        &\le
        2HB_{\mathrm{sum}}
        \sqrt{
        \left[
        \sum_{k'=1}^{k-1}
        \mathbf{u}^\top
        \bigl(\Lambda_h^{k-1}\bigr)^{-1}
        \mathbf{u}
        \right]
        \cdot
        \left[
        \sum_{k'=1}^{k-1}
        \psi(s_h^{k'},a_h^{k'})^\top
        \bigl(\Lambda_h^{k-1}\bigr)^{-1}
        \psi(s_h^{k'},a_h^{k'})
        \right]
        }
        \\
        &\le
        2HB_{\mathrm{sum}}
        \left\|
        \mathbf{u}
        \right\|_2
        \sqrt{\frac{dk}{\lambda'}}.
    \end{align*}
    The last inequality follows from $\Lambda_h^{k-1}\succeq \lambda'I$ and
    \begin{align*}
        \sum_{k'=1}^{k-1}
        \psi(s_h^{k'},a_h^{k'})^\top
        \bigl(\Lambda_h^{k-1}\bigr)^{-1}
        \psi(s_h^{k'},a_h^{k'})
        \le d .
    \end{align*}
    Therefore, by the dual characterization of the Euclidean norm,
    \begin{align*}
        \left\|
        \mathbf{w}_h^{k,(j)}
        \right\|_2
        =
        \max_{\mathbf{u}:\left\|\mathbf{u}\right\|_2=1}
        \left|
        \mathbf{u}^\top
        \mathbf{w}_h^{k,(j)}
        \right|
        \le
        2HB_{\mathrm{sum}}
        \sqrt{\frac{dk}{\lambda'}}.
    \end{align*}

    This completes the proof.
\end{proof}

Fix a constant $c_{\beta}>0$. For an absolute constant $C_{{LS}}>0$ independent of $c_\beta$.
Define the event
\begin{align}
\mathcal E_{\mathrm{LS}}
:=
\Bigg\{
&
\left\|
\sum_{k'=1}^{k-1}
\psi_h^{k'}
\left[
    V_{h+1}^{k,(j)}(s_{h+1}^{k'})
    -
    P_h V_{h+1}^{k,(j)}(s_h^{k'},a_h^{k'})
\right]
\right\|_{(\Lambda_h^{k-1})^{-1}}
\nonumber \\
&\qquad\qquad
\le
C_{{LS}} d H B_R \sqrt{\log
    \left(
        \frac{2(c_{\beta}+1)dKH}{\delta'}
    \right)},
\quad
\forall (k,h,j)\in[K]\times[H]\times[n]
\Bigg\}.
\label{eq:event_LS}
\end{align}

\begin{lemma}\label{lemma:self_normalized_V_concentration}
Let $c_{\beta}$ be the constant in the definition of $\beta_{\sumrm}$, and let $\lambda'=1$.
There exists an absolute constant $C_{{LS}}>0$, independent of
$c_{\beta}$, such that the event
$\mathcal E_{\mathrm{LS}}$ defined above satisfies
\[
    \Pr\left[
        \cK_{\mathrm{Norm}}
        \cap
        \mathcal E_{\mathrm{TS}}
        \cap
        \mathcal E_{\mathrm{LS}}^c
    \right]
    \le
    \delta' .
\]
\end{lemma}
\begin{proof}This proof follows the analysis of Lemma~B.3 in (Jin et al., \citeyear{jin20_linearmdp}).

On  event $\mathcal{E}_{\mathrm{TS}}\cap\cK_{\mathrm{Norm}}$, for all
$(k,h,j)\in[K]\times[H]\times[n]$, Lemma~\ref{lemma:bound_alg_weight_MDP_j^k} gives
\[
    \left\|\mathbf w_h^{k,(j)}\right\|_2
    \le
    2HB_{\mathrm{sum}} \sqrt{\frac{dK}{\lambda'}} .
\]
In addition, by the definition of $\Lambda_h^{k-1}$, we have
   $ \lambda_{\min}\bigl(\Lambda_h^{k-1}\bigr)\ge \lambda' .$

Hence, on event $\mathcal{E}_{\mathrm{TS}}\cap\cK_{\mathrm{Norm}}$, for all
$(k,h,j)\in[K]\times[H]\times[n]$, function
$V_{h+1}^{k,(j)}$ belongs to the function class characterized in
Lemma~\ref{lemma:cover_num_LSVI_V}.

Combining Lemmas~\ref{lemma:uniform_self_normalized} and
\ref{lemma:cover_num_LSVI_V}, we obtain that, with probability at least
$1-\delta'$, the following implication holds: when
$\mathcal E_{\mathrm{TS}}\cap\cK_{\mathrm{Norm}}$ occurs, for any fixed
$\epsilon>0$ and all $(k,h,j)\in[K]\times[H]\times[n]$,
\begin{align}
    &
    \left\|
    \sum_{k'=1}^{k-1}
    \psi_h^{k'}
    \left[
        V_{h+1}^{k,(j)}(s_{h+1}^{k'})
        -
        P_h V_{h+1}^{k,(j)}(s_h^{k'},a_h^{k'})
    \right]
    \right\|_{\bigl(\Lambda_h^{k-1}\bigr)^{-1}}^2
    \notag
    \\
    &\le
    4H^2B_{\mathrm{sum}}^2
    \left[
        \frac{d}{2}
        \log
        \left(
            \frac{k+\lambda'}{\lambda'}
        \right)
        +
        d
        \log
        \left(
            1+
            \frac{8HB_{\mathrm{sum}}\sqrt{dK}}{\epsilon\sqrt{\lambda'}}
        \right)
        +
        d^2
        \log
        \left(
            1+
            \frac{8d^{1/2}\beta_{\mathrm{sum}}^2}{\epsilon^2\lambda'}
        \right)
        +
        \log
        \left(
            \frac{H}{\delta'}
        \right)
    \right]
    +
    \frac{8k^2\epsilon^2}{\lambda'} .
    \label{eq:concentration_value_bound}
\end{align}
Now set $\lambda'=1$ and choose
\[
    \beta_{\mathrm{sum}}
    =
    c_{\beta}dHB_{\mathrm{sum}} \sqrt{\log\left(\frac{2dKH}{\delta'}\right)},
    \qquad
    \epsilon
    =
    \frac{dHB_{\mathrm{sum}}}{K}.
\]
Plugging these choices into Eq.~\eqref{eq:concentration_value_bound}, we conclude that there exists an absolute constant
$C_0>0$, independent of $c_\beta$, such that
\begin{align*}
    &
    \left\|
    \sum_{k'=1}^{k-1}
    \psi_h^{k'}
    \left[
        V_{h+1}^{k,(j)}(s_{h+1}^{k'})
        -
        P_h V_{h+1}^{k,(j)}(s_h^{k'},a_h^{k'})
    \right]
    \right\|_{\bigl(\Lambda_h^{k-1}\bigr)^{-1}}^2
    \le
    C_0 d^2H^2B_{\mathrm{sum}}^2
    \log
    \left(
        \frac{2(c_{\beta}+1)dKH}{\delta'}
    \right) .
\end{align*}
Choose the constant in the definition of $\mathcal E_{\mathrm{LS}}$ such that
$C_{LS}\ge \sqrt{C_0}$. Then
\[
    \left\|
    \sum_{k'=1}^{k-1}
    \psi_h^{k'}
    \left[
        V_{h+1}^{k,(j)}(s_{h+1}^{k'})
        -
        P_h V_{h+1}^{k,(j)}(s_h^{k'},a_h^{k'})
    \right]
    \right\|_{\bigl(\Lambda_h^{k-1}\bigr)^{-1}}
    \le
    C_{LS} dHB_{\mathrm{sum}}\sqrt{\log
    \left(
        \frac{2(c_{\beta}+1)dKH}{\delta'}
    \right)}.
\]
By the definition of $\mathcal E_{\mathrm{LS}}$, the event
    $\cK_{\mathrm{Norm}}
    \cap
    \mathcal E_{\mathrm{TS}}
    \cap
    \mathcal E_{\mathrm{LS}}^c$
can occur only if the high-probability implication above fails. Therefore,
$$
    \Pr\left[
        \cK_{\mathrm{Norm}}
        \cap
        \mathcal E_{\mathrm{TS}}
        \cap
        \mathcal E_{\mathrm{LS}}^c
    \right]
    \le
    \delta' .
$$
This completes the proof.
\end{proof}

\begin{lemma}[Upper Bound for estimator error in Original MDP]\label{lemma:upper_bound_LSVI-TS}
    Assume that event $\mathcal{E}_{\mathrm{TS}} \cap \cK_{\mathrm{Norm}}\cap\mathcal{E}_{\mathrm{LS}}$ holds.There exists a absolute constant $c_\beta$ that is the constant in definition of $\beta_{\mathrm{sum}}$.
    Then, for any policy $\pi$, for all
    $(s,a,h,k,j)\in\mathcal{S}\times\mathcal{A}\times[H]\times[K]\times[n]$, we have
    \begin{align*}
       \left|
        \left\langle
            \psi(s,a),
             \mathbf{w}_h^{k,(j)}
        \right\rangle
        -
        Q_h^{\pi}(s,a)
        -
        P_h
        \left(
            V_{h+1}^{k,(j)}
            -
            V_{h+1}^{\pi}
        \right)(s,a)
        -
        \left\langle
            \psi(s,a),
            \widetilde{\theta}_k^{(j)}-\theta^*
        \right\rangle
        \right|
        \\
        \le
        \beta_{\mathrm{sum}}
        \sqrt{
            \bigl(\psi(s,a)\bigr)^\top
            \bigl(\Lambda_h^{k-1}\bigr)^{-1}
            \psi(s,a)
        } .
    \end{align*}
\end{lemma}

\begin{proof}
    By Proposition~\ref{proposition:linear_Q_prime} and the Bellman equation, for any
    $(s,a,h)\in \mathcal{S}\times\mathcal{A}\times[H]$, the state-action value function satisfies
    \begin{align*}
        Q_h^{\pi}(s,a)
        :=
        \left\langle
            \psi(s,a),
            \mathbf{w}_h^{\pi}
        \right\rangle
        =
        \left(
            \langle\psi,\theta^*\rangle+ P_h V_{h+1}^{\pi}
        \right)(s,a).
    \end{align*}
    Therefore,
    \begin{align*}
        &\mathbf{w}_h^{k,(j)}-\mathbf{w}_h^{\pi}
        \\
        &=
        \bigl(\Lambda_h^{k-1}\bigr)^{-1}
        \sum_{k'=1}^{k-1}
        \psi_h^{k'}
        \left[
            \langle\psi_h^{k'},\widetilde{\theta}_k^{(j)} \rangle
            +
            V_{h+1}^{k,(j)}(s_{h+1}^{k'})
        \right]
        -
        \mathbf{w}_h^{\pi}
        \\
        &=
        \bigl(\Lambda_h^{k-1}\bigr)^{-1}
        \left\{
            -\lambda' \mathbf{w}_h^{\pi}
            +
            \sum_{k'=1}^{k-1}
            \psi_h^{k'}
            \left[
                V_{h+1}^{k,(j)}(s_{h+1}^{k'})
                -
                P_h V_{h+1}^{\pi}(s_h^{k'},a_h^{k'})
            \right]
            +
            \sum_{k'=1}^{k-1}
            \psi_h^{k'}
            \bigl(\psi_h^{k'}\bigr)^\top
            \bigl(\widetilde{\theta}_k^{(j)}-\theta^*\bigr)
        \right\}
        \\
        &=
        \underbrace{
        -\lambda'
        \bigl(\Lambda_h^{k-1}\bigr)^{-1}
        \mathbf{w}_h^{\pi}
        }_{q_1}
        +
        \underbrace{
        \bigl(\Lambda_h^{k-1}\bigr)^{-1}
        \sum_{k'=1}^{k-1}
        \psi_h^{k'}
        \left[
            V_{h+1}^{k,(j)}(s_{h+1}^{k'})
            -
            P_h V_{h+1}^{k,(j)}(s_h^{k'},a_h^{k'})
        \right]
        }_{q_2}
        \\
        &\quad
        +
        \underbrace{
        \bigl(\Lambda_h^{k-1}\bigr)^{-1}
        \sum_{k'=1}^{k-1}
        \psi_h^{k'}
        P_h
        \left(
            V_{h+1}^{k,(j)}
            -
            V_{h+1}^{\pi}
        \right)(s_h^{k'},a_h^{k'})
        }_{q_3}
        +
        \bigl(\widetilde{\theta}_k^{(j)}-\theta^*\bigr)
        \underbrace{
        -\lambda'
        \bigl(\Lambda_h^{k-1}\bigr)^{-1}
        \bigl(\widetilde{\theta}_k^{(j)}-\theta^*\bigr)
        }_{q_4}.
    \end{align*}

    We next control the four terms above separately. For the first term,
    \begin{align*}
        \left|
        \left\langle
            \psi(s,a),
            q_1
        \right\rangle
        \right|
        &=
        \lambda'
        \left|
        \left\langle
            \psi(s,a),
            \bigl(\Lambda_h^{k-1}\bigr)^{-1}
            \mathbf{w}_h^{\pi}
        \right\rangle
        \right|
        \\
        &\le
        \sqrt{\lambda'}
        \left\|
        \mathbf{w}_h^{\pi}
        \right\|_2
        \sqrt{
            \bigl(\psi(s,a)\bigr)^\top
            \bigl(\Lambda_h^{k-1}\bigr)^{-1}
            \psi(s,a)
        } .
    \end{align*}
    For the second term, on event $\mathcal{E}_{\mathrm{LS}}$, by Eq.~\eqref{eq:event_LS}, we have
    \begin{align*}
        \left|
        \left\langle
            \psi(s,a),
            q_2
        \right\rangle
        \right|
        \le
        C_{LS}\cdot dHB_{\mathrm{sum}}
        \sqrt{\log
    \left(
        \frac{2(c_{\beta}+1)dKH}{\delta'}
    \right)}
        \sqrt{
            \bigl(\psi(s,a)\bigr)^\top
            \bigl(\Lambda_h^{k-1}\bigr)^{-1}
            \psi(s,a)
        } .
    \end{align*}
    For the third term, using Assumption~\ref{assump:linear_mdp}, we obtain
    \begin{align*}
        \left\langle
            \psi(s,a),
            q_3
        \right\rangle
        &=
        \left\langle
            \psi(s,a),
            \bigl(\Lambda_h^{k-1}\bigr)^{-1}
            \sum_{k'=1}^{k-1}
            \psi_h^{k'}
            P_h
            \left(
                V_{h+1}^{k,(j)}
                -
                V_{h+1}^{\pi}
            \right)(s_h^{k'},a_h^{k'})
        \right\rangle
        \\
        &=
        \left\langle
            \psi(s,a),
            \bigl(\Lambda_h^{k-1}\bigr)^{-1}
            \sum_{k'=1}^{k-1}
            \psi_h^{k'}
            \bigl(\psi_h^{k'}\bigr)^\top
            \int
            \left(
                V_{h+1}^{k,(j)}
                -
                V_{h+1}^{\pi}
            \right)(s')
            d\mu_h^*(s')
        \right\rangle
        \\
        &=
        \underbrace{
        \left\langle
            \psi(s,a),
            \int
            \left(
                V_{h+1}^{k,(j)}
                -
                V_{h+1}^{\pi}
            \right)(s')
            d\mu_h^*(s')
        \right\rangle
        }_{p_1}
        \\
        &\quad
        \underbrace{
        -
        \lambda'
        \left\langle
            \psi(s,a),
            \bigl(\Lambda_h^{k-1}\bigr)^{-1}
            \int
            \left(
                V_{h+1}^{k,(j)}
                -
                V_{h+1}^{\pi}
            \right)(s')
            d\mu_h^*(s')
        \right\rangle
        }_{p_2}.
    \end{align*}
    By Assumption~\ref{assump:linear_mdp} and $B_{\mathrm{sum}}\ge r_{\max}$, we have
    \begin{align*}
        p_1
        =
        P_h
        \left(
            V_{h+1}^{k,(j)}
            -
            V_{h+1}^{\pi}
        \right)(s,a),
        \qquad
        |p_2|
        \le
        4HB_{\mathrm{sum}}\sqrt{d\lambda'}
        \sqrt{
            \bigl(\psi(s,a)\bigr)^\top
            \bigl(\Lambda_h^{k-1}\bigr)^{-1}
            \psi(s,a)
        } .
    \end{align*}
    For the fourth term, on event $\mathcal{E}_{\mathrm{TS}} \cap \cK_{\mathrm{Norm}}$, we have
    \begin{align*}
        \left|
        \left\langle
            \psi(s,a),
            q_4
        \right\rangle
        \right|
        &=
        \lambda'
        \left|
        \left\langle
            \psi(s,a),
            \bigl(\Lambda_h^{k-1}\bigr)^{-1}
            \bigl(
                \widetilde{\theta}_k^{(j)}
                -
                \theta^*
            \bigr)
        \right\rangle
        \right|
        \\
        &\le
        \lambda'
        \sqrt{
            \bigl(\psi(s,a)\bigr)^\top
            \bigl(\Lambda_h^{k-1}\bigr)^{-1}
            \psi(s,a)
        }
        \left\|
            \widetilde{\theta}_k^{(j)}
            -
            \theta^*
        \right\|_{\bigl(\Lambda_h^{k-1}\bigr)^{-1}}
        \\
        &\le
        \sqrt{\lambda'} B_{\mathrm{sum}}
        \sqrt{
            \bigl(\psi(s,a)\bigr)^\top
            \bigl(\Lambda_h^{k-1}\bigr)^{-1}
            \psi(s,a)
        } .
    \end{align*}
    Since
    \begin{align*}
        \left\langle
            \psi(s,a),
            \mathbf{w}_h^{k,(j)}
        \right\rangle
        -
        Q_h^{\pi}(s,a)
        =
        \left\langle
            \psi(s,a),
            q_1+q_2+q_3+q_4
        \right\rangle
        +
        \left\langle
            \psi(s,a),
            \widetilde{\theta}_k^{(j)}-\theta^*
        \right\rangle,
    \end{align*}
    by Lemma~\ref{lemma:bound_weights_primal_MDP} and our choice of $\lambda'$, there exists an absolute constant $c'>0$, independent of $c_\beta$, such that
    \begin{align}
        &
        \left|
        \left\langle
            \psi(s,a),
            \mathbf{w}_h^{k,(j)}
        \right\rangle
        -
        Q_h^{\pi}(s,a)
        -
        P_h
        \left(
            V_{h+1}^{k,(j)}
            -
            V_{h+1}^{\pi}
        \right)(s,a)
        -
        \left\langle
            \psi(s,a),
            \widetilde{\theta}_k^{(j)}-\theta^*
        \right\rangle
        \right|
        \notag
        \\
        &\le
        c'\cdot dHB_{\mathrm{sum}}
        \sqrt{\log
    \left(
        \frac{2(c_{\beta}+1)dKH}{\delta'}
    \right)}
        \sqrt{
            \bigl(\psi(s,a)\bigr)^\top
            \bigl(\Lambda_h^{k-1}\bigr)^{-1}
            \psi(s,a)
        } .
        \label{eq:c'_bound}
    \end{align}
    It remains to ensure that the right-hand side above is bounded by the bonus term with coefficient $\beta_{\mathrm{sum}}$.
    It is sufficient to choose $c_\beta$ such that
    \begin{align}
        c'
        \sqrt{
            \log\left(\frac{2dKH}{\delta'}\right)+\log(c_\beta+1)
        }
        \le
        c_\beta\sqrt{\log\left(\frac{2dKH}{\delta'}\right)},
        \label{eq:cbeta_choice}
    \end{align}
    Since $\log\left(\frac{2dKH}{\delta'}\right)\in[\log 2,\infty)$ and $c'$ is an absolute constant independent of $c_\beta$, we can find an absolute constant $c_\beta$ that is sufficiently large satisfying
    \[
        c'
        \sqrt{
            \log 2+\log(c_\beta+1)
        }
        \le
        c_\beta\sqrt{\log 2}.
    \]
    This choice guarantees Eq.~\eqref{eq:cbeta_choice}. Hence,
    \begin{align*}
        c'\cdot dHB_{\mathrm{sum}}
        \sqrt{\log
    \left(
        \frac{2(c_{\beta}+1)dKH}{\delta'}
    \right)}
        \le
        \beta_{\mathrm{sum}} .
    \end{align*}
    Combining this inequality with Eq.~\eqref{eq:c'_bound} proves the result.

    This completes the proof.
\end{proof}

\begin{lemma}[Upper Bound for estimator error in $\cM_k^{(j)}$]\label{lemma:upper_bound_LSVI_kj}
    Assume that event $\mathcal{E}_{\mathrm{TS}} \cap \cK_{\mathrm{Norm}}\cap\mathcal{E}_{\mathrm{LS}}$ holds.There exists a absolute constant $c_\beta$ that is the constant in definition of $\beta_{\mathrm{sum}}$.
    Then, for any policy $\pi$, for all
    $(s,a,h,k,j)\in\mathcal{S}\times\mathcal{A}\times[H]\times[K]\times[n]$, we have
    \begin{align*}
        \left|
        \left\langle
            \psi(s,a),
            \mathbf{w}_h^{k,(j)}
        \right\rangle
        -
        Q_h^{\pi;(k,j)}(s,a)
        -
        P_h
        \left(
            V_{h+1}^{k,(j)}
            -
            V_{h+1}^{\pi;(k,j)}
        \right)(s,a)
        \right|
        \\
        \le
        \beta_{\mathrm{sum}}
        \sqrt{
            \bigl(\psi(s,a)\bigr)^\top
            \bigl(\Lambda_h^{k-1}\bigr)^{-1}
            \psi(s,a)
        }.
    \end{align*}
\end{lemma}

\begin{proof}
    This proof follows the analysis of Lemma~B.4 in 
    (Jin et al., \citeyear{jin20_linearmdp}).
    
    Consider the MDP $\cM_k^{(j)}$ with reward parameter $\widetilde{\theta}_k^{(j)}$.
    By Proposition~\ref{prop:linearity_Q} and the Bellman equation, for any
    $(s,a,h)\in \mathcal{S}\times\mathcal{A}\times[H]$, the state-action value function satisfies
    \begin{align*}
        Q_h^{\pi;(k,j)}(s,a)
        :=
        \left\langle
            \psi(s,a),
            \mathbf{w}_h^{\pi;(k,j)}
        \right\rangle
        =
        \left(
            \langle\psi,\widetilde{\theta}_k^{(j)}\rangle
            +
            P_h V_{h+1}^{\pi;(k,j)}
        \right)(s,a).
    \end{align*}
    This implies
    \begin{align*}
        &\mathbf{w}_h^{k,(j)}-\mathbf{w}_h^{\pi;(k,j)}
        \\
        &=
        \bigl(\Lambda_h^{k-1}\bigr)^{-1}
        \sum_{k'=1}^{k-1}
        \psi_h^{k'}
        \left[
            \langle\psi_h^{k'},\widetilde{\theta}_k^{(j)} \rangle
            +
            V_{h+1}^{k,(j)}(s_{h+1}^{k'})
        \right]
        -
        \mathbf{w}_h^{\pi;(k,j)}
        \\
        &=
        \bigl(\Lambda_h^{k-1}\bigr)^{-1}
        \left\{
            -\lambda' \mathbf{w}_h^{\pi;(k,j)}
            +
            \sum_{k'=1}^{k-1}
            \psi_h^{k'}
            \left[
                V_{h+1}^{k,(j)}(s_{h+1}^{k'})
                -
                P_h V_{h+1}^{\pi;(k,j)}(s_h^{k'},a_h^{k'})
            \right]
        \right\}
        \\
        &=
        \underbrace{
        -\lambda'
        \bigl(\Lambda_h^{k-1}\bigr)^{-1}
        \mathbf{w}_h^{\pi;(k,j)}
        }_{q_1}
        +
        \underbrace{
        \bigl(\Lambda_h^{k-1}\bigr)^{-1}
        \sum_{k'=1}^{k-1}
        \psi_h^{k'}
        \left[
            V_{h+1}^{k,(j)}(s_{h+1}^{k'})
            -
            P_h V_{h+1}^{k,(j)}(s_h^{k'},a_h^{k'})
        \right]
        }_{q_2}
        \\
        &\quad
        +
        \underbrace{
        \bigl(\Lambda_h^{k-1}\bigr)^{-1}
        \sum_{k'=1}^{k-1}
        \psi_h^{k'}
        P_h
        \left(
            V_{h+1}^{k,(j)}
            -
            V_{h+1}^{\pi;(k,j)}
        \right)(s_h^{k'},a_h^{k'})
        }_{q_3}.
    \end{align*}

    We now bound these three terms separately. For the first term,
    \begin{align*}
        \left|
        \left\langle
            \psi(s,a),
            q_1
        \right\rangle
        \right|
        &=
        \lambda'
        \left|
        \left\langle
            \psi(s,a),
            \bigl(\Lambda_h^{k-1}\bigr)^{-1}
            \mathbf{w}_h^{\pi;(k,j)}
        \right\rangle
        \right|
        \\
        &\le
        \sqrt{\lambda'}
        \left\|
        \mathbf{w}_h^{\pi;(k,j)}
        \right\|_2
        \sqrt{
            \bigl(\psi(s,a)\bigr)^\top
            \bigl(\Lambda_h^{k-1}\bigr)^{-1}
            \psi(s,a)
        } .
    \end{align*}
    For the second term, on event $\mathcal{E}_{\mathrm{LS}}$, by Eq.~\eqref{eq:event_LS}, we have
    \begin{align*}
        \left|
        \left\langle
            \psi(s,a),
            q_2
        \right\rangle
        \right|
        \le
        C_{LS}\cdot dHB_{\mathrm{sum}}
        \sqrt{\log
    \left(
        \frac{2(c_{\beta}+1)dKH}{\delta'}
    \right)}
        \sqrt{
            \bigl(\psi(s,a)\bigr)^\top
            \bigl(\Lambda_h^{k-1}\bigr)^{-1}
            \psi(s,a)
        } .
    \end{align*}
    For the third term, by Assumption~\ref{assump:linear_mdp}, we have
    \begin{align*}
        \left\langle
            \psi(s,a),
            q_3
        \right\rangle
        &=
        \left\langle
            \psi(s,a),
            \bigl(\Lambda_h^{k-1}\bigr)^{-1}
            \sum_{k'=1}^{k-1}
            \psi_h^{k'}
            P_h
            \left(
                V_{h+1}^{k,(j)}
                -
                V_{h+1}^{\pi;(k,j)}
            \right)(s_h^{k'},a_h^{k'})
        \right\rangle
        \\
        &=
        \left\langle
            \psi(s,a),
            \bigl(\Lambda_h^{k-1}\bigr)^{-1}
            \sum_{k'=1}^{k-1}
            \psi_h^{k'}
            \bigl(\psi_h^{k'}\bigr)^\top
            \int
            \left(
                V_{h+1}^{k,(j)}
                -
                V_{h+1}^{\pi;(k,j)}
            \right)(s')
            d\mu_h^*(s')
        \right\rangle
        \\
        &=
        \underbrace{
        \left\langle
            \psi(s,a),
            \int
            \left(
                V_{h+1}^{k,(j)}
                -
                V_{h+1}^{\pi;(k,j)}
            \right)(s')
            d\mu_h^*(s')
        \right\rangle
        }_{p_1}
        \\
        &\quad
        \underbrace{
        -
        \lambda'
        \left\langle
            \psi(s,a),
            \bigl(\Lambda_h^{k-1}\bigr)^{-1}
            \int
            \left(
                V_{h+1}^{k,(j)}
                -
                V_{h+1}^{\pi;(k,j)}
            \right)(s')
            d\mu_h^*(s')
        \right\rangle
        }_{p_2}.
    \end{align*}
    By Assumption~\ref{assump:linear_mdp}, we have
    \begin{align*}
        p_1
        =
        P_h
        \left(
            V_{h+1}^{k,(j)}
            -
            V_{h+1}^{\pi;(k,j)}
        \right)(s,a),
        \qquad
        |p_2|
        \le
        4HB_{\mathrm{sum}}\sqrt{d\lambda'}
        \sqrt{
            \bigl(\psi(s,a)\bigr)^\top
            \bigl(\Lambda_h^{k-1}\bigr)^{-1}
            \psi(s,a)
        } .
    \end{align*}
    Since
    \[
        \left\langle
            \psi(s,a),
            \mathbf{w}_h^{k,(j)}
        \right\rangle
        -
        Q_h^{\pi;(k,j)}(s,a)
        =
        \left\langle
            \psi(s,a),
            q_1+q_2+q_3
        \right\rangle ,
    \]
    on event $\mathcal{E}_{\mathrm{TS}} \cap \cK_{\mathrm{Norm}}$, by Lemma~\ref{lemma:bound_Vweights_mdp_j^k} and our choice of $\lambda'$, there exists an absolute constant $c_0>0$ such that
    \begin{align}
        &
        \left|
        \left\langle
            \psi(s,a),
            \mathbf{w}_h^{k,(j)}
        \right\rangle
        -
        Q_h^{\pi;(k,j)}(s,a)
        -
        P_h
        \left(
            V_{h+1}^{k,(j)}
            -
            V_{h+1}^{\pi;(k,j)}
        \right)(s,a)
        \right|
        \notag
        \\
        &\le
        c_0\cdot dHB_{\mathrm{sum}}\sqrt{\log
    \left(
        \frac{2(c_{\beta}+1)dKH}{\delta'}
    \right)}
        \sqrt{
            \bigl(\psi(s,a)\bigr)^\top
            \bigl(\Lambda_h^{k-1}\bigr)^{-1}
            \psi(s,a)
        }\notag
        \\
        &\le
        \beta_{\mathrm{sum}}
        \sqrt{
            \bigl(\psi(s,a)\bigr)^\top
            \bigl(\Lambda_h^{k-1}\bigr)^{-1}
            \psi(s,a)
        } .
        \label{eq:c0_bound}
    \end{align}
    The last inequality follows by choosing $c_\beta$ sufficiently large so that Eq.~\eqref{eq:cbeta_choice} and Eq.~\eqref{eq:c0_bound} hold simultaneously.

    This completes the proof.
\end{proof}

Define event $\mathcal{Q}:=\lbr{ \forall k\in [K] ,\exists  j\in[n]:{\phi^{\pi^*}(s_1^k)}^\top \widetilde{\theta}_k^{(j)} \ge {\phi^{\pi^*}(s_1^k)}^\top \theta ^*}$.

\begin{lemma}\label{lemma:TS_p_optim}
Set
$    n
    =
    \left\lceil
    2\sqrt{2\pi e}
    \log\left(\frac{K}{\delta'}\right)
    \right\rceil .$
Then
\begin{align*}
    \Pr\left(\mathcal Q^c\cap \mathcal K_{\mathrm{Norm}}\right)
    \le
    \delta' .
\end{align*}
\end{lemma}

\begin{proof}
For each $k\in[K]$, recall
 $   X_k
    :=
    \phi^{\pi^*}(s_1^k)$
and $\cG_k$ denote the filtration including the history up to the end of episode $k-1$ and the initial state $s_1^k$ of episode $k$ are revealed. 
Then $X_k$, $\hat{\theta}_{k-1}^{\mathrm{sum}}$, and $\Sigma_{k-1}^{\mathrm{sum}}$ are fixed given $\cG_k$. 

For any $k>0$, define the local concentration event
\[
    \mathcal K_{\mathrm{Norm}}^{k-1}
    =
    \left\{
    \|\hat{\theta}_{k-1}^{\mathrm{sum}}-\theta^*\|_{\Sigma_{k-1}^{\mathrm{sum}}}
    \le
    \widetilde{\beta}(k-1)
    \right\}.
\]
On event $\mathcal K_{\mathrm{Norm}}^{k-1}$, by the Cauchy--Schwarz inequality,
\[
    X_k^\top(\theta^*-\hat{\theta}_{k-1}^{\mathrm{sum}})
    \le
    \widetilde{\beta}(k-1)
    \|X_k\|_{\left(\Sigma_{k-1}^{\mathrm{sum}}\right)^{-1}} .
\]

For each $j\in[n]$, conditional on $\cG_k$, we have
\[
    \widetilde{\theta}_k^{(j)}
    =
    \hat{\theta}_{k-1}^{\mathrm{sum}}+\xi_k^{(j)},
    \qquad
    \xi_k^{(j)}
    \sim
    \mathcal N
    \left(
    0,
    \bigl(\widetilde{\beta}(k-1)\bigr)^2
    \left(\Sigma_{k-1}^{\mathrm{sum}}\right)^{-1}
    \right),
\]
and the samples are conditionally independent. Therefore,
\[
    X_k^\top \xi_k^{(j)}
    \sim
    \mathcal N
    \left(
    0,
    \bigl(\widetilde{\beta}(k-1)\bigr)^2
    \|X_k\|_{\left(\Sigma_{k-1}^{\mathrm{sum}}\right)^{-1}}^2
    \right).
\]
Consequently, on $\mathcal K_{\mathrm{Norm}}^{k-1}$, following the same anti-concentration argument as in Lemma~\ref{lemma:lin_Gaussian_anti_concentration}, we obtain
\[
    \Pr\left[
    X_k^\top\widetilde{\theta}_k^{(j)}
    \ge
    X_k^\top\theta^*
    \mid
    \cG_k
    \right]
    \ge
    \Pr(Z\ge 1)
    \ge
    \frac{1}{2\sqrt{2\pi e}}
    =:
    p_0,
\]
where $Z\sim\mathcal N(0,1)$. If
$\|X_k\|_{\left(\Sigma_{k-1}^{\mathrm{sum}}\right)^{-1}}=0$, then the same lower bound holds trivially.

Define
\[
    \mathcal Q_k
    =
    \left\{
    \exists j\in[n]:
    X_k^\top\widetilde{\theta}_k^{(j)}
    \ge
    X_k^\top\theta^*
    \right\}.
\]
By the conditional independence of $\{\widetilde{\theta}_k^{(j)}\}_{j\in[n]}$ given $\cG_k$, we have
\[
    \mathbf 1_{\mathcal K_{\mathrm{Norm}}^{k-1}}
    \Pr(\mathcal Q_k^c\mid \cG_k)
    \le
    \mathbf 1_{\mathcal K_{\mathrm{Norm}}^{k-1}}
    (1-p_0)^n
    \le
    \mathbf 1_{\mathcal K_{\mathrm{Norm}}^{k-1}}
    \exp(-np_0).
\]
Since $\mathcal K_{\mathrm{Norm}}^{k-1}$ is also measurable to $\cG_k$, taking tower rule of expectation on both sides, we obtain
\[
    \Pr\left(\mathcal Q_k^c\cap \mathcal K_{\mathrm{Norm}}^{k-1}\right)
    \le
    \exp(-np_0).
\]
Moreover,
   $ \mathcal K_{\mathrm{Norm}}
    =
    \bigcap_{k>0} \mathcal K_{\mathrm{Norm}}^{k-1}
    \subseteq
    \mathcal K_{\mathrm{Norm}}^{k-1}.$
Thus,
\[
    \Pr\left(\mathcal Q_k^c\cap \mathcal K_{\mathrm{Norm}}\right)
    \le
    \exp(-np_0).
\]
Since $\mathcal Q=\bigcap_{k=1}^K\mathcal Q_k$, by the union bound,
\[
    \Pr\left(\mathcal Q^c\cap \mathcal K_{\mathrm{Norm}}\right)
    \le
    \sum_{k=1}^K
    \Pr\left(\mathcal Q_k^c\cap \mathcal K_{\mathrm{Norm}}\right)
    \le
    K\exp(-np_0).
\]
By the choice of $n$, we have
$    K\exp(-np_0)
    \le
    \delta' .$
Therefore,
\[
    \Pr\left(\mathcal Q^c\cap \mathcal K_{\mathrm{Norm}}\right)
    \le
    \delta' .
\]

This completes the proof.
\end{proof}

Let ${\pi^*}_k^{(j)}$, $Q_{h}^{{\pi^*}_k^{(j)};(k,j)}$, and 
$V_{h}^{{\pi^*}_k^{(j)};(k,j)}$ denote the optimal policy, state-action value function, and state value function of MDP $\cM_k^{(j)}$.

\begin{lemma}\label{lemma:UCB_MDP_jk}
   Assume that event $\mathcal{E}_{\mathrm{TS}} \cap \cK_{\mathrm{Norm}}\cap\mathcal{E}_{\mathrm{LS}}$ holds. Then, for all 
   $(s,a,h,k,j)\in \mathcal{S} \times \mathcal{A} \times [H] \times [K] \times [n]$, we have
\[
Q_h^{k,(j)}(s,a) 
\ge 
Q_h^{{\pi^*}_k^{(j)};(k,j)}(s,a).
\]
\end{lemma}

\begin{proof}This proof follows the analysis of Lemma B.5 in (Jin et al.,~\citeyear{jin20_linearmdp}).

We prove the claim by backward induction on $h$.
At the last step $H$, we have 
\[
V_{H+1}^{k,(j)} = V_{H+1}^{{\pi^*}_k^{(j)};(k,j)} = 0 .
\]
Applying Lemma~\ref{lemma:upper_bound_LSVI_kj} with $\pi={\pi^*}_k^{(j)}$, we obtain
\[
\left| 
\left\langle 
\psi(s,a), 
\mathbf{w}_H^{k,(j)} 
\right\rangle 
-
Q_H^{{\pi^*}_k^{(j)};(k,j)}(s,a) 
\right|
\le
\beta_{\mathrm{sum}}
\sqrt{ 
\bigl(\psi(s,a)\bigr)^\top 
\bigl(\Lambda_H^{k-1}\bigr)^{-1} 
\psi(s,a) 
}.
\]
Since $Q_H^{{\pi^*}_k^{(j)};(k,j)}(s,a) \in [-HB_{\mathrm{sum}},HB_{\mathrm{sum}}]$, it follows that
\begin{align*}
    Q_H^{{\pi^*}_k^{(j)};(k,j)}(s,a) 
&\le 
\operatorname{clip}_{[-HB_{\mathrm{sum}}, HB_{\mathrm{sum}}]}
        \left(
        \bigl(\mathbf{w}_H^{k,(j)}\bigr)^\top \psi(s,a)
        +
        \beta_{\mathrm{sum}}
        \sqrt{
        \bigl(\psi(s,a)\bigr)^\top
        \bigl(\Lambda_H^{k-1}\bigr)^{-1}
        \psi(s,a)
        }
        \right)
        \\
&=
Q_H^{k,(j)}(s,a).
\end{align*}

Now suppose that the claim holds at step $h+1$. We prove it for step $h$.
Again, applying Lemma~\ref{lemma:upper_bound_LSVI_kj} with $\pi={\pi^*}_k^{(j)}$, we have
\begin{align*}
    \left| 
\left\langle 
\psi(s,a), 
\mathbf{w}_h^{k,(j)} 
\right\rangle 
-
Q_h^{{\pi^*}_k^{(j)};(k,j)}(s,a)
-
P_h
\left(
V_{h+1}^{k,(j)}
-
V_{h+1}^{{\pi^*}_k^{(j)};(k,j)}
\right)(s,a)
\right|
\\
\le
\beta_{\mathrm{sum}}
\sqrt{
\bigl(\psi(s,a)\bigr)^\top 
\bigl(\Lambda_h^{k-1}\bigr)^{-1} 
\psi(s,a)
}.
\end{align*}

By the induction hypothesis,
\[V_{h+1}^{k,(j)}(s')
\ge
V_{h+1}^{{\pi^*}_k^{(j)};(k,j)}(s'),\]
and hence
$P_h
\left(
V_{h+1}^{k,(j)}
-
V_{h+1}^{{\pi^*}_k^{(j)};(k,j)}
\right)(s,a)
\ge
0.$

Therefore, we have
\begin{align*}
    Q_h^{{\pi^*}_k^{(j)};(k,j)}(s,a) 
&\le 
\operatorname{clip}_{[-HB_{\mathrm{sum}}, HB_{\mathrm{sum}}]}
        \left(
        \bigl(\mathbf{w}_h^{k,(j)}\bigr)^\top \psi(s,a)
        +
        \beta_{\mathrm{sum}}
        \sqrt{
        \bigl(\psi(s,a)\bigr)^\top
        \bigl(\Lambda_h^{k-1}\bigr)^{-1}
        \psi(s,a)
        }
        \right)
        \\
&=
Q_h^{k,(j)}(s,a).
\end{align*}

This completes the proof.
\end{proof}

\begin{lemma}[Optimism]\label{lemma:LSVI_TS_optim}
Assume that event $\mathcal{E}_{\mathrm{TS}} \cap \cK_{\mathrm{{Norm}}}\cap\mathcal{E}_{\mathrm{LS}}\cap \mathcal{Q} $ holds. For any episode \(k\) and realized  initial state $s_1^k$,  let $j^+$ be selected by Algorithm~\ref{alg:unified-seg-lsvi-ts}  such that
    $j^+
    =
    \arg\max_{j\in[n]}
    V_1^{k,(j)}(s_1^k).$
Then,  we have
\begin{align*}
    V_1^{k,(j^+)}(s_1^k)
    \ge
    V_1^*(s_1^k).
\end{align*}
\end{lemma}
\begin{proof}
    On event $\mathcal{E}_{\mathrm{TS}} \cap \cK_{\mathrm{Norm}}\cap\mathcal{E}_{\mathrm{LS}}$, by Lemma~\ref{lemma:UCB_MDP_jk}, we have
    \[
        Q_h^{k,(j)}(s,a)
        \ge
        Q_h^{{\pi^*}_k^{(j)};(k,j)}(s,a),
        \qquad
        \forall (s,a,h,k,j)
        \in
        \mathcal{S}\times\mathcal{A}\times[H]\times[K]\times[n].
    \]
    Given $s_1^k$, it follows that
    \begin{align*}
        V_1^{k,(j)}(s^k_1)
        &=
        \max_a Q_1^{k,(j)}(s^k_1,a)
        \ge
        Q_1^{k,(j)}
        \left(
            s_1^k,
            {\pi^*}_k^{(j)}(s_1^k)
        \right)
\ge
        V_1^{{\pi^*}_k^{(j)};(k,j)}(s_1^k)
        \ge
        V_1^{\pi^*;(k,j)}(s_1^k),
    \end{align*}
    where the last inequality follows from the definition of the optimal state value function
    $V_1^{{\pi^*}_k^{(j)};(k,j)}$ in MDP $\cM_k^{(j)}$, and $\pi^*$ denotes the optimal policy of the original MDP.

    Moreover, under the same transition kernel, we can write
    $V_1^{\pi^*;(k,j)}(s_1^k)$ and $V_1^*(s_1^k)$ as
    \[
        V_1^{\pi^*;(k,j)}(s_1^k)
        =
        \left(\phi^{\pi^*}(s_1^k)\right)^\top
        \widetilde{\theta}_k^{(j)},
        \qquad
        V_1^*(s_1^k)
        =
        \left(\phi^{\pi^*}(s_1^k)\right)^\top
        \theta^* .
    \]
    On event $\mathcal Q$, for the given episode $k$, there exists $j'\in[n]$ such that
    \begin{align*}
        \left(\phi^{\pi^*}(s_1^k)\right)^\top
        \widetilde{\theta}_k^{(j')}
        \ge
        \left(\phi^{\pi^*}(s_1^k)\right)^\top
        \theta^* .
    \end{align*}
    Hence,
    \begin{align*}
        V_1^{\pi^*;(k,j')}(s_1^k)
        \ge
        V_1^*(s_1^k).
    \end{align*}

    By the definition of $j^+$ in Algorithm~\ref{alg:unified-seg-lsvi-ts}, we have
    \begin{align*}
        V_1^{k,(j^+)}(s^k_1)
        \ge
        V_1^{k,(j')}(s^k_1)
        \ge
        V_1^{\pi^*;(k,j')}(s_1^k)
        \ge
        V_1^*(s_1^k).
    \end{align*}
    This proves the optimism.
\end{proof}

\begin{lemma}\label{lemma:recursive_formula}
    Assume that event 
    $\mathcal{E}_{\mathrm{TS}} \cap \cK_{\mathrm{Norm}}\cap\mathcal{E}_{\mathrm{LS}}$ holds.
    For any episode $k$, let $j^+$ and $\pi^k$ be generated by 
    Algorithm~\ref{alg:unified-seg-lsvi-ts}. Define
    \[
        \Gamma_h^k
        :=
        V_h^{k,(j^+)}(s_h^k)
        -
        V_h^{\pi^k}(s_h^k),
    \]
    and
    \[
        \zeta_{h+1}^k
        :=
        \mathbb{E}
        \left[
            \Gamma_{h+1}^k
            \mid
            s_h^k,a_h^k
        \right]
        -
        \Gamma_{h+1}^k .
    \]
    Then, for any $(k,h)\in[K]\times[H]$, we have
    \begin{align*}
        \Gamma_h^k
        \le
        \Gamma_{h+1}^k
        +
        \zeta_{h+1}^k
        +
        2\beta_{\mathrm{sum}}
        \sqrt{
            \bigl(\psi_h^k\bigr)^\top
            \bigl(\Lambda_h^{k-1}\bigr)^{-1}
            \psi_h^k
        } 
        +
        \left\langle
            \psi_h^k,
            \widetilde{\theta}_k^{(j^+)}-\theta^*
        \right\rangle .
    \end{align*}
\end{lemma}

\begin{proof}This proof follows the analysis of Lemma B.6 in (Jin et al., \citeyear{jin20_linearmdp}).
    
    On event $\mathcal{E}_{\mathrm{TS}} \cap \cK_{\mathrm{Norm}}\cap\mathcal{E}_{\mathrm{LS}}$, by Lemma~\ref{lemma:UCB_MDP_jk} and its proof, the un-clipped optimistic estimate satisfies
    \[
        \left\langle
            \psi(s,a),
            \mathbf{w}_h^{k,(j)}
        \right\rangle
        +
        \beta_{\mathrm{sum}}
        \sqrt{
            \bigl(\psi(s,a)\bigr)^\top
            \bigl(\Lambda_h^{k-1}\bigr)^{-1}
            \psi(s,a)
        }
        \ge
        Q_h^{{\pi^*}_k^{(j)};(k,j)}(s,a)
        \ge
        -HB_{\mathrm{sum}} .
    \]
    Hence, the lower clipping operation is inactive. Therefore, using Lemma~\ref{lemma:upper_bound_LSVI-TS}, for any
    $(s,a,h,k)\in\mathcal{S}\times\mathcal{A}\times[H]\times[K]$, we have
    \begin{align*}
        Q_h^{k,(j^+)}(s,a)
        -
        Q_h^{\pi^k}(s,a)
        &\le
        P_h
        \left(
            V_{h+1}^{k,(j^+)}
            -
            V_{h+1}^{\pi^k}
        \right)(s,a)
        +
        2\beta_{\mathrm{sum}}
        \sqrt{
            \bigl(\psi(s,a)\bigr)^\top
            \bigl(\Lambda_h^{k-1}\bigr)^{-1}
            \psi(s,a)
        }
        \\
        &+
        \left\langle
            \psi(s,a),
            \widetilde{\theta}_k^{(j^+)}-\theta^*
        \right\rangle .
    \end{align*}
    By the action-selection rule in Algorithm~\ref{alg:unified-seg-lsvi-ts} and the definition of the state value function under $\pi^k$, we have
    \begin{align*}
        \Gamma_h^k
        &=
        V_h^{k,(j^+)}(s_h^k)
        -
        V_h^{\pi^k}(s_h^k)
    =
        Q_h^{k,(j^+)}(s_h^k,a_h^k)
        -
        Q_h^{\pi^k}(s_h^k,a_h^k).
    \end{align*}
    Substituting $(s,a)=(s_h^k,a_h^k)$ into the previous inequality, we have
    \begin{align*}
        \Gamma_h^k
        &\le
        P_h
        \left(
            V_{h+1}^{k,(j^+)}
            -
            V_{h+1}^{\pi^k}
        \right)(s_h^k,a_h^k)
        +
        2\beta_{\mathrm{sum}}
        \sqrt{
            \bigl(\psi_h^k\bigr)^\top
            \bigl(\Lambda_h^{k-1}\bigr)^{-1}
            \psi_h^k
        }
        +
        \left\langle
            \psi_h^k,
            \widetilde{\theta}_k^{(j^+)}-\theta^*
        \right\rangle
        \\
        &=
        \Gamma_{h+1}^k
        +
        \zeta_{h+1}^k
        +
        2\beta_{\mathrm{sum}}
        \sqrt{
            \bigl(\psi_h^k\bigr)^\top
            \bigl(\Lambda_h^{k-1}\bigr)^{-1}
            \psi_h^k
        }
        +
        \left\langle
            \psi_h^k,
            \widetilde{\theta}_k^{(j^+)}-\theta^*
        \right\rangle .
    \end{align*}

    This completes the proof.
\end{proof}

Recall that \(F_{k-1}^{+}\) denotes the filtration including the history up to the end of episode $k-1$ and the Gaussian noises \(\{\xi_k^{(j)}\}_{j=1}^n\) for episode \(k\) are revealed, but before the initial state \(s_1^k\) is observed. For \(h\in[H]\) and \(F_{k,h}\) denotes the filtration including \(F_{k-1}^{+}\), the trajectory information up to step \(h\) are revealed.

% We use the notation of Lemma~\ref{lemma:recursive_formula}. 

We use the notation of Lemma~\ref{lemma:recursive_formula}, define the event
\begin{align}
\mathcal{E}_{\mathrm{KS}}
:=
\Bigg\{
&\left|
\sum_{k'=1}^{k}
\sum_{h=1}^{H}
\zeta_{h+1}^{k'}
\right|
\le
4HB_{\mathrm{sum}}
\sqrt{
2kH\log\left(\frac{2K}{\delta'}\right)
},
\nonumber\\
&\left|
\sum_{k'=1}^{k}
\left(
\left[
V_1^*(s_1^{k'})
-
V_1^{\pi^{k'}}(s_1^{k'})
\right]
-
\mathbb E_{s\sim\rho}
\left[
V_1^*(s)
-
V_1^{\pi^{k'}}(s)
\right]
\right)
\right|
\nonumber\\
&
\le
2Hr_{\max}
\sqrt{
2k\log\left(\frac{2K}{\delta'}\right)
},
\quad
\forall k\in[K].
\Bigg\}. 
\label{eq:some_concerntration}
\end{align}

\begin{lemma}\label{lemma:event_KS}
It holds that
\begin{align*}
    \Pr\left[ \mathcal{E}_{\mathrm{KS}} \right] \ge 1-2\delta' .
\end{align*}
\end{lemma}

\begin{proof}
We prove the two inequalities in the definition of
$\mathcal E_{\mathrm{KS}}$ separately.

Conditional on \(F_{k,h}\), the functions
\(V_{h+1}^{k,(j^+)}\) and \(V_{h+1}^{\pi^k}\) are fixed. Hence, by the definition of
\(\Gamma_{h+1}^k\) and \(\zeta_{h+1}^k\), we have
\[
\mathbb E
\left[
    \zeta_{h+1}^k
    \mid
    F_{k,h}
\right]
=0.
\]
Moreover, since both state value functions are bounded in
\([-HB_{\mathrm{sum}},HB_{\mathrm{sum}}]\), we have
\[
|\zeta_{h+1}^k|
\le
4HB_{\mathrm{sum}}.
\]
Therefore, for any fixed \(k\in[K]\), by the Azuma--Hoeffding inequality,
\[
\Pr\left(
\left|
\sum_{k'=1}^{k}
\sum_{h=1}^{H}
\zeta_{h+1}^{k'}
\right|
>
4HB_{\mathrm{sum}}
\sqrt{
2kH\log\left(\frac{2K}{\delta'}\right)
}
\right)
\le
\frac{\delta'}{K}.
\]
Taking a union bound over \(k\in[K]\), we obtain that with probability at
least \(1-\delta'\), for all \(k\in[K]\),
\[
\left|
\sum_{k'=1}^{k}
\sum_{h=1}^{H}
\zeta_{h+1}^{k'}
\right|
\le
4HB_{\mathrm{sum}}
\sqrt{
2kH\log\left(\frac{2K}{\delta'}\right)
}.
\]

Next, we construct the concentration inequality for the randomness of the
initial states. We clarify the meaning of
\[
\mathbb E_{s\sim\rho}
\left[
V_1^*(s)-V_1^{\pi^k}(s)
\right].
\]

Conditional on \(F_{k-1}^{+}\), the history and the Gaussian perturbations are fixed. For an initial state \(s\), \(V_1^{\pi^k}(s)\) denotes the value of the policy obtained by applying the same selection rule of Algorithm~\ref{alg:unified-seg-lsvi-ts} with initial state \(s\). Under this convention, the function
\[
s
\mapsto
V_1^*(s)-V_1^{\pi^k}(s)
\]
is fixed conditional on \(F_{k-1}^{+}\). 

Recall
$\Delta_k(s)
:=
V_1^*(s)-V_1^{\pi^k}(s).$
Since \(s_1^k\sim\rho\) is sampled independently of
\(F_{k-1}^{+}\), we have
\[
\mathbb E
\left[
\Delta_k(s_1^k)
-
\mathbb E_{s\sim\rho}
[
\Delta_k(s)
]
\mid
F_{k-1}^{+}
\right]
=0.
\]
Furthermore, since
$V_1^*(s),\ V_1^{\pi^k}(s)
\in
[-Hr_{\max},Hr_{\max}],$
we have
$0
\le
\Delta_k(s)
\le
2Hr_{\max},$
and therefore
\[
\left|
\Delta_k(s_1^k)
-
\mathbb E_{s\sim\rho}
[
\Delta_k(s)
]
\right|
\le
2Hr_{\max}.
\]
Thus, for any fixed \(k\in[K]\), application of the
Azuma--Hoeffding inequality gives
\[
\Pr\left(
\left|
\sum_{k'=1}^{k}
\left(
\Delta_{k'}(s_1^{k'})
-
\mathbb E_{s\sim\rho}
[
\Delta_{k'}(s)
]
\right)
\right|
>
2Hr_{\max}
\sqrt{
2k\log\left(\frac{2K}{\delta'}\right)
}
\right)
\le
\frac{\delta'}{K}.
\]
Taking a union bound over \(k\in[K]\), we obtain that with probability at
least \(1-\delta'\), for all \(k\in[K]\),
\[
\left|
\sum_{k'=1}^{k}
\left(
\left[
V_1^*(s_1^{k'})
-
V_1^{\pi^{k'}}(s_1^{k'})
\right]
-
\mathbb E_{s\sim\rho}
\left[
V_1^*(s)
-
V_1^{\pi^{k'}}(s)
\right]
\right)
\right|
\le
2Hr_{\max}
\sqrt{
2k\log\left(\frac{2K}{\delta'}\right)
}.
\]
Combining the two concentration events by a union bound yields
\[
\Pr\left[
\mathcal E_{\mathrm{KS}}
\right]
\ge
1-2\delta' .
\]

This completes the proof.
\end{proof}

\begin{lemma}
    \label{lemma:seg_weighted_elliptical_potential}
    For any $K \ge 1$ , we have 
   \begin{align*}
       \sum_{k=1}^{K} \sum_{i=1}^{m} \nbr{ \phi^{\tau^k_i}}_{\left(\Sigma_{k-1}^{\mathrm{sum}}\right)^{-1}}\le \sqrt{ 2KH d \cdot \max\lbr{ \frac{H}{\lambda}, 1} \cdot \log \sbr{ 1+ \frac{  KH }{ \lambda d  } } }
   \end{align*}
\end{lemma}

\begin{proof}
For any $K\ge 1$,we have 
\begin{align*}
    \sum_{i=1}^{m} \nbr{\sqrt{\frac{m}{H}}  \phi^{\tau^k_i}}_{\left(\Sigma_{k-1}^{\mathrm{sum}}\right)^{-1}}^2 
		&\leq m \cdot \frac{H}{m} \cdot \frac{1}{\lambda} \le \frac{H}{\lambda}
\end{align*}
We have 
     \begin{align*}
        \sum_{k=1}^{K} \sum_{i=1}^{m} \nbr{ \phi^{\tau^k_i}}_{\left(\Sigma_{k-1}^{\mathrm{sum}}\right)^{-1}}&=\sum_{k=1}^{K} \sum_{i=1}^{m} \sqrt{\frac{H}{m}}\nbr{ \sqrt{\frac{m}{H}}\phi^{\tau^k_i}}_{\left(\Sigma_{k-1}^{\mathrm{sum}}\right)^{-1}}  
        \\&\leq \sqrt{ KH   \sum_{k=1}^{K} \sum_{i=1}^{m} \nbr{\sqrt{\frac{m}{H}} \phi^{\tau^k_i}}_{\left(\Sigma_{k-1}^{\mathrm{sum}}\right)^{-1}}^2 }
		\nonumber\\
		&{\leq} \sqrt{ 2KH \cdot \max\lbr{ \frac{H}{  \lambda}, 1} \cdot \sum_{k=1}^{K} \log \sbr{ 1+ \sum_{i=1}^{m} \nbr{ \sqrt{\frac{m}{H}}\phi^{\tau^k_i}}_{\left(\Sigma_{k-1}^{\mathrm{sum}}\right)^{-1}}^2} }
		\nonumber\\ 
		&\leq \sqrt{ 2KH d \cdot \max\lbr{ \frac{H}{\lambda}, 1} \cdot \log \sbr{ 1+ \frac{  KH }{ \lambda d  } } } 
    \end{align*}
\end{proof}

\begin{proof}[Proof of Theorem~\ref{thm:seg_sum_unknown_tran_d}]
On the event 
$\mathcal{E}_{\mathrm{TS}} \cap \cK_{\mathrm{Norm}}\cap\mathcal{E}_{\mathrm{LS}} \cap \mathcal{Q}\cap \mathcal{E}_{\mathrm{KS}}$, we have
\begin{align*}
    \cR(K)
    &=
    \sum_{k=1}^K
    \left(
    V_1^*-V_1^{\pi^k} 
    \right)
    \overset{\textup{(a)}}{\le}
    \sum_{k=1}^K
    \left(
    V_1^*(s_1^k)-V_1^{\pi^k}(s_1^k)
    \right)
    +
    2Hr_{\max}
    \sqrt{
    2K\log\left(\frac{2K}{\delta'}\right) 
    }
    \\
    &\overset{\textup{(b)}}{\le}
    \sum_{k=1}^{K} \Gamma_1^k 
    +
    2Hr_{\max}
    \sqrt{
    2K\log\left(\frac{2K}{\delta'}\right) 
    }
    \\
    &\overset{\textup{(c)}}{\le}
    \sum_{k=1}^{K}\sum_{h=1}^{H} \zeta_{h+1}^k
    +
    2\beta_{\mathrm{sum}}
    \sum_{k=1}^{K}\sum_{h=1}^{H}
    \sqrt{
        \bigl(\psi_h^k\bigr)^\top
        \bigl(\Lambda_h^{k-1}\bigr)^{-1}
        \psi_h^k
    }
    \\
    &\quad
    +
    \sum_{k=1}^{K}\sum_{i=1}^{m}
    \left\langle
        \phi^{\tau_i^k},
        \widetilde{\theta}_k^{(j^+)}-\theta^*
    \right\rangle
    +
    2Hr_{\max}
    \sqrt{
    2K\log\left(\frac{2K}{\delta'}\right) 
    } .
\end{align*}
Here, inequality \textup{(a)} follows from Eq.~\eqref{eq:some_concerntration}, inequality \textup{(b)} follows from Lemma~\ref{lemma:LSVI_TS_optim}, and inequality \textup{(c)} follows from Lemma~\ref{lemma:recursive_formula} and the identity 
$\sum_{h=1}^{H}\psi_h^k=\sum_{i=1}^{m}\phi^{\tau_i^k}$.

For the first term, by Eq.~\eqref{eq:some_concerntration}, we have
\begin{align*}
    \sum_{k=1}^{K}
    \sum_{h=1}^{H}
    \zeta_{h+1}^{k}
    \le
    4HB_{\mathrm{sum}}
    \sqrt{
    2KH\log\left(\frac{2K}{\delta'}\right)
    } .
\end{align*}
For the second term, by the Cauchy--Schwarz inequality and Lemma~\ref{lemma:elliptical_sum_bound}, we have
\begin{align*}
    \sum_{k=1}^{K}\sum_{h=1}^{H}
    \sqrt{
    \bigl(\psi_h^k\bigr)^\top
    \bigl(\Lambda_h^{k-1}\bigr)^{-1}
    \psi_h^k
    }
    \le
    H\sqrt{
    2dK\log\left(1+\frac{K}{d}\right)
    } .
\end{align*}
For the third term, by the Cauchy--Schwarz inequality and Lemma~\ref{lemma:seg_weighted_elliptical_potential}, we obtain
\begin{align*}
    \sum_{k=1}^{K}\sum_{i=1}^{m}
    \left\langle
        \phi^{\tau_i^k},
        \widetilde{\theta}_k^{(j^+)}-\theta^*
    \right\rangle
    &\le
    \sum_{k=1}^{K}\sum_{i=1}^{m}
    \nbr{\phi^{\tau_i^k}}_{\left(\Sigma_{k-1}^{\mathrm{sum}}\right)^{-1}}
    \nbr{
        \widetilde{\theta}_k^{(j^+)}
        -
        \theta^*
    }_{\Sigma_{k-1}^{\mathrm{sum}}}
    \\
    &\le
    \sum_{k=1}^{K}\sum_{i=1}^{m}
    \nbr{\phi^{\tau_i^k}}_{\left(\Sigma_{k-1}^{\mathrm{sum}}\right)^{-1}}
    \left(
    \nbr{
        \hat{\theta}_{k-1}^{\mathrm{sum}}
        -
        \theta^*
    }_{\Sigma_{k-1}^{\mathrm{sum}}}
    +
    \widetilde{\beta}(k-1)
    \nbr{g_k^{(j^+)}}_2
    \right)
    \\
    &\le
    \widetilde{\beta}(K)
    \left(
        \sqrt{d}
        +
        2\sqrt{\log{\frac{Kn}{\delta'}}}
        +
        1
    \right)
    \sqrt{
        2KHd
        \cdot
        \max\lbr{ \frac{H}{\lambda}, 1}
        \cdot
        \log \sbr{
        1+ \frac{KH}{\lambda d}
        }
    } .
\end{align*}
Combining the three bounds above, we have
\begin{align}
\cR(K)
&\le
2Hr_{\max}
\sqrt{
2K\log\left(\frac{2K}{\delta'}\right) 
}
+
4HB_{\mathrm{sum}}
\sqrt{
2KH\log\left(\frac{2K}{\delta'}\right)
}
+
2\beta_{\mathrm{sum}} H
\sqrt{
2dK\log\left(1+\frac{K}{d}\right)
}
\notag
\\
&\quad
+
\widetilde{\beta}(K)
\left(
    \sqrt{d}
    +
    2\sqrt{\log{\frac{Kn}{\delta'}}}
    +
    1
\right)
\sqrt{
    2KHd
    \cdot
    \max\lbr{ \frac{H}{\lambda}, 1}
    \cdot
    \log \sbr{
    1+ \frac{KH}{\lambda d}
    }
}.
\label{regret_mid}
\end{align}
Plugging
\[
    \widetilde{\beta}(k)
    :=
    \sqrt{
        d  \left(\log\left(1+\frac{kH}{\lambda d}\right)
        +
        2\log\left(\frac{1}{\delta'}\right) \right)
    }
    +
    r_{\max} \sqrt{\lambda d},
\]
\[
    B_{\mathrm{sum}}
    :=
    \frac{
        \widetilde{\beta}(K)
        \left(
            \sqrt{d}
            +
            2\sqrt{\log{\frac{Kn}{\delta'}}}
            +
            1
        \right)
    }{\sqrt{\lambda}}
    +
    r_{\max}\sqrt{d},
    \qquad
    \beta_{\mathrm{sum}}
    =
    c_{\beta} d H B_{\mathrm{sum}}
    \sqrt{
        \log\left(\frac{2dKH}{\delta'}\right)
    },
\]
into Eq.~\eqref{regret_mid}, we obtain
\begin{align*}
\cR(K)
    &=
    \widetilde O\Bigg(
d^{5/2}H^2\sqrt K
\left(
    r_{\max}
    +
    \frac{1}{\sqrt{\lambda}}
\right)
+
d
\left(
    1+r_{\max}\sqrt{\lambda}
\right)
\sqrt{
    KHd
    \max\left\{
        \frac{H}{\lambda},
        1
    \right\}
}
\Bigg).
\end{align*}

It remains to lower bound the probability of the event on which the above regret bound holds. We have
\begin{align*}
   &\Pr\left[
    \mathcal{E}_{\mathrm{TS}}
    \cap
    \cK_{\mathrm{Norm}}
    \cap
    \mathcal{E}_{\mathrm{LS}}
    \cap
    \mathcal{Q}
    \cap
    \mathcal{E}_{\mathrm{KS}}
    \right]
    \\
    &=
    1-
    \Pr\left[
    \left(
    \mathcal{E}_{\mathrm{TS}}
    \cap
    \cK_{\mathrm{Norm}}
    \cap
    \mathcal{E}_{\mathrm{LS}}
    \cap
    \mathcal{Q}
    \cap
    \mathcal{E}_{\mathrm{KS}}
    \right)^c
    \right]
    \\
    &=
    1-
    \Pr\left[
    \mathcal{E}_{\mathrm{TS}}^c
    \cup
    \cK_{\mathrm{Norm}}^c
    \cup
    \left(
        \mathcal{E}_{\mathrm{TS}}
        \cap
        \cK_{\mathrm{Norm}}
        \cap
        \mathcal{E}_{\mathrm{LS}}^c
    \right)
    \cup
    \left(
        \cK_{\mathrm{Norm}}
        \cap
        \mathcal Q^c
    \right)
    \cup
    \mathcal{E}_{\mathrm{KS}}^c
    \right]
    \\
    &\ge
    1-
    \Pr(\mathcal{E}_{\mathrm{TS}}^c)
    -
    \Pr(\cK_{\mathrm{Norm}}^c)
    -
    \Pr\left(
        \mathcal{E}_{\mathrm{TS}}
        \cap
        \cK_{\mathrm{Norm}}
        \cap
        \mathcal{E}_{\mathrm{LS}}^c
    \right)
    -
    \Pr\left(
        \cK_{\mathrm{Norm}}
        \cap
        \mathcal Q^c
    \right)
    -
    \Pr(\mathcal{E}_{\mathrm{KS}}^c)
    \\
    &\ge
    1-
    \delta'
    -
    \delta'
    -
    \delta'
    -
    \delta'
    -
    2\delta'
    \\
    &=
    1-6\delta' .
\end{align*}
Setting $\delta'=\delta/6$, the regret bound holds with probability at least $1-\delta$.

This completes the proof.
\end{proof}

\section{Proof for Variable Binary Feedback with Known Transitions}\label{secE}

In this section, we prove the regret upper bound of variable binary feedback with Known transitions. Broadly, the proof follows the similar argument as in Appendix~\ref{app:upper_binary_known}, with several technical details adapted to accommodate different segmentation schemes.

For any $k>0$ and $\theta \in \Theta$, define
\begin{align}
	Z^U_k &:= \sum_{k'=1}^{k} \sum_{i=1}^{m_{k'}} \varepsilon_{k',i} \cdot \phi^{\tau^{k'}_i} , \nonumber
	\\
	g^U_{k}(\theta) &:= \sum_{k'=1}^{k} \sum_{i=1}^{m_{k'}} \sig( (\phi^{\tau^{k'}_i})^\top\theta ) \cdot \phi^{\tau^{k'}_i} + \lambda \theta , \label{eq:uq_def_g_k}
	\\
	\Lambda^U_{k}(\theta) &:=  \sum_{k'=1}^{k} \sum_{i=1}^{m_{k'}} \sig'( (\phi^{\tau^{k'}_i})^\top\theta ) \cdot \phi^{\tau^{k'}_i} (\phi^{\tau^{k'}_i})^\top + \lambda I . \label{eq:uq_def_Lambda_k}
\end{align}

\begin{lemma} \label{lemma:uq_Lambda_k}
	For any $k>0$ and $\theta \in \Theta$, we have
	\begin{align*}
		\det(\Lambda^U_k(\theta)) \leq \sbr{ \frac{\sig'_{\max}\sum_{k'=1}^{k}\sum_{i=1}^{m_{k'}}  \sbr{\ell^{k'}_{i}}^2}{d} + \lambda  }^{d}.
	\end{align*}
\end{lemma}
\begin{proof}
	For any $k>0$, by AM-GM equality, we have
	\begin{align*}
		\det(\Lambda^U_k(\theta)) &\leq \sbr{ \frac{\trace(\Lambda^U_k(\theta))}{d} }^{d}
		\\
		&\leq \sbr{ \frac{1}{d} \cdot \sbr{ \sig'_{\max} \cdot\sum_{k'=1}^{k} \sum_{i=1}^{m_{k'}}  \sbr{\ell^{k'}_{i}}^2 + \lambda d } }^{d}
		\\
		&= \sbr{ \frac{\sig'_{\max}\sum_{k'=1}^{k}\sum_{i=1}^{m_{k'}}  \sbr{\ell^{k'}_{i}}^2}{d} + \lambda  }^{d} .
	\end{align*}
\end{proof}

After executing episode $k$, the whole trajectory
$\tau^k=(s_1^k,a_1^k,\dots,s_H^k,a_H^k)$
is observed. The algorithm then chooses the number of segments $m_k$ and the partition
$0=t_0^k<t_1^k<\cdots<t_{m_k}^k=H.$

Let $\mathcal H_k$ denote the filtration containing all information available after the trajectory of episode $k$ has been observed and its segmentation has been determined, but before any segment-level feedback is revealed.

For any \(k>0\) and \(i\in[m_k]\), let
\(\varepsilon_{k,i}:=y_i^k-\sig((\phi^{\tau_i^k})^\top\theta^*)\)
denote the noise of the binary feedback. and $v_{k,i}^2:=\mathbb E[\varepsilon_{k,i}^2\mid\mathcal H_{k}]=
\sig((\phi^{\tau_i^k})^\top\theta^*)
\bigl(1-\sig((\phi^{\tau_i^k})^\top\theta^*)\bigr)
=
\sig'((\phi^{\tau_i^k})^\top\theta^*).$
Then, we have
\begin{align*}
	\Lambda^U_{k}(\theta^*) &:= \sum_{k'=1}^{k} \sum_{i=1}^{m_{k'}} \sig'( (\phi^{\tau^{k'}_i})^\top \theta^* ) \cdot \phi^{\tau^{k'}_i} (\phi^{\tau^{k'}_i})^\top + \lambda I 
	\\
	&= \sum_{k'=1}^{k} \sum_{i=1}^{m_{k'}} v_{k',i}^2 \cdot \phi^{\tau^{k'}_i} (\phi^{\tau^{k'}_i})^\top + \lambda I .
\end{align*}

\begin{lemma}[Concentration of Noises under Binary Feedback with Unequal Segment]\label{lemma:uq_noise_seq_martingale}
	With probability at least $1-\delta'$, for any $k>0$,
	\begin{align*}
		\nbr{\sum_{k'=1}^{k} \sum_{i=1}^{m_{k'}} \varepsilon_{k',i} \cdot \phi^{\tau^{k'}_i}}_{{\left(\Lambda^U_{k}(\theta^*)\right)}^{-1}} \leq\frac{\sqrt{\lambda}}{2} + \frac{d}{\sqrt{\lambda}} \log \sbr{ \frac{4}{\delta'}  \sbr{ 1 + \frac{\sig'_{\max}\sum_{k'=1}^{k}\sum_{i=1}^{m_{k'}}  \sbr{\ell^{k'}_{i}}^2}{\lambda d}    }  }  .
	\end{align*}
\end{lemma}

\begin{proof}
    According to Theorem 1 in \citep{faury2020improved}, we have that with probability at least $1-\delta'$, for any $k>0$ ,
	\begin{align*}
		\nbr{\sum_{k'=1}^{k} \sum_{i=1}^{m_{k'}} \varepsilon_{k',i} \cdot \phi^{\tau^{k'}_i}}_{{\left(\Lambda^U_{k}(\theta^*)\right)}^{-1}}  &\leq \frac{\sqrt{\lambda}}{2} + \frac{2}{\sqrt{\lambda}} \log \sbr{ \frac{\det(\Lambda^U_{k}(\theta^*))^{\frac{1}{2}} \cdot \lambda^{-\frac{d}{2}}}{\delta'}  } + \frac{2}{\sqrt{\lambda}} d \log(2)
		\\
		&\overset{\textup{(a)}}{\leq} \frac{\sqrt{\lambda}}{2} \!+\! \frac{2}{\sqrt{\lambda}} \log \sbr{ \frac{1}{\delta'}  \sbr{ 1 + \frac{\sig'_{\max}\sum_{k'=1}^{k}\sum_{i=1}^{m_{k'}}  \sbr{\ell^{k'}_{i}}^2}{\lambda d}   }^{\frac{d}{2}}  }   \!+\! \frac{2}{\sqrt{\lambda}} d \log(2)
		\\
		&\leq \frac{\sqrt{\lambda}}{2} + \frac{d}{\sqrt{\lambda}} \log \sbr{ \frac{1}{\delta'}  \sbr{ 1 +\frac{\sig'_{\max}\sum_{k'=1}^{k}\sum_{i=1}^{m_{k'}}  \sbr{\ell^{k'}_{i}}^2}{\lambda d}    }  }   + \frac{2}{\sqrt{\lambda}} d \log(2)
		\\
		&\leq \frac{\sqrt{\lambda}}{2} + \frac{d}{\sqrt{\lambda}} \log \sbr{ \frac{4}{\delta'}  \sbr{ 1 + \frac{\sig'_{\max}\sum_{k'=1}^{k}\sum_{i=1}^{m_{k'}}  \sbr{\ell^{k'}_{i}}^2}{\lambda d}    }  } ,
	\end{align*}
	where (a) uses Lemma~\ref{lemma:uq_Lambda_k}.
\end{proof}

Define event
\begin{align*}
	\cE_{\mathrm{U}}:=\lbr{ \nbr{g^U_k(\hat{\theta}_k) - g^U_k(\theta^*) }_{{\left(\Lambda^U_{k}(\theta^*)\right)}^{-1}} \leq \omega_U(k), \  \forall k>0 } .
\end{align*}
\begin{lemma}
	It holds that
	\begin{align*}
		\Pr \mbr{ \cE_{\mathrm{U}} } \geq 1-\delta' .
	\end{align*}
\end{lemma}

\begin{proof}
	This proof is similar to that for Lemma 8 in 
    (Faury et al., \citeyear{faury2020improved}).
	
	Define
	\begin{align*}
		\cL^{U}_k(\theta) \!:=\! - \sbr{ \sum_{k'=1}^{k} \sum_{i=1}^{m_{k'}} \sbr{ y^{k'}_i \cdot \log\sbr{\sig((\phi^{\tau^{k'}_i})^\top \theta)} \!+\! (1 \!-\! y^{k'}_i) \cdot \log\sbr{1 \!-\! \sig((\phi^{\tau^{k'}_i})^\top \theta) } } \!-\! \frac{1}{2} \lambda \|\theta\|_2^2 } .
	\end{align*}
	Recall that $\hat{\theta}_k = \argmin_{\theta} \cL^{U}_k(\theta)$.
	Since $\nabla \cL^{U}_k(\hat{\theta}_k)=0$ and $\sig'(x)=\sig(x)(1-\sig(x))$, we have
	\begin{align*}
		\underbrace{\sum_{k'=1}^{k} \sum_{i=1}^{m_{k'}} \sig( (\phi^{\tau^{k'}_i})^\top \hat{\theta}_k ) \cdot \phi^{\tau^{k'}_i} + \lambda \hat{\theta}_k}_{g^U_k(\hat{\theta}_k)} = 	\sum_{k'=1}^{k} \sum_{i=1}^{m_{k'}} y^{k'}_i \cdot \phi^{\tau^{k'}_i} .
	\end{align*}
	
	Hence, we have
	\begin{align}
		g^U_k(\hat{\theta}_k) - g^U_k(\theta^*) &= \sum_{k'=1}^{k} \sum_{i=1}^{m_{k'}} y^{k'}_i \cdot \phi^{\tau^{k'}_i} - \sbr{ \sum_{k'=1}^{k} \sum_{i=1}^{m_{k'}} \sig( (\phi^{\tau^{k'}_i})^\top \theta^* ) \cdot \phi^{\tau^{k'}_i} + \lambda \theta^* }
		\nonumber\\
		&= \sum_{k'=1}^{k} \sum_{i=1}^{m_{k'}} \varepsilon_{k',i} \cdot \phi^{\tau^{k'}_i} -  \lambda \theta^* . \label{eq:uq_g_hat_theta-g_theta_star}
	\end{align}

	Then, using Lemma~\ref{lemma:uq_noise_seq_martingale}, we have that with probability at least $1-\delta'$, for any $k>0$,
	\begin{align*}
		\nbr{ g^U_k(\hat{\theta}_k) - g^U_k(\theta^*) }_{{\left(\Lambda^U_k(\theta^*)\right)}^{-1}} &\leq \nbr{ \sum_{k'=1}^{k} \sum_{i=1}^{m_{k'}} \varepsilon_{k',i} \cdot \phi^{\tau^{k'}_i} }_{{\left(\Lambda^U_k(\theta^*)\right)}^{-1}} + r_{\max} \sqrt{\lambda d} 
		\\
		&\leq \frac{\sqrt{\lambda}}{2} + \frac{d}{\sqrt{\lambda}} \log \sbr{ \frac{4}{\delta'}  \sbr{ 1 + \frac{\sig'_{\max}\sum_{k'=1}^{k}\sum_{i=1}^{m_{k'}}  \sbr{\ell^{k'}_{i}}^2}{\lambda d}    }  } + r_{\max} \sqrt{\lambda d}  
		\\
		&= \omega_U(k) .
	\end{align*}
    \end{proof}
	
\begin{lemma} \label{lemma:uq_transform_Lambda_Sigma}
	For any $k\geq 1$ and $\theta \in \Theta$, we have
	\begin{align*}
		%			\Gamma_k(\theta_1,\theta_2) &\succeq \sbr{1+\frac{2Hr_{\max}}{m}}^{-1} \Lambda_k(\theta_1) ,
		%			\\
		\Sigma^U_k &\preceq  \Lambda^U_k(\theta) .
		%			\\
		%			\Gamma_k(\theta_1,\theta_2) &\succeq \frac{1}{\alpha} \Sigma_k .
	\end{align*}
\end{lemma}
\begin{proof}
	For any $\theta \in \Theta$, we have
	\begin{align*}
		\sig'(\ell_{\tau^{\seg}} r_{\max})\le \sig'((\phi^{\tau^{\seg}})^\top \theta) .
	\end{align*}
	
	Then, it holds that
	\begin{align*}
		\Sigma^U_k &= \sum_{k'=1}^{k} \sum_{i=1}^{m_{k'}} \sig'(\ell_i^{k'} r_{\max}) \phi^{\tau^{k'}_i} (\phi^{\tau^{k'}_i})^\top +\lambda I
		\\
		&\preceq \sum_{k'=1}^{k} \sum_{i=1}^{m_{k'}} \sig'( (\phi^{\tau^{k'}_i})^\top\theta ) \cdot \phi^{\tau^{k'}_i} (\phi^{\tau^{k'}_i})^\top + \lambda I 
        \\
        &= \Lambda^U_k(\theta).
	\end{align*}
\end{proof}
Recall that for any $\phi \in \R^{d}$ and $\theta_1,\theta_2 \in \Theta$, define
\begin{align*}
	b(\phi,\theta_1,\theta_2)&:=\int_{z=0}^{1} \mu'( (1-z) \cdot \phi^\top \theta_1 + z \cdot \phi^\top \theta_2) dz .
\end{align*}
For any $k>0$ and $\theta_1,\theta_2 \in \Theta$, define
\begin{align*}
	\Gamma^U_k(\theta_1,\theta_2)&:= \sum_{k'=1}^{k} \sum_{i=1}^{m_{k'}} b(\phi^{\tau_i^{k'}},\theta_1,\theta_2) \cdot \phi^{\tau^{k'}_i} (\phi^{\tau^{k'}_i})^\top + \lambda I .
\end{align*}
In the definitions of $b(\phi,\theta_1,\theta_2)$ and $\Gamma^U_k(\theta_1,\theta_2)$, $\theta_1$ and $\theta_2$ have the same roles and can be interchanged.

\begin{lemma} \label{lemma:uq_two_equalities}
	For any $\phi \in \R^{d}$ and $\theta_1,\theta_2 \in \Theta$, we have
	\begin{align*}
		\sig(\phi^\top \theta_1) - \sig(\phi^\top \theta_2) = b(\phi,\theta_2,\theta_1) \cdot \phi^\top (\theta_1-\theta_2) .
	\end{align*}
	
	In addition, for any $k>0$ and $\theta_1,\theta_2 \in \Theta$, we have
	\begin{align*}
		\nbr{\theta_1-\theta_2}_{\Gamma^U_{k}(\theta_2,\theta_1)} = \nbr{ g^U_k(\theta_1) - g^U_k(\theta_2) }_{{\left(\Gamma^U_{k}(\theta_2,\theta_1)\right)}^{-1}} .
	\end{align*}
\end{lemma}
\begin{proof}
	The first statement follows from the mean-value theorem.
	
	Then, using the first statement, we have that for any $k>0$,
	\begin{align*}
		g^U_k(\theta_1) - g^U_k(\theta_2) &= \sum_{k'=1}^{k} \sum_{i=1}^{m_{k'}} \sbr{\sig( (\phi^{\tau^{k'}_i})^\top \theta_1 ) - \sig( (\phi^{\tau^{k'}_i})^\top \theta_2 )} \cdot \phi^{\tau^{k'}_i} + \lambda \sbr{ \theta_1 - \theta_2 }
		\\
		&= \sum_{k'=1}^{k} \sum_{i=1}^{m_{k'}} b(\phi^{\tau^{k'}_i},\theta_2,\theta_1) \cdot \phi^{\tau^{k'}_i} (\phi^{\tau^{k'}_i})^\top (\theta_1-\theta_2) + \lambda \sbr{ \theta_1 - \theta_2 }
		\\
		&= \Gamma^U_{k}(\theta_2,\theta_1) \cdot (\theta_1-\theta_2) ,
	\end{align*}
	and thus
	\begin{align*}
		\nbr{\theta_1-\theta_2}_{\Gamma^U_{k}(\theta_2,\theta_1)} &= \sqrt{ \sbr{\theta_1-\theta_2}^\top \cdot \Gamma^U_{k}(\theta_2,\theta_1) \cdot \sbr{\theta_1-\theta_2}  }
		\\
		&= \sqrt{ \sbr{\theta_1-\theta_2}^\top \cdot \Gamma^U_{k}(\theta_2,\theta_1) \cdot {\left(\Gamma^U_{k}(\theta_2,\theta_1)\right)}^{-1} \cdot \Gamma^U_{k}(\theta_2,\theta_1) \cdot \sbr{\theta_1-\theta_2}  }
		\\
		&= \nbr{ g^U_k(\theta_1) - g^U_k(\theta_2) }_{{\left(\Gamma^U_{k}(\theta_2,\theta_1)\right)}^{-1}} ,
	\end{align*}
	which gives the second statement.
\end{proof}

Recall that for any $k>0$, $Z^U_k = \sum_{k'=1}^{k} \sum_{i=1}^{m_{k'}} \varepsilon_{k',i} \cdot \phi^{\tau^{k'}_i}  ,L(k)
    :=
    \max_{\substack{k'\leq k\\i\in[m_{k'}]}}
    \ell_i^{k'}$.
\begin{lemma} \label{lemma:uq_eff_bound_Z_k}
	For any $k>0$, we have
	\begin{align*}
		\Gamma^U_k(\theta^*,\hat{\theta}_k) &\succeq  \sbr{ 1 \!+\! L(k) r_{\max} \sqrt{d} \!+\! \frac{L(k)}{\sqrt{\lambda}} \nbr{ Z^U_k }_{{\left(\Gamma^U_k(\theta^*,\hat{\theta}_k)\right)}^{-1}} }^{\!\!-1} \Lambda^U_k(\theta^*) ,
		\\
		\nbr{Z^U_k}_{{\left(\Gamma^U_k(\theta^*,\hat{\theta}_k)\right)}^{-1}} &\leq \sqrt{ 1+  L(k)r_{\max} \sqrt{d} } \nbr{Z^U_k}_{{\left(\Lambda^U_k(\theta^*)\right)}^{-1}} + \frac{L(k)}{\sqrt{\lambda}}  \nbr{Z^U_k}_{{\left(\Lambda^U_k(\theta^*)\right)}^{-1}}^2 .
	\end{align*}
	
	Furthermore, assuming that event $\cE_{\mathrm{U}}$ holds, we have
	\begin{align*}
		\nbr{Z^U_k}_{{\left(\Gamma^U_k(\theta^*,\hat{\theta}_k)\right)}^{-1}} \leq \sqrt{ 1+   L(k)r_{\max} \sqrt{d} }  \cdot \omega_U(k) + \frac{ L(k)}{\sqrt{\lambda}} \cdot \omega_U(k)^2 .
	\end{align*}
\end{lemma}

\begin{proof}
	This proof follows the analysis of Proposition 6 and Corollary 5 in ((Russac et al., \citeyear{russac2021self}).
	
	From Eq.~\eqref{eq:uq_g_hat_theta-g_theta_star}, we have that for any $k>0$,
	\begin{align*}
		g^U_k(\hat{\theta}_k) - g^U_k(\theta^*) = Z^U_k -  \lambda \theta^* .
	\end{align*}
	
	Using Lemma~\ref{lemma:self_concordance}, we have that for any $\phi \in \R^{d}$ such that $\|\phi\|_2 \leq L_{\phi}$, 
	\begin{align*}
		b(\phi,\theta^*,\hat{\theta}_k) &\geq \sbr{ 1+\abr{\phi^\top (\theta^*-\hat{\theta}_k)} }^{-1} \sig'(\phi^\top \theta^*)
		\\
		&= \sbr{ 1+\abr{\phi^\top {\left(\Gamma^U_k(\theta^*,\hat{\theta}_k)\right)}^{-1}\cdot (g^U_k(\theta^*)-g^U_k(\hat{\theta}_k))} }^{-1} \sig'(\phi^\top \theta^*)
		\\
		&\geq \sbr{ 1+\nbr{\phi}_{ {\left(\Gamma^U_k(\theta^*,\hat{\theta}_k)\right)}^{-1}}  \nbr{g^U_k(\theta^*)-g^U_k(\hat{\theta}_k)}_{{\left(\Gamma^U_k(\theta^*,\hat{\theta}_k)\right)}^{-1}} }^{-1} \sig'(\phi^\top \theta^*)
		\\
		&\geq \sbr{ 1+\frac{L_{\phi}}{\sqrt{\lambda}} \nbr{g^U_k(\theta^*)-g^U_k(\hat{\theta}_k)}_{{\left(\Gamma^U_k(\theta^*,\hat{\theta}_k)\right)}^{-1}} }^{-1} \sig'(\phi^\top \theta^*)
		\\
		&= \sbr{ 1+\frac{L_{\phi}}{\sqrt{\lambda}} \nbr{ Z^U_k-\lambda \theta^* }_{{\left(\Gamma^U_k(\theta^*,\hat{\theta}_k)\right)}^{-1}} }^{-1} \sig'(\phi^\top \theta^*)
		\\
		&\geq \sbr{ 1+ L_{\phi} r_{\max} \sqrt{d}   + \frac{L_{\phi}}{\sqrt{\lambda}} \nbr{ Z^U_k }_{{\left(\Gamma^U_k(\theta^*,\hat{\theta}_k)\right)}^{-1}} }^{-1} \sig'(\phi^\top \theta^*) .
	\end{align*}
	
	Using the above equation with $\phi=\phi^{\tau^{k'}_i}$ and $L_{\phi} = L(k)$, we have
	\begin{align*}
		\Gamma^U_k(\theta^*,\hat{\theta}_k)&:= \sum_{k'=1}^{k} \sum_{i=1}^{m_{k'}} b(\phi^{\tau^{k'}_i},\theta^*,\hat{\theta}_k) \cdot \phi^{\tau^{k'}_i} (\phi^{\tau^{k'}_i})^\top + \lambda I 
		\\
		&\succeq \!\!\sum_{k'=1}^{k} \!\sum_{i=1}^{m_{k'}} \sbr{ 1 \!+\!  L(k) r_{\max} \sqrt{d} \!+\! \frac{ L(k)}{\sqrt{\lambda}} \nbr{ Z^U_k }_{{\left(\Gamma^U_k(\theta^*,\hat{\theta}_k)\right)}^{-1}} }^{\!\!-1} \!\!\!\! \sig'(\phi^\top \theta^*) \cdot \phi^{\tau^{k'}_i} (\phi^{\tau^{k'}_i})^{\!\top} \!\!+\! \lambda I 
		\\
		&=  \sbr{ 1 \!+\! L(k) r_{\max} \sqrt{d} \!+\! \frac{L(k)}{\sqrt{\lambda}} \nbr{ Z^U_k }_{{\left(\Gamma^U_k(\theta^*,\hat{\theta}_k)\right)}^{-1}} }^{\!\!-1} \Lambda^U_k(\theta^*) .
	\end{align*}
	
	This implies
	\begin{align*}
		\nbr{Z^U_k}_{{\left(\Gamma^U_k(\theta^*,\hat{\theta}_k)\right)}^{-1}}^2 \leq  \sbr{ 1 \!+\! L(k) r_{\max} \sqrt{d} \!+\! \frac{L(k)}{\sqrt{\lambda}} \nbr{ Z^U_k }_{{\left(\Gamma^U_k(\theta^*,\hat{\theta}_k)\right)}^{-1}} }\nbr{Z^U_k}_{{\left(\Lambda^U_k(\theta^*)\right)}^{-1}}^2 ,
	\end{align*}
	which is equivalent to
	\begin{align*}	\nbr{Z^U_k}_{{\left(\Gamma^U_k(\theta^*,\hat{\theta}_k)\right)}^{-1}}^2 - \frac{L(k)}{\sqrt{\lambda}}  \nbr{Z^U_k}_{{\left(\Lambda^U_k(\theta^*)\right)}^{-1}}^2 \nbr{ Z^U_k }_{{\left(\Gamma^U_k(\theta^*,\hat{\theta}_k)\right)}^{-1}} - \sbr{ 1+  L(k) r_{\max} \sqrt{d} } \nbr{Z^U_k}_{{\left(\Lambda^U_k(\theta^*)\right)}^{-1}}^2 \leq 0 .
	\end{align*}
	
	By analysis of quadratic functions, we have
	\begin{align*}
		\nbr{Z^U_k}_{{\left(\Gamma^U_k(\theta^*,\hat{\theta}_k)\right)}^{-1}} &\leq \sqrt{ 1+  L(k)r_{\max} \sqrt{d} } \nbr{Z^U_k}_{{\left(\Lambda^U_k(\theta^*)\right)}^{-1}} + \frac{L(k)}{\sqrt{\lambda}}  \nbr{Z^U_k}_{{\left(\Lambda^U_k(\theta^*)\right)}^{-1}}^2 
		\\
		&\leq\sqrt{ 1+  L(k)r_{\max} \sqrt{d} }  \cdot \omega_U(k) + \frac{L(k)}{\sqrt{\lambda}} \cdot \omega_U(k)^2  .
	\end{align*}

\end{proof}

\begin{lemma}[Concentration of $\phi^\top \hat{\theta}_k$ under Unequal Segment Binary Feedback] \label{lemma:uq_confence_interval_proj_free}
	Assume that event $\cE_{\mathrm{U}}$ holds. Then, for any $k>0$ and $\phi \in \R^{d}$,
	\begin{align*}
		|\phi^\top \theta^* - \phi^\top \hat{\theta}_k| \leq \nu_U(k) \nbr{\phi}_{{\left(\Sigma^U_k\right)}^{-1}} .
	\end{align*}
\end{lemma}
\begin{proof}
	We have
	\begin{align*}
		&\quad\ |\phi^\top \theta^* - \phi^\top \hat{\theta}_k| 
		\\
		&\le 	\nbr{\phi}_{{\left(\Gamma^U_k(\theta^*,\hat{\theta}_k)\right)}^{-1}} \nbr{\theta^* - \hat{\theta}_k}_{\Gamma^U_k(\theta^*,\hat{\theta}_k)}
		\nonumber\\
		&\overset{\textup{(a)}}{\leq} \sqrt{ 1 \!+\! L(k) r_{\max} \sqrt{d} \!+\! \frac{L(k)}{\sqrt{\lambda}} \nbr{ Z^U_k }_{{\left(\Gamma^U_k(\theta^*,\hat{\theta}_k)\right)}^{-1}} } \nbr{\phi}_{{\left(\Lambda^U_k(\theta^*)\right)}^{-1}} \nbr{g^U_k(\theta^*) - g^U_k(\hat{\theta}_k)}_{{\left(\Gamma^U_k(\theta^*,\hat{\theta}_k)\right)}^{-1}}
		\nonumber\\
		&= \sqrt{ 1 \!+\! L(k) r_{\max} \sqrt{d} \!+\! \frac{L(k)}{\sqrt{\lambda}} \nbr{ Z^U_k }_{{\left(\Gamma^U_k(\theta^*,\hat{\theta}_k)\right)}^{-1}} } \nbr{\phi}_{{\left(\Lambda^U_k(\theta^*)\right)}^{-1}} \nbr{ Z^U_k-\lambda \theta^* }_{{\left(\Gamma^U_k(\theta^*,\hat{\theta}_k)\right)}^{-1}}
		\nonumber\\
		&\leq \sqrt{ 1 \!+\! L(k) r_{\max} \sqrt{d} \!+\! \frac{L(k)}{\sqrt{\lambda}} \nbr{ Z^U_k }_{{\left(\Gamma^U_k(\theta^*,\hat{\theta}_k)\right)}^{-1}} } \nbr{\phi}_{{\left(\Lambda^U_k(\theta^*)\right)}^{-1}} \sbr{\nbr{ Z^U_k }_{{\left(\Gamma^U_k(\theta^*,\hat{\theta}_k)\right)}^{-1}} \!+\! r_{\max} \sqrt{\lambda d}  }
		\nonumber\\
		&= \frac{ \sqrt{\lambda}}{ L(k)} \sqrt{ 1+   L(k) r_{\max} \sqrt{d}  + \frac{ L(k)}{\sqrt{\lambda}} \nbr{ Z^U_k }_{{\left(\Gamma^U_k(\theta^*,\hat{\theta}_k)\right)}^{-1}} } \nbr{\phi}_{{\left(\Lambda^U_k(\theta^*)\right)}^{-1}} \cdot 
		\\
		&\quad \sbr{ \frac{ L(k)}{ \sqrt{\lambda}} \nbr{ Z^U_k }_{{\left(\Gamma^U_k(\theta^*,\hat{\theta}_k)\right)}^{-1}} +  L(k)r_{\max}\sqrt{d} } 
		\nonumber\\
		&\leq  \frac{ \sqrt{\lambda}}{ L(k)} \sbr{ 1+   L(k) r_{\max} \sqrt{d}  + \frac{ L(k)}{\sqrt{\lambda}} \nbr{ Z^U_k }_{{\left(\Gamma^U_k(\theta^*,\hat{\theta}_k)\right)}^{-1}} }^{\frac{3}{2}} \nbr{\phi}_{{\left(\Lambda^U_k(\theta^*)\right)}^{-1}}  
		\nonumber\\
		&\overset{\textup{(b)}}{\leq} \!\frac{ \sqrt{\lambda}}{ L(k)} \!\! \sbr{ 1+   L(k) r_{\max} \sqrt{d}   \!+\! \frac{ L(k)}{\sqrt{\lambda}} \sbr{ \sqrt{ 1+   L(k)r_{\max} \sqrt{d} }  \cdot \omega_U(k) + \frac{ L(k)}{\sqrt{\lambda}} \cdot \omega_U(k)^2  } }^{\!\!\frac{3}{2}} \!\!\nbr{\phi}_{{\left(\Lambda^U_k(\theta^*)\right)}^{-1}}  
		\\
		&\overset{\textup{(c)}}{\leq} \nu_U(k) \nbr{\phi}_{{\left(\Sigma^U_k\right)}^{-1}} ,
	\end{align*}
	where inequality (a) is due to Lemmas~\ref{lemma:uq_two_equalities} and \ref{lemma:uq_eff_bound_Z_k}, inequality (b) follows from  Lemmas~\ref{lemma:uq_transform_Lambda_Sigma}, and  inequality (c) follows from Lemma~\ref{lemma:uq_eff_bound_Z_k}.
\end{proof}

Recall that $F_k$ is the filtration containing all information accumulated
through the end of episode $k$. 

\begin{lemma}\label{lemma:uq_Gaussian_anti_concentration}
	Assume that event $\cE_{\mathrm{U}}$ holds. Then, for any $k>0$ and $F_{k-1}$-measurable random variable $X \in \R^{d}$, we have
	\begin{align*}
		\Pr\mbr{ X^\top \widetilde{\theta}_k > X^\top \theta^* \ |\ F_{k-1} } \geq \frac{1}{2\sqrt{2\pi e}} .
	\end{align*}
\end{lemma}
\begin{proof}
	This proof is originated from the analysis of Lemma 11 in (Efroni et al., \citeyear{efroni2021reinforcement}). 
	
	Using Lemma~\ref{lemma:uq_confence_interval_proj_free}, we have that for any $k>0$,
	\begin{align*}
		|X^\top \theta^* - X^\top \hat{\theta}_{k-1}| \leq  \nu_U(k-1) \nbr{X}_{{\left(\Sigma^U_{k-1}\right)}^{-1}}  .
	\end{align*}
	
	It holds that
	\begin{align*}
		&\quad\ \Pr\mbr{ X^\top \widetilde{\theta}_k > X^\top \theta^* \ |\ F_{k-1} }
		\\
		&= \Pr\mbr{ \frac{X^\top \widetilde{\theta}_k - X^\top \hat{\theta}_{k-1} }{ \nu_U(k-1) \|X\|_{{\left(\Sigma^U_{k-1}\right)}^{-1}} } > \frac{X^\top \theta^* - X^\top \hat{\theta}_{k-1} }{  \nu_U(k-1) \|X\|_{{\left(\Sigma^U_{k-1}\right)}^{-1}} } \ |\ F_{k-1} } .
	\end{align*}
	
	Here given $F_{k-1}$, $X^\top \tilde{\theta}_k - X^\top \hat{\theta}_{k-1} = X^\top \xi_k$ is a Gaussian random variable with mean $0$ and standard deviation $ \nu_U(k-1) \|X\|_{{\left(\Sigma^U_{k-1}\right)}^{-1}}$.
	
	Since when event $\cE_{\mathrm{U}}$ holds,
	\begin{align*}
		\frac{X^\top \theta^* - X^\top \hat{\theta}_{k-1}}{ \nu_U(k-1) \|X\|_{{\left(\Sigma^U_{k-1}\right)}^{-1}} }  \leq \frac{\nu_U(k-1) \|X\|_{{\left(\Sigma^U_{k-1}\right)}^{-1}} }{ \nu_U(k-1) \|X\|_{{\left(\Sigma^U_{k-1}\right)}^{-1}} } =1 ,
	\end{align*}
	we have
	\begin{align*}
		\Pr\mbr{ X^\top \tilde{\theta}_k > X^\top \theta^* \ |\ F_{k-1} }
		&\geq \Pr\mbr{ \frac{X^\top \tilde{\theta}_k - X^\top \hat{\theta}_{k-1}}{ \nu_U(k-1) \|X\|_{{\left(\Sigma^U_{k-1}\right)}^{-1}} } > 1 \ |\ F_{k-1} }
		\\
		&= \Pr\mbr{ \frac{ X^\top \xi_k }{\nu_U(k-1) \|X\|_{{\left(\Sigma^U_{k-1}\right)}^{-1}} } > 1 \ |\ F_{k-1} }
		\\
		&\overset{\textup{(a)}}{\geq} \frac{1}{2\sqrt{2\pi e}} ,
	\end{align*}
	where inequality (a) comes from that if $Z \sim \cN(0,1)$,  $\Pr[Z>z]\geq \frac{1}{\sqrt{2\pi}} \cdot \frac{z}{1+z^2} e^{-\frac{z^2}{2}}$((Borjesson
and Sundberg,~\citeyear{borjesson1979simple}).
\end{proof}

For any $k>0$ and $\delta_k \in (0,1)$, define event
\begin{align*}
	\cM^\mathrm{U}_k(\delta_k):=\lbr{ \forall \phi \in \R^{d} :\  |\phi^\top \xi_k| \leq \nu_U(k-1) \sbr{ \sqrt{d} + 2 \sqrt{\log\sbr{\frac{1}{\delta_k}}} } \nbr{\phi}_{{\left(\Sigma^U_{k-1}\right)}^{-1}} } .
\end{align*}

\begin{lemma} \label{lemma:uq_X_top_xi_k}
	For any $k>0$ and $\delta_k \in (0,1)$, we have
	\begin{align*}
		\Pr\mbr{ \cM^\mathrm{U}_k(\delta_k) \ |\ F_{k-1} } \geq 1 - \delta_k .
	\end{align*}
	
	In addition, for a random variable $X \in \R^{d}$ such that $\|X\|_{{\left(\Sigma^U_{k-1}\right)}^{-1}} \leq L_{X}$, we have
	\begin{align*}
		\ex\mbr{|X^\top \xi_k| \ | F_{k-1}} 
		&\leq \nu_U(k-1) \sbr{ \sqrt{d} + 2 \sqrt{\log\sbr{\frac{1}{\delta_k}}} } \ex\mbr{ \nbr{X}_{{\left(\Sigma^U_{k-1}\right)}^{-1}} \ | F_{k-1}} 
		\\
		&\quad + \nu_U(k-1) \cdot L_{X} \sqrt{ d \delta_k } .
	\end{align*}
\end{lemma}
\begin{proof}
	This proof is similar to the analysis of Lemma 13 in (Efroni et al., \citeyear{efroni2021reinforcement}).
	
	First, we prove the first statement.
	
	For any $\phi \in \R^{d}$, we have
	\begin{align}
		|\phi^\top \xi_k| &= |\phi^\top {\left(\Sigma^U_{k-1}\right)}^{-\frac{1}{2}} {\left(\Sigma^U_{k-1}\right)}^{\frac{1}{2}} \xi_k|
		\nonumber\\
		&\leq \nbr{ {\left(\Sigma^U_{k-1}\right)}^{-\frac{1}{2}} \phi}_2 \nbr{{\left(\Sigma^U_{k-1}\right)}^{\frac{1}{2}} \xi_k}_2
		\nonumber\\
		&= \nu_U(k-1) \nbr{\phi}_{{\left(\Sigma^U_{k-1}\right)}^{-1}} \nbr{\frac{1}{\nu_U(k-1)} {\left(\Sigma^U_{k-1}\right)}^{\frac{1}{2}} \xi_k}_2 . \label{eq:known_tran_phi_xi_decomposition_uq}
	\end{align}
	
	Since given $F_{k-1}$, $\frac{1}{ \nu_U(k-1)} {\left(\Sigma^U_{k-1}\right)}^{\frac{1}{2}} \xi_k \in \R^{d}$ is a vector with each entry being a standard Gaussian random variable, we have that $\|\frac{1}{\nu_U(k-1)} {\left(\Sigma^U_{k-1}\right)}^{\frac{1}{2}} \xi_k\|_2$ is chi-distributed with parameter $d$. 
	
	Then, using Lemma 1 in(Laurent and Massart,~\citeyear{laurent2000adaptive}), we have that with probability at least $1-\delta_k$,
	\begin{align*}
		\nbr{\frac{1}{\nu_U(k-1)} {\left(\Sigma^U_{k-1}\right)}^{\frac{1}{2}} \xi_k}_2 &\leq \sqrt{d + 2\sqrt{d \log\sbr{\frac{1}{\delta_k}}} + 2 \log\sbr{\frac{1}{\delta_k}} }
		\\
		&= \sqrt{ \sbr{ \sqrt{d} + \sqrt{\log\sbr{\frac{1}{\delta_k}}} }^2 +  \log\sbr{\frac{1}{\delta_k}} } 
		\\
		&\leq \sqrt{d} + 2 \sqrt{\log\sbr{\frac{1}{\delta_k}}} .
	\end{align*}
	
	Next, we prove the second statement.
	
	For a random variable $X \in \R^{d}$, we have
	\begin{align*}
		\ex\mbr{|X^\top \xi_k| \ | F_{k-1}} 
		&= \Pr\mbr{ \cM^\mathrm{U}_k(\delta_k) } \cdot \ex\mbr{|X^\top \xi_k| \ | F_{k-1}, \cM^\mathrm{U}_k(\delta_k)} 
		\\
		&\quad + \Pr\mbr{ \left(\cM^\mathrm{U}_k(\delta_k)\right)^c } \cdot  \ex\mbr{|X^\top \xi_k| \ | F_{k-1}, \left(\cM^\mathrm{U}_k(\delta_k)\right)^c }
		\\
		&\leq \nu_U(k-1) \sbr{ \sqrt{d} + 2 \sqrt{\log\sbr{\frac{1}{\delta_k}}} } \ex\mbr{ \nbr{X}_{{\left(\Sigma^U_{k-1}\right)}^{-1}} \ | F_{k-1}} 
		\\
		&\quad + \sqrt{\Pr\mbr{ \left(\cM^\mathrm{U}_k(\delta_k)\right)^c } \cdot  \ex\mbr{|X^\top \xi_k|^2 \ | F_{k-1}, \left(\cM^\mathrm{U}_k(\delta_k)\right)^c} }
		\\
		&\overset{\textup{(a)}}{\leq}\nu_U(k-1) \sbr{ \sqrt{d} + 2 \sqrt{\log\sbr{\frac{1}{\delta_k}}} } \ex\mbr{ \nbr{X}_{{\left(\Sigma^U_{k-1}\right)}^{-1}} \ | F_{k-1}} 
		\\
		&\quad +\nu_U(k-1) \sqrt{ \delta_k  \ex\mbr{ \nbr{X}^2_{{\left(\Sigma^U_{k-1}\right)}^{-1}} \cdot \nbr{\frac{1}{\nu_U(k-1)} {\left(\Sigma^U_{k-1}\right)}^{\frac{1}{2}} \xi_k}_2^2 \ | F_{k-1}}  }
				\\
		&\leq \nu_U(k-1) \sbr{ \sqrt{d} + 2 \sqrt{\log\sbr{\frac{1}{\delta_k}}} } \ex\mbr{ \nbr{X}_{{\left(\Sigma^U_{k-1}\right)}^{-1}} \ | F_{k-1}} 
		\\
		&\quad +\nu_U(k-1) \sqrt{ \delta_k L_{X}^2 \ex\mbr{  \nbr{\frac{1}{ \nu_U(k-1)} {\left(\Sigma^U_{k-1}\right)}^{\frac{1}{2}} \xi_k}_2^2 \ | F_{k-1}}  }
		\\
		&\overset{\textup{(b)}}{\leq}\nu_U(k-1) \sbr{ \sqrt{d} + 2 \sqrt{\log\sbr{\frac{1}{\delta_k}}} } \ex\mbr{ \nbr{X}_{{\left(\Sigma^U_{k-1}\right)}^{-1}} \ | F_{k-1}} 
		\\
		&\quad +\nu_U(k-1) \cdot L_{X} \sqrt{ d \delta_k }  .
	\end{align*}
	Here inequality (a) follows from the Cauchy-Schwarz inequality. Inequality (b) is due to the fact that given $F_{k-1}$, $\|\frac{1}{\nu_U(k-1)} {\left(\Sigma^U_{k-1}\right)}^{\frac{1}{2}} \xi_k\|_2$ is chi-distributed with parameter $d$, and then $\ex[ \|\frac{1}{\nu_U(k-1)} {\left(\Sigma^U_{k-1}\right)}^{\frac{1}{2}} \xi_k\|_2^2 \ | F_{k-1}] = d$.
\end{proof}

For any $k \geq 1$, let $\tau^k$ denote the trajectory in episode $k$. 
Define event 
\begin{align}
	\cF^{\bi\mathrm{U}}_{\textup{KTran}}:=\Bigg\{& \abr{\sum_{k'=1}^{k} \sbr{ \ex\mbr{ \nbr{ \phi^{\tau^{k'}}}_{{\left(\Sigma^U_{k'-1}\right)}^{-1}} | F_{k'-1}} - \nbr{ \phi^{\tau^{k'}}}_{{\left(\Sigma^U_{k'-1}\right)}^{-1}} } } \leq 4H\sqrt{ \frac{k}{\lambda} \log\sbr{\frac{4k}{\delta'}} } ,
	\nonumber\\
	&\abr{\sum_{k'=1}^{k} \sbr{ \ex\mbr{ (\phi^{\pi^{k'}})^\top \theta^* | F_{k'-1} } - (\phi^{\pi^{k'}})^\top \theta^* }} \leq 4 H r_{\max} \sqrt{ k \log\sbr{\frac{4k}{\delta'}} } ,\ \forall k>0 \Bigg\} .
\end{align}

\begin{lemma} \label{lemma:cF_binary}
	It holds that
	\begin{align*}
		\Pr \mbr{ \cF^{\bi\mathrm{U}}_{\textup{KTran}} } \geq 1-2\delta' .
	\end{align*}
\end{lemma}
\begin{proof}
	We prove the first inequality as follows.
	
	For any $k'\geq 1$, we have that 
    \[\nbr{ \phi^{\tau^{k'}}}_{{\left(\Sigma^U_{k'-1}\right)}^{-1}} \leq \frac{H}{\sqrt{ \lambda}}\]
    
    , and then \[\ex \left[\nbr{ \phi^{\tau^{k'}}}_{{\left(\Sigma^U_{k'-1}\right)}^{-1}} | F_{k'-1}\right] - \nbr{ \phi^{\tau^{k'}}}_{{\left(\Sigma^U_{k'-1}\right)}^{-1}}| \leq \frac{2H}{\sqrt{\lambda}}\].

	Using the Azuma-Hoeffding inequality, we have that for any fixed $k>0$, with probability at least $1-\frac{\delta'}{2k^2}$,
	\begin{align*}
		\abr{\sum_{k'=1}^{k} \sbr{ \ex\mbr{ \nbr{ \phi^{\tau^{k'}}}_{{\left(\Sigma^U_{k'-1}\right)}^{-1}} | F_{k'-1}} - \nbr{ \phi^{\tau^{k'}}}_{{\left(\Sigma^U_{k'-1}\right)}^{-1}} } } &\leq \sqrt{ 2 \cdot \frac{4H^2}{ \lambda} \cdot k \log\sbr{\frac{4k^2}{\delta'}} } .
	\end{align*}
	
	Since $\sum_{k=1}^{\infty} \frac{\delta'}{2k^2} \leq \delta'$, by a union bound over $k$, we have that with probability at least $1-\delta'$, for any $k\geq 1$,
	\begin{align*}
		\abr{\sum_{k'=1}^{k} \sbr{ \ex\mbr{ \nbr{ \phi^{\tau^{k'}}}_{{\left(\Sigma^U_{k'-1}\right)}^{-1}} | F_{k'-1}} - \nbr{ \phi^{\tau^{k'}}}_{{\left(\Sigma^U_{k'-1}\right)}^{-1}} } } &\leq \sqrt{ 2 \cdot \frac{4H^2}{ \lambda} \cdot k \log\sbr{\frac{4k^2}{\delta'}} } 
		\\
		&\leq 4H \sqrt{ \frac{k}{ \lambda} \log\sbr{\frac{4k}{\delta'}} } .
	\end{align*}
	The second inequality can be obtained by a similar argument and the fact that $|(\phi^{\pi^{k}})^\top \theta^*| \leq Hr_{\max}$ for any $k>0$.
\end{proof}

\begin{lemma} \label{lemma:uq_sum_phi_seg_binary}
	For any $K\geq 1$, we have
	\begin{align*}
		\sum_{k=1}^{K} \sum_{i=1}^{m_{k}} \nbr{ \phi^{\tau^k_i}}_{{\left(\Sigma^U_{k-1}\right)}^{-1}} \leq  \sqrt{2\sum_{k=1}^{K}\max\lbr{ \frac{\sum_{i=1}^{m_{k}}(\ell_i^k)^2}{4\lambda}, 1}\sum_{i=1}^{m_{k}}\frac{1}{\sig'(\ell_i^{k}r_{\max})} } \cdot \sqrt{d\log \left(1+\frac{\sum_{k=1}^{K}\sum_{i=1}^{m_{k}}(\ell_i^k)^2}{4\lambda d}\right)} .
	\end{align*}
\end{lemma}
\begin{proof}		
	We have
	\begin{align*}
		\sum_{k=1}^{K} \sum_{i=1}^{m_{k}} \nbr{ \phi^{\tau^k_i}}_{{\left(\Sigma^U_{k-1}\right)}^{-1}} &=
        \sum_{k=1}^{K} \sum_{i=1}^{m_{k}} \frac{1}{\sqrt{\sig'(\ell_i^{k}r_{\max})}}\nbr{ \sqrt{\sig'(\ell_i^{k}r_{\max})}\phi^{\tau^k_i}}_{{\left(\Sigma^U_{k-1}\right)}^{-1}}
        \\
        &\leq \sum_{k=1}^{K} \sqrt{\sum_{i=1}^{m_{k}}\frac{1}{\sig'(\ell_i^{k}r_{\max})}\sum_{i=1}^{m_{k}}\nbr{ \sqrt{\sig'(\ell_i^{k}r_{\max})}\phi^{\tau^k_i}}^2_{{\left(\Sigma^U_{k-1}\right)}^{-1}}}
        \\
        &\overset{\textup{(a)}} \leq \sum_{k=1}^{K}\sqrt{2\max\lbr{ \frac{\sum_{i=1}^{m_{k}}(\ell_i^k)^2}{4\lambda}, 1}\sum_{i=1}^{m_{k}}\frac{1}{\sig'(\ell_i^{k}r_{\max})}}
        \\
        &\cdot
        \sqrt{\log\left(1+\sum_{i=1}^{m_{k}}\nbr{ \sqrt{\sig'(\ell_i^{k}r_{\max})}\phi^{\tau^k_i}}^2_{{\left(\Sigma^U_{k-1}\right)}^{-1}}\right)}
        \\
        & \leq \sqrt{2\sum_{k=1}^{K}\max\lbr{ \frac{\sum_{i=1}^{m_{k}}(\ell_i^k)^2}{4\lambda}, 1}\sum_{i=1}^{m_{k}}\frac{1}{\sig'(\ell_i^{k}r_{\max})} }
        \\
        &\cdot
        \sqrt{\sum_{k=1}^{K}\log\left(1+\sum_{i=1}^{m_{k}}\nbr{ \sqrt{\sig'(\ell_i^{k}r_{\max})}\phi^{\tau^k_i}}^2_{{\left(\Sigma^U_{k-1}\right)}^{-1}}\right)}
        \\
        &\le \sqrt{2\sum_{k=1}^{K} \max\lbr{ \frac{\sum_{i=1}^{m_{k}}(\ell_i^k)^2}{4\lambda}, 1}\sum_{i=1}^{m_{k}}\frac{1}{\sig'(\ell_i^{k}r_{\max})}} \cdot \sqrt{\log \left(\frac{\det(\Sigma^U_K)}{\det(\lambda I)}\right)}
        \\
        &\le \sqrt{2\sum_{k=1}^{K}\max\lbr{ \frac{\sum_{i=1}^{m_{k}}(\ell_i^k)^2}{4\lambda}, 1}\sum_{i=1}^{m_{k}}\frac{1}{\sig'(\ell_i^{k}r_{\max})} } 
        \\
        &\cdot \sqrt{d\log \left(1+\frac{\sum_{k=1}^{K}\sum_{i=1}^{m_{k}}(\ell_i^k)^2}{4\lambda d}\right)}
	\end{align*}
	where inequality (a) is due to that for any $x \in [0,c]$ with constant $c\geq0$, it holds that $x \leq 2 \max\{c,1\} \cdot \log(1+x)$.
\end{proof}

\begin{proof}[Proof of Theorem~\ref{thm:regret-binary-variable-known}]
    Let $\delta'=\delta/3$. Then,
$        \Pr\mbr{\cE_{\mathrm{U}} \cap \cF^{\bi\mathrm{U}}_{\textup{KTran}}}
        \geq 1-\delta .$
    Therefore, it remains to prove the desired regret bound on the event
    $\cE_{\mathrm{U}} \cap \cF^{\bi\mathrm{U}}_{\textup{KTran}}$.

    Assume that event $\cE_{\mathrm{U}} \cap \cF^{\bi\mathrm{U}}_{\textup{KTran}}$ holds. Then,
    \begin{align}
        \cR(K)
        &=
        \sum_{k=1}^{K}
        \sbr{
            (\phi^{\pi^*})^\top \theta^*
            -
            (\phi^{\pi^k})^\top \theta^*
        }
        \nonumber\\
        &=
        \sum_{k=1}^{K}
        \sbr{
            \ex\mbr{
                (\phi^{\pi^*})^\top \theta^*
                -
                (\phi^{\pi^k})^\top \theta^*
                \mid F_{k-1}
            }
            +
            \ex\mbr{
                (\phi^{\pi^k})^\top \theta^*
                \mid F_{k-1}
            }
            -
            (\phi^{\pi^k})^\top \theta^*
        }
        \nonumber\\
        &\leq
        \sum_{k=1}^{K}
        \ex\mbr{
            (\phi^{\pi^*})^\top \theta^*
            -
            (\phi^{\pi^k})^\top \theta^*
            \mid F_{k-1}
        }
        +
        4Hr_{\max}
        \sqrt{
            K\log\sbr{\frac{4K}{\delta'}}
        } .
        \label{eq:uq_known_tran_regret_decomposition}
    \end{align}

    For the first term, we further decompose
    \begin{align}
        &\quad
        \sum_{k=1}^{K}
        \ex\mbr{
            (\phi^{\pi^*})^\top \theta^*
            -
            (\phi^{\pi^k})^\top \theta^*
            \mid F_{k-1}
        }
        \nonumber\\
        &=
        \sum_{k=1}^{K}
        \sbr{
            \ex\mbr{
                (\phi^{\pi^*})^\top \theta^*
                -
                (\phi^{\pi^k})^\top \tilde{\theta}_k
                \mid F_{k-1}
            }
            +
            \ex\mbr{
                (\phi^{\pi^k})^\top \tilde{\theta}_k
                -
                (\phi^{\pi^k})^\top \theta^*
                \mid F_{k-1}
            }
        } .
        \label{eq:uq_known_tran_ex_phi_star-phi_pi_k_theta}
    \end{align}

    We first show that
    \begin{align}
        &\quad
        \ex\mbr{
            (\phi^{\pi^*})^{\top}\theta^*
            -
            (\phi^{\pi^k})^{\top}\tilde{\theta}_k
            \mid F_{k-1}
        }
        \nonumber\\
        &\leq
        2\sqrt{2\pi e}
        \cdot
        \ex\mbr{
            \sbr{
                (\phi^{\pi^k})^{\top}\tilde{\theta}_k
                -
                \ex\mbr{
                    (\phi^{\pi^k})^{\top}\tilde{\theta}_k
                    \mid F_{k-1}
                }
            }^{+}
            \mid F_{k-1}
        } .
        \label{eq:uq_known_tran_fist_term_leq_positive}
    \end{align}

    If
$        \ex\mbr{
            (\phi^{\pi^*})^\top \theta^*
            -
            (\phi^{\pi^k})^\top \tilde{\theta}_k
            \mid F_{k-1}
        }
        <0 ,$
    then Eq.~\eqref{eq:uq_known_tran_fist_term_leq_positive} holds trivially.
    Otherwise, let
    \begin{align*}
        z
        :=
        \ex\mbr{
            (\phi^{\pi^*})^\top \theta^*
            -
            (\phi^{\pi^k})^\top \tilde{\theta}_k
            \mid F_{k-1}
        } .
    \end{align*}
    Then,
    \begin{align*}
        &\quad
        \ex\mbr{
            \sbr{
                (\phi^{\pi^k})^\top \tilde{\theta}_k
                -
                \ex\mbr{
                    (\phi^{\pi^k})^\top \tilde{\theta}_k
                    \mid F_{k-1}
                }
            }^{+}
            \mid F_{k-1}
        }
        \\
        &\geq
        z
        \Pr\mbr{
            (\phi^{\pi^k})^\top \tilde{\theta}_k
            -
            \ex\mbr{
                (\phi^{\pi^k})^\top \tilde{\theta}_k
                \mid F_{k-1}
            }
            \geq z
            \mid F_{k-1}
        }
        \\
        &\geq
        \ex\mbr{
            (\phi^{\pi^*})^\top \theta^*
            -
            (\phi^{\pi^k})^\top \tilde{\theta}_k
            \mid F_{k-1}
        }
        \cdot
        \Pr\mbr{
            (\phi^{\pi^k})^\top \tilde{\theta}_k
            \geq
            (\phi^{\pi^*})^\top \theta^*
            \mid F_{k-1}
        }
        \\
        &\overset{\textup{(a)}}{\geq}
        \ex\mbr{
            (\phi^{\pi^*})^\top \theta^*
            -
            (\phi^{\pi^k})^\top \tilde{\theta}_k
            \mid F_{k-1}
        }
        \cdot
        \Pr\mbr{
            (\phi^{\pi^*})^\top \tilde{\theta}_k
            \geq
            (\phi^{\pi^*})^\top \theta^*
            \mid F_{k-1}
        }
        \\
        &\overset{\textup{(b)}}{\geq}
        \ex\mbr{
            (\phi^{\pi^*})^\top \theta^*
            -
            (\phi^{\pi^k})^\top \tilde{\theta}_k
            \mid F_{k-1}
        }
        \cdot
        \frac{1}{2\sqrt{2\pi e}},
    \end{align*}
    where inequality \textup{(a)} follows from the definition of $\pi^k$,
    and inequality \textup{(b)} follows from
    Lemma~\ref{lemma:uq_Gaussian_anti_concentration}. This proves
    Eq.~\eqref{eq:uq_known_tran_fist_term_leq_positive}.

    Let $\xi'_k\in\R^d$ be an independent copy of $\xi_k$ conditioned on
    $F_{k-1}$. Applying Lemma~\ref{lemma:lin_xi_xi'} with
    $\tilde{P}=P$, $x_{k-1}=\hat{\theta}_{k-1}$, and
    $\tilde{\pi}^k=\pi^k$, we obtain
    \begin{align*}
        &\quad
        \ex\mbr{
            (\phi^{\pi^*})^\top \theta^*
            -
            (\phi^{\pi^k})^\top \tilde{\theta}_k
            \mid F_{k-1}
        }
        \\
        &\leq
        2\sqrt{2\pi e}
        \cdot
        \ex\mbr{
            \sbr{
                (\phi^{\pi^k})^\top \tilde{\theta}_k
                -
                \ex\mbr{
                    (\phi^{\pi^k})^\top \tilde{\theta}_k
                    \mid F_{k-1}
                }
            }^{+}
            \mid F_{k-1}
        }
        \\
        &\leq
        2\sqrt{2\pi e}
        \cdot
        \ex\mbr{
            \left|(\phi^{\pi^k})^\top \xi_k\right|
            +
            \left|(\phi^{\pi^k})^\top \xi'_k\right|
            \mid F_{k-1}
        } .
    \end{align*}

    Substituting the above bound into
    Eq.~\eqref{eq:uq_known_tran_ex_phi_star-phi_pi_k_theta}, and then applying
    Lemma~\ref{lemma:uq_X_top_xi_k} with
    $\delta_k=1/k^4$ and $L_X=H/\sqrt{\lambda}$, we have
    \begin{align}
        &\quad
        \sum_{k=1}^{K}
        \ex\mbr{
            (\phi^{\pi^*})^\top \theta^*
            -
            (\phi^{\pi^k})^\top \theta^*
            \mid F_{k-1}
        }
        \nonumber\\
        &=
        \sum_{k=1}^{K}
        \ex\mbr{
            (\phi^{\pi^*})^\top \theta^*
            -
            (\phi^{\pi^k})^\top \tilde{\theta}_k
            +
            (\phi^{\pi^k})^\top \tilde{\theta}_k
            -
            (\phi^{\pi^k})^\top \theta^*
            \mid F_{k-1}
        }
        \nonumber\\
        &\leq
        \sum_{k=1}^{K}
        \bigg(
        2\sqrt{2\pi e}
        \cdot
        \ex\mbr{
            \left|(\phi^{\pi^k})^\top \xi_k\right|
            +
            \left|(\phi^{\pi^k})^\top \xi'_k\right|
            \mid F_{k-1}
        }
        \nonumber\\
        &\qquad
        +
        \ex\mbr{
            (\phi^{\pi^k})^\top
            \sbr{
                \hat{\theta}_{k-1}
                +
                \xi_k
            }
            -
            (\phi^{\pi^k})^\top \theta^*
            \mid F_{k-1}
        }
        \bigg)
        \nonumber\\
        &\leq
        \sum_{k=1}^{K}
        \bigg(
        \sbr{2\sqrt{2\pi e}+1}
        \cdot
        \ex\mbr{
            \left|(\phi^{\pi^k})^\top \xi_k\right|
            \mid F_{k-1}
        }
        +
        2\sqrt{2\pi e}
        \cdot
        \ex\mbr{
            \left|(\phi^{\pi^k})^\top \xi'_k\right|
            \mid F_{k-1}
        }
        \nonumber\\
        &\qquad
        +
        \ex\mbr{
            \left|
            (\phi^{\pi^k})^\top \hat{\theta}_{k-1}
            -
            (\phi^{\pi^k})^\top \theta^*
            \right|
            \mid F_{k-1}
        }
        \bigg)
        \nonumber\\
        &\overset{\textup{(a)}}{\leq}
        \sum_{k=1}^{K}
        \Bigg(
        \sbr{4\sqrt{2\pi e}+2}
        \cdot
        \nu_U(k-1)
        \sbr{
            \sqrt{d}
            +
            4\sqrt{\log\sbr{k}}
        }
        \cdot
        \ex\mbr{
            \nbr{\phi^{\pi^k}}_{\left(\Sigma^U_{k-1}\right)^{-1}}
            \mid F_{k-1}
        }
        \nonumber\\
        &\qquad
        +
        \sbr{4\sqrt{2\pi e}+1}
        \cdot
        \nu_U(k-1)
        \frac{\sqrt{d}}{k^2}
        \cdot
        \frac{H}{\sqrt{\lambda}}
        \Bigg),
        \label{eq:uq_ex_phi_star_theta_star-phi_k_theta_star}
    \end{align}
    where inequality \textup{(a)} follows from
    Lemmas~\ref{lemma:uq_confence_interval_proj_free}
    and~\ref{lemma:uq_X_top_xi_k}.

    Moreover, by the definition of event
    $\cF^{\bi\mathrm{U}}_{\textup{KTran}}$ and
    Lemma~\ref{lemma:uq_sum_phi_seg_binary}, we have
    \begin{align}
        &\quad
        \sum_{k=1}^{K}
        \ex\mbr{
            \nbr{\phi^{\pi^k}}_{\left(\Sigma^U_{k-1}\right)^{-1}}
            \mid F_{k-1}
        }
        \nonumber\\
        &\overset{\textup{(a)}}{\leq}
        \sum_{k=1}^{K}
        \ex\mbr{
            \nbr{\phi^{\tau^k}}_{\left(\Sigma^U_{k-1}\right)^{-1}}
            \mid F_{k-1}
        }
        \nonumber\\
        &=
        \sum_{k=1}^{K}
        \sbr{
            \ex\mbr{
                \nbr{\phi^{\tau^k}}_{\left(\Sigma^U_{k-1}\right)^{-1}}
                \mid F_{k-1}
            }
            -
            \nbr{\phi^{\tau^k}}_{\left(\Sigma^U_{k-1}\right)^{-1}}
        }
        +
        \sum_{k=1}^{K}
        \nbr{\phi^{\tau^k}}_{\left(\Sigma^U_{k-1}\right)^{-1}}
        \nonumber\\
        &\leq
        4H
        \sqrt{
            \frac{K}{\lambda}
            \log\sbr{\frac{4K}{\delta'}}
        }
        +
        \sum_{k=1}^{K}
        \sum_{i=1}^{m_k}
        \nbr{\phi^{\tau_i^k}}_{\left(\Sigma^U_{k-1}\right)^{-1}}
        \nonumber\\
        &\leq
        4H
        \sqrt{
            \frac{K}{\lambda}
            \log\sbr{\frac{4K}{\delta'}}
        }
        \nonumber\\
        &\quad
        +
        \sqrt{
            2
            \sum_{k=1}^{K} \max\lbr{
                \frac{\sum_{i=1}^{m_k}(\ell_i^k)^2}{4\lambda},
                1
            }
            \sum_{i=1}^{m_k}
            \frac{1}{\sig'(\ell_i^k r_{\max})}
        }
        \cdot
        \sqrt{
            d
            \log\left(
                1+
                \frac{
                    \sum_{k=1}^{K}
                    \sum_{i=1}^{m_k}
                    (\ell_i^k)^2
                }{
                    4\lambda d
                }
            \right)
        } .
        \label{eq:uq_sum_expected_norm}
    \end{align}
    Here inequality \textup{(a)} follows from Jensen's inequality,
    the identity
$        \phi^{\pi^k}
        =
        \ex\mbr{
            \phi^{\tau^k}
            \mid \widetilde F_k
        },$
    and the tower property of conditional expectation.

    It can be verified that $\nu_U(k)$ is non-decreasing in $k$.
    Therefore, combining
    Eq.~\eqref{eq:uq_ex_phi_star_theta_star-phi_k_theta_star},
    Eq.~\eqref{eq:uq_sum_expected_norm}, and
    Eq.~\eqref{eq:uq_known_tran_regret_decomposition}, and using
    $\sum_{k=1}^{K}k^{-2}\leq 2$, we obtain
    \begin{align*}
        \cR(K)
        &\leq
        \sbr{4\sqrt{2\pi e}+2}
        \cdot
        \nu_U(K)
        \sbr{
            \sqrt{d}
            +
            4\sqrt{\log\sbr{K}}
        }
        \\
        &\quad
        \cdot
        \Bigg[
        4H
        \sqrt{
            \frac{K}{\lambda}
            \log\sbr{\frac{4K}{\delta'}}
        }
        \\
        &\qquad
        +
        \sqrt{
            2
            \sum_{k=1}^{K}\max\lbr{
                \frac{\sum_{i=1}^{m_k}(\ell_i^k)^2}{4\lambda},
                1
            }
            \sum_{i=1}^{m_k}
            \frac{1}{\sig'(\ell_i^k r_{\max})}
        }
        \cdot
        \sqrt{
            d
            \log\left(
                1+
                \frac{
                    \sum_{k=1}^{K}
                    \sum_{i=1}^{m_k}
                    (\ell_i^k)^2
                }{
                    4\lambda d
                }
            \right)
        }
        \Bigg]
        \\
        &\quad
        +
        2\sbr{4\sqrt{2\pi e}+1}
        H
        \cdot
        \nu_U(K)
        \sqrt{\frac{d}{\lambda}}
        +
        4Hr_{\max}
        \sqrt{
            K\log\sbr{\frac{4K}{\delta'}}
        }
        \\
        &\overset{\textup{(a)}}{=}
        \tilde O\Bigg(
    \nu_U(K)d
    \sqrt{
    \sum_{k=1}^{K}\max\left\{
    \frac{\sum_{i=1}^{m_k}(\ell_i^k)^2}{4\lambda},
    1
    \right\}
    \sum_{i=1}^{m_k}
    \frac{1}{\sig'(\ell_i^k r_{\max})}
    }
    \Bigg),
    \end{align*}

This completes the proof.
\end{proof}

\section{Some Technical Tools}\label{secF}

\begin{lemma}[Self-concordance, Lemma 9 in 
(Faury et al.,~\citeyear{faury2020improved})] \label{lemma:self_concordance}
	For any $x_1,x_2 \in \R$, we have
	\begin{align*}
		\mu'(x_1) \frac{1-\exp(-|x_1-x_2|)}{|x_1-x_2|} \leq \int_{z=0}^{1} \mu'( (1-z)x_1 + zx_2 ) dz \leq \mu'(x_1) \frac{\exp(|x_1-x_2|)-1}{|x_1-x_2|} .
	\end{align*}
	
	Furthermore, we have
	\begin{align*}
		\int_{z=0}^{1} \mu'( (1-z)x_1 + zx_2 ) dz \geq \frac{\mu'(x_1)}{1+|x_1-x_2|} .
	\end{align*}
\end{lemma}

\begin{theorem}[Concentration of Self-Normalized Processes (Abbasi-Yadkori et al., \citeyear{abbasi2011improved})]
\label{thm:self_normalized_process}
Let $\{\varepsilon_t\}_{t=1}^{\infty}$ be a real-valued stochastic process with 
corresponding filtration $\{\mathcal{F}_t\}_{t=0}^{\infty}$. Suppose that 
$\varepsilon_t\mid \mathcal{F}_{t-1}$ is zero-mean and $\sigma$-subGaussian, i.e.,
\begin{align*}
    \mathbb{E}\left[\varepsilon_t \mid \mathcal{F}_{t-1}\right] = 0,
\end{align*}
and for any $\lambda \in \mathbb{R}$,
\begin{align*}
    \mathbb{E}\left[
        \exp\left(\lambda \varepsilon_t\right)
        \mid \mathcal{F}_{t-1}
    \right]
    \le
    \exp\left(\frac{\lambda^2\sigma^2}{2}\right).
\end{align*}
Let $\{\psi_t\}_{t=0}^{\infty}$ be an $\mathbb{R}^d$-valued stochastic process such that 
$\psi_t \in \mathcal{F}_{t-1}$. Assume $\Lambda_0$ is a $d\times d$ positive definite matrix, and define
\begin{align*}
    \Lambda_t
    =
    \Lambda_0
    +
    \sum_{s=1}^{t}
    \psi_s \psi_s^\top .
\end{align*}
Then for any $\delta' \in (0,1) $, with probability at least $1-\delta'$, for all $t\ge 0$, we have
\begin{align*}
    \left\|
        \sum_{s=1}^{t}
        \psi_s \varepsilon_s
    \right\|_{\Lambda_t^{-1}}^2
    \le
    2\sigma^2
    \log
    \left(
        \frac{
            \det(\Lambda_t)^{1/2}
            \det(\Lambda_0)^{-1/2}
        }{
            \delta'
        }
    \right).
\end{align*}
\end{theorem}

\begin{lemma}[Uniform self-normalized concentration]\label{lemma:uniform_self_normalized}
    Let $\{s_t\}_{t=1}^{\infty}$ be a stochastic process on the state space
    $\mathcal{S}$ adapted to the filtration $\{\cF_t\}_{t=0}^{\infty}$.
    Let $\{\psi_t\}_{t=1}^{\infty}$ be an $\mathbb{R}^d$-valued predictable
    stochastic process such that $\psi_t\in\cF_{t-1}$ and
    $\|\psi_t\|_2\le 1$ for all $t\ge 1$.
    For any $\lambda'>0$, define
    \begin{align*}
        \Lambda_T
        =
        \lambda' I
        +
        \sum_{t=1}^{T}
        \psi_t\psi_t^\top ,
        \qquad T\ge 0 .
    \end{align*}
    Let $\mathcal{V}$ be a class of functions $V:\mathcal{S}\to\mathbb{R}$
    satisfying $\sup_s |V(s)|\le HB_R$ for every $V\in\mathcal{V}$.
    Then, for any $\epsilon>0$, with probability at least
    $1-\delta'$, the following inequality holds simultaneously for all
    $T\ge 0$ and all $V\in\mathcal{V}$:
    \begin{align*}
        \left\|
        \sum_{t=1}^{T}
        \psi_t
        \left\{
            V(s_t)
            -
            \mathbb{E}
            \left[
                V(s_t)
                \mid
                \cF_{t-1}
            \right]
        \right\}
        \right\|_{\Lambda_T^{-1}}^2
        \le
        4H^2B_R^2
        \left[
            \frac{d}{2}
            \log
            \left(
                \frac{T+\lambda'}{\lambda'}
            \right)
            +
            \log
            \left(
                \frac{\mathcal{N}_{\epsilon}}{\delta'}
            \right)
        \right]
        +
        \frac{8T^2\epsilon^2}{\lambda'} ,
    \end{align*}
    where $\mathcal{N}_{\epsilon}$ denotes the $\epsilon$-covering number of
    $\mathcal{V}$ with respect to the distance
           $ \operatorname{dist}(V,V')
        =
        \sup_s |V(s)-V'(s)| .$
\end{lemma}

\begin{proof}
    This proof follows the analysis of Lemma~D.4 in 
    (Jin et al.,~\citeyear{jin20_linearmdp}).
    
    For any $V\in\mathcal{V}$, there exists a function $\widetilde{V}$ in the $\epsilon$-covering of $\mathcal{V}$ such that
    \begin{align*}
        V
        =
        \widetilde{V}
        +
        \Delta_V
        \qquad
        \text{and}
        \qquad
        \sup_s |\Delta_V(s)|
        \le
        \epsilon .
    \end{align*}
    Thus, we have the following decomposition:
    \begin{align*}
        &
        \left\|
        \sum_{t=1}^{T}
        \psi_{t}
        \left\{
            V(s_{t})
            -
            \mathbb{E}
            \left[
                V(s_{t})
                \mid
                \cF_{t-1}
            \right]
        \right\}
        \right\|_{\Lambda_T^{-1}}^2
        \\
        &\le
        2
        \left\|
        \sum_{t=1}^{T}
        \psi_{t}
        \left\{
            \widetilde{V}(s_{t})
            -
            \mathbb{E}
            \left[
                \widetilde{V}(s_{t})
                \mid
                \cF_{t-1}
            \right]
        \right\}
        \right\|_{\Lambda_T^{-1}}^2
        \quad
        +
        2
        \left\|
        \sum_{t=1}^{T}
        \psi_{t}
        \left\{
            \Delta_V(s_{t})
            -
            \mathbb{E}
            \left[
                \Delta_V(s_{t})
                \mid
                \cF_{t-1}
            \right]
        \right\}
        \right\|_{\Lambda_T^{-1}}^2 .
    \end{align*}
    We first control the term associated with $\widetilde{V}$.
    Applying Theorem~\ref{thm:self_normalized_process} and taking a union bound over the $\epsilon$-covering, we obtain that, with probability at least $1-\delta'$,
    \begin{align*}
        \left\|
        \sum_{t=1}^{T}
        \psi_{t}
        \left\{
            \widetilde{V}(s_{t})
            -
            \mathbb{E}
            \left[
                \widetilde{V}(s_{t})
                \mid
                \cF_{t-1}
            \right]
        \right\}
        \right\|_{\Lambda_T^{-1}}^2
        \le
        2H^2B_R^2
        \left[
            \frac{d}{2}
            \log
            \left(
                \frac{T+\lambda'}{\lambda'}
            \right)
            +
            \log
            \left(
                \frac{\mathcal{N}_{\epsilon}}{\delta'}
            \right)
        \right].
    \end{align*}
    
    It remains to bound the approximation error term. Since $\sup_s |\Delta_V(s)|\le \epsilon$, we have
    \begin{align*}
        \left|
            \Delta_V(s_{t})
            -
            \mathbb{E}
            \left[
                \Delta_V(s_{t})
                \mid
                \cF_{t-1}
            \right]
        \right|
        \le
        2\epsilon .
    \end{align*}
    Hence,
    \begin{align*}
        \left\|
        \sum_{t=1}^{T}
        \psi_{t}
        \left\{
            \Delta_V(s_{t})
            -
            \mathbb{E}
            \left[
                \Delta_V(s_{t})
                \mid
                \cF_{t-1}
            \right]
        \right\}
        \right\|_{\Lambda_T^{-1}}^2
        &\le
        \frac{1}{\lambda'}
        \left\|
        \sum_{t=1}^{T}
        \psi_{t}
        \left\{
            \Delta_V(s_{t})
            -
            \mathbb{E}
            \left[
                \Delta_V(s_{t})
                \mid
                \cF_{t-1}
            \right]
        \right\}
        \right\|_2^2
        \\
        &\le
        \frac{1}{\lambda'}
        \left(
            \sum_{t=1}^{T}
            \|\psi_{t}\|_2
            \left|
                \Delta_V(s_{t})
                -
                \mathbb{E}
                \left[
                    \Delta_V(s_{t})
                    \mid
                    \cF_{t-1}
                \right]
            \right|
        \right)^2
        \\
        &\le
        \frac{4T^2\epsilon^2}{\lambda'}.
    \end{align*}
    Combining the two bounds above gives
    \begin{align*}
        \left\|
        \sum_{t=1}^{T}
        \psi_{t}
        \left\{
            V(s_{t})
            -
            \mathbb{E}
            \left[
                V(s_{t})
                \mid
                \cF_{t-1}
            \right]
        \right\}
        \right\|_{\Lambda_T^{-1}}^2
        \le
        4H^2B_R^2
        \left[
            \frac{d}{2}
            \log
            \left(
                \frac{T+\lambda'}{\lambda'}
            \right)
            +
            \log
            \left(
                \frac{\mathcal{N}_{\epsilon}}{\delta'}
            \right)
        \right]
        +
        \frac{8T^2\epsilon^2}{\lambda'}.
    \end{align*}

    This completes the proof.
\end{proof}

\begin{lemma}
    [Covering Number of Euclidean Ball]\label{lemma:covering_number_ball}
    For any $\epsilon > 0$, the $\epsilon$-covering number of the Euclidean ball in $\mathbb{R}^d$ with radius $R>0$ is upper bounded by
    \begin{align*}
        \left(1+\frac{2R}{\epsilon}\right)^d .
    \end{align*}
\end{lemma}

\begin{lemma}
\label{lemma:cover_num_LSVI_V}
    Let $\mathcal{V}$ be the class of state value functions
    $V:\mathcal{S}\to\mathbb{R}$ such that every $V\in\mathcal{V}$ represents
    \begin{align*}
        V(\cdot)
        =
        \max_{a\in\mathcal{A}}
        \operatorname{clip}_{[-HB_R,\,HB_R]}
        \left(
        \mathbf{w}^{\top} \psi(\cdot,a)
        +
        \beta
        \left[
        \bigl(\psi(\cdot,a)\bigr)^\top
        \Lambda^{-1}
        \psi(\cdot,a)
        \right]^{1/2}
        \right),
    \end{align*}
    where the parameters $(\mathbf{w},\beta,\Lambda)$ satisfy
    \begin{align*}
        \|\mathbf{w}\|_2\le W,
        \qquad
        \beta\in[0,B],
        \qquad
        \lambda_{\min}(\Lambda)\ge \lambda' .
    \end{align*}
    Suppose that $\|\psi(s,a)\|_2\le 1$ for all
    $(s,a)\in\mathcal{S}\times\mathcal{A}$.
    For any $\epsilon>0$, let $\mathcal{N}_{\epsilon}$ denote the
    $\epsilon$-covering number of $\mathcal{V}$ with respect to the distance
$        \operatorname{dist}(V,V')
        =
        \sup_{s\in\mathcal{S}} |V(s)-V'(s)|.$
    Then,
    \begin{align*}
        \log \mathcal{N}_{\epsilon}
        \le
        d\log\left(1+\frac{4W}{\epsilon}\right)
        +
        d^2
        \log
        \left(
            1+
            \frac{8d^{1/2}B^2}{\lambda'\epsilon^2}
        \right).
    \end{align*}
\end{lemma}

\begin{proof}This proof follows the analysis of Lemma D.6 in (Jin et al.,~\citeyear{jin20_linearmdp}).
    
    We first reparametrize the function class $\mathcal{V}$ by setting
    $
        A
        =
        \beta^2
        \Lambda^{-1}.
    $
    Then each function in $\mathcal{V}$ can be written as
    \begin{align}\label{eq:LSVI_V_form}
        V(\cdot)
        =
        \max_{a\in\mathcal{A}}
        \operatorname{clip}_{[-HB_R,\,HB_R]}
        \left( 
        \mathbf{w}^{\top} \psi(\cdot,a)
        +
        \sqrt{
        \bigl(\psi(\cdot,a)\bigr)^\top
        A
        \psi(\cdot,a)
        }
        \right)
    \end{align}
    with $\|\mathbf{w}\|_2\le W$ and
      $  \|A\|_{\mathrm{op}}
        \le
        B^2\bigl(\lambda'\bigr)^{-1}.$

    For any two functions $V_1,V_2\in\mathcal{V}$, let them take the above form with parameters
    $(\mathbf{w}_1,A_1)$ and $(\mathbf{w}_2,A_2)$, respectively. Since both $\operatorname{clip}_{[-HB_R,HB_R]}(\cdot)$ and $\max_{a\in\mathcal{A}}$ are non-expansive under the sup-norm, we have
    \begin{align*}
        \operatorname{dist}(V_1,V_2)
        &\le
        \sup_{s,a}
        \left|
        \left[
        \mathbf{w}_1^\top \psi(s,a)
        +
        \sqrt{
        \bigl(\psi(s,a)\bigr)^\top A_1\psi(s,a)
        }
        \right]
        -
        \left[
        \mathbf{w}_2^\top \psi(s,a)
        +
        \sqrt{
        \bigl(\psi(s,a)\bigr)^\top A_2\psi(s,a)
        }
        \right]
        \right|
        \\
        &\le
        \sup_{\psi:\|\psi\|_2\le 1}
        \left|
        (\mathbf{w}_1-\mathbf{w}_2)^\top \psi
        \right|
        +
        \sup_{\psi:\|\psi\|_2\le 1}
        \left|
        \sqrt{\psi^\top A_1\psi}
        -
        \sqrt{\psi^\top A_2\psi}
        \right|
        \\
        &\le
        \|\mathbf{w}_1-\mathbf{w}_2\|_2
        +
        \sup_{\psi:\|\psi\|_2\le 1}
        \sqrt{
        \left|
        \psi^\top
        \bigl(A_1-A_2\bigr)
        \psi
        \right|
        }
        \\
        &=
        \|\mathbf{w}_1-\mathbf{w}_2\|_2
        +
        \sqrt{\|A_1-A_2\|_{\mathrm{op}}}
        \\
        &\le
        \|\mathbf{w}_1-\mathbf{w}_2\|_2
        +
        \sqrt{\|A_1-A_2\|_F}.
    \end{align*}
    The second last inequality follows from
        $|\sqrt{x}-\sqrt{y}|
        \le
        \sqrt{|x-y|}$
    for any $x,y\ge 0$. 

   Let $\mathcal{C}_{\mathbf{w}}$ be an $\epsilon/2$-cover of the Euclidean ball
    \begin{align*}
        \left\{
        \mathbf{w}\in\mathbb{R}^d:
        \|\mathbf{w}\|_2\le W
        \right\}
    \end{align*}
    in the Euclidean norm. Similarly, let $\mathcal{C}_A$ be an $\epsilon^2/4$-cover of the Frobenius ball
    \begin{align*}
        \left\{
        A\in\mathbb{R}^{d\times d}:
        \|A\|_F
        \le
        d^{1/2}B^2\bigl(\lambda'\bigr)^{-1}
        \right\}
    \end{align*}
   in the Frobenius norm. By Lemma~\ref{lemma:covering_number_ball}, we have
    \begin{align*}
        |\mathcal{C}_{\mathbf{w}}|
        \le
        \left(1+\frac{4W}{\epsilon}\right)^d,
        \qquad
        |\mathcal{C}_A|
        \le
        \left[
            1+
            \frac{8d^{1/2}B^2}{\lambda'\epsilon^2}
        \right]^{d^2}.
    \end{align*}

    By the distance bound above, for any $V_1\in\mathcal{V}$, there exist
    $\mathbf{w}_2\in\mathcal{C}_{\mathbf{w}}$ and $A_2\in\mathcal{C}_A$
    such that the corresponding function $V_2$, parameterized by
    $(\mathbf{w}_2,A_2)$, satisfies
        $\operatorname{dist}(V_1,V_2)
        \le
        \epsilon .$
    Therefore,
    \begin{align*}
        \log \mathcal{N}_{\epsilon}
        &\le
        \log |\mathcal{C}_{\mathbf{w}}|
        +
        \log |\mathcal{C}_A|
        \\
        &\le
        d\log\left(1+\frac{4W}{\epsilon}\right)
        +
        d^2
        \log
        \left[
            1+
            \frac{8d^{1/2}B^2}{\lambda'\epsilon^2}
        \right].
    \end{align*}

    This completes the proof.
\end{proof}

\end{appendices}
\bibliography{sn-bibliography}% common bib file

\end{document}